%% file: main.tex
\documentclass{article}

\usepackage[nonatbib, preprint]{neurips_2022}

\usepackage[utf8]{inputenc} 
\usepackage[T1]{fontenc}    
\usepackage{hyperref}       
\usepackage{url}            
\usepackage{booktabs}       
\usepackage{amsfonts}       
\usepackage{nicefrac}       
\usepackage{microtype}      
\usepackage{xcolor}         
\usepackage[numbers, compress]{natbib}
\usepackage{amsmath}
\usepackage{amsthm}
\usepackage{graphicx}
\usepackage{color}
\usepackage{wrapfig}
\usepackage{algorithm}
\usepackage{algorithmicx}
\usepackage{algpseudocode}
\algrenewcommand\algorithmicprocedure{}

\usepackage{multirow}

\usepackage{custom_formats}

\hypersetup{
pdftitle={Optimizing Data Collection for Machine Learning},
pdfsubject={Optimizing Data Collection for Machine Learning}
}

\input{macros}

\title{Optimizing Data Collection for Machine Learning}

\author{%
  Rafid Mahmood$^1$ ~~~ James Lucas$^1$ ~~~ Jose M. Alvarez$^1$ ~~~ Sanja Fidler$^{1, 2, 3}$ ~~~ Marc T. Law$^1$ \\
  $^1$NVIDIA ~~ $^2$University of Toronto ~~ $^3$Vector Institute \\
  \texttt{\{rmahmood, jalucas, josea, sfidler, marcl\}@nvidia.com} 
  \\ Project Page:   \url{https://nv-tlabs.github.io/LearnOptimizeCollect/}
}

\begin{document}

\maketitle

\input{sections/abstract}

\input{sections/introduction}

\input{sections/related}
\input{sections/problem_def}

\input{sections/optimization_method}

\input{sections/theory}

\input{sections/multivariate}

\input{sections/results}

\input{sections/conclusion}

\begin{ack}
We thank Jonathan Lorraine, Daiqing Li, Jonah Philion, David Acuna, and Zhiding Yu, as well as the anonymous reviewers and meta-reviewers for valuable feedback on earlier versions of this paper.


\end{ack}



{
\small

\bibliographystyle{unsrtnat}
\bibliography{refs}



}

\section*{Checklist}

\begin{enumerate}

\item For all authors...
\begin{enumerate}
  \item Do the main claims made in the abstract and introduction accurately reflect the paper's contributions and scope?
    \answerYes{}
  \item Did you describe the limitations of your work?
    \answerYes{See the Conclusion.}
  \item Did you discuss any potential negative societal impacts of your work?
    \answerYes{See the Conclusion.}
  \item Have you read the ethics review guidelines and ensured that your paper conforms to them?
    \answerYes{}
\end{enumerate}

\item If you are including theoretical results...
\begin{enumerate}
  \item Did you state the full set of assumptions of all theoretical results?
    \answerYes{See Assumption~\ref{ass:cdf_pdf_exists} and the theorem statements.}
        \item Did you include complete proofs of all theoretical results?
    \answerYes{See Appendix~\ref{sec:app_math_properties_one_d}.}
\end{enumerate}

\item If you ran experiments...
\begin{enumerate}
  \item Did you include the code, data, and instructions needed to reproduce the main experimental results (either in the supplemental material or as a URL)?
    \answerNo{The code is proprietary. The data is publicly available. See Algorithm~\ref{alg:data_collection} and Appendix~\ref{sec:app_experiment_results} for implementation details.}
  \item Did you specify all the training details (e.g., data splits, hyperparameters, how they were chosen)?
    \answerYes{See Appendix~\ref{sec:app_experiments} and~\ref{sec:app_experiment_results} for details.}
        \item Did you report error bars (e.g., with respect to the random seed after running experiments multiple times)?
    \answerYes{All experiments were over five seeds. See the for standard deviation ranges.}
        \item Did you include the total amount of compute and the type of resources used (e.g., type of GPUs, internal cluster, or cloud provider)?
    \answerYes{See Appendix~\ref{sec:app_experiments}.}
\end{enumerate}

\item If you are using existing assets (e.g., code, data, models) or curating/releasing new assets...
\begin{enumerate}
  \item If your work uses existing assets, did you cite the creators?
    \answerYes{We cite all data sets and model architectures in Appendix~\ref{sec:app_experiments}.}
  \item Did you mention the license of the assets?
    \answerNA{All data sets and models are publicly available.}
  \item Did you include any new assets either in the supplemental material or as a URL?
    \answerNo{}
  \item Did you discuss whether and how consent was obtained from people whose data you're using/curating?
    \answerNA{}
  \item Did you discuss whether the data you are using/curating contains personally identifiable information or offensive content?
    \answerNA{}
\end{enumerate}

\item If you used crowdsourcing or conducted research with human subjects...
\begin{enumerate}
  \item Did you include the full text of instructions given to participants and screenshots, if applicable?
    \answerNA{}
  \item Did you describe any potential participant risks, with links to Institutional Review Board (IRB) approvals, if applicable?
    \answerNA{}
  \item Did you include the estimated hourly wage paid to participants and the total amount spent on participant compensation?
    \answerNA{}
\end{enumerate}

\end{enumerate}



\clearpage
\appendix

\input{sections/app_notation}
\input{sections/app_optimization_algo}

\input{sections/app_proofs}

\input{sections/app_multivariate}

\input{sections/app_experiments}

\input{sections/app_experiment_results}

\end{document}

%% file: macros.tex
\newif\ifcomments
\commentstrue

\newcommand{\fir}[1]{\mathbf{#1}}
\newcommand{\seco}[1]{\underline{#1}}

\newcommand{\rebut}[1]{{#1}}

%% file: sections/abstract.tex
\begin{abstract}
Modern deep learning systems require huge data sets to achieve impressive performance, but there is little guidance on how much or what kind of data to collect. Over-collecting data incurs unnecessary present costs, while under-collecting may incur future costs and delay workflows. We propose a new paradigm for modeling the data collection workflow as a formal \emph{optimal data collection problem} that allows designers to specify performance targets, collection costs, a time horizon, and penalties for failing to meet the targets. Additionally, this formulation generalizes to tasks requiring multiple data sources, such as labeled and unlabeled data used in semi-supervised learning. To solve our problem, we develop Learn-Optimize-Collect (LOC), which minimizes expected future collection costs. Finally, we numerically compare our framework to the conventional baseline of estimating data requirements by extrapolating from neural scaling laws. We significantly reduce the risks of failing to meet desired performance targets on several classification, segmentation, and detection tasks, while maintaining low total collection costs.
\end{abstract}

%% file: sections/introduction.tex
\section{Introduction}
\label{sec:introduction}

When deploying a deep learning model in an industrial application, designers often mandate that the model must meet a pre-determined baseline performance, such as a target metric over a validation data set. 
For example, an object detector may require a certain minimum mean average precision before being deployed in a safety-critical setting. 
One of the most effective ways of meeting target performances is by collecting more training data for a given model.

Determining how much data is needed to meet performance targets can impact costs and development delays.
Overestimating the data requirement incurs excess costs from collection, cleaning, and annotation. 
For instance, annotating segmentation masks for a driving data set takes between $15$ to $40$ seconds per object. For $100{,}000$ images the annotation could require between $170$ and $460$ days-equivalent of time~\citep{acuna2018efficient, mahmood2022howmuch}.
On the other hand, collecting too little data may incur future costs and workflow delays from having to collect more later. 
For example, in medical imaging applications, this means further clinical data acquisition rounds
that require expensive clinician time. 
In the worst case, designers may even realize that a project is infeasible only after collecting insufficient data.

The growing literature on sample complexity in machine learning has identified neural scaling laws that scale model performance with data set sizes according to power laws
\citep{frey1999modeling, gu2001modelling, hestness2017deep, rosenfeld2019constructive, kaplan2020scaling, hoiem2021learning, bahri2021explaining, bisla2021theoretical}.
For instance,~\citet{rosenfeld2019constructive} fit power law functions on the performance statistics of small data sets to extrapolate the learning curve with more data.
In contrast,~\citet{mahmood2022howmuch} consider estimating data requirements and show that even 
small errors in a power law model of the learning curve can translate to massively over- or underestimating how much data is needed.
Beyond this, different data sources have different costs and scale differently with performance~\citep{mikami2021scaling,acuna2021f,Prakash_2021_ICCV,acuna2022domain}.
For example, although unlabeled data may be easier to collect than labeled data, some semi-supervised learning tasks may need an order of magnitude more unlabeled data to match the performance of a small labeled set. 
Thus, collecting more data based only on estimation will fail to capture uncertainty and collection costs.

In this paper, we propose a new paradigm for modeling the data collection workflow as an \emph{optimal data collection problem}. Here, a designer must minimize the cost of collecting enough data to obtain a model capable of a desired performance score.
They have multiple collection rounds, where after each round, they re-evaluate the model and decide how much more data to order.
The data has per-sample costs and moreover, the designer pays a penalty if they fail to meet the target score within a finite horizon. 
Using this formal framework, we develop an optimization approach for minimizing the expected future collection costs and show that this problem can be optimized in each collection round via gradient descent. 
Furthermore, our optimization problem immediately generalizes to decisions over multiple data sources (e.g., unlabeled, long-tail, cross-domain, synthetic) that have different costs and impacts on performance. 
Finally, we demonstrate the value of optimization over na\"{i}vely estimating data set requirements (e.g.,~\citep{mahmood2022howmuch}) for several machine learning tasks and data sets. 

Our contributions are as follows.
(1) We propose the optimal data collection problem in machine learning, which formalizes data collection workflows.
(2) We introduce Learn-Optimize-Collect (LOC), a learning-and-optimizing framework that minimizes future collection costs, can be solved via gradient descent, and has analytic solutions in some settings. 
(3) We generalize the data collection problem and LOC to a multi-variate setting where different types of data have different costs. To the best of our knowledge, this is the first exploration of data collection with general multiple data sets in machine learning, covering for example, semi-supervised and long-tail learning. 
(4) We perform experiments over classification, segmentation, and detection tasks to show, on average, approximately a $2\times$ reduction in the chances of failing to meet performance targets, versus estimation baselines. 

%% file: sections/related.tex
\section{Related work}
\label{sec:setup_related_work}

\begin{figure*}[!t]
\begin{center}
\begin{minipage}{0.84\linewidth}
\includegraphics[width=1\textwidth]{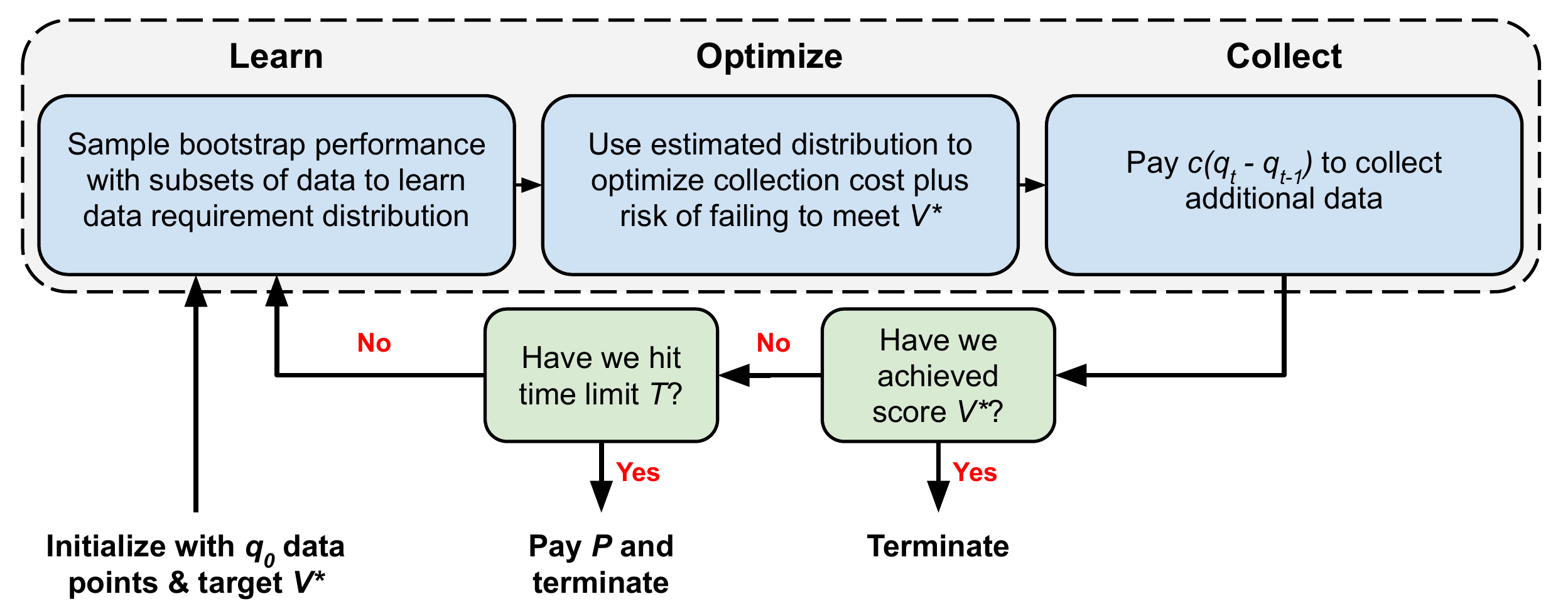} 
\end{minipage}
    \vspace{-2mm}
    \caption{\label{fig:flowchart} In the optimal data collection problem, we iteratively determine the amount of data that we should have, pay to collect the additional data, and then re-evaluate our model. 
    Our approach, Learn-Optimize-Collect, optimizes for the minimum amount of data $q^*_t$ to collect.
    }
\end{center}
\vspace{-7mm}
\end{figure*}

\noindent \textbf{Neural Scaling Laws.}
\rebut{
According to the neural scaling law literature, the performance of a model on a validation set scales with the size of the training data set $q$ via a power law $V \propto \theta_0 q^{\theta_1}$~\citep{hestness2017deep, rosenfeld2019constructive, hoiem2021learning, bahri2021explaining, bisla2021theoretical, jones2003introduction, sun2017revisiting,  figueroa2012predicting, viering2021shape, zhai2021scaling}.
\citet{hestness2017deep} observe this property over vision, language, and audio tasks,~\citet{bahri2021explaining} develop a theoretical relationship under assumptions on over-parametrization and the Lipschitz continuity of the loss, model, and data, and~\citet{rosenfeld2019constructive} estimate power laws using smaller data sets and models to extrapolate future performance.
}
Multi-variate scaling laws have also been considered for some specific tasks, for example in transfer learning from synthetic to real data sets~\citep{mikami2021scaling}.
Finally,~\citet{mahmood2022howmuch} explore data collection by estimating the minimum amount of data needed to meet a given target performance over multiple rounds.
Our paper extends these prior studies by developing an optimization problem to minimize the expected total cost of data collected. Specifically, we incorporate the uncertainty in any regression estimate of data requirements and further generalize to multiple data sources with different costs.

\noindent \textbf{Active Learning.} 
\rebut{
In active learning, a model sequentially collects data by selecting new subsets of an unlabeled data pool to label under a pre-determined labeling budget that replenishes after each round~\citep{settles2009active, sener2017active, yoo2019learning, sinha2019variational, mahmood2021low}. In contrast, our work focuses on systematically determining an optimal collection budget. After determining how much data to collect, we can use active learning techniques to collect the desired amount of data.
}

\noindent \textbf{Statistical Learning Theory.}
\rebut{
Theoretical analysis of the sample complexity of machine learning models is typically only tight asymptotically, but some recent work have empirically analyzed these relationships~\citep{jiang2020neurips, jiang2021methods}.
Particularly,~\citet{bisla2021theoretical} study generalization bounds for deep neural networks, provide empirical validation, and suggest using them to estimate data requirements. 
In contrast, our paper formally explores the consequences of collection costs on data requirements.
}

\noindent \textbf{Optimal Experiment Design.}
The topic of how to collect data, select samples, and design scientific experiments or controlled trials is well-studied in econometrics~\citep{smith1918standard, cohn1993neural, emery1998optimal}. 
For example,~\citet{bertsimas2015power} optimize the assignment of samples into control and trial groups to minimize inter-group variances.
Most recently,~\citet{carneiro2020optimal} optimize how many samples and covariates to collect in a statistical experiment by minimizing a treatment effect estimation error or maximizing $t$-test power.
However, our focus on industrial machine learning applications differs from experiment design by having target performance metrics and continual rounds of collection and modeling. 

%% file: sections/problem_def.tex
\section{Main Problem}
\label{sec:data_collection_problem}

In this section, we give a motivating example before introducing the formal data collection problem.
We include a table of notation in Appendix~\ref{sec:app_notation}.

\textbf{Motivating Example.}
\emph{A startup is developing an object detector for use in autonomous vehicles within the next $T=5$ years.
Their model must achieve a mean Average Precision greater than $V^*=95\%$ on a pre-determined validation set or else they will lose an expected profit of $P = \$1,000,000$. 
Collecting training data requires employing drivers to record videos and annotators to label the data, where the marginal cost of obtaining each image is approximately $c = \$1$. 
In order to manage annual finances, the startup must plan how much data to collect at the beginning of each year.
}

Let $z \sim p(z)$ be data drawn from a distribution $p$. For instance, $z := (x, y)$ may correspond to images $x$ and labels $y$. 
Consider a prediction problem for which we train a model with a data set $\dataset$ of points sampled from $p(z)$. Let $V(\dataset)$ be a score function evaluating the model trained on $\dataset$. 

\textbf{Optimal Data Collection.}
We possess an initial data set $\dataset_{q_0} := \{z_i\}_{i=1}^{q_0}$ of $q_0$ points; we omit the subscript on $\dataset$ referring to its size when it is obvious.
Our problem is defined by a target score $V^* > V(\dataset_{q_0})$, a cost-per-sample $c$ of collection, a horizon of $T$ rounds, and a penalty $P$ for failure. 
At the end of each round $t \in \{1, \dots, T\}$, let $q_t$ be the current amount of data collected.
Our goal is to minimize the total collection cost while building a model that can achieve the target score:
\begin{align}\nonumber
     &\min_{q_1, \dots, q_T} \; c (q_T - q_0) + P \Ind \{ V(\dataset_{q_T}) < V^* \} \quad\;\;\;\;\;\;\;\;\;\; \st \; q_0 \leq q_1 \leq \cdots \leq q_T \\ \label{eq:formal_problem}
    =&\min_{q_1, \dots, q_T} \; c \sum_{t=1}^T(q_t - q_{t-1}) + P \Ind \{ V(\dataset_{q_T}) < V^* \} \quad\;\; \st \; q_0 \leq q_1 \leq \cdots \leq q_T
\end{align}
The collection cost is measured by the difference in data set size between the final and the 0-th round $c ( q_T - q_0 ) = c \sum_{t=1}^T(q_t - q_{t-1})$, 
Because we collect data iteratively over multiple rounds (see Figure \ref{fig:flowchart}), we break \eqref{eq:formal_problem} into the sum of differences per round. 
Specifically in each round, we 
\begin{enumerate}
    \item Decide to grow the data set to $q_t \geq q_{t-1}$ points by sampling $\hat\dataset := \{ \hat{z}_i \}_{i=1}^{q_t - q_{t-1}} \sim p(z)$.  Pay a cost $c (q_t - q_{t-1})$ and update $\dataset \gets \dataset \cup \hat\dataset$.
    
    \item Train the model and evaluate the score. If $V(\dataset) \geq V^*$, then terminate. 
    
    \item If $t = T$, then pay the penalty $P$ and terminate. Otherwise, repeat for the next round.
\end{enumerate}
%

The model score typically increases monotonically with data set size~\citep{hestness2017deep, rosenfeld2019constructive}. This means that the minimum cost strategy for~\eqref{eq:formal_problem} is to collect just enough data such that $V(\dataset_{q_T}) = V^*$.
We can estimate this minimum data requirement by modeling the score function as a stochastic process. 
Let $V_q := V(\dataset_q)$ and let $\{ V_q \}_{q \in \field{Z}_+}$ be a stochastic process whose indices represent training set sizes in different rounds. 
Then, collecting data in each round yields a sequence of subsampled data sets $\dataset_{q_{t-1}} \subset \dataset_{q_t}$ and 
their performances $V(\dataset_{q_t})$. 
The minimum data requirement is the stopping time
\begin{align}\label{eq:min_dat_req}
    D^* := \argmin_q \left\{q \;|\; V_q \geq V^* \right\}.
\end{align}
which is a random variable giving the first time that we pass the target. 
Note that $q^*_1 = \cdots = q^*_T = D^*$ is a minimum cost solution to the optimal data collection problem, incurring a total cost $c (D^* - q_0)$\footnote{We assume that $c (D^* - q_0) < P$, since otherwise the optimal strategy would be to collect no data.}.

Estimating $D^*$ using past observations of the learning curve is difficult since we have only $T$ rounds. Further,~\citet{mahmood2022howmuch} empirically show that small errors in fitting the learning curve can cause massive over- or under-collection. Thus, robust policies must capture the uncertainty of estimation.

%% file: sections/optimization_method.tex
\section{Learn-Optimize-Collect (LOC)}
\label{sec:optimization_framework}



Our solution approach, which we refer to as Learn-Optimize-Collect (LOC), minimizes the total collection cost while incorporating the uncertainty of estimating $D^*$. 
Although $D^*$ is a discrete random variable, it is realized typically on the order of thousands or greater. 
To simplify our problem and ensure differentiability, we assume that $D^*$ is continuous and has a well-defined density.
\begin{assumption}\label{ass:cdf_pdf_exists}
The random variable $D^*$ is absolutely continuous and has a cumulative density function (CDF) $F(q)$ and probability density function (PDF) $f(q) := dF(q)/dq$.
\end{assumption}
%
In Section~\ref{sec:optimization_framework_opt_problem}, we first develop an optimization model when given access to the CDF $f(q)$ and PDF $F(q)$. 
In Section~\ref{sec:optimization_framework_learning}, we estimate these distributions and combine them with the optimization model.
Finally in Section~\ref{sec:compare_vs_CVPR}, we delineate our optimization approach from prior regression methods.

\subsection{Optimization Model}
\label{sec:optimization_framework_opt_problem}

We propose an optimization problem that for any $t$, can simultaneously solve for the optimal amounts of data to collect $q_t, \dots, q_{T}$ in all future rounds. Consider $t=1$ 
and to develop intuition, suppose we know a priori the exact stopping time $D^*$. 
Then, problem~\eqref{eq:formal_problem} can be re-written as
\begin{equation} \label{eq:original_prob}
    \min_{q_1, \cdots q_T} \; L(q_1, \dots, q_T; D^*)  \qquad\qquad \st \; q_0 \leq q_1 \leq \cdots \leq q_T
\end{equation}
where the objective function is defined recursively as follows
\begin{align*}
    L(q_1, \dots, q_T; D^*) &:= c (q_1 - q_0) + \Ind \{ q_1 < D^* \} \Big( c (q_2 - q_1) + \Ind \{ q_2 < D^* \} \Big( c (q_3 - q_2) \dots \\
        & \qquad \qquad \;\; \dots + \Ind \{ q_{T-1} < D^* \} \Big( c (q_T - q_{T-1}) + P \Ind \{q_T < D^*\} \Big)\cdots\Big)\Big) \\
        &= c \sum_{t=1}^T (q_t - q_{t-1}) \prod_{s=1}^{t-1} \Ind \{ q_s < D^* \}  + P \prod_{t=1}^T  \Ind \{ q_s < D^* \} \\
        &= c \sum_{t=1}^T (q_t - q_{t-1}) \Ind \{ q_{t-1} < D^* \}  + P \Ind \{ q_T < D^* \}.
\end{align*}
The objective differs slightly from \eqref{eq:formal_problem} due to the indicator terms, which ensure that once we collect enough data, we terminate the problem.
The second line follows from gathering the terms. The third line follows from observing that $q_1 \leq q_2 \leq \cdots \leq q_T$ are constrained.

In practice, we do not know $D^*$ a priori since it is an unobserved random variable.
Instead, suppose we have access to the CDF $F(q)$. Then, we take the expectation over the objective $\field{E} [ L(q_1, \dots, q_T; D^*) ]$ to formulate a \emph{stochastic optimization problem} for determining how much data to collect:
\begin{equation} \label{eq:expected_risk_prob}
    \min_{q_1, \cdots q_T} \; c \sum_{t=1}^T (q_t - q_{t-1}) \left( 1 - F(q_{t-1}) \right) + P \left( 1 - F(q_T) \right) \quad \st \; q_0 \leq q_1 \leq \cdots \leq q_T.
\end{equation}
Note that the collection variables should be discrete $q_1, \dots, q_T \in \field{Z}_+$, but similar to the modeling of $D^*$, we relax the integrality requirement, optimize over continuous variables, and round the final solutions.
Furthermore, although problem~\eqref{eq:expected_risk_prob} is constrained, we can re-formulate it with variables $d_t := q_t - q_{t-1}$; this consequently replaces the current constraints with only non-negativity constraints $d_t \geq 0$.
Finally due to Assumption~\ref{ass:cdf_pdf_exists}, problem~\eqref{eq:expected_risk_prob} can be optimized via gradient descent.

\subsection{Learning and Optimizing the Data Requirement}
\label{sec:optimization_framework_learning}

Solving problem~\eqref{eq:expected_risk_prob} requires access to the true distribution $F(q)$, which we do not have in reality. 
In each round, given a current training data set $\dataset_{q_t}$ of $q_t$ points, we must estimate these distribution functions $F(q)$ and $f(q)$ and then incorporate them into our optimization problem. 


Given a current data set $\dataset_{q_t}$, we may sample an increasing sequence of $R$ subsets
$\dataset_{q_t/R} \subset \dataset_{2q_t/R} \subset \cdots \subset \dataset_{q_t}$, fit our model to each subset, and compute the scores to obtain a data set of the learning curve $\set{R} := \{ (rq_t/R, V(\dataset_{rq_t/R})) \}_{r=1}^R$. In order to model the distribution of $D^*$, we can take $B$ bootstrap resamples of $\set{R}$ to fit a series of regression functions and obtain corresponding estimates $\{ \hat{D}_b \}_{b=1}^B$. \rebut{Given a set of estimates of the data requirement, we estimate the PDF via Kernel Density Estimation (KDE). Finally to fit the CDF, we numerically integrate the PDF. }

In our complete framework, LOC, we first estimate $F(q)$ and $f(q)$. We then use these models to solve problem~\eqref{eq:expected_risk_prob}.
Note that in the $t$-th round of collection, we fix the prior decision variables $q_1, \dots q_{t-1}$ constant. 
Finally, we collect data as determined by the optimal solution $q^*_t$ to problem~\eqref{eq:expected_risk_prob}.
Full details of the learning and optimization steps, including the complete Algorithm, are in Appendix~\ref{sec:app_opt}.

\subsection{Comparison to~\citet{mahmood2022howmuch}}
\label{sec:compare_vs_CVPR}

Our prediction model extends the previous approach of~\citet{mahmood2022howmuch}, who consider only point estimation of $D^*$.
They (i) build the set $\set{R}$, (ii) fit a parametric function $\hat{v}(q; \btheta)$ to $\set{R}$ via least-squares minimization, and (iii) solve for $\hat{D} = \argmin_{q} \{ q \;|\; \hat{v}(q; \btheta) \geq V^* \}$. 
They use several parametric functions from the neural scaling law literature, including the power law function (i.e.,  $\hat{v}(q;\btheta) := \theta_0 q^{\theta_1} + \theta_2$~\citep{mahmood2022howmuch,hoiem2021learning} where $\btheta := \{ \theta_0, \theta_1, \theta_2 \}$), and use an ad hoc correction factor obtained by trial and error on past tasks to help decrease the failure rate.
Instead, we take bootstrap samples of $\set{R}$ to fit multiple regression functions, estimate a distribution for $\hat{D}$, and incorporate them into our novel optimization model.
Finally, we show in the next two sections that our optimization problem has analytic solutions and extends to multiple sources.

%% file: sections/theory.tex
\section{Analytic Solutions for the $T=1$ Setting}
\label{sec:main_theory}

In this section, we explore analytic solutions for problem~\eqref{eq:expected_risk_prob}. 
The unobservable $D^*$ and sequential decision-making nature suggest this problem can be formulated as a Partially Observable Markov Decision Process (POMDP) with an infinite state and action space (see Appendix~\ref{sec:app_mdp_alternatives}), but such problems rarely permit exact solution methods~\citep{zhang2022dynamic}.
Nonetheless, we can derive exact solutions for the simple case of a single $T=1$ round, re-stated below
\begin{equation}\label{eq:expected_risk_prob_one_round}
    \min_{q_1} \; c(q_1 - q_0) + P(1 - F(q_1)) \qquad\qquad\qquad \st \; q_0 \leq q_1
\end{equation}
\begin{thm}\label{thm:t1_has_analytic_sol}
    Assume $F(q)$ is strictly increasing and continuous. If there exists $q_1 \geq q_0, \hat\epsilon \geq 0$ where  
    \begin{align}\label{eq:expected_risk_prob_one_round_required_condition}
        \frac{c}{P} \leq \frac{F(q_1) - F(q_0)}{q_1 - q_0}, \qquad \hat\epsilon \leq 1 - F(q_0), \qquad P = c / f ( F^{-1}(1 - \hat\epsilon) )
    \end{align}
    then there exists an $\epsilon \leq 1 - F(q_0)$ that satisfies
    $P = c / f ( F^{-1}(1 - \epsilon) )$ and
    an optimal solution to the corresponding problem~\eqref{eq:expected_risk_prob_one_round} is $q_1^* := F^{-1}(1 - \epsilon)$. 
    Otherwise, the optimal solution is $q_1^* := q_0$.
\end{thm}
When the penalty $P$ is specified via a failure risk $\epsilon$, the optimal solution to problem~\eqref{eq:expected_risk_prob_one_round} is equal to a quantile of the distribution of $D^*$. 
We defer the proof and some auxiliary results to Appendix~\ref{sec:app_proofs}.

Theorem~\ref{thm:t1_has_analytic_sol} further provides guidelines on choosing values for the cost and penalty parameters.  
While $c$ is the dollar-value cost per-sample, which includes acquisition, cleaning, and annotation, $P$ can reflect their inherent regret or opportunity cost of failing to meet their target score. 
A designer can accept a risk $\epsilon$ of failing to collect enough data $\mathrm{Pr} \{ q^* < D^* \} = \epsilon$. 
From Theorem~\ref{thm:t1_has_analytic_sol}, their optimal strategy should be to collect $F^{-1}(1- \epsilon)$ points, which is also the optimal solution to problem~\eqref{eq:expected_risk_prob_one_round}.

%% file: sections/multivariate.tex


\section{The Multi-variate LOC: Collecting Data from Multiple Sources}
\label{sec:multivariate}

So far, we have assumed that a designer only chooses how much data to collect and must pay a fixed per-sample collection cost. 
We now explore the multi-variate extension of the data collection problem where there are different types of data with different costs. 
For example, consider long-tail learning where samples for some rare classes are harder to obtain and thus, more expensive~\citep{haixiang2017learning}, semi-supervised learning where labeling data may cost more than collecting unlabeled data~\citep{van2020survey},
or domain adaptation where a source data set is easier to obtain than a target set~\citep{ben2010theory}. 
In this section, we highlight our main formulation and defer the complete multi-variate LOC to Appendix~\ref{sec:app_multivariate}.

Consider $K \in \field{N}$ data sources (e.g., $K=2$ with labeled and unlabeled) and for each $k \in \{1, \dots, K\}$, let $z^k \sim p_k(z^k)$ be data drawn from their distribution. We train a model with a data set $\dataset := \cup_{k=1}^K \dataset^k$ where each $\dataset^k$ contains points of the $k$-th source. 
The performance or score function of our model is $V(\dataset^1, \dots, \dataset^K)$. For each $k$, we initialize with $q_0^k$ points. Let $\bq_0 = (q^1_0, \dots, q^K_0)^\tpose$ denote the vector of data set sizes and let $\bc = (c^1, \dots, c^K)^\tpose$ denote costs (i.e., $c^k$ is the cost of collecting data from $p_k(z^k)$). 
Given a target $V^*$, penalty $P$, and $T$ rounds, we want to minimize the total cost of collection
\begin{equation*}
    \min_{\bq_1, \dots, \bq_T} \; \bc^\tpose \sum_{t=1}^T(\bq_t - \bq_{t-1}) + P \Ind \{ V(\dataset_{q_T^1}, \dots, \dataset_{q_T^K}) < V^* \} \quad\;\; \st \; \bq_0 \leq \bq_1 \leq \bq_2 \leq \cdots \leq \bq_T
\end{equation*}
We follow the same steps shown in Section~\ref{sec:optimization_framework} for this problem. First, the learning curve is now a stochastic process $\{ V_{\bq} \}_{\bq \in \field{Z}_+^K}$ indexed in $K$ dimensions. Next, the multi-variate analogue of the minimum data requirement in~\eqref{eq:min_dat_req} is the minimum cost amount of data needed to meet the target: 
%
\begin{align*}
    \bD^* := \argmin_{\bq} \left\{ \bc^\tpose \bq \;|\; V_{\bq} \geq V^* \right\}
\end{align*}
We randomly pick a unique solution to break ties.
From Assumption~\ref{ass:cdf_pdf_exists}, $\bD^*$ is a random vector with a PDF $f(\bq)$ and a CDF $F(\bq) := \int^{\bq}_{\bzero} f(\bhq) d\bhq$. 
Finally, the multi-variate analogue of problem~\eqref{eq:expected_risk_prob} is 
%
\begin{align}\label{eq:expected_risk_multivariate}
    \min_{\bq_1, \cdots, \bq_T} \; \bc^\tpose \sum_{t=1}^T (\bq_t - \bq_{t-1}) \left( 1 - F(\bq_{t-1}) \right) + P \left( 1 - F(\bq_T) \right) \; \st \; \bq_0 \leq \bq_1 \leq \cdots \leq \bq_T
\end{align}
The Multi-variate LOC requires multi-variate PDFs, which we can fit in the same way as discussed in Section~\ref{sec:optimization_framework_learning}. 
However, we now need multi-variate regression functions that can accommodate different types of data.
In Appendix~\ref{sec:app_multivariate}, we propose an additive family of power law regression functions that can handle an arbitrary number of $K$ sources. In our experiments, we also generalize the estimation approach of \citet{mahmood2022howmuch} to the multi-source setting for comparison.

%% file: sections/results.tex
\section{Empirical Results}
\label{sec:numerical_results}

We explore the data collection problem over two sets of experiments covering single-variate $K=1$ (Section~\ref{sec:optimization_framework}) and multi-variate $K = 2$ (Section~\ref{sec:multivariate}) problems. We consider image classification, segmentation, and object detection tasks. 
For every data set and task, LOC significantly reduces the number of instances where we fail to meet a data requirement $V^*$, while incurring a competitive cost with respect to the conventional baseline of na\"{i}vely estimating the data requirement~\citep{mahmood2022howmuch}. 

In this section, we summarize the main results. We detail our data collection and experiment setup in Appendix~\ref{sec:app_experiments}. We expand our full results \rebut{and experiments with additional baselines} in Appendix~\ref{sec:app_experiment_results} .

\subsection{Data and Methods}

When $K=1$, the designer decides how much data to sample without controlling the type of data. We explore classification on CIFAR-10~\citep{krizhevsky2009learning}, CIFAR-100~\citep{krizhevsky2009learning}, and ImageNet~\citep{deng2009imagenet}, where we train ResNets~\citep{he2016deep} to meet a target validation accuracy. 
We explore semantic segmentation using Deeplabv3~\citep{chen2017rethinking} on BDD100K~\citep{yu2020bdd100k}, which is a large-scale driving data set, as well as Bird's-Eye-View (BEV) segmentation on nuScenes~\citep{nuscenes2019} using the `Lift Splat' architecture~\citep{liftsplat}; for both tasks, we desire a target mean intersection-over-union (IoU).
We explore 2-D object detection on PASCAL VOC~\citep{pascal-voc-2007, pascal-voc-2012} using SSD300~\citep{liu2016ssd}, where we evaluate mean average precision (mAP).

When $K=2$, the designer collects two types of data with different costs.
We first divide CIFAR-100 into two subsets containing data from the first and last 50 classes, respectively. 
Here, we assume that the first 50 classes are more expensive to collect than the last; this mimics a real-world scenario where collecting data for some classes (e.g., long-tail) is more expensive than others.
We then explore semi-supervised learning on BDD100K where the labeled subset of this data is more expensive than the unlabeled data; the cost difference between these two types is equal to the cost of data annotation.

We use a simulation model of the deep learning workflow following the procedure of~\citet{mahmood2022howmuch}, to approximate the true problem while simplifying the experiments (see Appendix~\ref{sec:app_experiments} for full details). 
To avoid repeatedly sampling data, re-training a model, and evaluating the score, each simulation uses a piecewise-linear approximation of a `ground truth' learning curve that returns model performance as a function of data set size.
In our problems, we initialize with $q_0 = 10\%$ of the full data set (we use $20\%$ for VOC). 
Then in each round, we solve for the amount of data to collect and then call the piecewise-linear learning curve to obtain the current score.

We compare LOC against the conventional estimation approach of~\citet{mahmood2022howmuch} who fit a regression model to the learning curve statistics, extrapolate the learning curve for larger data sets, and then solve for the minimum data requirement under this extrapolation.
There are many different regression models that can be used to fit learning curves~\citep{jones2003introduction, figueroa2012predicting, hestness2017deep, hoiem2021learning}. Since power laws are the most commonly studied approach in the neural scaling law literature, we focus on these.
In Appendix \ref{sec:app_experiment_alternate_regression_functions}, we show that our optimization approach can be incorporated with other regression models.

\subsection{Main Results}

We consider $T=1, 3, 5$ rounds and $V^* \in [V(\dataset_{q_0}) + 1, V(\dataset)]$ targets, where $\dataset$ is the entire data set.
We evaluate all methods on (i) the failure rate, which is how often the method fails to achieve the given $V^*$ within $T$ rounds, and (ii) the cost ratio, which is the suboptimality of an algorithm for solving problem \eqref{eq:expected_risk_prob}, i.e., $\bc^\tpose (\bq^*_T - \bq_0) / \bc^\tpose (\bD^* - \bq_0) - 1$. Note that the suboptimality does not count the penalty for failure since this would distort the average metrics.
For $K=1$, we also measure the ratio of points collected $q_T^* / D^*$.
Although there is a natural trade-off between low cost ratio (under-collecting) and failure rate (over-collecting), we emphasize that our goal is to have low cost but with zero chance of failure.

\begin{figure*}[!t]
\begin{center}
\begin{minipage}{0.16\linewidth}\includegraphics[width=1\textwidth]{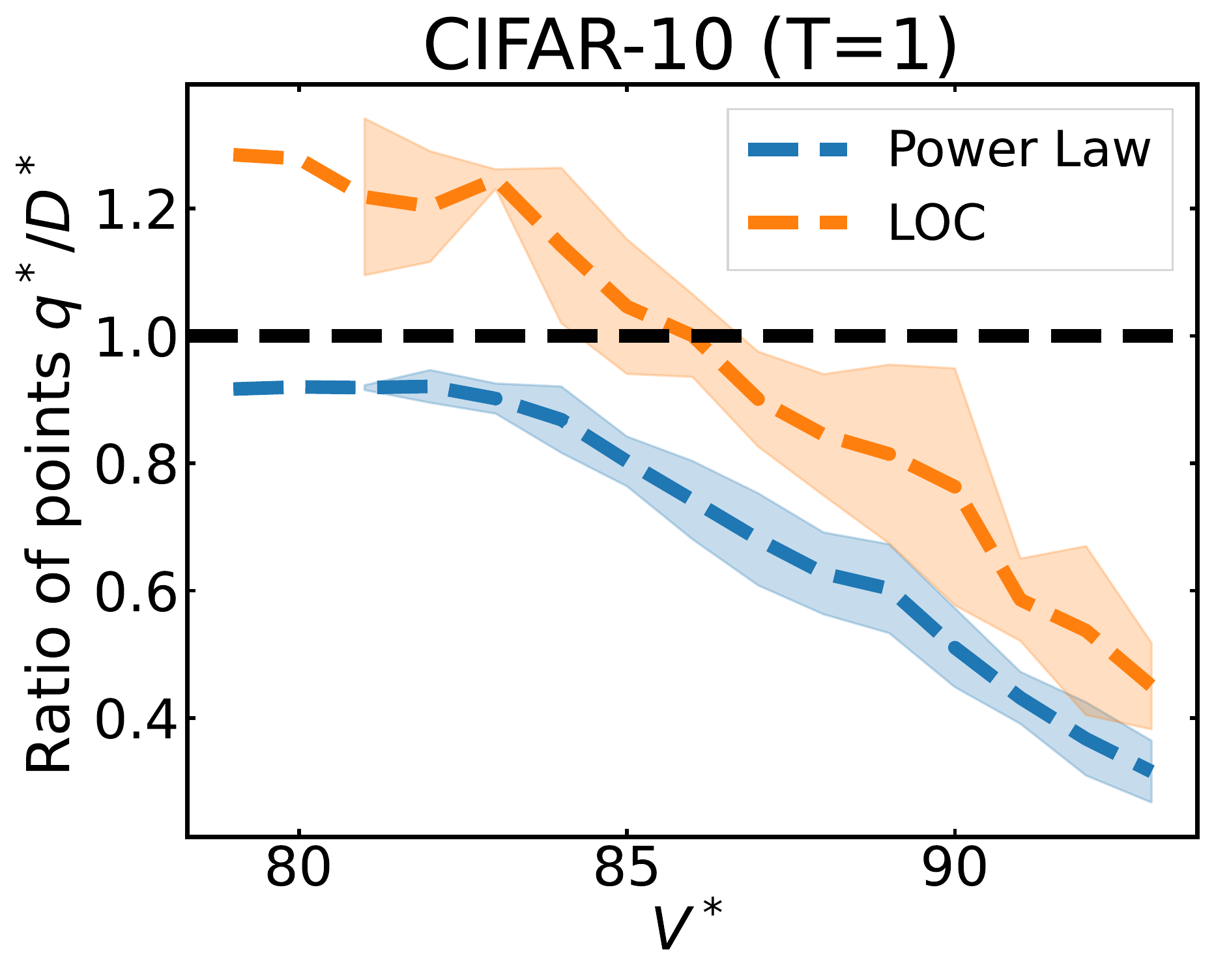}\end{minipage}
\begin{minipage}{0.16\linewidth}\includegraphics[width=1\textwidth]{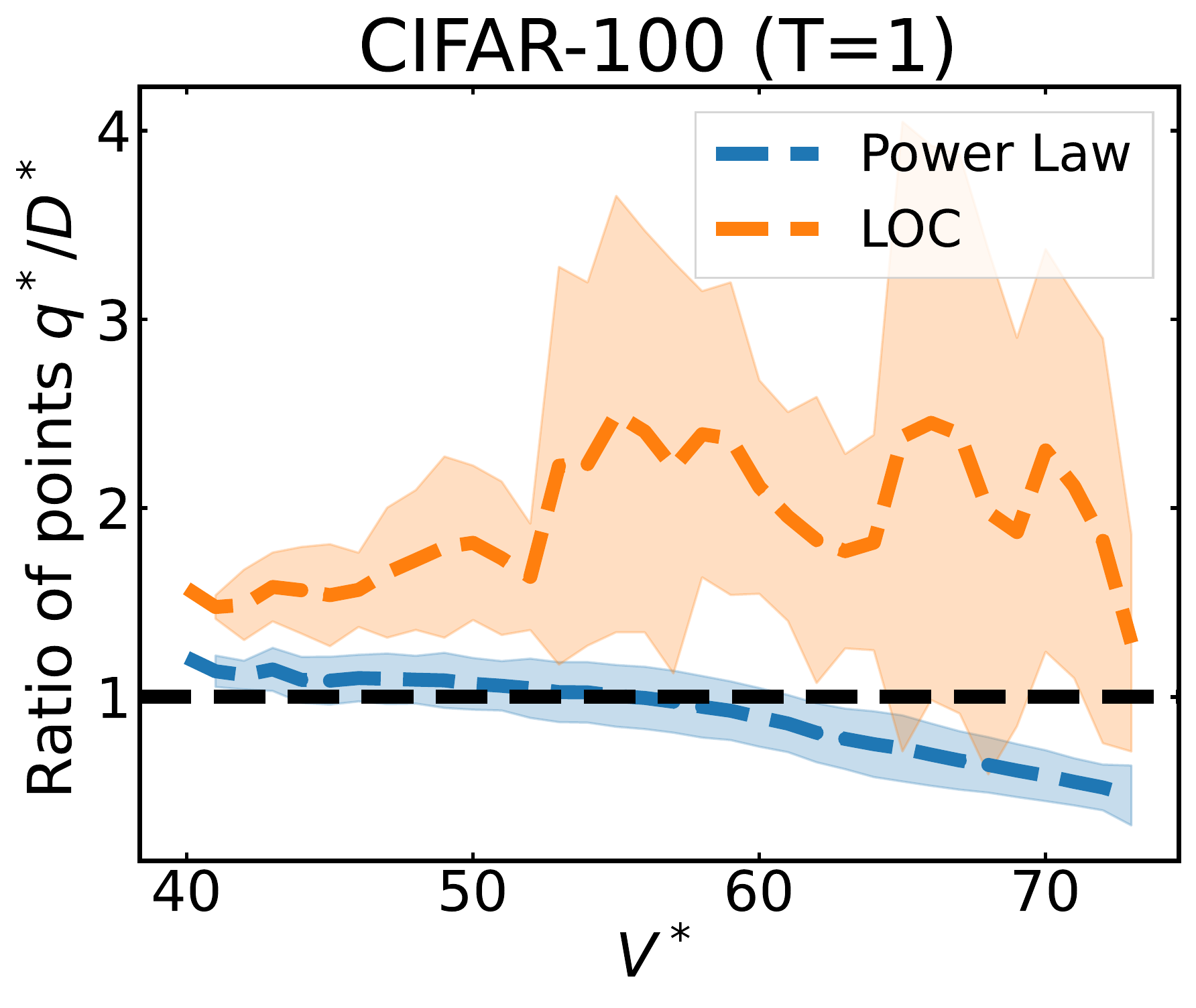}\end{minipage}
\begin{minipage}{0.16\linewidth}\includegraphics[width=1\textwidth]{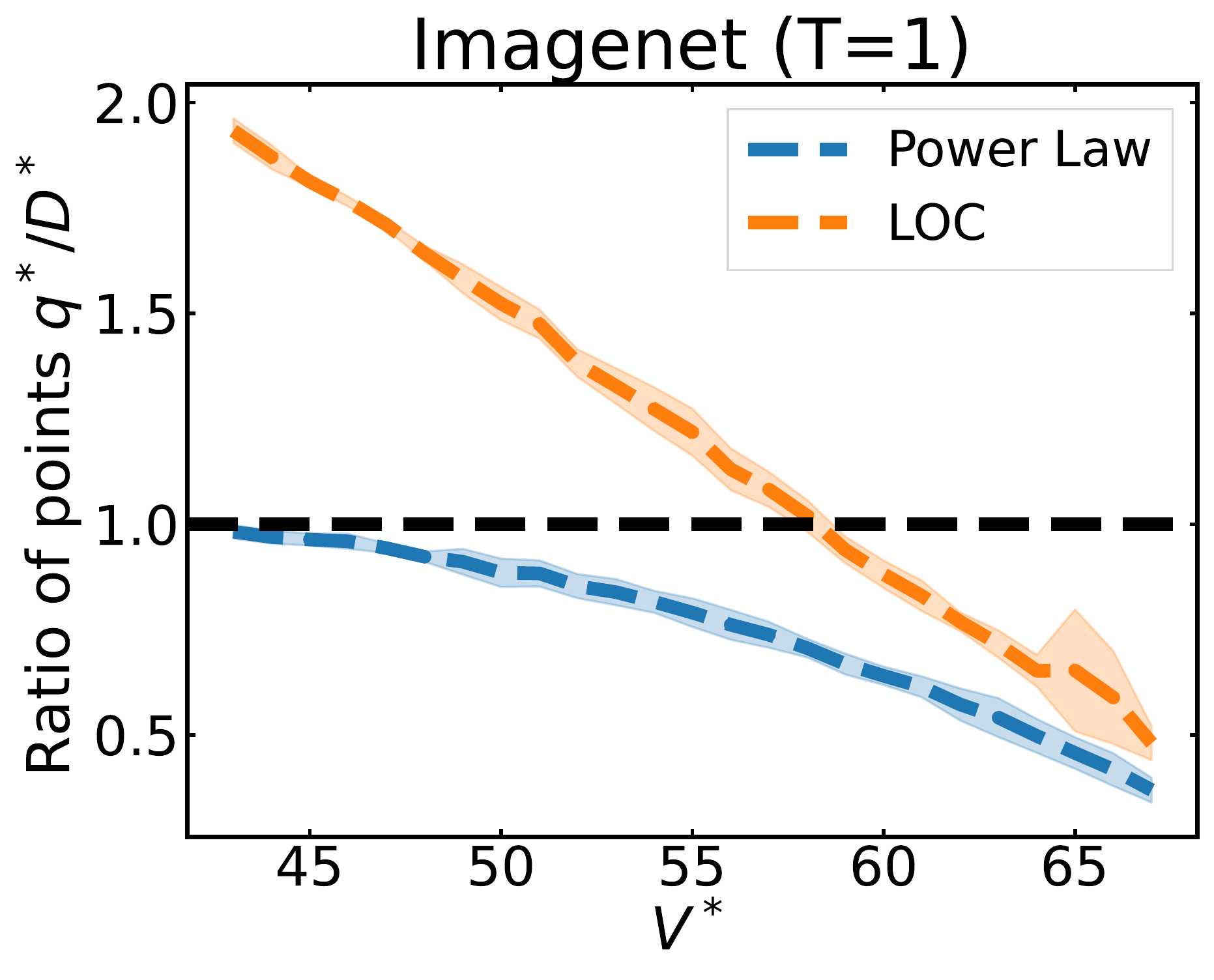}\end{minipage}
\begin{minipage}{0.16\linewidth}\includegraphics[width=1\textwidth]{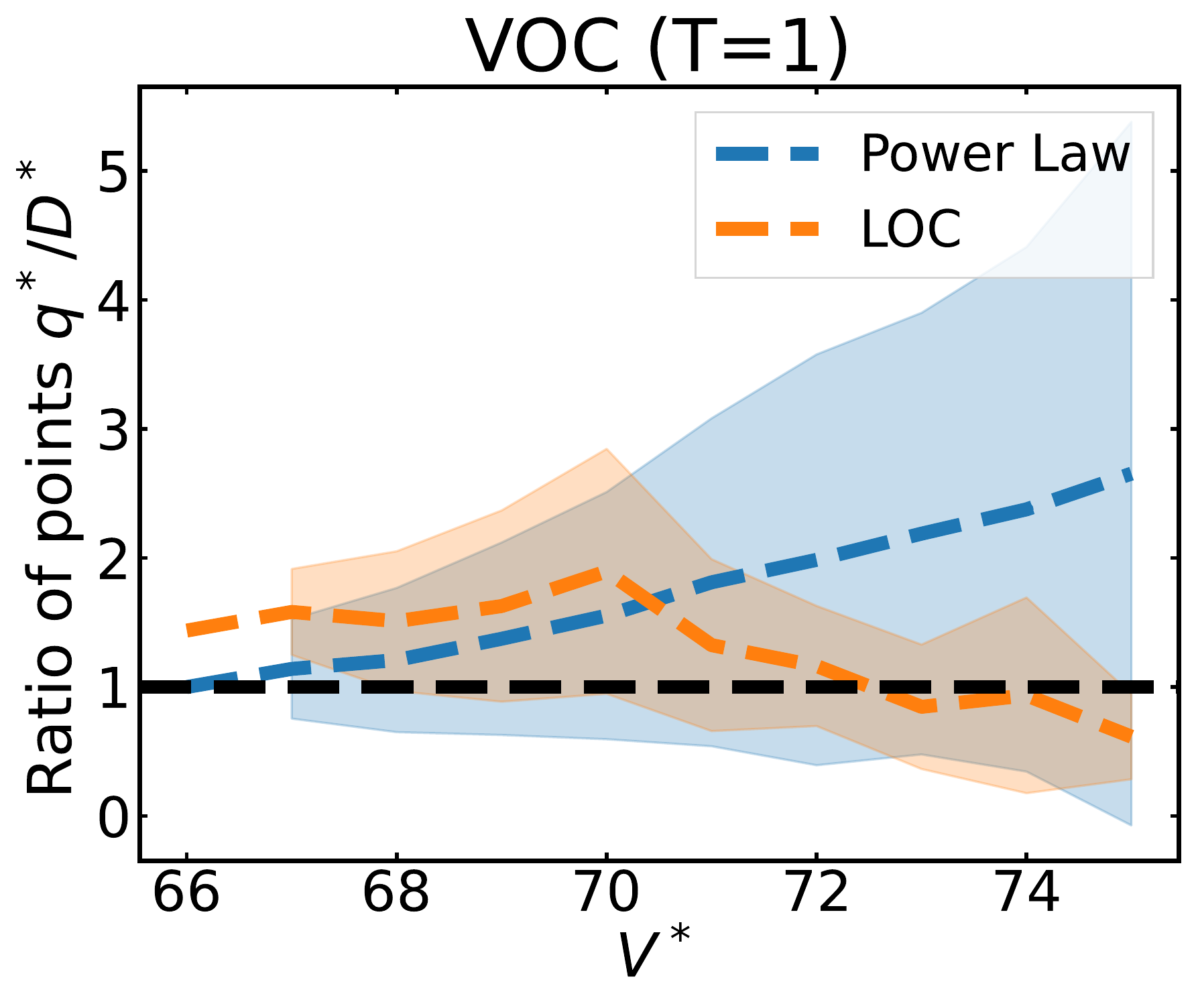}\end{minipage}
\begin{minipage}{0.16\linewidth}\includegraphics[width=1\textwidth]{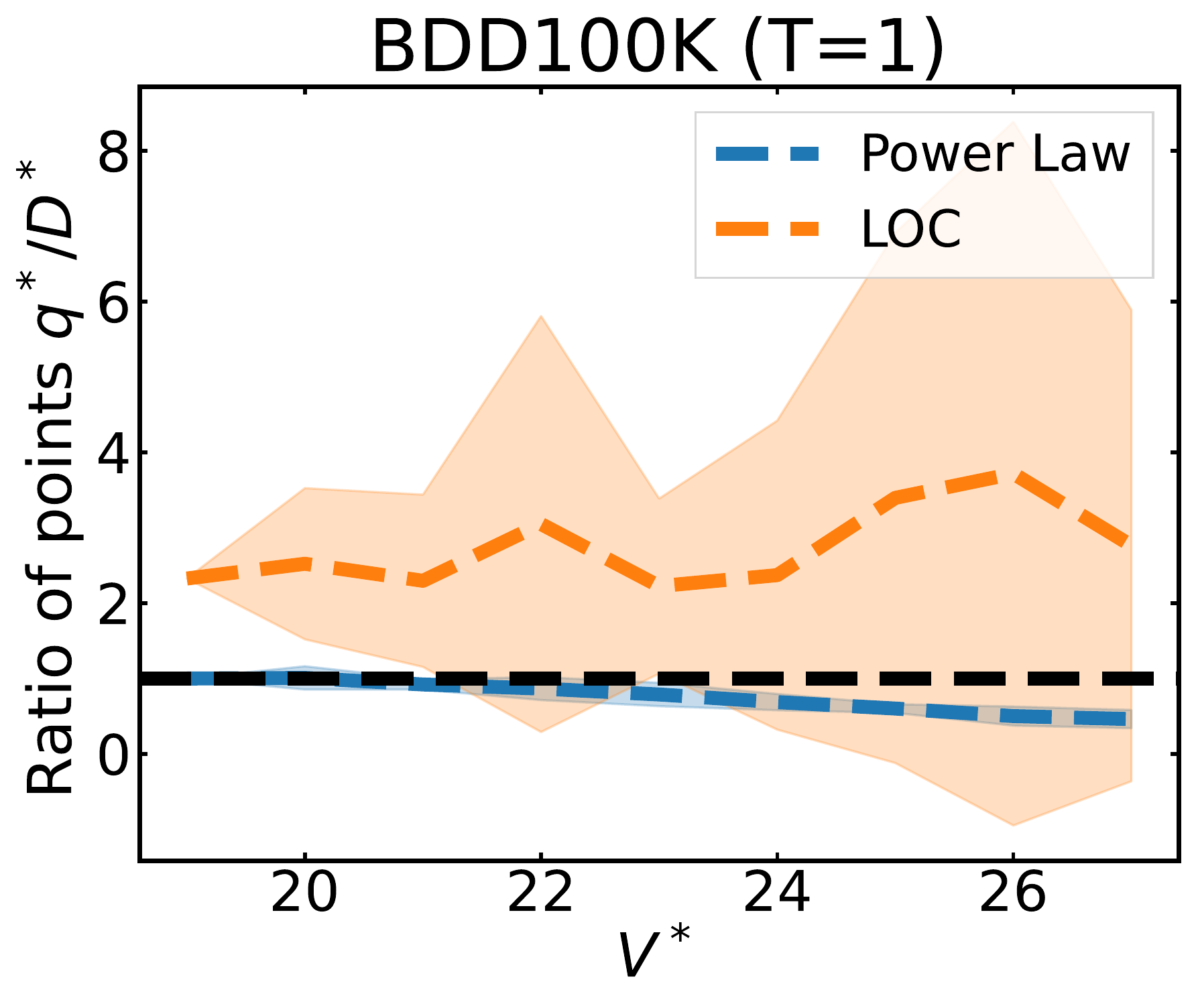}\end{minipage}
\begin{minipage}{0.16\linewidth}\includegraphics[width=1\textwidth]{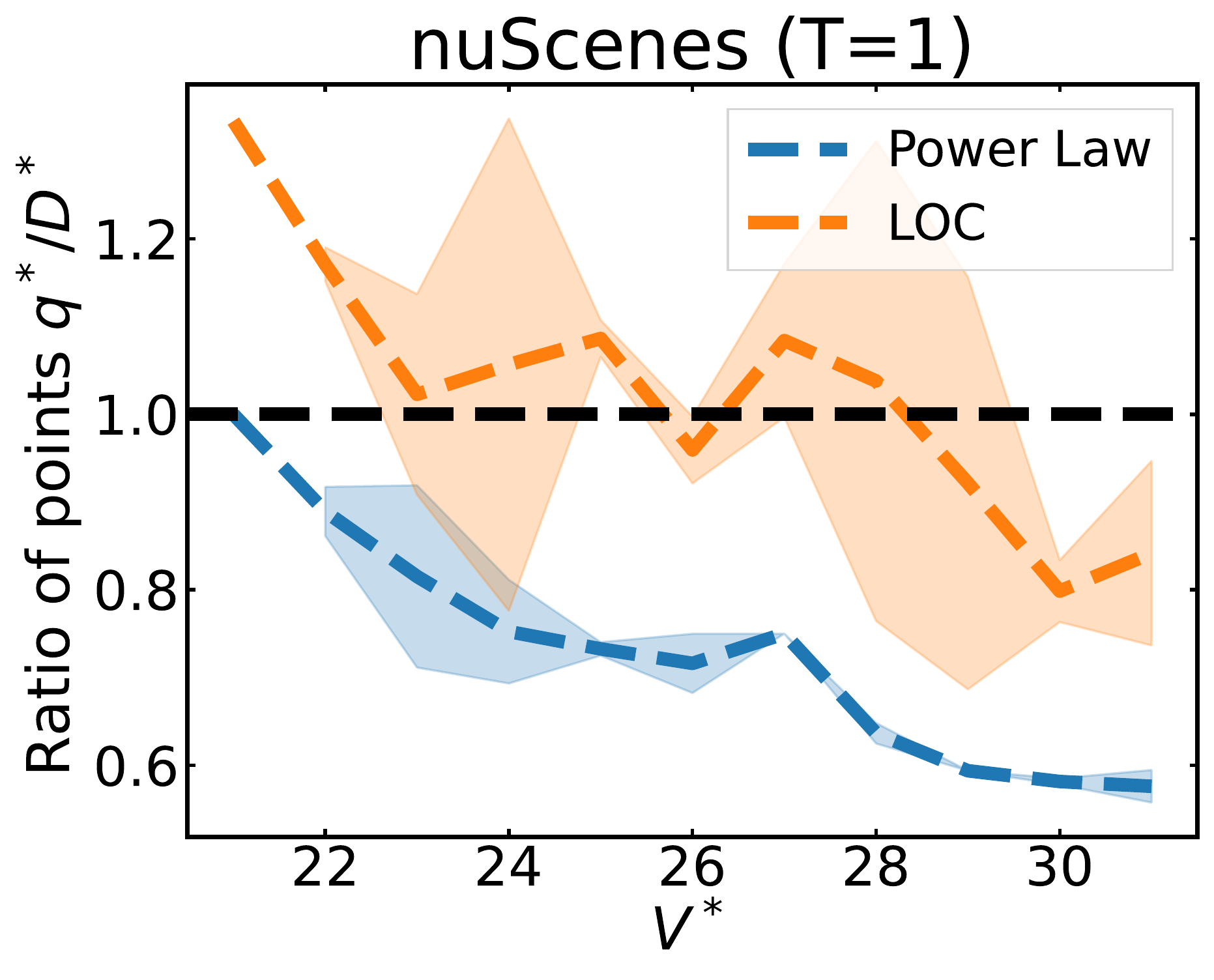}\end{minipage}
\begin{minipage}{0.16\linewidth}\includegraphics[width=1\textwidth]{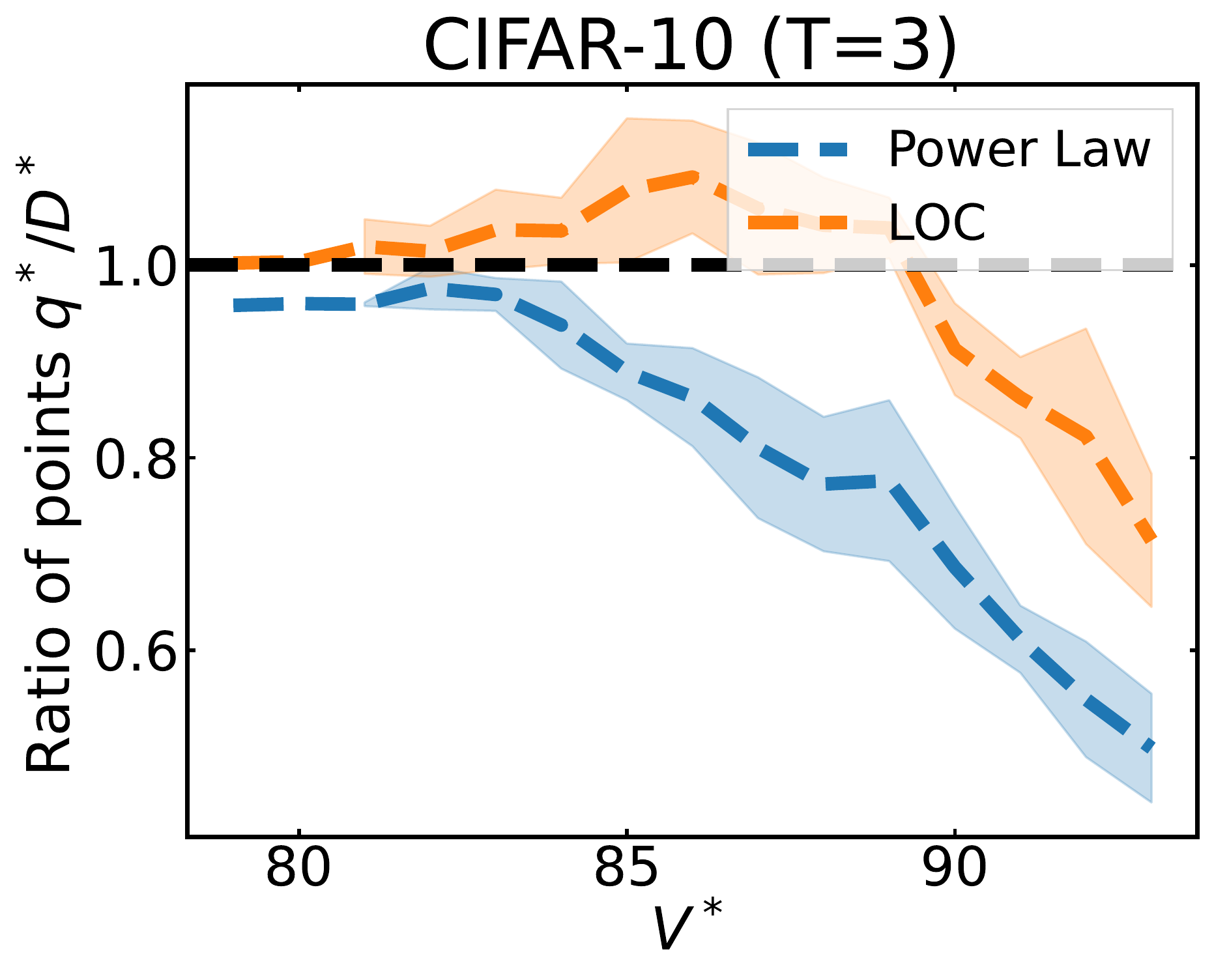}\end{minipage}
\begin{minipage}{0.16\linewidth}\includegraphics[width=1\textwidth]{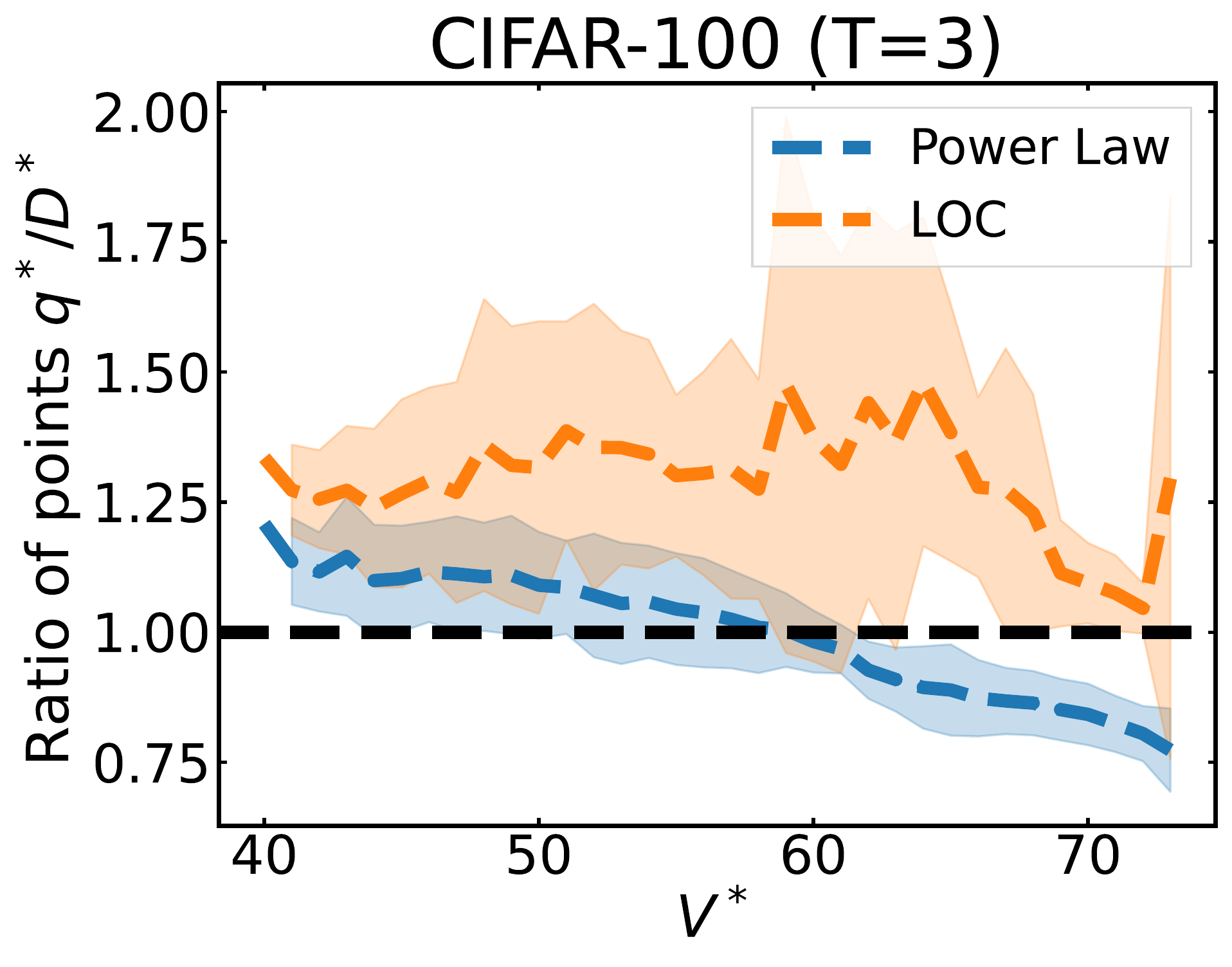}\end{minipage}
\begin{minipage}{0.16\linewidth}\includegraphics[width=1\textwidth]{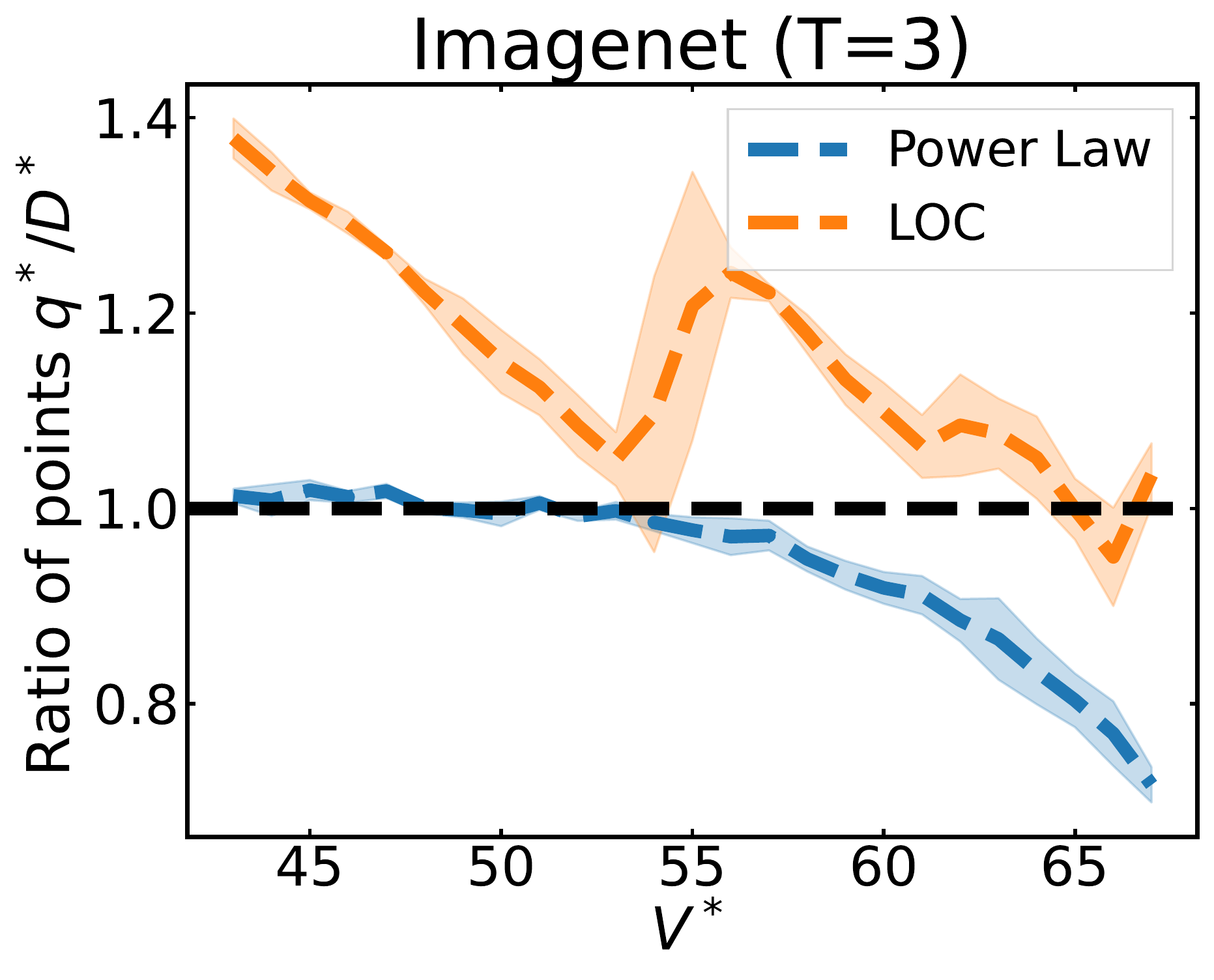}\end{minipage}
\begin{minipage}{0.16\linewidth}\includegraphics[width=1\textwidth]{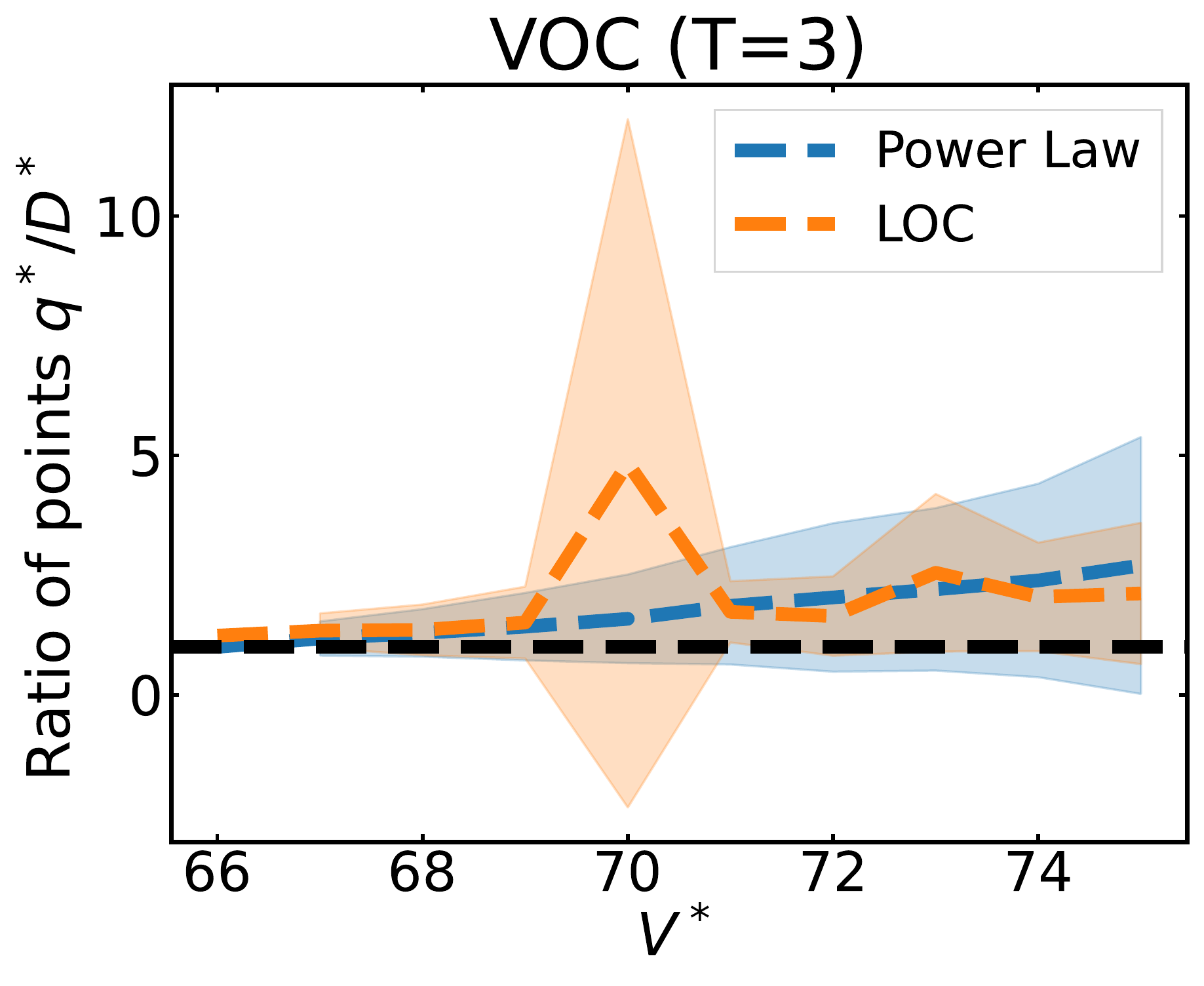}\end{minipage}
\begin{minipage}{0.16\linewidth}\includegraphics[width=1\textwidth]{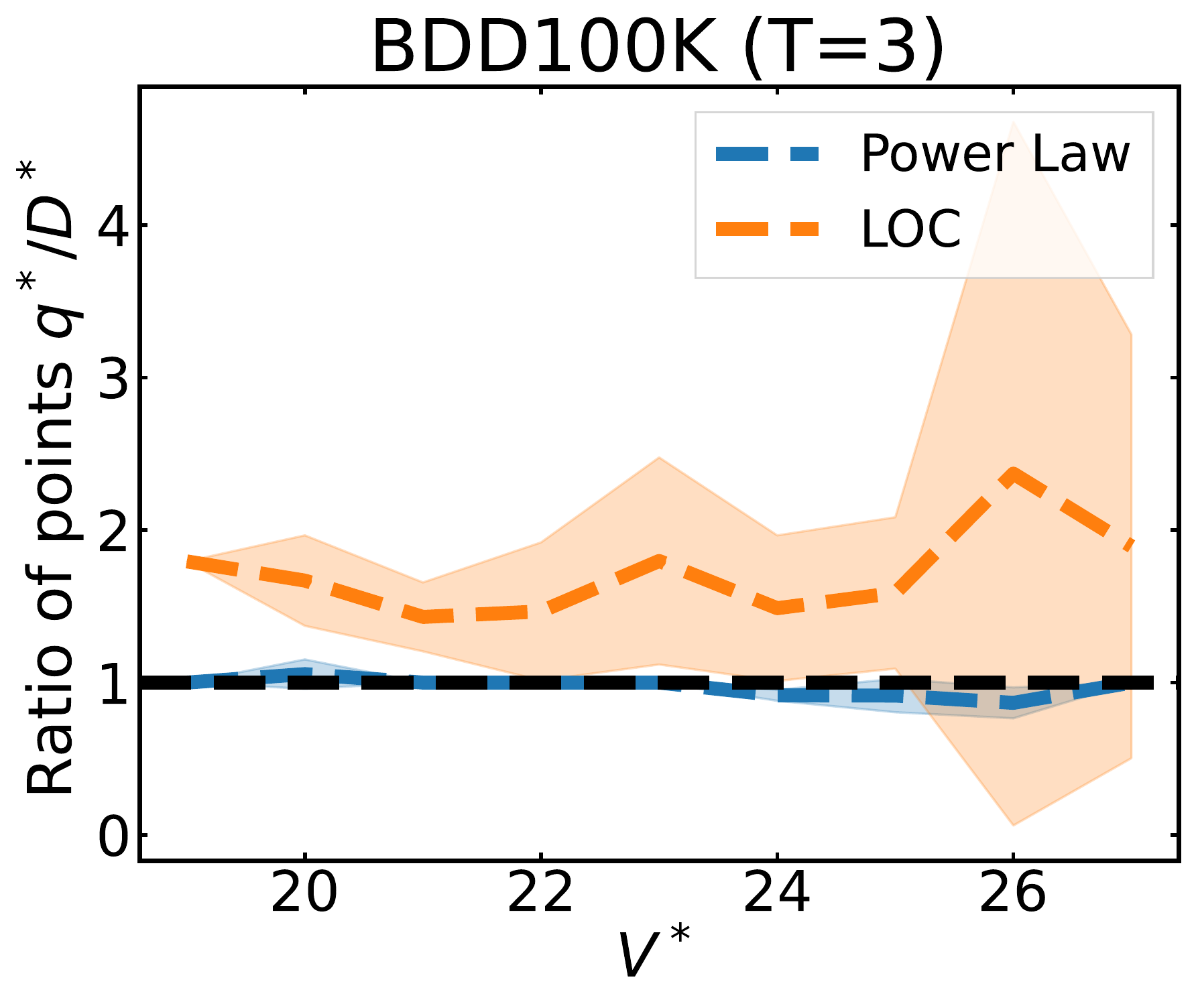}\end{minipage}
\begin{minipage}{0.16\linewidth}\includegraphics[width=1\textwidth]{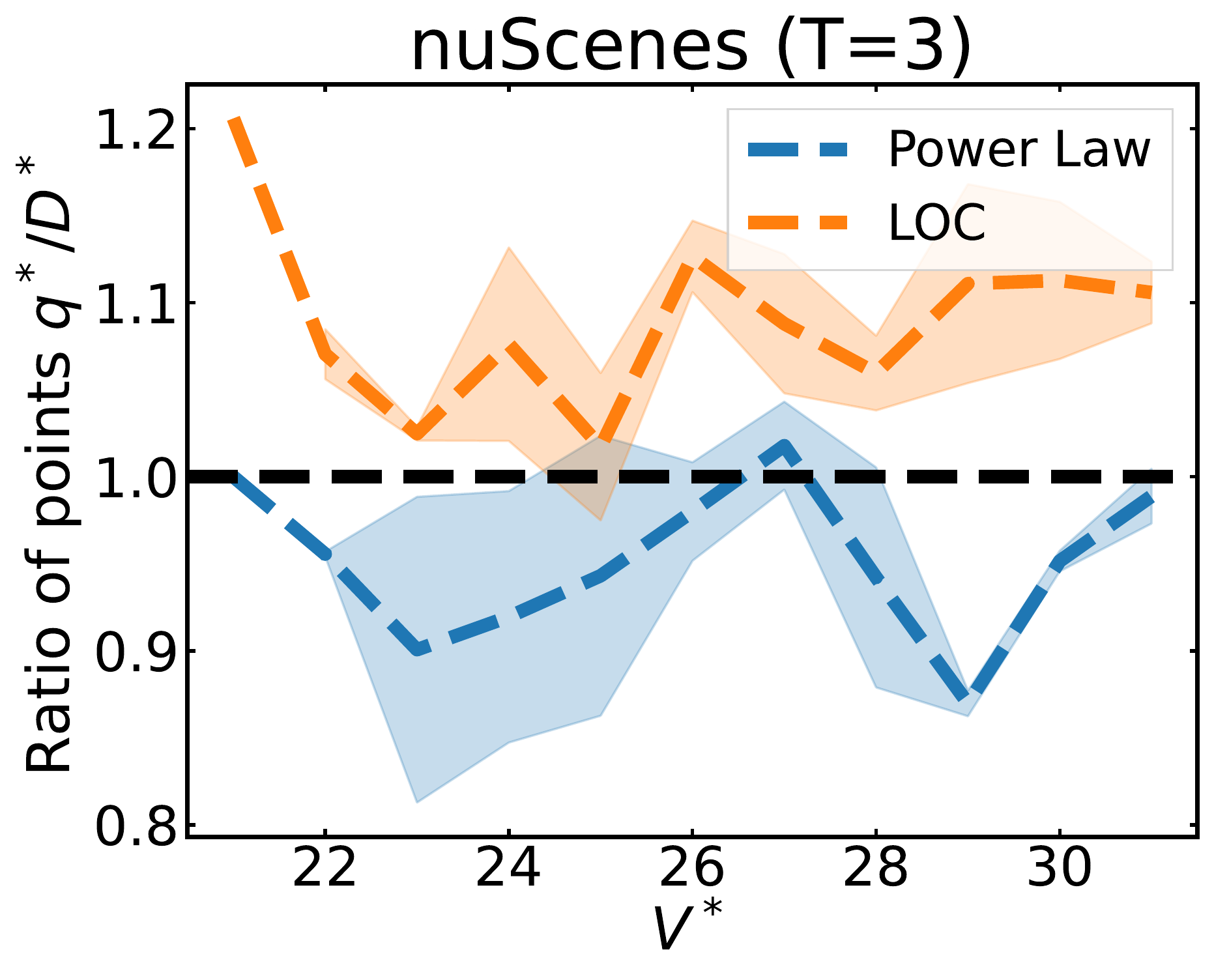}\end{minipage}
\begin{minipage}{0.16\linewidth}\includegraphics[width=1\textwidth]{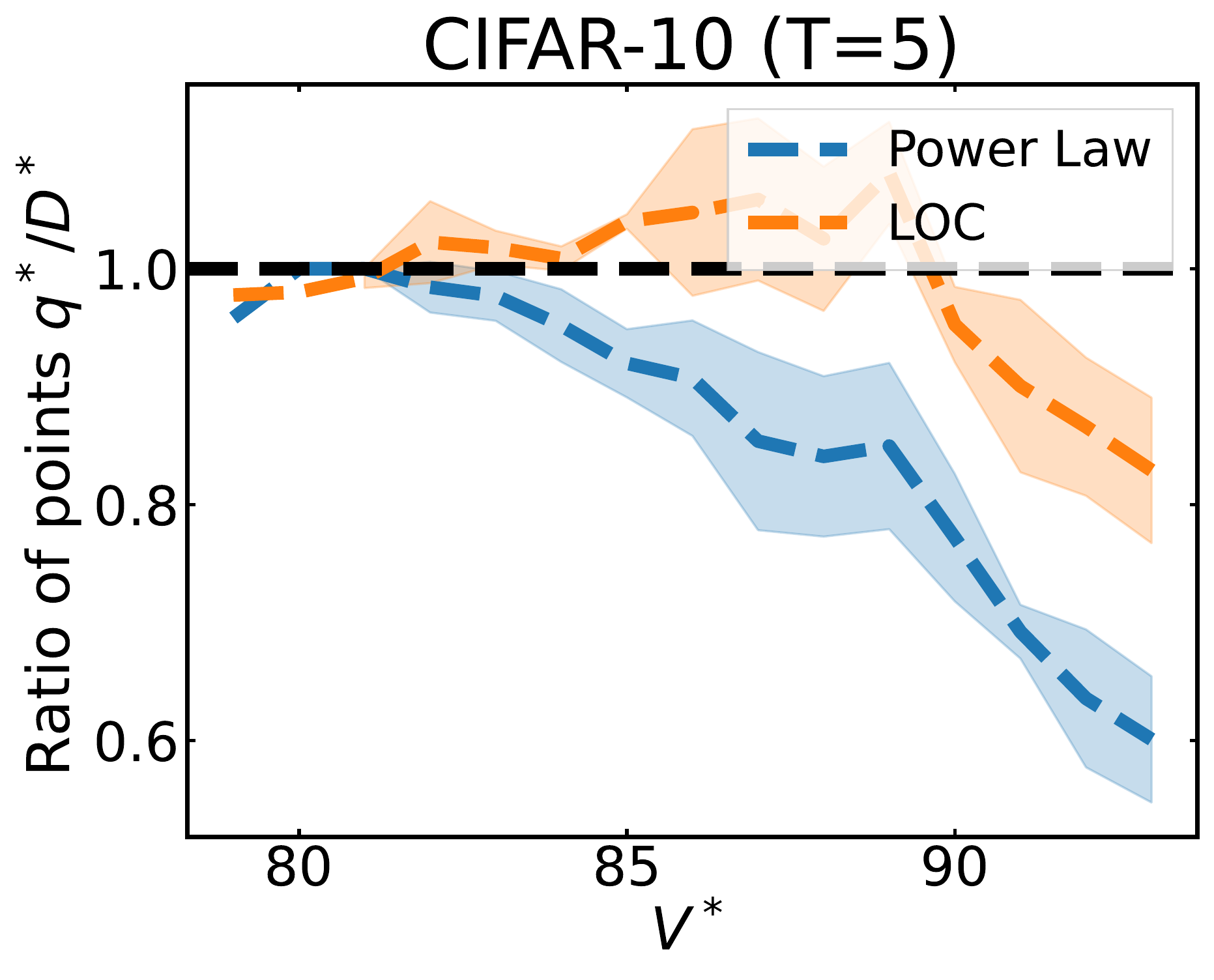}\end{minipage}
\begin{minipage}{0.16\linewidth}\includegraphics[width=1\textwidth]{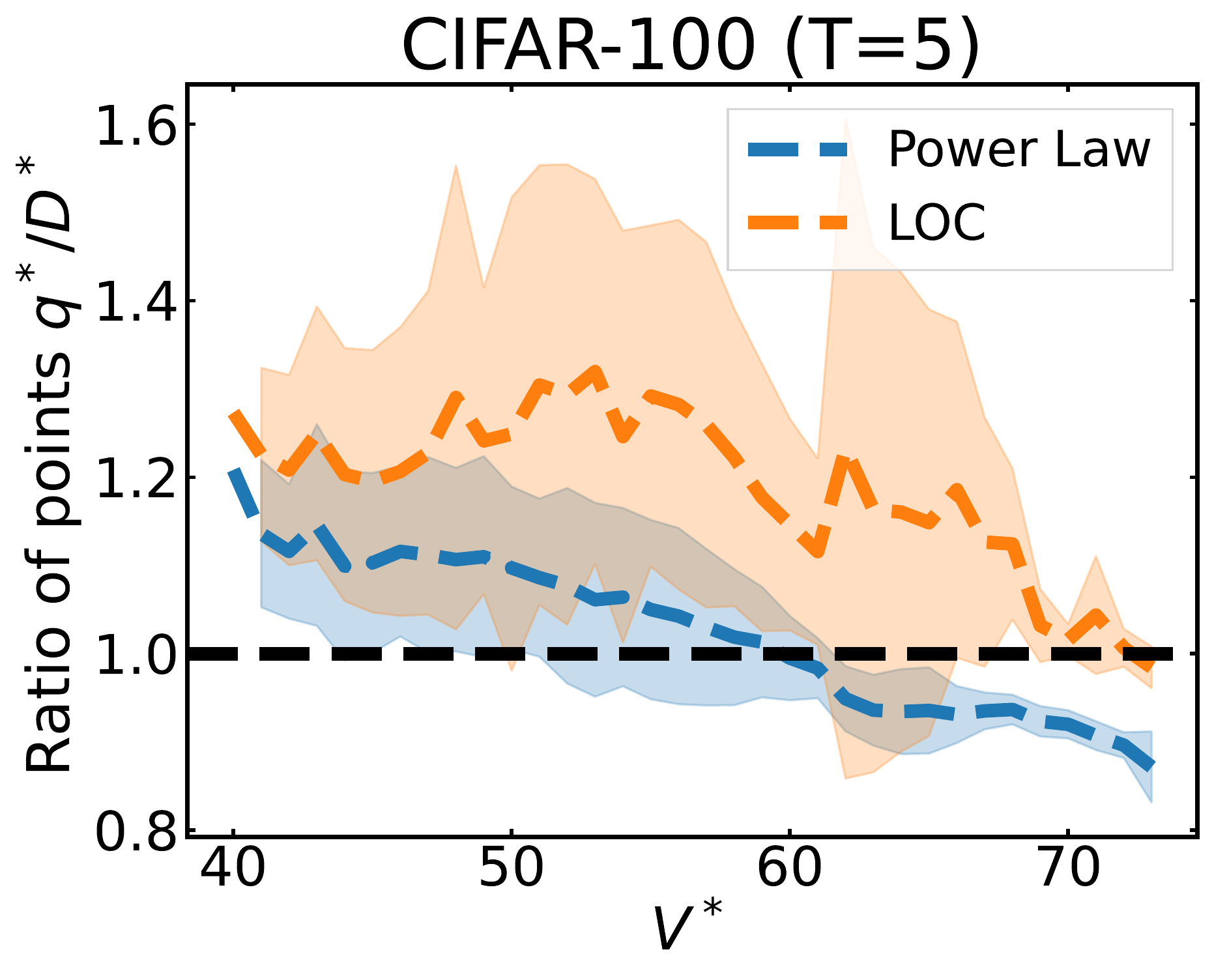}\end{minipage}
\begin{minipage}{0.16\linewidth}\includegraphics[width=1\textwidth]{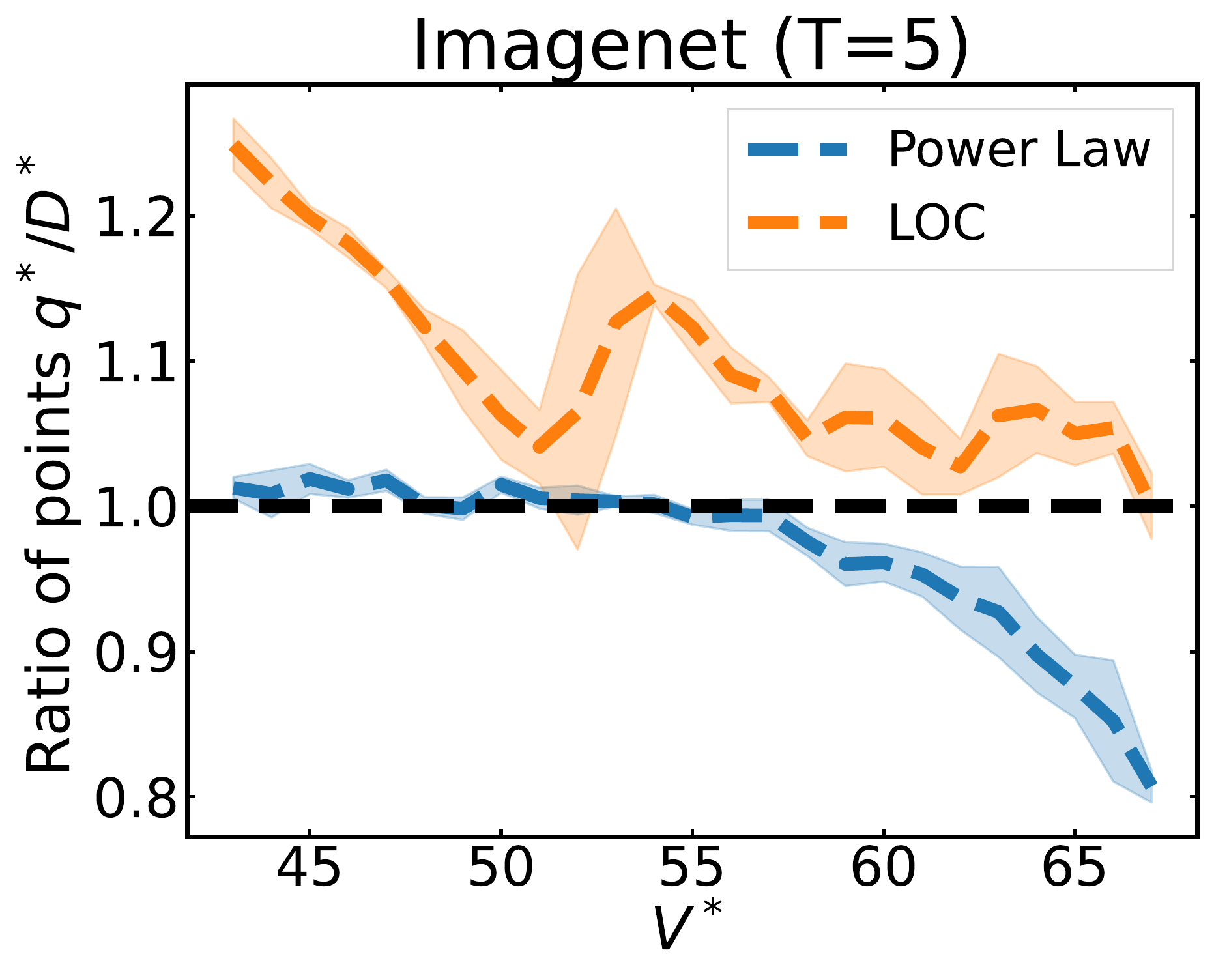}\end{minipage}
\begin{minipage}{0.16\linewidth}\includegraphics[width=1\textwidth]{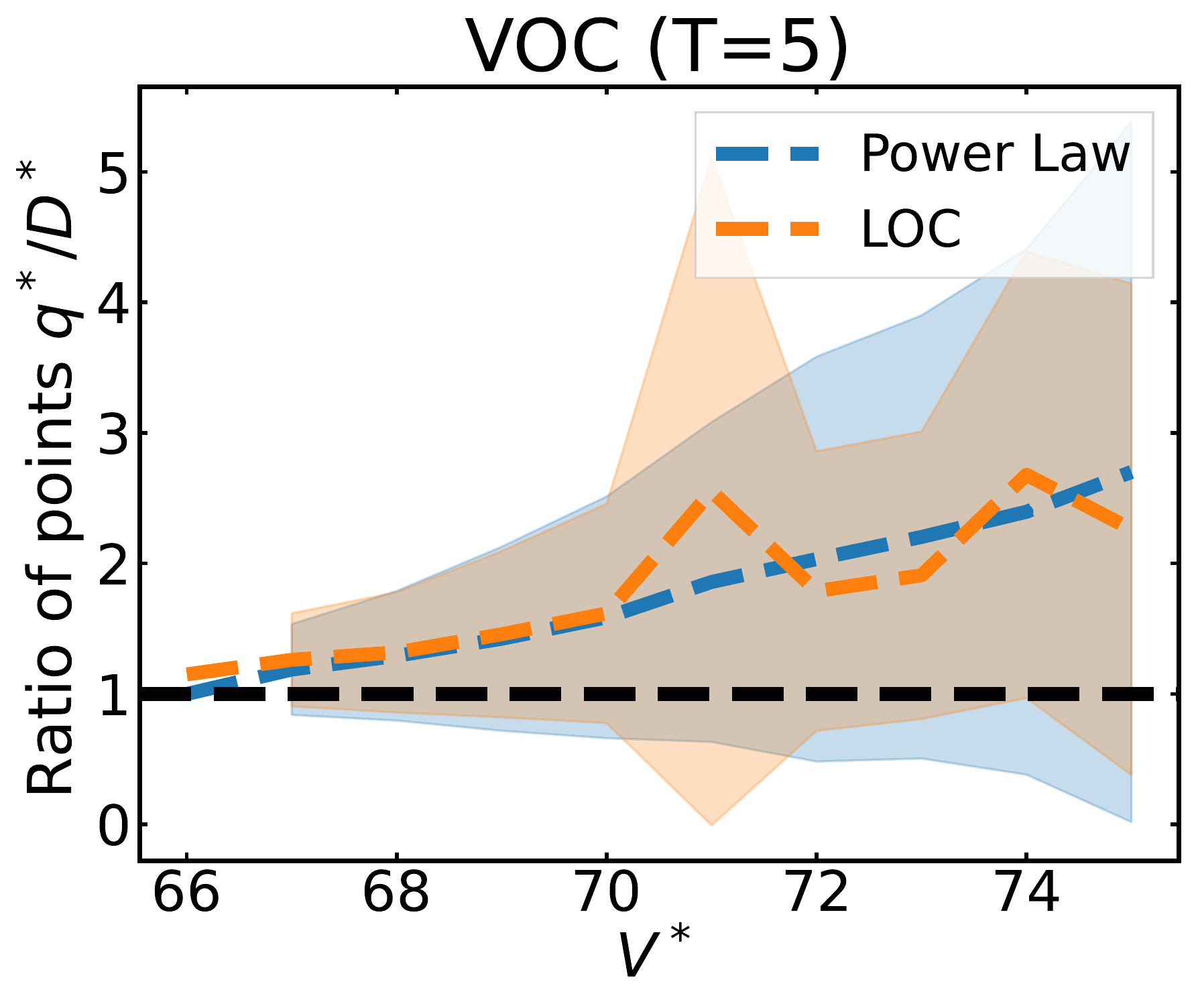}\end{minipage}
\begin{minipage}{0.16\linewidth}\includegraphics[width=1\textwidth]{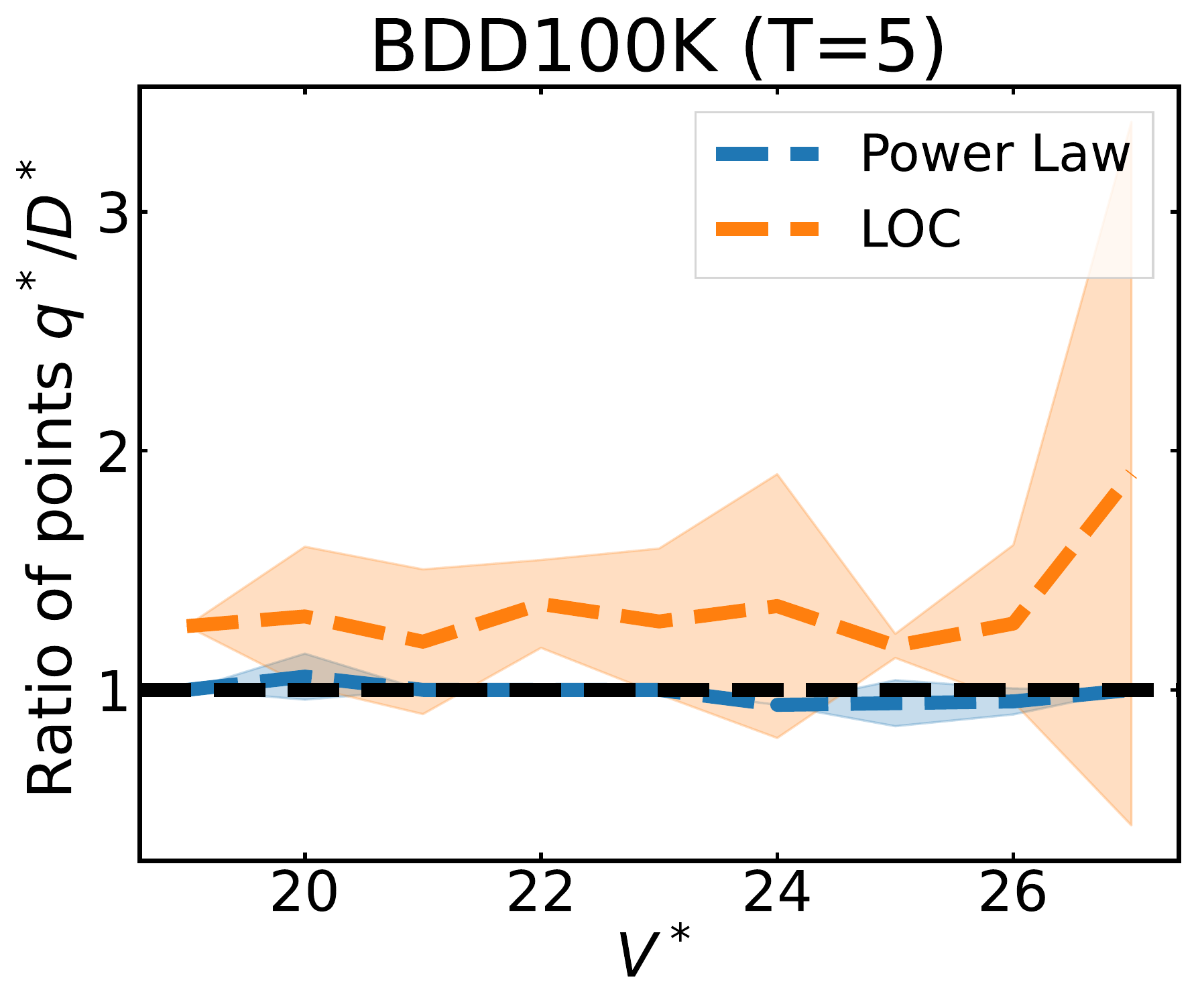}\end{minipage}
\begin{minipage}{0.16\linewidth}\includegraphics[width=1\textwidth]{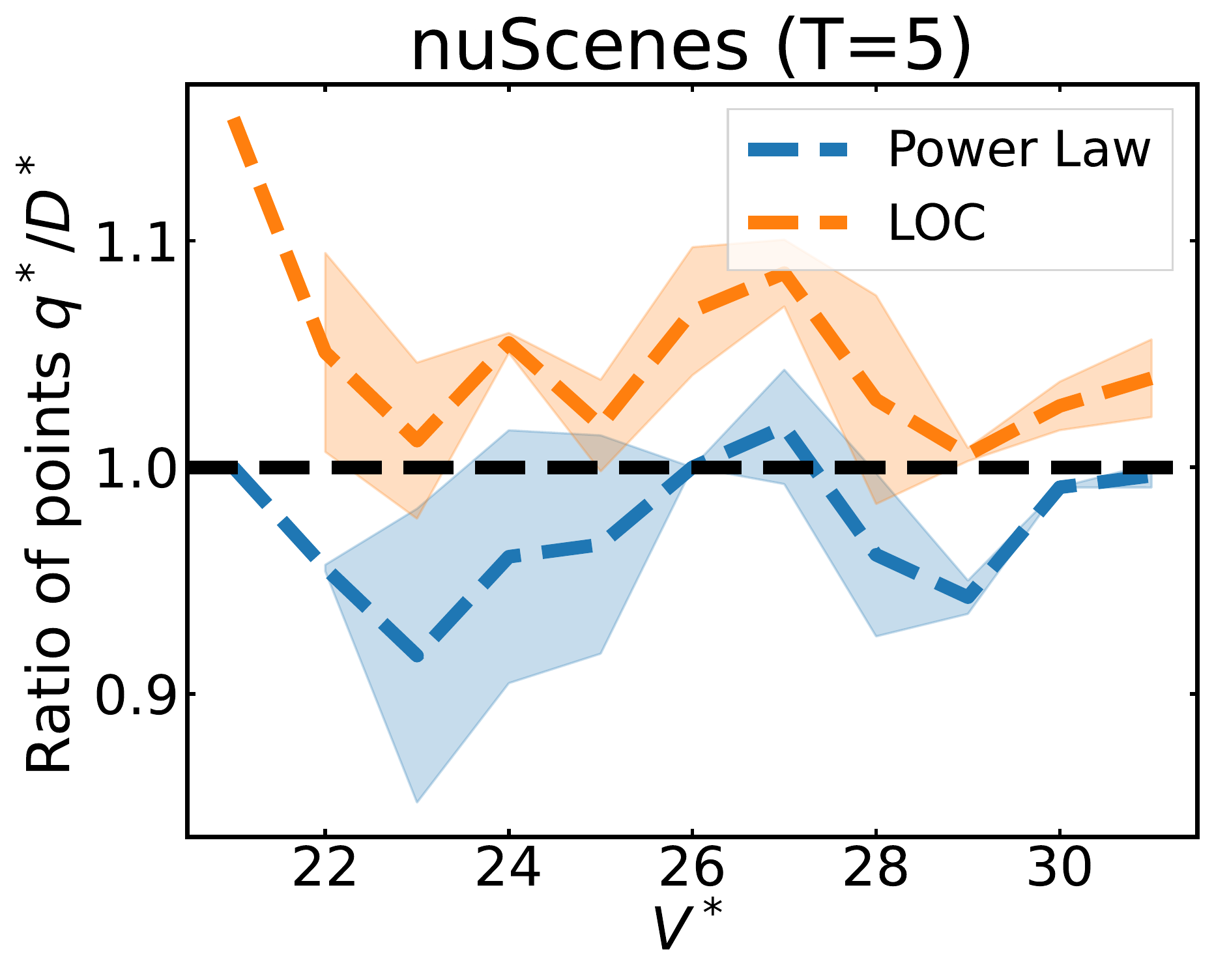}\end{minipage}
\vspace{-4mm}
\caption{\label{fig:head_to_head}
Mean $\pm$ standard deviation of 5 seeds of the ratio of data collected $q_T^* / D^*$ for different $V^*$. 
The rows correspond to $T = 1, 3, 5$ and the columns to different data sets. The black line corresponds to collecting exactly the minimum data requirement. 
\rebut{LOC almost always remains slightly above the black line, meaning we rarely fail to meet the target.}
}
\end{center}
\vspace{-5mm}
\end{figure*}

\begin{table}[!t]
\begin{minipage}{0.66\linewidth}
\resizebox{\linewidth}{!}{
\begin{tabular}{clccccc}
\toprule
  & Data set &  $T$ & \multicolumn{2}{c}{Power Law Regression} & \multicolumn{2}{c}{LOC} \\ \cmidrule(lr){4-5}\cmidrule(lr){6-7}
  &          &      & Failure rate & Cost ratio & Failure rate & Cost ratio \\
\midrule
\multirow{9}{*}{\rotatebox[origin=c]{90}{Class.}}    
 & \multirow{3}{*}{CIFAR-10}    &  1 &   $100\%$ &             $-$ &        $\fir{60\%}$ &    $0.19$ \\
 &                              &  3 &   $ 95\%$ &          $0.00$ &        $\fir{32\%}$ &    $0.05$ \\
 &                              &  5 &   $ 86\%$ &          $0.00$ &        $\fir{29\%}$ &    $0.03$ \\
  \cmidrule{2-7} 
& \multirow{3}{*}{CIFAR-100}    &  1 &   $  56\%$ &          $0.12$ &        $\fir{4\%}$ &    $0.99$ \\
&                               &  3 &   $  48\%$ &          $0.10$ &        $\fir{3\%}$ &    $0.31$ \\
&                               &  5 &   $  48\%$ &          $0.10$ &        $\fir{2\%}$ &    $0.19$ \\
  \cmidrule{2-7} 
 & \multirow{3}{*}{Imagenet}    &  1 &   $  99\%$ &          $0.00$ &        $\fir{37\%}$ &    $0.49$ \\
 &                              &  3 &   $  75\%$ &          $0.01$ &        $\fir{ 5\%}$ &    $0.16$ \\
 &                              &  5 &   $  56\%$ &          $0.01$ &        $\fir{ 2\%}$ &    $0.10$ \\
 \midrule
  \multirow{6}{*}{\rotatebox[origin=c]{90}{Seg.}}    
  & \multirow{3}{*}{BDD100K}    &  1 &   $  77\%$ &          $0.03$ &        $\fir{12\%}$ &    $2.03$ \\
  &                             &  3 &   $  31\%$ &          $0.00$ &        $\fir{ 0\%}$ &    $0.72$ \\
  &                             &  5 &   $  23\%$ &          $0.01$ &        $\fir{ 0\%}$ &    $0.35$ \\
  \cmidrule{2-7} 
 & \multirow{3}{*}{nuScenes}    &  1 &   $  95\%$ &          $0.00$ &        $\fir{52\%}$ &    $0.16$ \\
 &                              &  3 &   $  71\%$ &          $0.01$ &        $\fir{ 0\%}$ &    $0.09$ \\
 &                              &  5 &   $  62\%$ &          $0.00$ &        $\fir{ 0\%}$ &    $0.04$ \\
 \midrule
 \multirow{3}{*}{\rotatebox[origin=c]{90}{Det.}}    
      & \multirow{3}{*}{VOC}    &  1 &   $  36\%$ &          $1.24$ &        $\fir{25\%}$ &      $0.56$ \\
      &                         &  3 &   $   8\%$ &          $0.88$ &        $\fir{ 0\%}$ &      $1.10$ \\
      &                         &  5 &   $   6\%$ &          $0.86$ &        $\fir{ 0\%}$ &      $0.84$ \\
\bottomrule
\end{tabular}
}
\end{minipage}
\hfill
\begin{minipage}{0.31\linewidth}
    \caption{Average cost ratio $\bc^\tpose (\bq^*_T - \bq_0) / \bc^\tpose (\bD^* - \bq_0) - 1$ and failure rate measured over a range of $V^*$ for each $T$ and data set. 
    We fix $c = 1$ and $P=10^7$ ($P=10^6$ for VOC and $P=10^8$ for ImageNet).
    The best performing failure rate for each setting is bolded. The cost ratio is measured only for instances that achieve $V^*$. LOC consistently reduces the average failure rate, often down to $0\%$, while keeping the average cost ratio almost always below $1$ (i.e., spending at most $2\times$ the optimal amount).}
    \label{tab:one_d_summary}
\end{minipage}
\vspace{-2mm}
\end{table}

\begin{figure*}[!t]
\begin{center}
\begin{minipage}{0.3\linewidth}\includegraphics[width=1\textwidth]{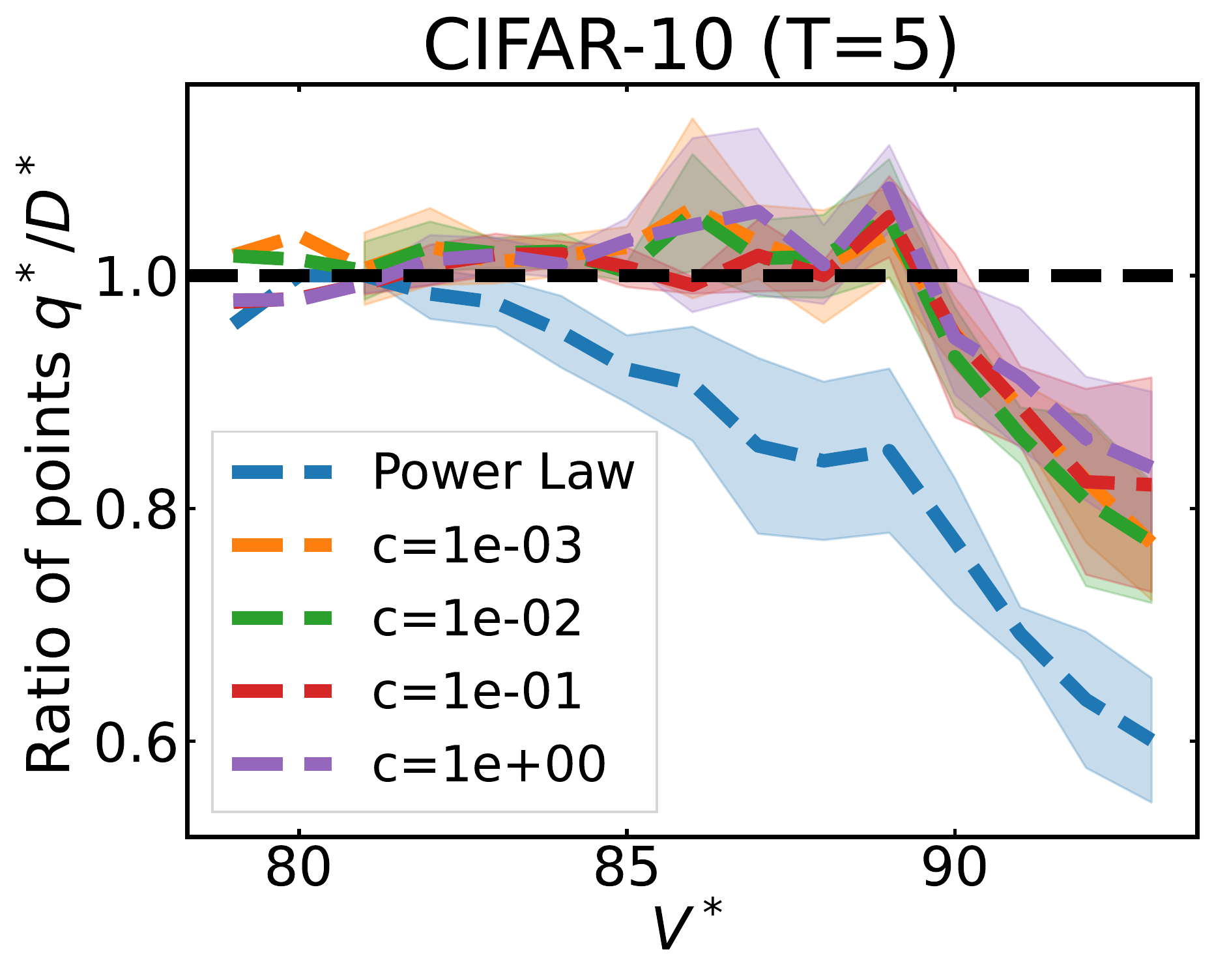}\end{minipage}
\begin{minipage}{0.3\linewidth}\includegraphics[width=1\textwidth]{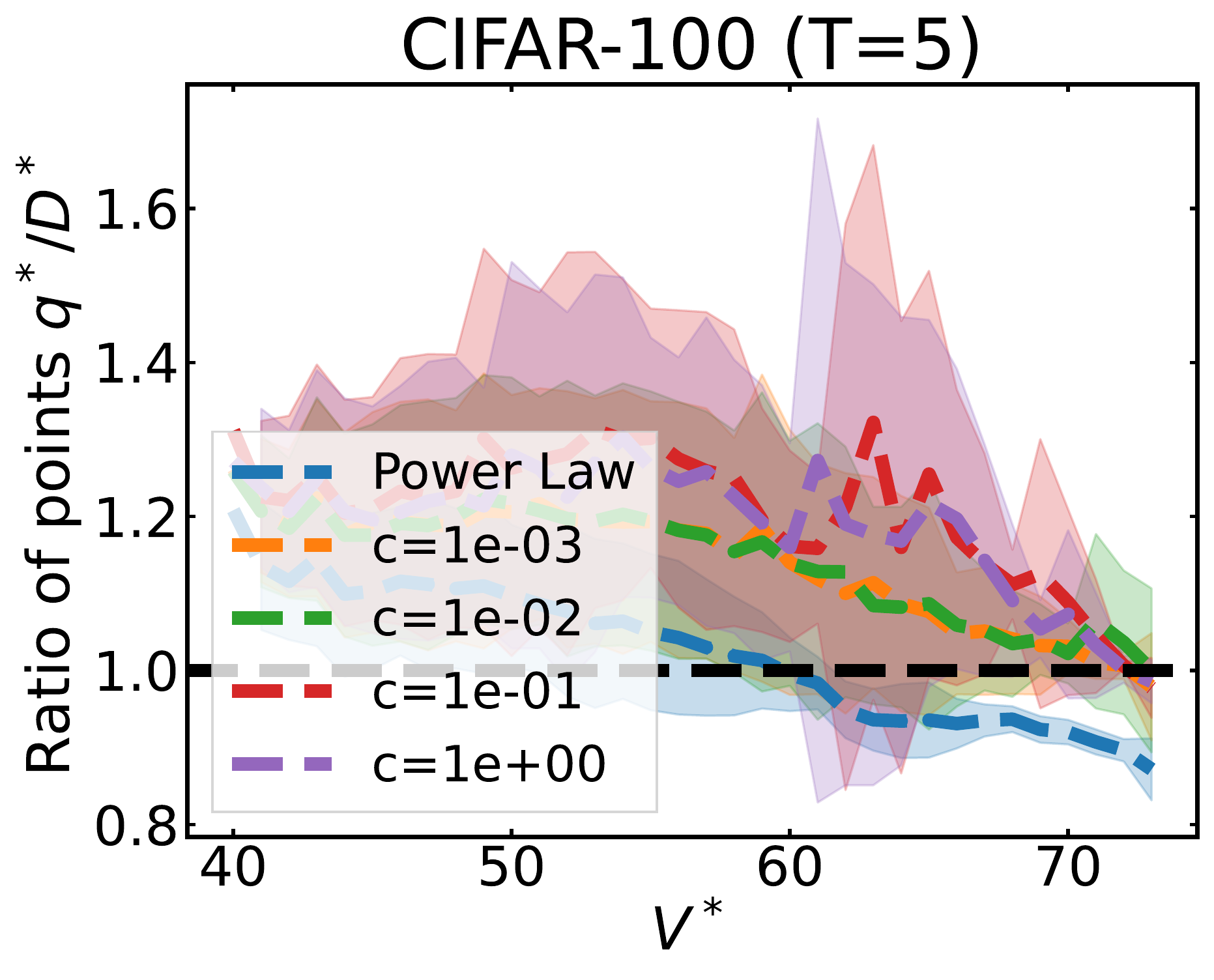}\end{minipage}
\begin{minipage}{0.3\linewidth}\includegraphics[width=1\textwidth]{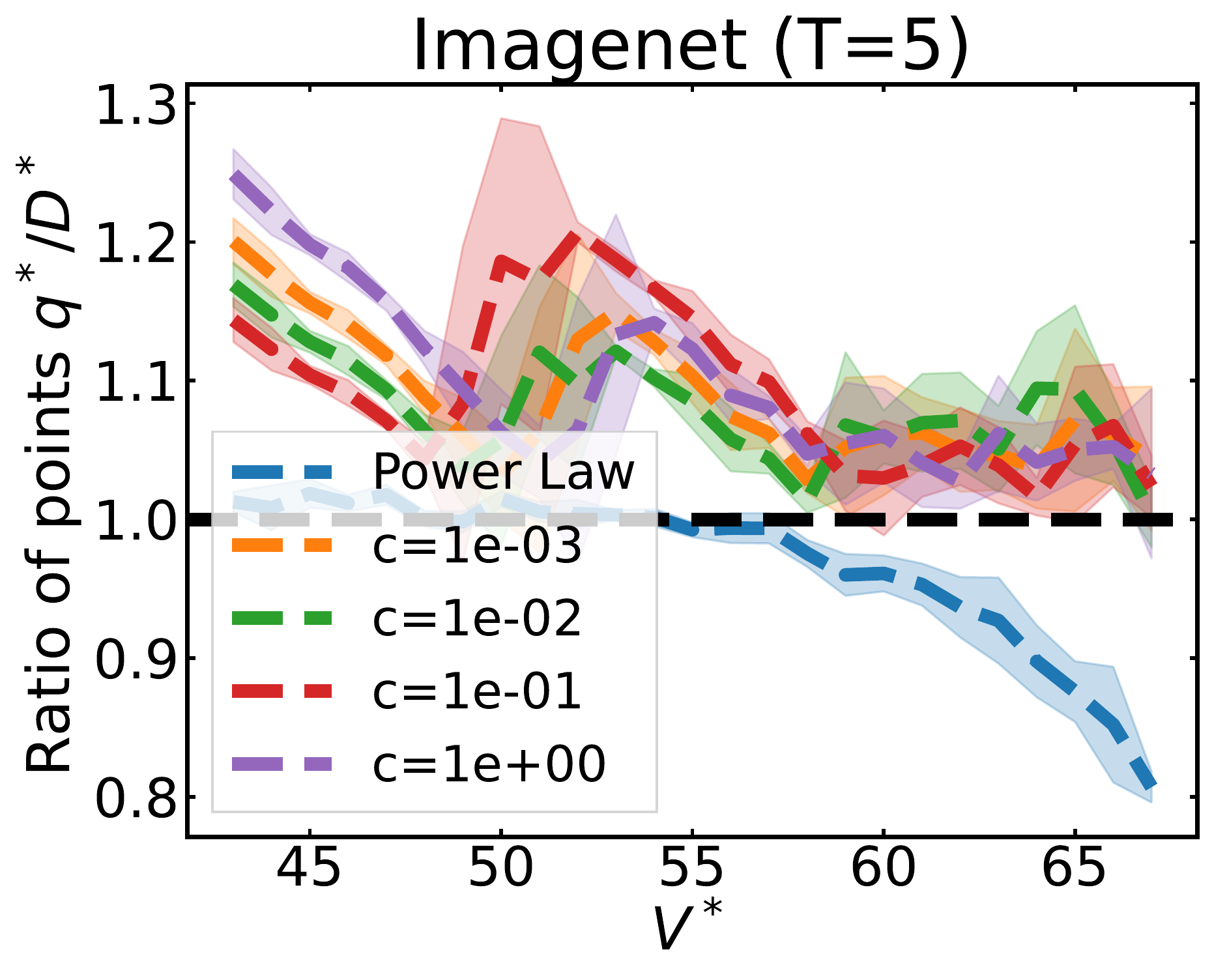}\end{minipage}
\vspace{-1em}
\begin{minipage}{0.3\linewidth}\includegraphics[width=1\textwidth]{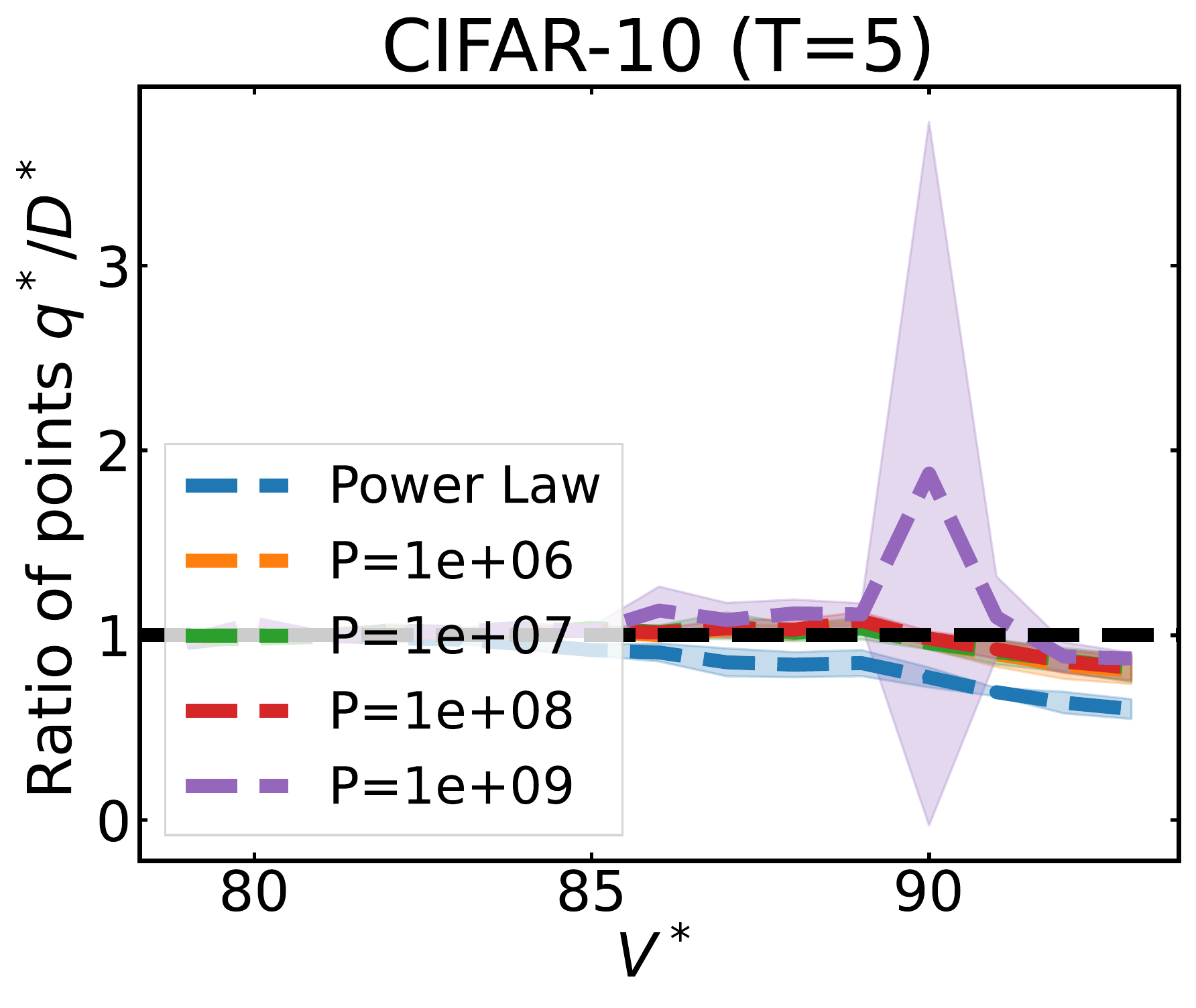}\end{minipage}
\begin{minipage}{0.3\linewidth}\includegraphics[width=1\textwidth]{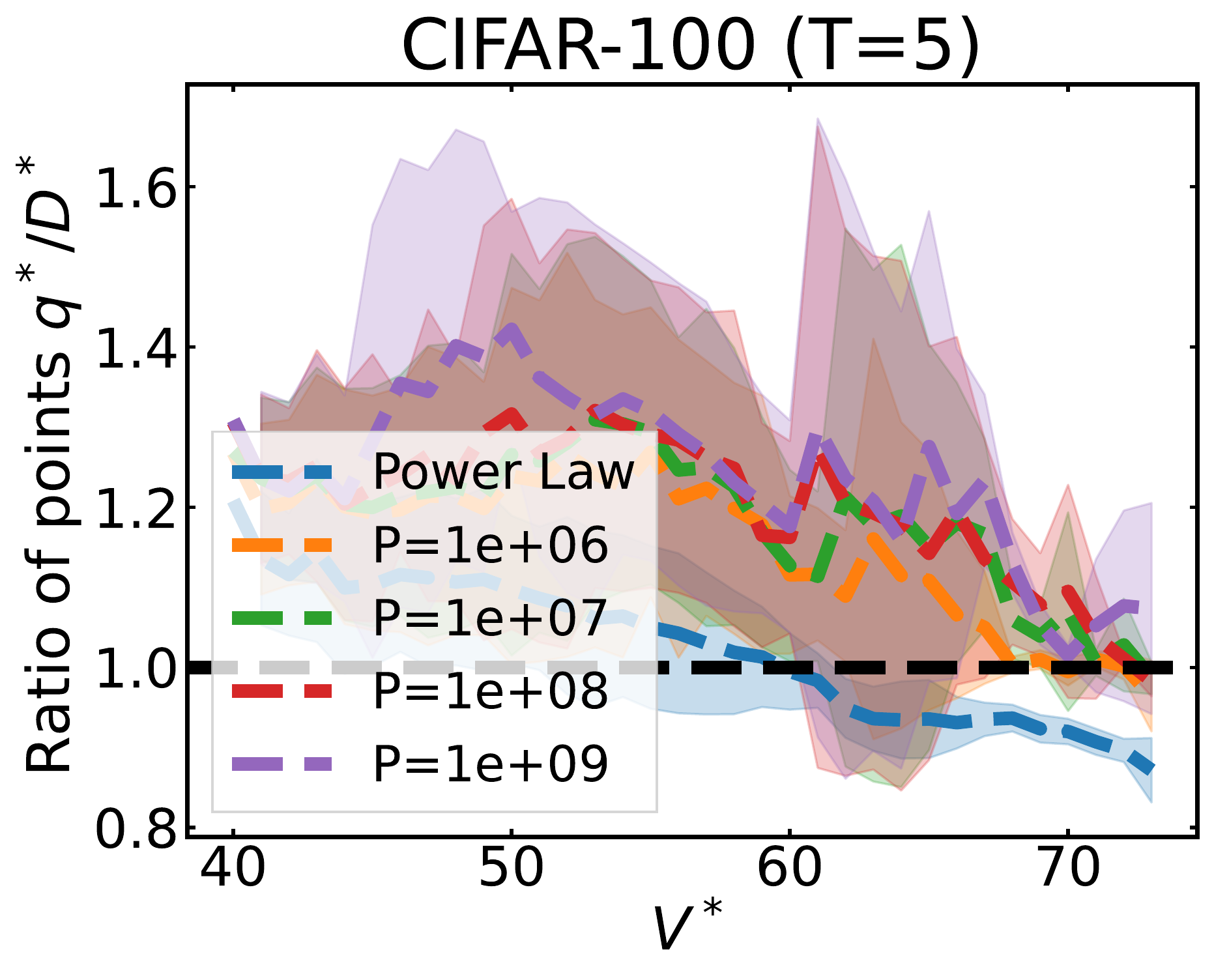}\end{minipage}
\begin{minipage}{0.3\linewidth}\includegraphics[width=1\textwidth]{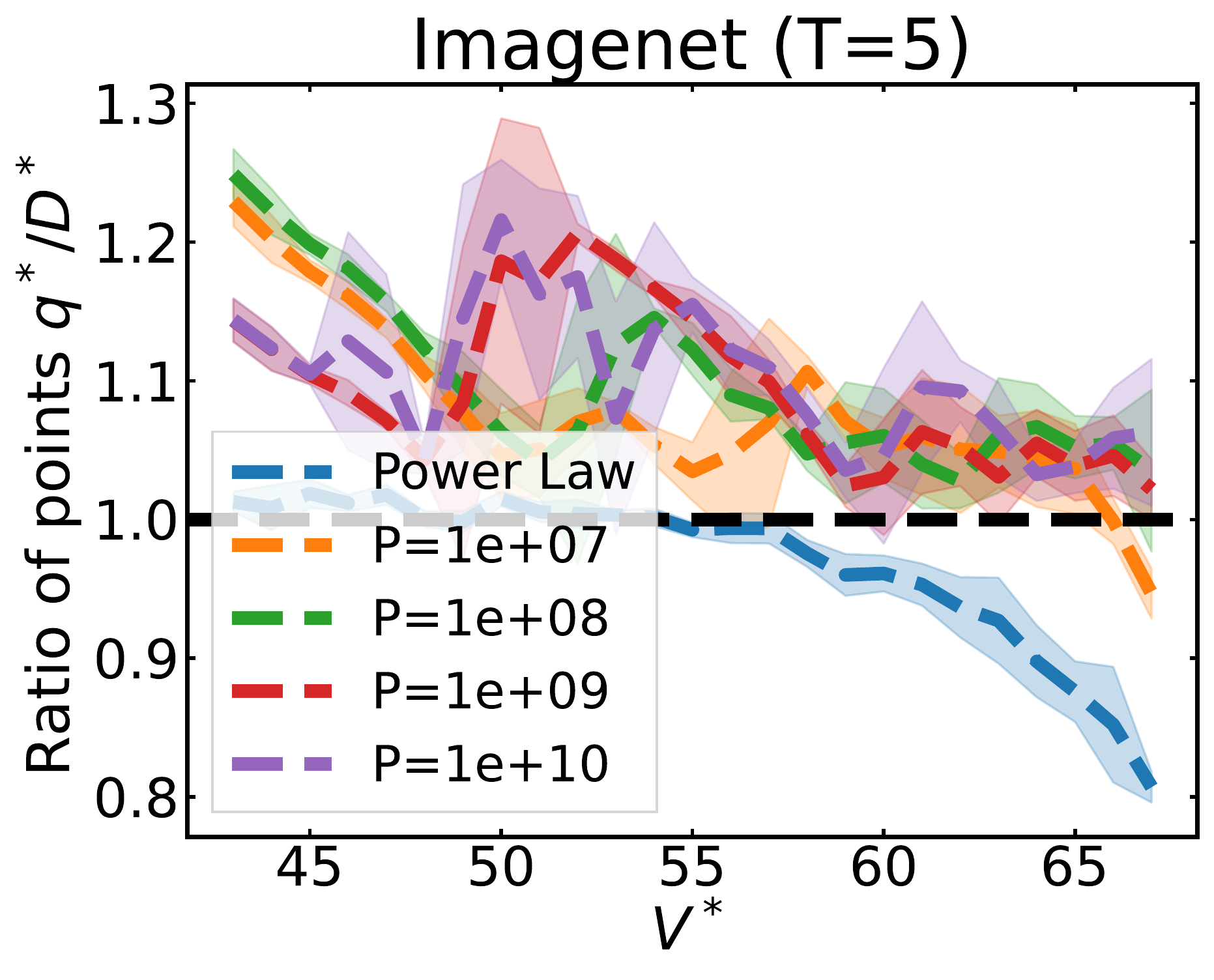}\end{minipage}
\vspace{-1em}
\begin{minipage}{0.3\linewidth}\includegraphics[width=1\textwidth]{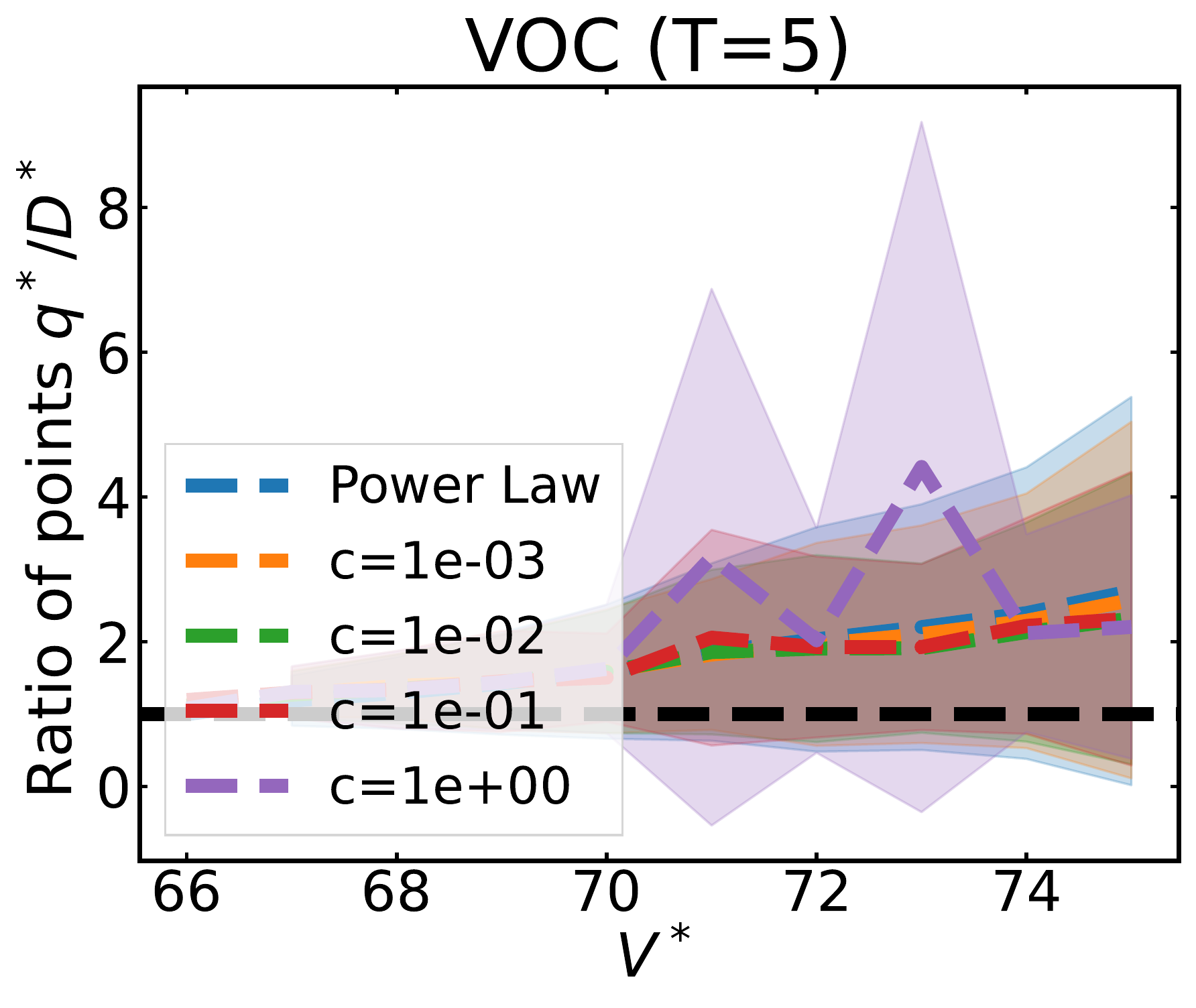}\end{minipage}
\begin{minipage}{0.3\linewidth}\includegraphics[width=1\textwidth]{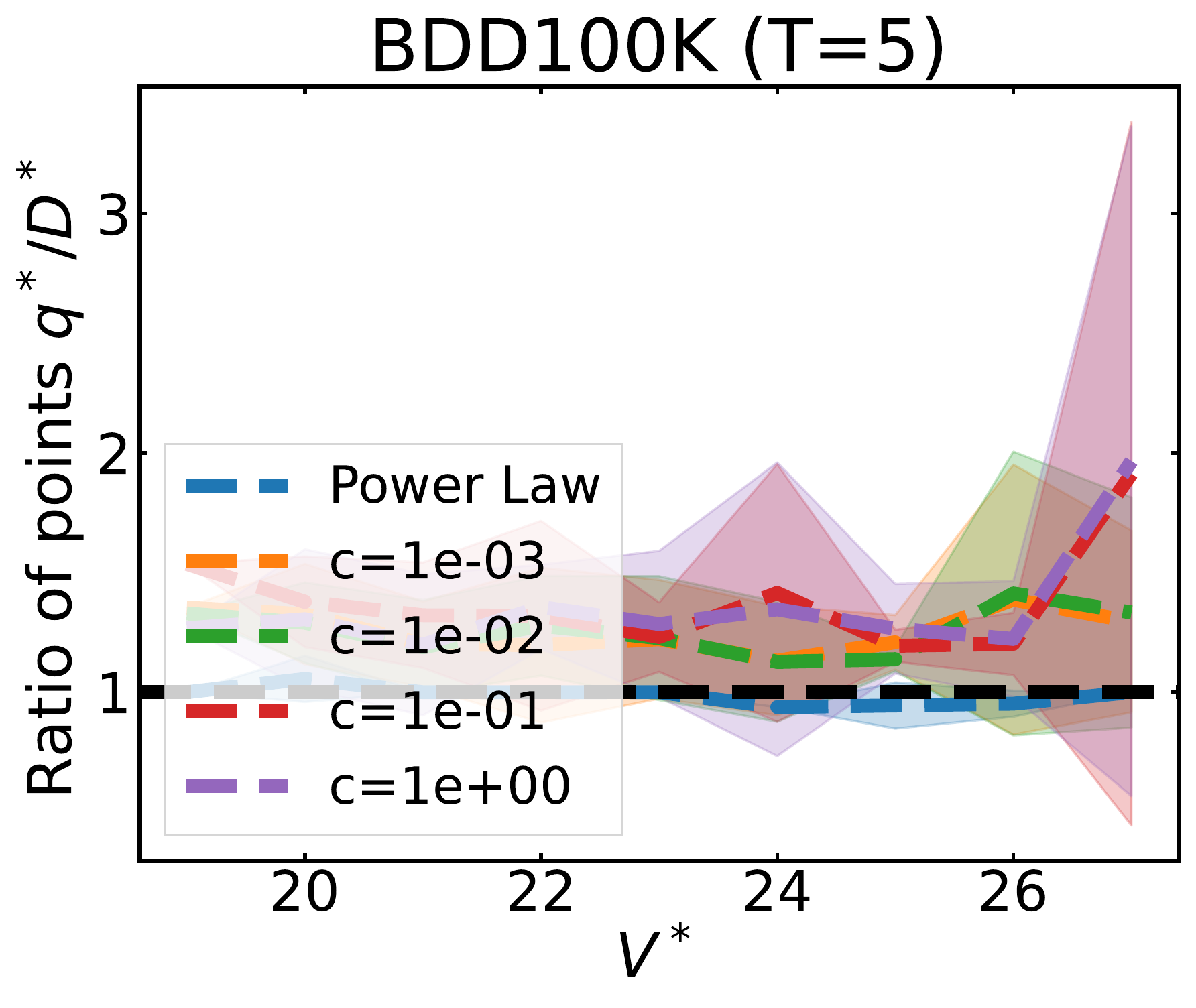}\end{minipage}
\begin{minipage}{0.3\linewidth}\includegraphics[width=1\textwidth]{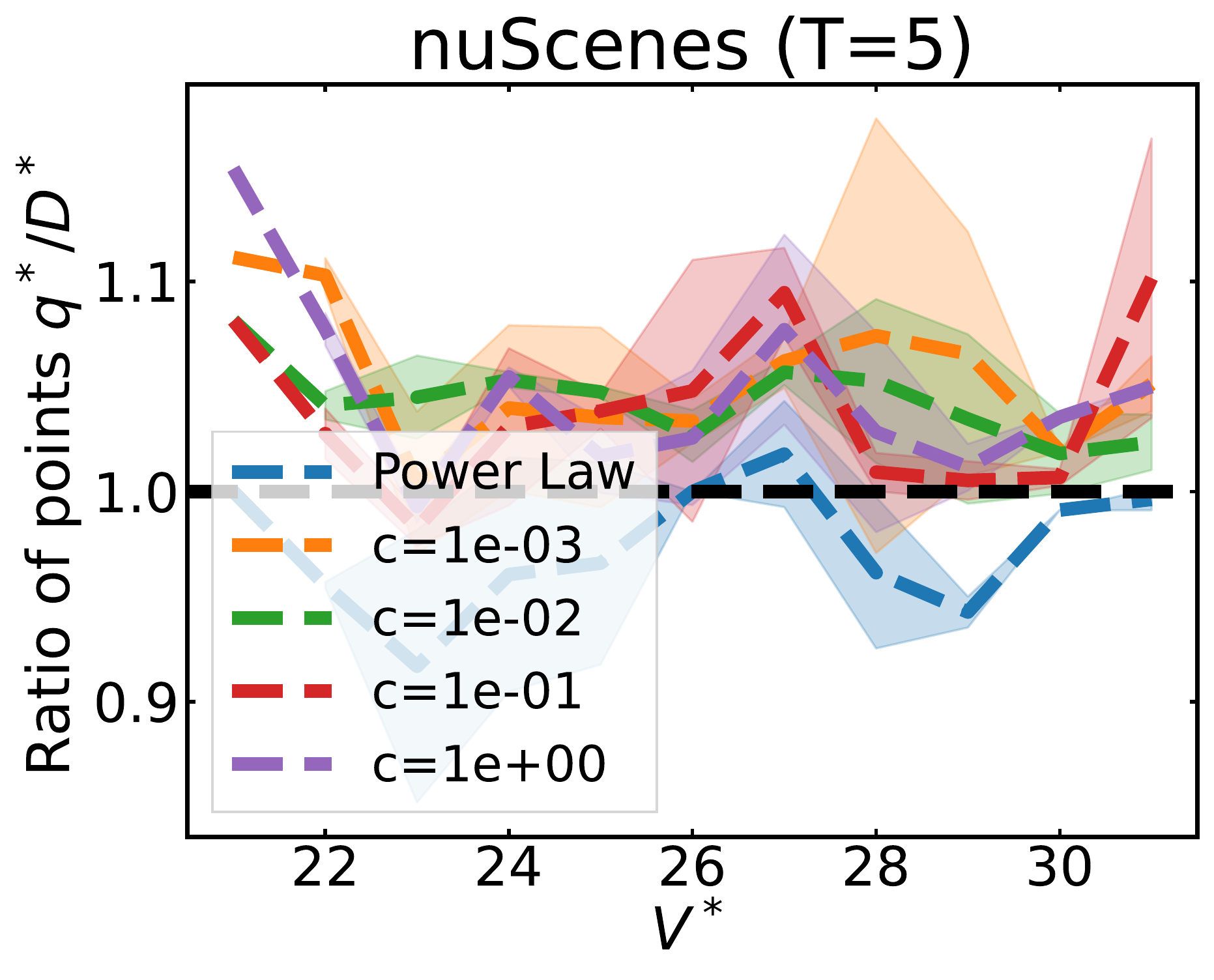}\end{minipage}
\vspace{-1em}
\begin{minipage}{0.3\linewidth}\includegraphics[width=1\textwidth]{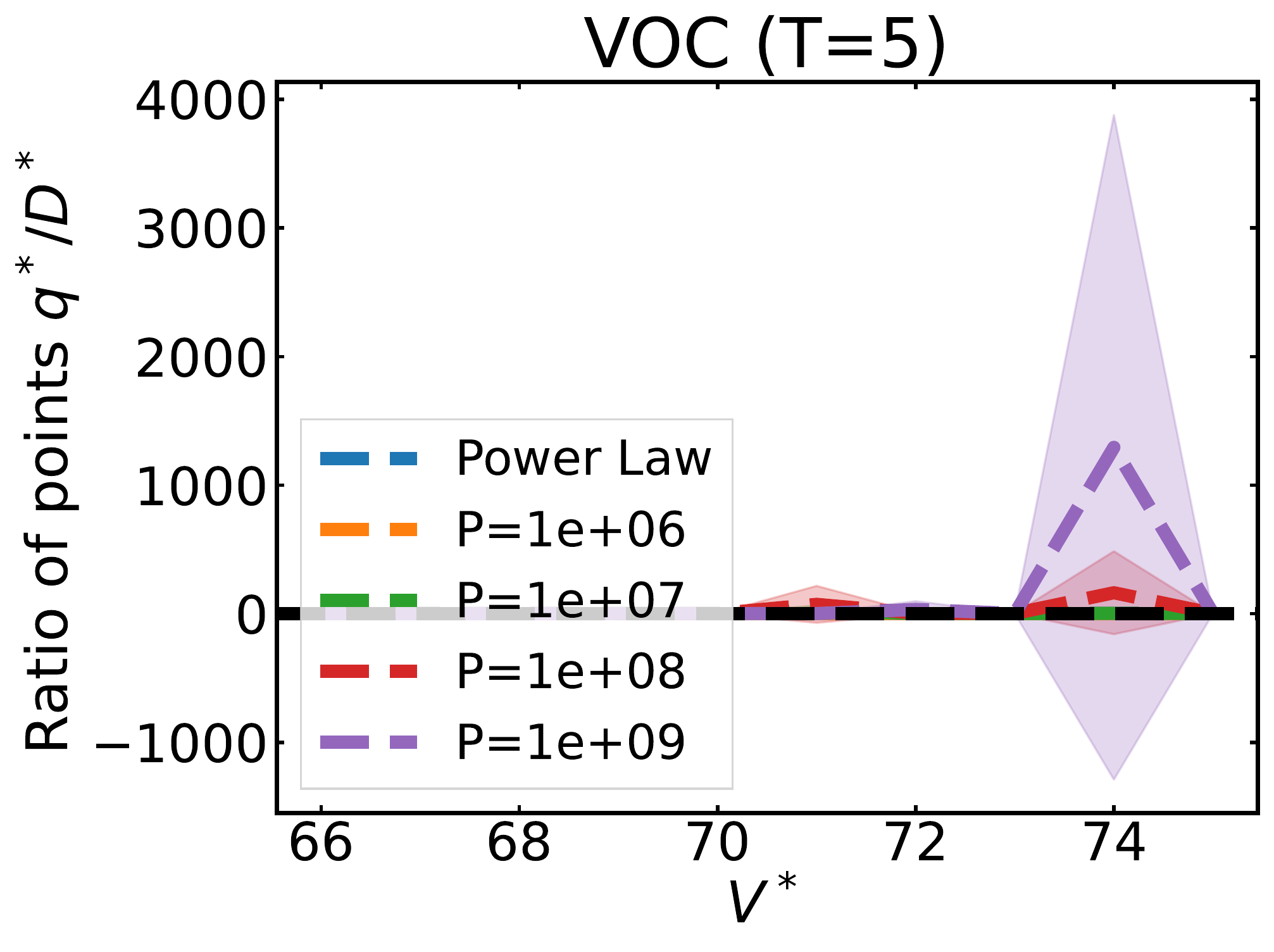}\end{minipage}
\begin{minipage}{0.3\linewidth}\includegraphics[width=1\textwidth]{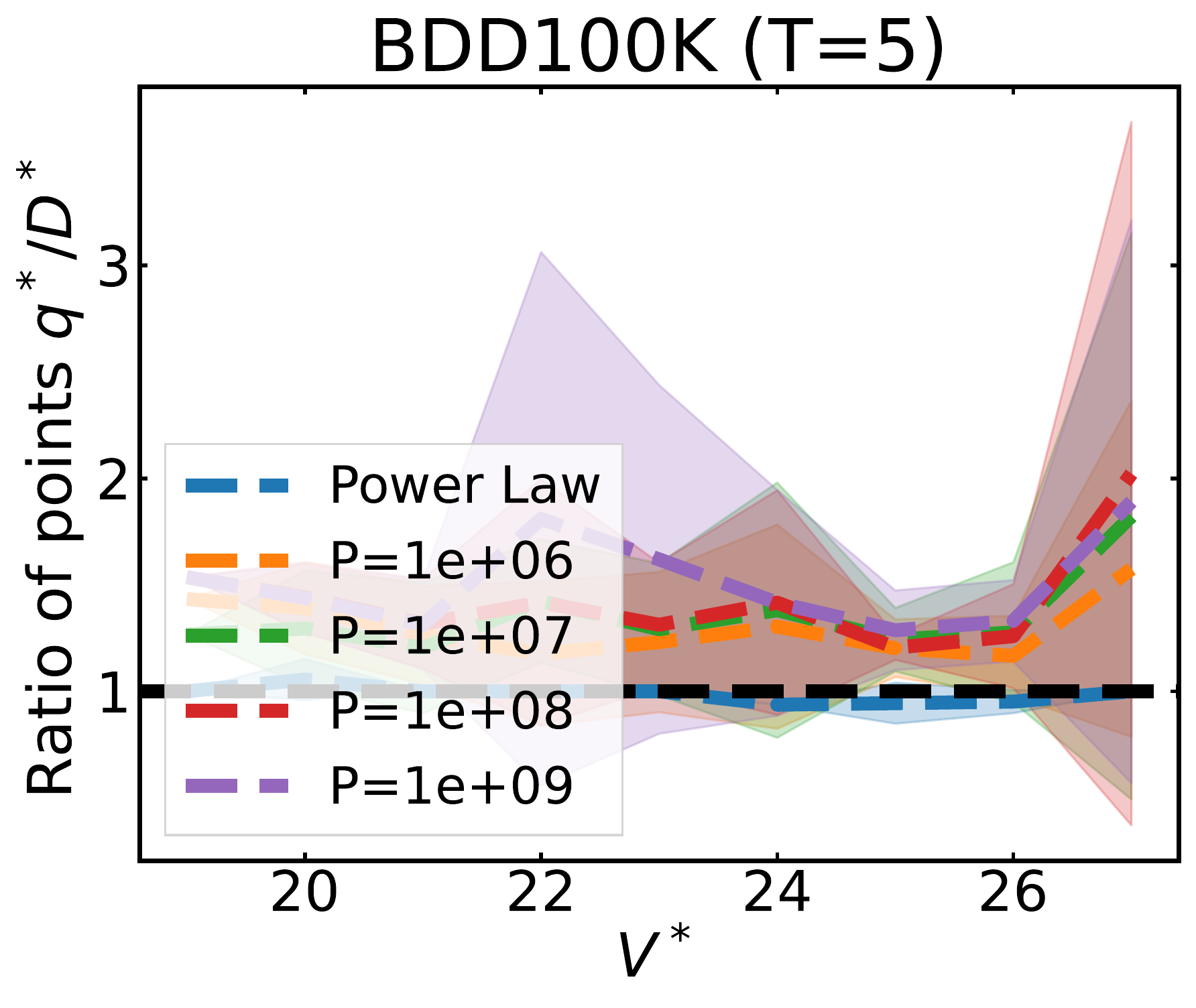}\end{minipage}
\begin{minipage}{0.3\linewidth}\includegraphics[width=1\textwidth]{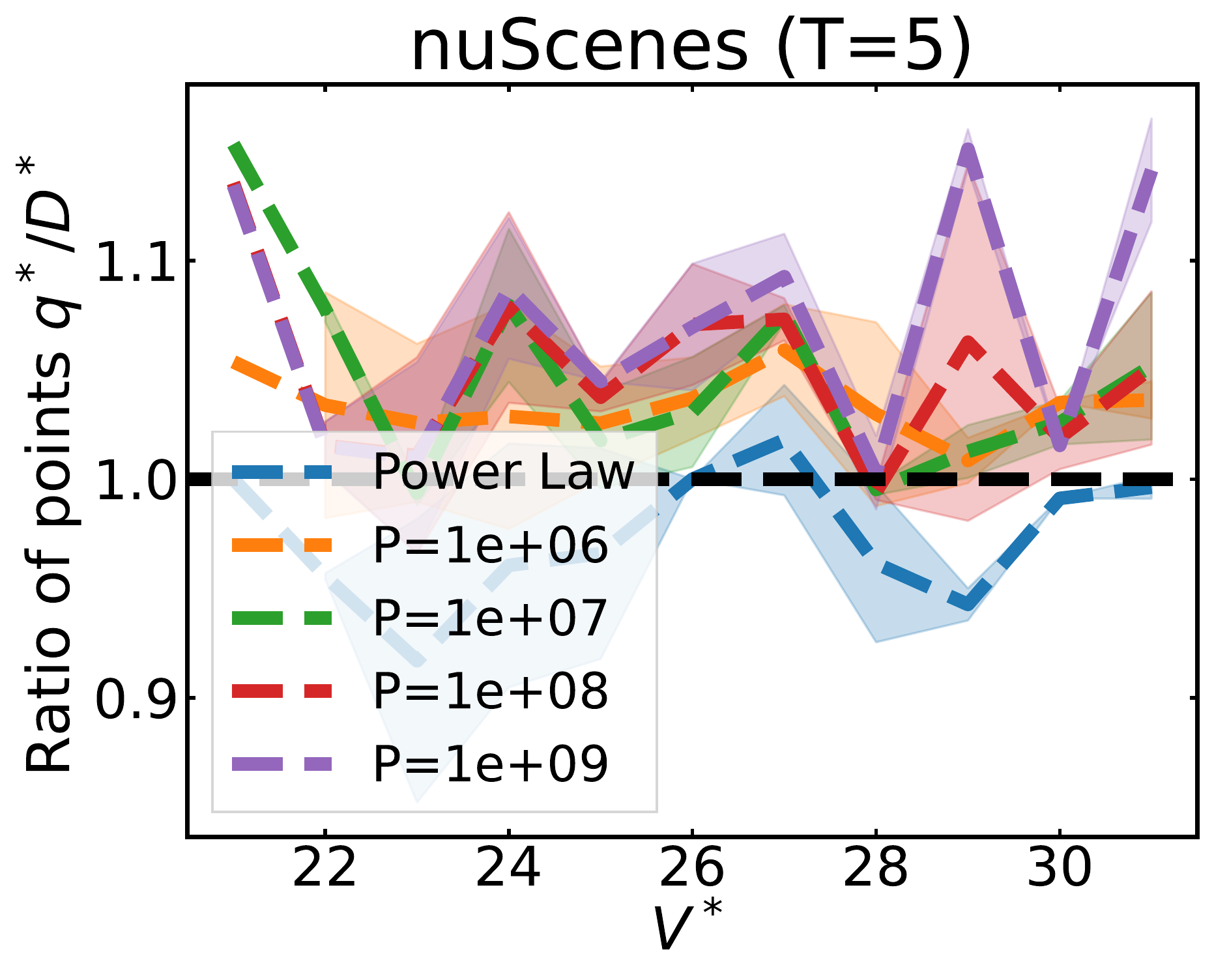}\end{minipage}
%
\caption{\label{fig:sweep_cost_and_penalty}
Mean $\pm$ standard deviation of 5 seeds of the ratio of data collected $q_T^* / D^*$ for different $V^*$ and fixed $T=5$.
\emph{Rows 1 \& 3:} We sweep the cost parameter from $0.001$ to $1$ and fix $P=10^7$. 
\emph{Rows 2 \& 4:} We sweep the penalty parameter from $10^6$ to $10^9$ and fix $c = 1$. 
The dashed black line corresponds to collecting exactly the minimum data requirement.
See Appendix~\ref{sec:app_experiment_results} for all $T$.
}
\end{center}
\vspace{-6mm}
\end{figure*}

\textbf{The Value of Optimization over Estimation when $K=1$.~}
Figure~\ref{fig:head_to_head} compares LOC versus the corresponding power law regression baseline when $c=1$ and $P=10^7$ ($P=10^6$ for VOC and $P=10^8$ for ImageNet). If a curve is below the black line, then it failed to collect enough data to meet the target.
LOC consistently remains above this black line for most settings.
In contrast, even with up to $T=5$ rounds, collecting data based only on regression estimates leads to failure.

Table~\ref{tab:one_d_summary} aggregates the failure rates and cost ratios for each setting. 
To summarize, LOC fails at less than $10\%$ of instances for $12/18$ settings, whereas regression fails over $30\%$ for $15/18$ settings.
In particular, regression nearly always under-collects data when given a single $T=1$ round. Here, LOC reduces the risk of under-collecting by $40\%$ to $90\%$ over the baseline.
While this leads to a marginal increase in costs, our cost ratios are consistently less than $0.5$ for $12/18$ settings, meaning that we spend at most $50\%$ more than the true minimum cost.

We remark that previously,~\citet{mahmood2022howmuch} observed that incorrect regression estimates necessitated real machine learning workflows to collect data over multiple rounds. 
Instead, with LOC, we can make significantly improved data collection decisions even with a single round.

\textbf{Robustness to Cost and Penalty Parameters (see Appendix~\ref{sec:app_experiment_robustness} for details).~}
Figure~\ref{fig:sweep_cost_and_penalty} evaluates the ratio of points collected for $T=5$ when the cost and the penalty of the optimization problem are varied. Our algorithm is robust to variations in these parameters, as LOC retains the same shape and scale for almost every parameter setting and data set. 
Further, LOC consistently remains above the horizontal $1$ line, showing that even after varying $c$ and $P$, we do not fail as frequently as the baseline.
Finally, validating Theorem~\ref{thm:t1_has_analytic_sol}, the penalty parameter $P$ provides natural control over the amount of data collected. As we increase $P$, the ratio of data collected increases consistently.

\begin{figure*}[!t]
\begin{center}
\begin{minipage}{0.24\linewidth}\includegraphics[width=1\textwidth]{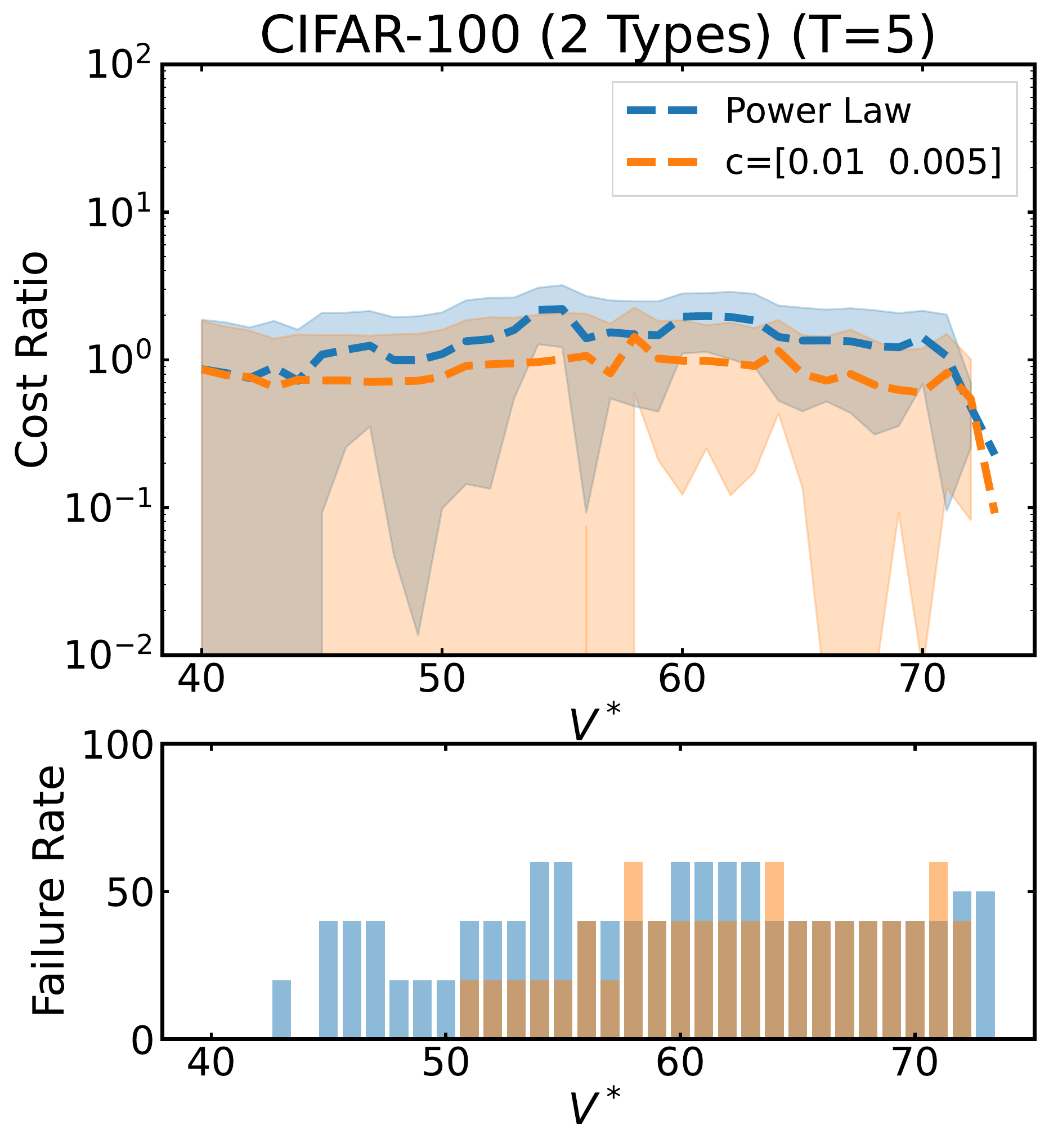} \end{minipage}
\begin{minipage}{0.24\linewidth}\includegraphics[width=1\textwidth]{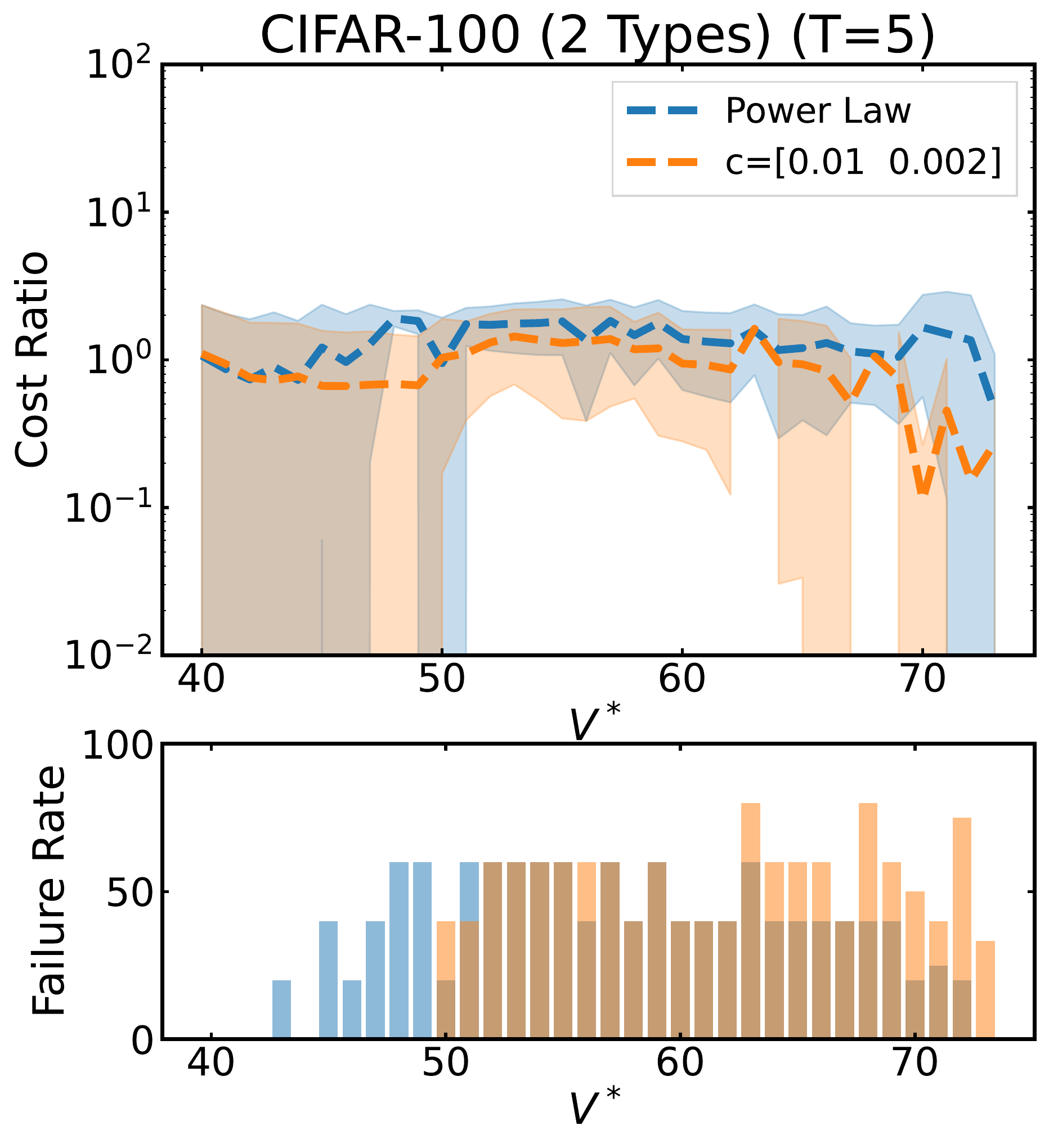} \end{minipage}
\begin{minipage}{0.24\linewidth}\includegraphics[width=1\textwidth]{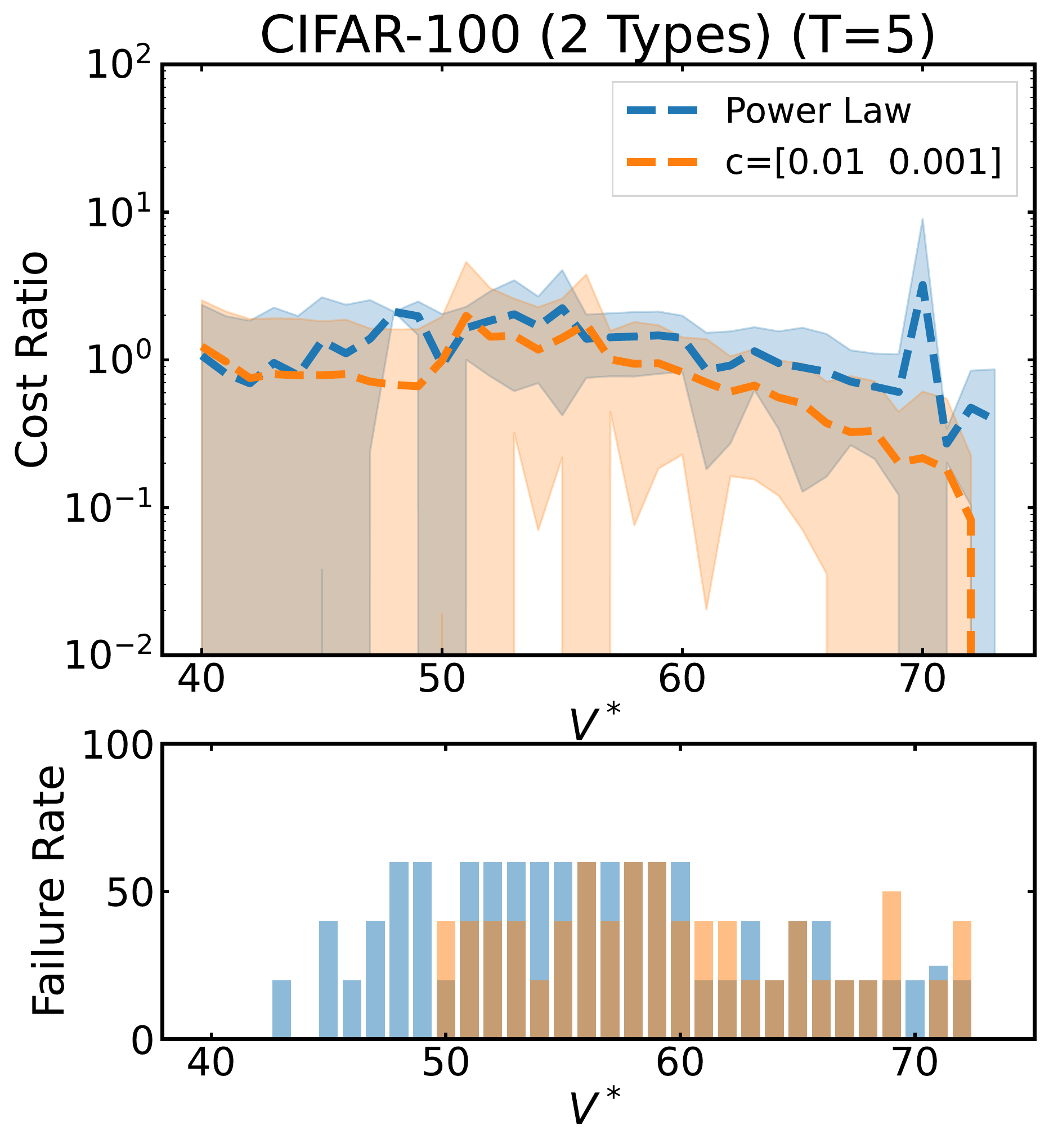} \end{minipage}
\begin{minipage}{0.24\linewidth}\includegraphics[width=1\textwidth]{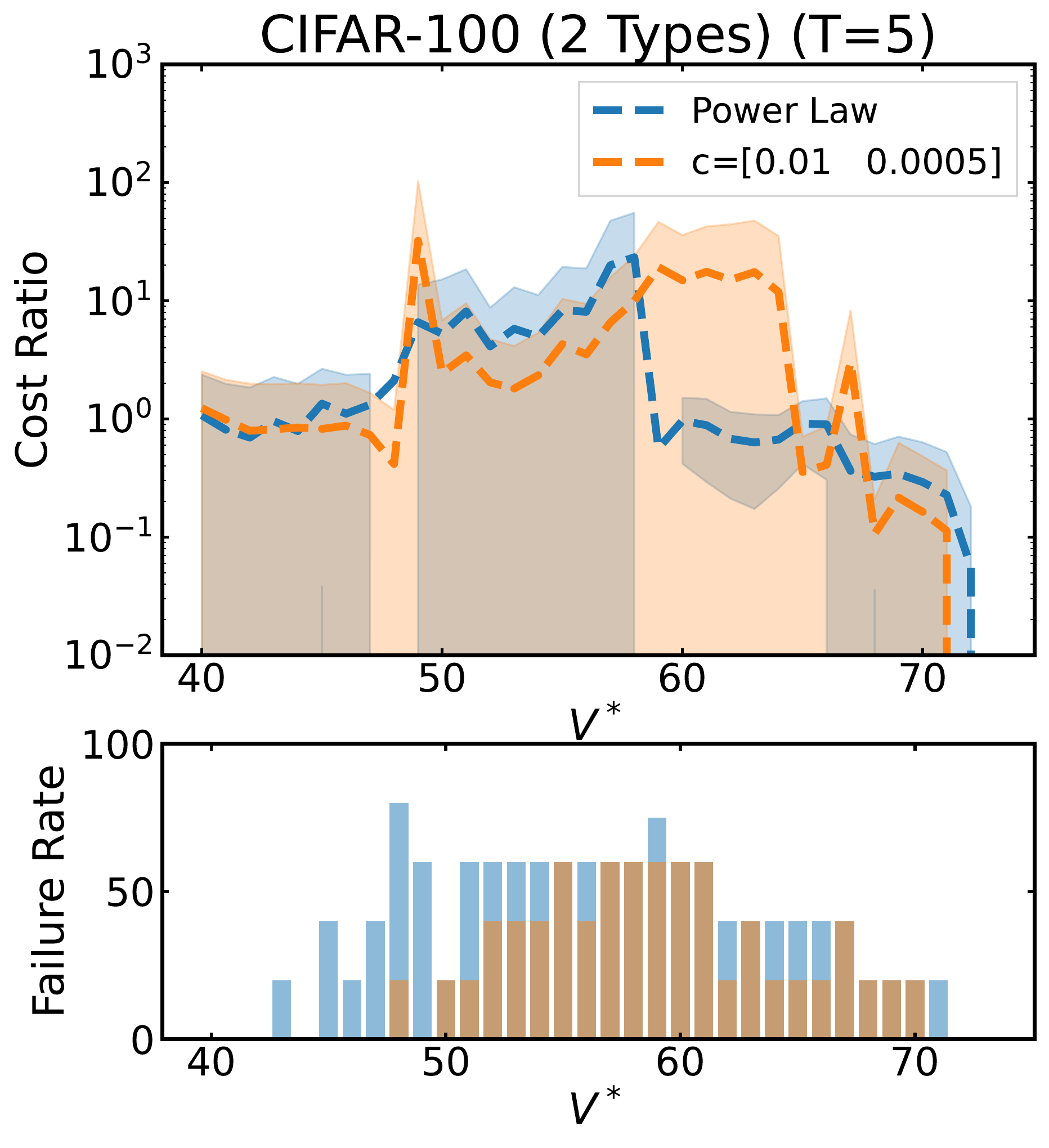} \end{minipage}
\begin{minipage}{0.24\linewidth}\includegraphics[width=1\textwidth]{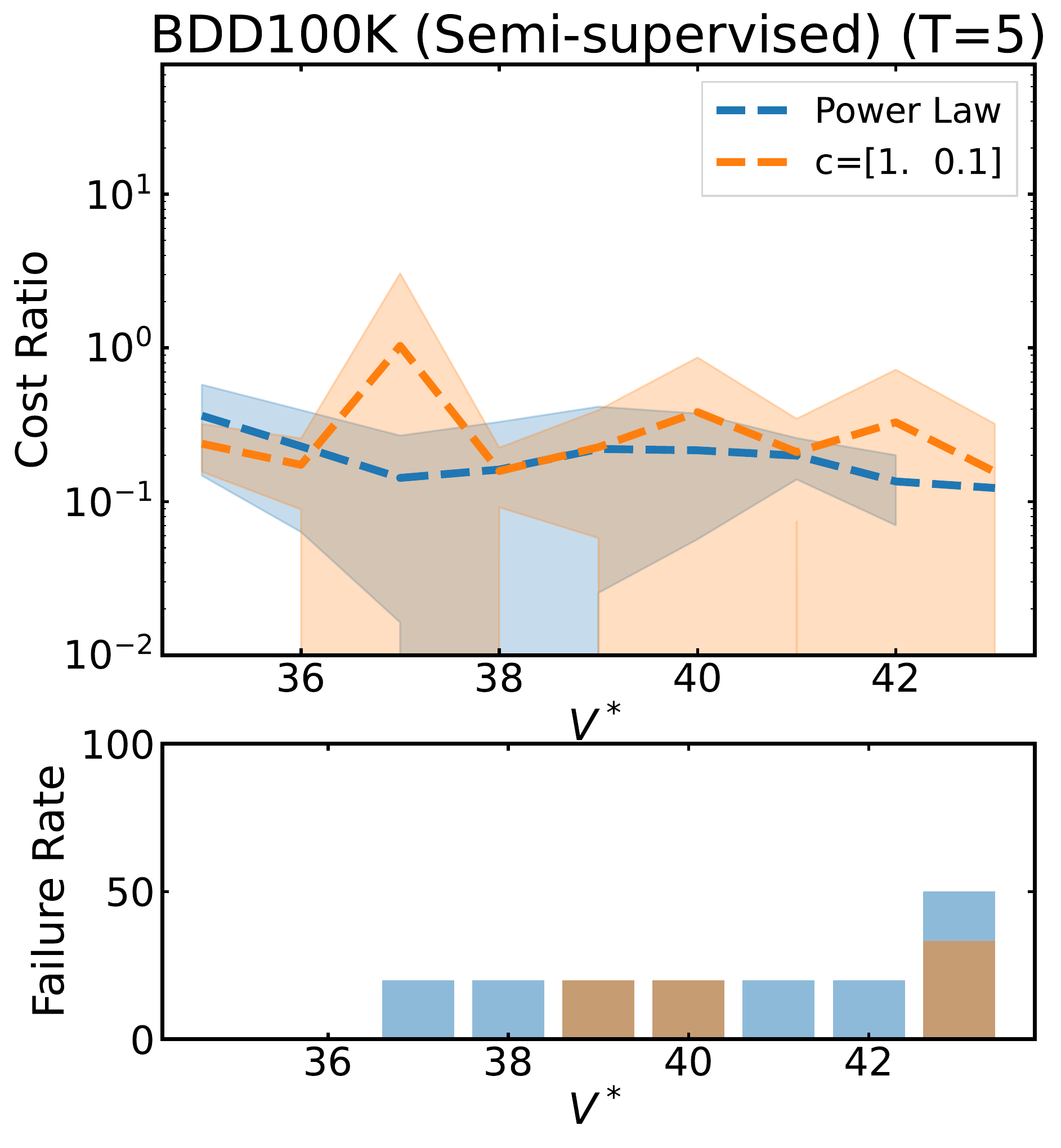} \end{minipage}
\begin{minipage}{0.24\linewidth}\includegraphics[width=1\textwidth]{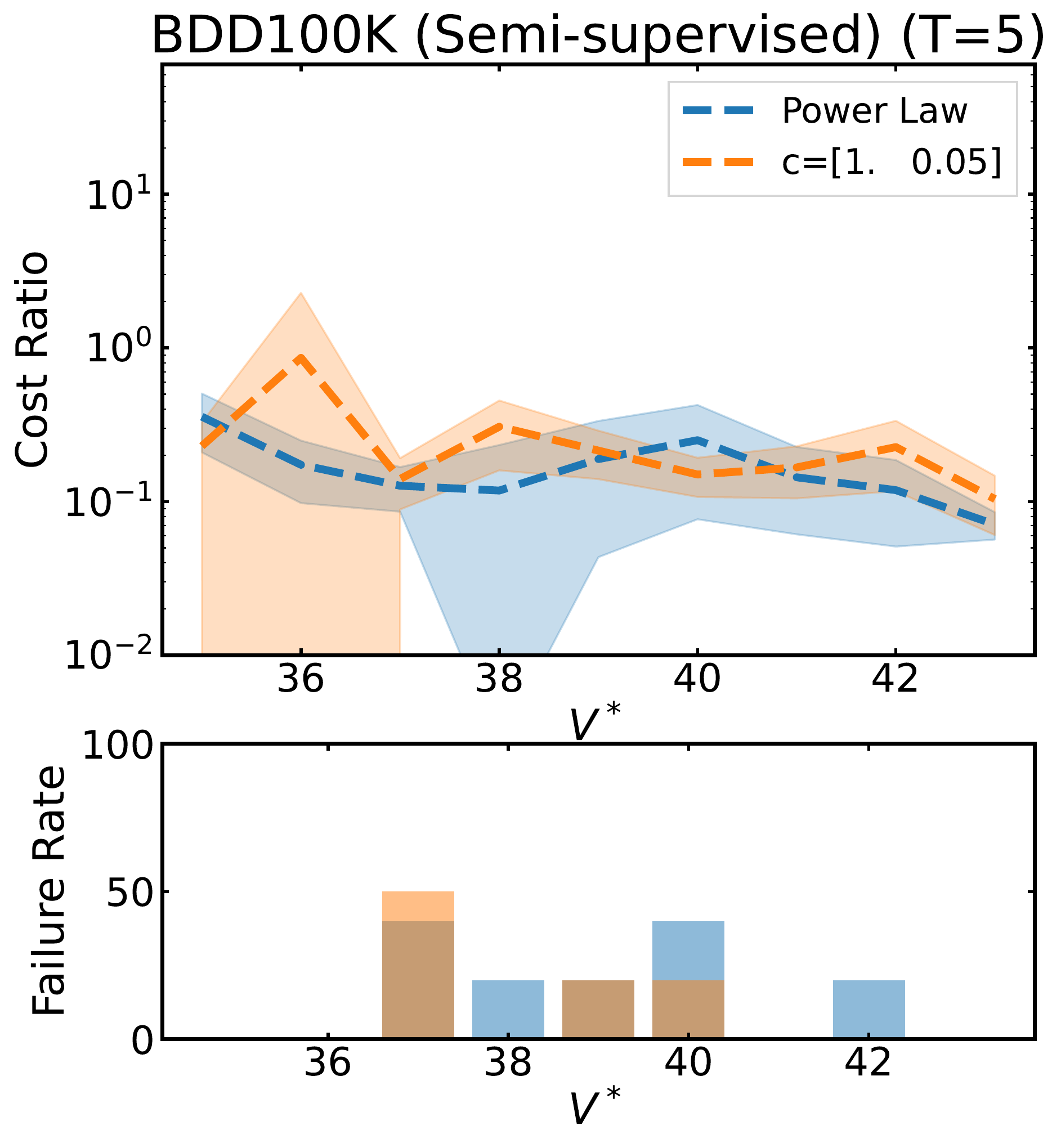} \end{minipage}
\begin{minipage}{0.24\linewidth}\includegraphics[width=1\textwidth]{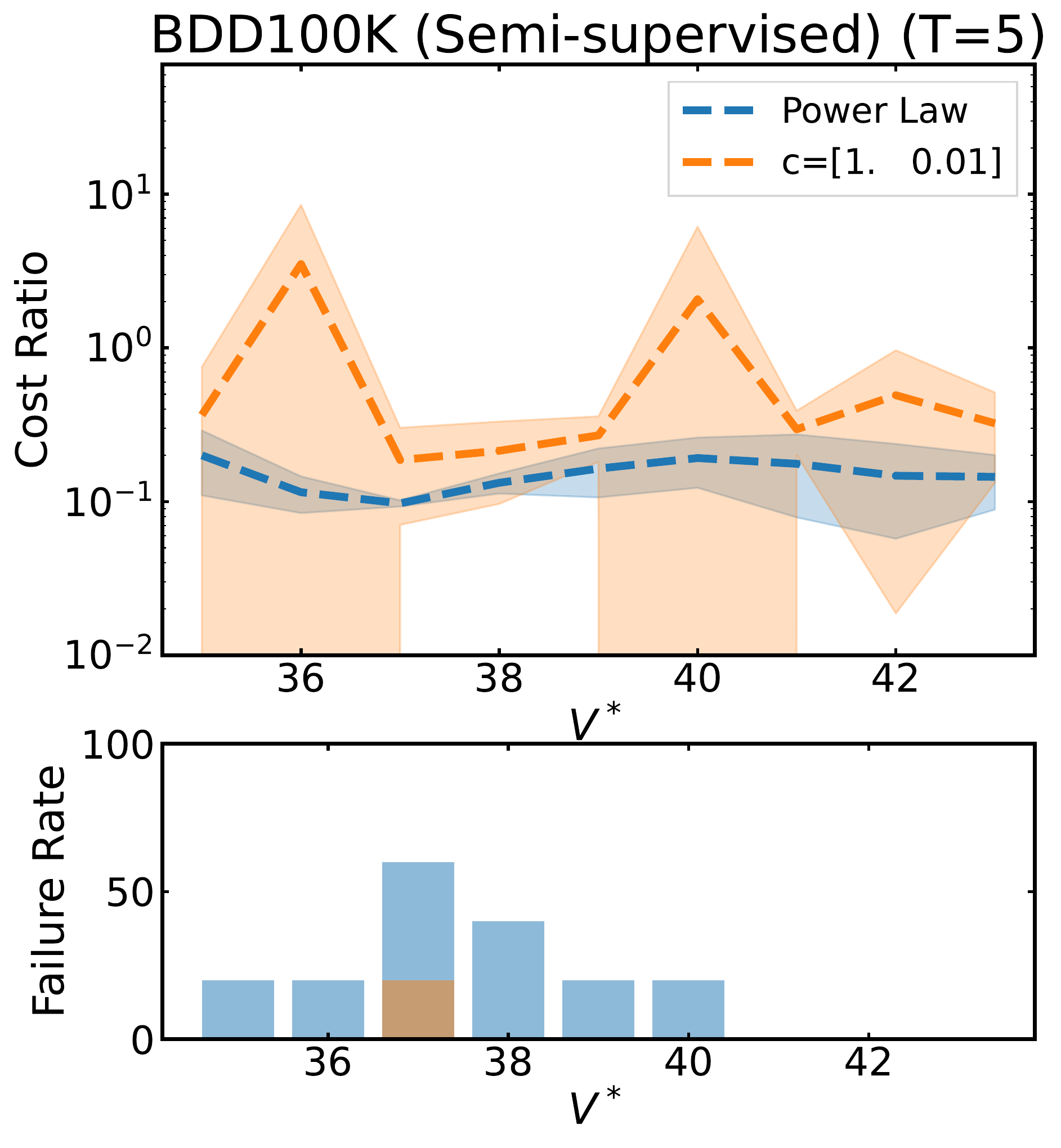} \end{minipage}
\begin{minipage}{0.24\linewidth}\includegraphics[width=1\textwidth]{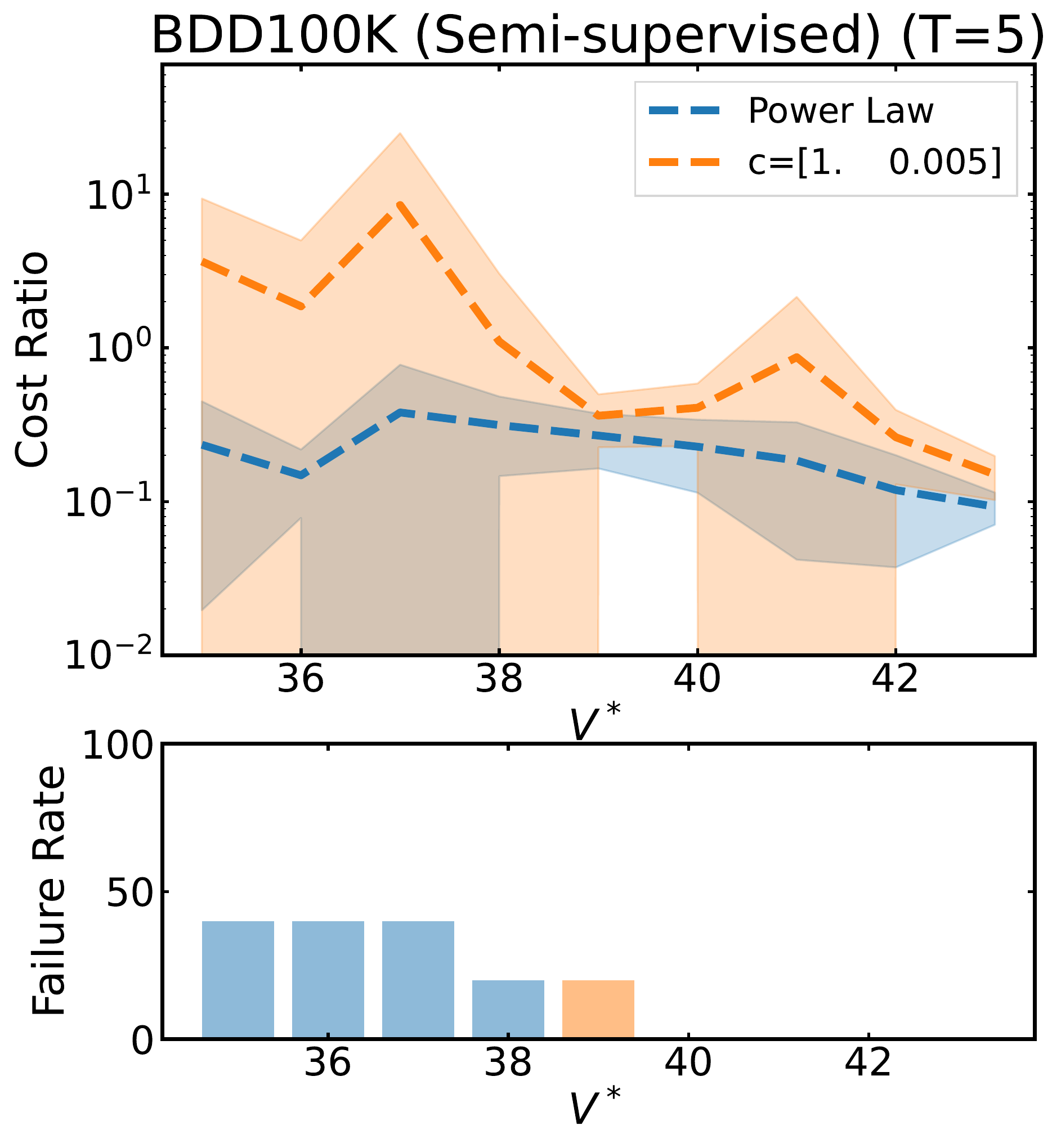} \end{minipage}
\vspace{-4mm}
\caption{\label{fig:two_d_head_to_head}
Mean $\pm$ standard deviation over 5 seeds of the cost ratio $\bc^\tpose (\bq^*_T - \bq_0) / \bc^\tpose (\bD^* - \bq_0) - 1$ and failure rate for different $V^*$, after removing $99$-th percentile outliers.
The columns correspond to scenarios where the first set $c^1$ costs increasingly more than the second $c^2$. 
See Appendix~\ref{sec:app_experiment_results} for all $T$.
}
\end{center}
\vspace{-5mm}
\end{figure*}

\textbf{The Value of Optimization over Estimation when $K=2$ (Appendix~\ref{sec:app_experiment_k_equals_two}).~}
Figure~\ref{fig:two_d_head_to_head} compares LOC versus regression at $T=5$ with different costs, showing that we maintain a similar cost ratio to the regression alternative, but with lower failure rates. 
Table~\ref{tab:two_d_summary} aggregates failure rates and cost ratios for all settings, showing LOC consistently achieves lower failure rates for nearly all settings of $T$. 
When $T=5$, LOC also achieves lower cost ratios versus regression on CIFAR-100, meaning that with multiple rounds of collection, we can ensure meeting performance requirements while paying nearly the optimal amount of data.
However, solving the optimization problem is generally more difficult as $K$ increases, and we sometimes over-collect data by large margins.
In practice, these outliers can be identified from common sense (e.g., if a policy suggests collecting more data than we can reasonably afford, then we would not use the policy suggestion). 
Consequently, we report these results after removing the $99$-th percentile outliers with respect to total cost for both methods. 
Nonetheless, this challenge remains when $T=1$, particularly for CIFAR-100.

\begin{table}[t]
    \centering
\begin{minipage}{0.68\linewidth}
\resizebox{\linewidth}{!}{
\begin{tabular}{lcccccc}
\toprule
   Data set & $T$ & Cost & \multicolumn{2}{c}{Power Law Regression} & \multicolumn{2}{c}{LOC} \\ \cmidrule(lr){4-5}\cmidrule(lr){6-7}
            &  &    & Failure rate & Cost ratio & Failure rate & Cost ratio \\
\midrule
 \multirow{12}{*}{\rotatebox[origin=c]{90}{CIFAR-100 (2 Types)}} 
    &  \multirow{4}{*}{1} & $(0.01, 0.0005)$&             $62\%$ &               0.89 &        $\fir{40\%}$ &             41.80 \\
    &                     & $(0.01, 0.001)$ &             $58\%$ &               1.19 &        $\fir{46\%}$ &             9.85 \\
    &                     & $(0.01, 0.002)$ &             $56\%$ &               1.55 &        $\fir{54\%}$ &             6.98 \\
    &                     & $(0.01, 0.005)$ &             $54\%$ &               1.65 &        $\fir{33\%}$ &             4.43 \\ \cmidrule{2-7}
    &  \multirow{4}{*}{3} & $(0.01, 0.0005)$&             $43\%$ &               3.47 &        $\fir{30\%}$ &             4.88 \\
    &                     & $(0.01, 0.001)$ &             $45\%$ &               1.22 &        $\fir{43\%}$ &             1.31 \\
    &                     & $(0.01, 0.002)$ &             $45\%$ &               1.47 &        $\fir{44\%}$ &             1.21 \\
    &                     & $(0.01, 0.005)$ &             $38\%$ &               1.31 &        $\fir{36\%}$ &             1.17 \\ \cmidrule{2-7}
    &  \multirow{4}{*}{5} & $(0.01, 0.0005)$&             $38\%$ &               3.31 &        $\fir{24\%}$ &             5.19 \\
    &                     & $(0.01, 0.001)$ &             $35\%$ &               1.22 &        $\fir{24\%}$ &             0.79 \\
    &                     & $(0.01, 0.002)$ &       $\fir{37\%}$ &               1.33 &        $     38\% $ &             0.90 \\
    &                     & $(0.01, 0.005)$ &             $36\%$ &               1.30 &        $\fir{24\%}$ &             0.82 \\
 \midrule
   \multirow{12}{*}{\rotatebox[origin=c]{90}{BDD100K (Semi-supervised)}} 
    &  \multirow{4}{*}{1} &  $(1, 0.005)$ &         $86\%$ &                 0.11 &        $\fir{44\%}$ &             7.02    \\
    &                     &   $(1, 0.01)$ &         $79\%$ &                 0.15 &        $\fir{30\%}$ &            13.47    \\
    &                     &   $(1, 0.05)$ &         $72\%$ &                 0.19 &        $\fir{49\%}$ &             1.02    \\
    &                     &    $(1, 0.1)$ &         $70\%$ &                 0.19 &        $\fir{65\%}$ &             0.40    \\ \cmidrule{2-7}
    &  \multirow{4}{*}{3} &  $(1, 0.005)$ &         $23\%$ &                 0.18 &        $\fir{ 7\%}$ &             1.20    \\
    &                     &   $(1, 0.01)$ &         $21\%$ &                 0.15 &        $\fir{ 7\%}$ &             2.57    \\
    &                     &   $(1, 0.05)$ &         $26\%$ &                 0.18 &        $\fir{23\%}$ &             0.50    \\
    &                     &    $(1, 0.1)$ &         $26\%$ &                 0.21 &        $\fir{30\%}$ &             0.15    \\ \cmidrule{2-7}
    &  \multirow{4}{*}{5} &  $(1, 0.005)$ &         $16\%$ &                 0.22 &        $\fir{ 2\%}$ &             1.91    \\
    &                     &   $(1, 0.01)$ &         $21\%$ &                 0.15 &        $\fir{ 2\%}$ &             0.86    \\
    &                     &   $(1, 0.05)$ &         $16\%$ &                 0.17 &        $\fir{ 9\%}$ &             0.27    \\
    &                     &    $(1, 0.1)$ &         $16\%$ &                 0.20 &        $\fir{ 7\%}$ &             0.32    \\
\bottomrule
\end{tabular}
}
\end{minipage}
\hfill
\begin{minipage}{0.295\linewidth}
    \caption{
        Average cost ratio $\bc^\tpose (\bq^*_T - \bq_0) / \bc^\tpose (\bD^* - \bq_0) - 1$ and failure rate  over different $V^*$ for each $T$ and $\bc$, after removing $99$-th percentile outliers. We fix $P=10^{13}$ for CIFAR-100 and $P = 10^{8}$ for BDD100K.
        The best performing failure rate for each setting is bolded. The cost ratio is measured over instances that achieve $V^*$. 
        LOC consistently reduces the average failure rate, and for $T > 1$, preserves the cost ratio. 
        Further, LOC is more robust to uneven costs than regression.}
    \label{tab:two_d_summary}
\end{minipage}
\end{table}

%% file: sections/conclusion.tex
\section{Discussion}
\label{sec:conclusion}

We develop a rigorous framework for optimizing data collection workflows in machine learning applications, by introducing an optimal data collection problem that captures the uncertainty in estimating data requirements. 
We generalize this problem to more realistic settings where multiple data sources incur different collection costs. 
We validate our solution algorithm, LOC, on six data sets covering classification, segmentation, and detection tasks to show that we consistently meet pre-determined performance metrics regardless of costs and time horizons.

Our approach relies on estimating the CDF and PDF of the minimum data requirement, which is a challenging problem, especially with multiple data sources. 
Nonetheless, LOC can be deployed on top of future advances in estimating neural scaling laws.
Further, we allow practitioners to input problem-specific costs and penalties, but these quantities may not always be readily available.
We provide some theoretical insight into parameter selection and show that LOC is robust to these parameters.
Finally, our empirical analysis focuses on computer vision, but we expect our approach to be viable in other domains governed by scaling laws.

Improving data collection practices yields potentially positive and negative societal impacts. 
LOC reduces the collection of extraneous data, which can, in turn, reduce the environmental costs of training models.
On the other hand, equitable data collection should also be considered in real-world data collection practices that involve humans. 
We envision a potential future work to incorporate privacy and fairness constraints to prevent over- or under-sampling of protected groups.
Finally, our method is guided by a score function on a held-out validation set. Biases in this set may be exacerbated when optimizing data collection to meet target performance.

There is a folklore observation that over $80\%$ of industry machine learning projects fail to reach production, often due to insufficient, noisy, or inappropriate data~\citep{van_der_meulen_mccall_2018, venturebeat_2019}. 
Our experiments verify this by showing that na\"{i}vely estimating data requirements will often yield failures to meet target performances. 
We believe that robust data collection policies obtained via LOC can reduce failures while further guiding practitioners on how to manage both costs and time.

%% file: sections/app_notation.tex
\section{Table of Notation}
\label{sec:app_notation}

\begin{table}[!h]
    \centering
    \def\arraystretch{1.5}
    \resizebox{\linewidth}{!}{
    \begin{tabular}{c l}
        \toprule
        $T \in \mathbb{N}$ & Total number of rounds of collection.\\
        $V^* \in \mathbb{R}$ & Target performance.\\
        $P \in \mathbb{N}$ & Penalty for failing to reach target performance.\\
        $c \in \mathbb{N}$ & Cost of per-item collection.\\
        $q_t \in \mathbb{N}, \: t \in \{0,\ldots,T\}$ & Number of data points in round $t$.\\
        $\set{D}_q$ & Data set of size $q$.\\
        $V(\set{D})$ & Valuation of model trained on data set $\set{D}$.\\
        $V_q := V(\set{D}_q)$ & Short-hand notation for model valuation.\\
        $D^* := \argmin_q \left\{q \;|\; V_q \geq V^* \right\}$ & The minimum data requirement.\\
        $F(q) = P(D^* < q)$ & The CDF of the minimum data requirement.\\
        $f(q) = dF(q)/dq$ & The PDF of the minimum data requirement.\\
        $L(q_1,\ldots,q_T;D^*)$ & The stochastic objective function for data collection.\\
        $d_t \geq 0$ & Amount of data to collect in round $t$.\\
        $\set{R} := \left\{ (rq_t / R, V(\dataset_{rq_t/R})) \right\}_{r=1}^R $ & The regression set representing the learning curve of a given model and data set.\\
        $K \in \mathbb{N}$ & Number of data sources in multi-source setting.\\
        $\mathbf{q}_t \in \mathbb{N}^K$ & Vector containing number of data points per-source in round $t$.\\
        $\mathbf{c} \in \mathbb{N}^K$ & Vector of per-source costs.\\
        
        \bottomrule
    \end{tabular}
    }
    \caption{Table of notation used throughout paper.}
    \label{tab:notation}
\end{table}

%% file: sections/app_optimization_algo.tex
\section{Learn-Optimize-Collect}
\label{sec:app_opt}

Algorithm~\ref{alg:data_collection} summarizes the complete workflow of LOC within the data collection problem. 
Given a target score $V^*$, we collect data until we have met the target or until $T$ rounds have passed. 
In this section, we first expand our complete LOC algorithm. 
We then provide further details on the regression procedure used to estimate the data requirement distributions. 
Finally, we introduce a more practical reformulation of our optimization problem~\eqref{eq:expected_risk_prob}, which is what we later use in our numerical experiments.

\subsection{Details on the Learning and Optimizing Approach}

Our approach for determining how much data to collect consists of three steps. 
The first step is to collect performance statistics by measuring the training dynamics of the current data set $\dataset_{q_t}$. Here, we sample subsets of the current data set, train our model on the subset, and evaluate the score. We repeat this process for up to $R$ subsets, where $R$ denotes the size of our performance statistics regression data set $\set{R}$.

The second step is to model the data requirement distribution $D^*$. We perform this by creating bootstrap samples from $\set{R}$, and then fitting a regression model $\hat{v}(q;\btheta)$ that can estimate the model score as a function of data set size. We then invert this regression function by solving for the minimum $\hat{q}$ for which the model will predict that we have exceeded the target performance. We repeat this process to obtain $B$ bootstrap estimates of $D^*$. We can then fit any density estimation model to approximate a probability density and a cumulative density function. In this paper, we focus on Kernel Density Estimation (KDE) and Gaussian Mixture models since such models can be easily fit and provide functional forms of the PDF.

Finally in the third step, we solve our optimization problem~\eqref{eq:expected_risk_prob} via gradient descent. This problem yields the optimal data set sizes $q_1^*, \dots, q_T^*$ that we should have at the end of each round. Furthermore, if we are in the $t$-th round for $t > 1$, we freeze the values for $q_1, \dots, q_{t-1}$ to the data set sizes that we have observed in the previous rounds. 
Upon solving this problem, we then collect data until we have $q_t$ samples, and then re-train our model to evaluate our current state.

\begin{algorithm}[t!]
\footnotesize
\caption{Learn-Optimize-Collect (LOC)}\label{alg:data_collection}
\begin{algorithmic}[1]
\State \textbf{Input:}  Initial data set $\dataset_{q_0}$, Sampling distribution $p(z)$, Score function $V(\dataset)$, Target score $V^*$, Maximum rounds $T$, Cost $c$, Penalty $P$, Regression model $\hv(q;\btheta)$, Regression set size $R$, Density Estimation model $f(q)$, Number of bootstrap samples $B$. 
\State Initialize round $t = 0$, loss $L = 0$
\Repeat
    \Procedure{Collect Performance Statistics}{}
    \State Initialize $\set{R} = \emptyset$, $\dataset_{0} = \emptyset$
    \For{$r \in \{1, \dots, R\}$}
        \State Sub-sample $\dataset_{q_tr/R} \subset \dataset_{q_t}$ by augmenting to $\dataset_{q_t(r-1)/R}$
        \State Evaluate $V(\dataset_{q_tr/R})$ and update $\set{R} \gets \set{R} \cup \{ ( q_tr/R , V(\dataset_{q_tr/R}) \}$
    \EndFor
    \EndProcedure
    \Procedure{Estimate Data Requirement Distribution}{}
    \State Initialize $\set{Q} = \emptyset$
    \For{$b \in \{1, \dots, B\}$}
        \State Create bootstrap $\set{R}_b$ by sub-sampling with replacement from $\set{R}$
        \State Fit regression model
            $\btheta^* = \argmin_{\btheta} \sum_{(q, v) \in \set{R}_b} \left(v - \hat{v}(q; \btheta)\right)^2$
        \State Estimate data requirement
            $\hat{q}_b = \argmin_{q} \left\{ q \;|\; \hat{v}(q; \btheta^*) \geq V^* \right\}$
        %
        \State Update $\set{Q} \gets \set{Q} \cup \{ \hat{q}_b \}$
    \EndFor
    \State Fit Density Estimation model $f(q)$ using the empirical distribution $\set{Q}$ and let $F(q) = \int_{0}^q f(q)dq$
    \EndProcedure
    \Procedure{Collect Data}{}
        \State Solve problem~\eqref{eq:expected_risk_prob} using Density Estimation models $f(q)$ and $F(q)$ to obtain $q^*_1, \dots, q^*_T$.
        \State Sample $\hat\dataset \sim p(z)$ until $|\hat\dataset| = q_{t} - q_{t-1}$
        \State Update loss $L \gets L + c (q_{t} - q_{t-1}) $
        \State Update $\dataset_{q_{t}} = \dataset_{q_{t-1}} \cup \hat\dataset$
    \EndProcedure
    \State Evaluate $V(\dataset_{q_t})$ and update $t \gets t + 1$ 
\Until{$V(\dataset_{q_t}) \geq V^*$ or $t = T$}
\If{$V(\dataset_{q_t}) < V^*$}
    \State Update loss $L \gets L + P$
\EndIf
\State \textbf{Output:} Final collected data set $\dataset_{q_t}$, loss $L$
\end{algorithmic}
\end{algorithm}

\subsection{Regression Model for $D^*$}

To estimate the data requirement, we build a regression model of the learning curve and then invert this curve. 
The classical learning curve literature proposes using structured functions to regress on the learning curve.
For example with neural scaling laws, the most common function is a power law model $\hat{v}(q;\btheta) := \theta_0 q^{\theta_1} + \theta_2$ where $\btheta := \{ \theta_0, \theta_1, \theta_2 \}$~\citep{hestness2017deep, viering2021shape, sun2017revisiting, rosenfeld2019constructive, bahri2021explaining, bisla2021theoretical, zhai2021scaling, hoiem2021learning, mahmood2022howmuch}. 
As a result, the main experiments of our paper use this model.

We fit each regression function by minimizing a least squares problem using the Levenberg-Marquardt algorithm as implemented by Scipy~\cite{more1978levenberg, 2020SciPy-NMeth}. The parameters for each function are initialized to either $1$ or $0$ depending on if they are product or bias terms. To further help fit the data, we use weighted least squares where each subsequent point is weighted twice as much as the previous point. This ensures that our regression model is tuned to better fit the curve for larger $q$.

\subsection{Alternate Re-formulation of Problem~\eqref{eq:expected_risk_prob}}

Although problem~\eqref{eq:expected_risk_prob} has a differentiable objective function, it includes a set of lower bound constraints, meaning that a vanilla gradient descent algorithm may not immediately apply.
A naive solution algorithm may be to use projected gradient descent. 
Instead, we show that the problem can be reformulated to remove these constraints.

For each $t$, let $d_t = q_t - q_{t-1}$ denote the additional amount of data collected in each round. Note that we can recursively re-define $q_t = q_0 + \sum_{r=1}^t d_r$. Furthermore, the ordering inequality constraints can all be re-written to non-negativity constraints $d_t \geq 0$. We now re-write problem~\eqref{eq:expected_risk_prob} using these new variables:
\begin{align*}
    \min_{d_1, \cdots d_T} \quad & c \sum_{t=1}^T d_t \left( 1 - F\left(q_0 + \sum_{r=1}^{t-1} d_r \right) \right) + P \left( 1 - F\left(q_0 + \sum_{r=1}^T d_r\right) \right) \\ 
    \st\; \quad & d_1, \dots, d_T \geq 0.
\end{align*}
The objective in the above formulation remains differentiable. Moreover, the non-negativity constraints can be removed by re-writing $d_t \gets \mathrm{softplus}(d_t)$ for all $t$. As a result, we can solve this problem by using an off-the-shelf gradient descent algorithm. 
Finally we note that in our experiments, we implemented a projected gradient descent based on~\eqref{eq:expected_risk_prob} and the above gradient descent algorithm. We found that the above streamlined approach to generate slightly better final solutions.

%% file: sections/app_proofs.tex
\section{Details on the Main Theory}
\label{sec:app_math_properties_one_d}

In this section, we explore mathematical properties and challenges of the optimal data collection problem. In Section~\ref{sec:app_mdp_alternatives}, we discuss the drawbacks of alternative approaches to our modeling framework. Then in Section~\ref{sec:app_proofs}, we prove our main theorem demonstrating an analytic solution in the single-round problem.

\subsection{Markov Decision Process Alternatives}
\label{sec:app_mdp_alternatives}

The data collection problem requires sequential decision-making in terms of collecting additional data in each round. A natural modeling approach for such problems is via Markov Decision Processes (MDPs).
However, MDP techniques are challenging for this problem due to (i) the unobservability of the state, and (ii) an infinite state space. 
That is, until we have collected enough data to meet the target, we do not know how much data is the minimum, which can furthermore be arbitrarily large. 
Here, we sketch a potential alternative MDP framework and highlight the core challenge.

Our sequential decision-making problem
can be written as a Partially Observable Markov Decision Process (POMDP)~\citep{puterman2014markov, bertsekas2012dynamic}. Furthermore, because this state variable is constant throughout the collection problem, we can write it as an EK `Learning-and-Doing' model~\citep{easley1988controlling}. Such POMDPs are defined by the tuple $(\Theta, \set{A}, \set{S}, p, r_t)$, 
where the state space characterizes the data requirement $D^* \in \Theta := \field{R}_+$, the action space characterizes the additional data collected $d_t := (q_t - q_{t-1}) \in \set{A} := \field{R}_+$, and the observation set $\set{S} := \{0,1\}$ characterizes a binary variable $\Ind \{ V(\dataset_{q_t}) \geq V^* \} = \Ind \{ q_t \geq D^* \}$. 
Furthermore, $p( \cdot | D^*, d_t)$ is the observation transition probability and $r_t(\cdot)$ is the reward function where $r_t(q_t, q_{t-1}, D^*) := -c (q_t -  q_{t-1})$ for $t \leq T$ and $r_{T+1}(q_t, q_{t-1}, D^*) := -P \Ind \{ q_T < D^* \}$.

Because the state variable is unobserved, POMDPs are typically solved by using a belief distribution of the state variable to average the reward in the value function.
When both the state and the action space of a POMDP are finite,~\citet{smallwood1973optimal} show that this value function is piecewise-linear and can be solved by exact methods, albeit under a curse of dimensionality with respect to these spaces.
Unfortunately, for a general POMDP with an infinite state and action space, 
these methods do not apply and we most often resort to approximation techniques~\citep{zhao2021active}. 
Moreover in our case, approximations based on discretizing the state and action space fall prey to the curse of dimensionality.

Alternatively, we may naively consider applying reinforcement learning.
However, note that real-world data collection tasks do not contain the requisite sizes of learning data or generalizable simulation mechanisms that are staples in reinforcement learning techniques. 
These challenges, coupled with the goal of delivering practical managerial guidelines for data collection operations,
motivate us to explore easy-to-implement techniques for optimizing data collection.

\subsection{Proof of Theorem~\ref{thm:t1_has_analytic_sol}}
\label{sec:app_proofs}

Our main theorem states that the one-round problem has an analytic solution. 
However, the proof requires several auxiliary results.
For clarity, we first reproduce the theorem.

\noindent\textbf{Theorem 1 (Repeated).}
\emph{
    Consider the one-round problem
    \begin{align*}
        \min_{q_1} \; c(q_1 - q_0) + P(1 - F(q_1)) \quad\quad \st \; q_0 \leq q_1
    \end{align*}
    Assume $F(q)$ is strictly increasing and continuous. For any $\epsilon$ such that $F(q_0) < 1 - \epsilon$, let
    $P := c / f ( F^{-1}(1 - \epsilon) )$. The optimal solution to the corresponding problem~\eqref{eq:expected_risk_prob_one_round} is $q_1^* = F^{-1}(1 - \epsilon)$. Furthermore, the optimal solution satisfies $F(q_1^*) = 1 - \epsilon$.
}

The assumption of a strictly increasing and continuous cumulative density function is needed to ensure that the data requirement distribution has a well-defined quantile function $F^{-1}(p) := \inf_{q} \{ q \;|\; F(q) \geq p \}$, where the optimal solution for any $p \in (0, 1)$ is unique.

The proof for Theorem~\ref{thm:t1_has_analytic_sol} relies on equating the one-round optimization problem to the following constrained optimization problem:
\begin{align}\label{eq:constrained_expected_risk}
    \min_{q_1} \; c(q_1 - q_0) \qquad \st \; F(q_1) \geq 1 -\epsilon \;, \;\; q_0 \leq q_1
\end{align}
where $\epsilon > 0$ is a pre-determined parameter. This above problem~\eqref{eq:constrained_expected_risk} is solving for the least amount of additional data to collect such that with probability at least $1 - \epsilon$, we collect above the minimum data requirement.
We first characterize the properties of problem~\eqref{eq:constrained_expected_risk}.
\begin{lem}\label{lem:constrained_expected_risk_is_convex}
    Problem~\eqref{eq:constrained_expected_risk} is a convex optimization problem.
\end{lem}
\begin{proof}
    We only need to prove that the set $\{ q_1 \;|\; F(q_1) \geq 1 - \epsilon \}$ is a convex set, since the objective and remaining constraint are convex. 
    Since $F(q)$ is a monotonically non-decreasing function in $q$, for any $\theta \in [0, 1]$ and $q < \hat{q}$ that satisfy the CDF constraint, we have
    \begin{align*}
        F(\theta q + (1-\theta) \hat{q}) \geq F(q) \geq 1 - \epsilon.
    \end{align*}
    Because the convex combination of any two points is in the set, the set must be convex.
\end{proof}
\begin{lem}\label{lem:constrained_expected_risk_has_opt}
    The optimal solution to problem~\eqref{eq:constrained_expected_risk} is
    \begin{align*}
        q^* = 
        \begin{cases}
            F^{-1}(1 - \epsilon), & \text{ if } F(q) < 1 - \epsilon \\
            q_0 & \text{ otherwise}
        \end{cases}
    \end{align*}    
\end{lem}
\begin{proof}
    First consider the case where $F(q_0) \geq 1 - \epsilon$. Then, $q_0$ is a feasible solution to problem~\eqref{eq:constrained_expected_risk}. Furthermore due to the second constraint, any $q < q_0$ is infeasible. Since $q_0$ minimizes the objective, it is optimal.
    
    Next consider the case where $F(q_0) < 1 - \epsilon$. Then, let $q_1 = F^{-1}(1 - \epsilon) = \inf \{ q \;|\; F(q) \geq 1 - \epsilon \}$ be the smallest solution that satisfies the CDF constraint. By the monotonicity of $F(q)$, it follows that $q_1 > q_0$. Therefore, this solution minimizes the objective.
\end{proof}

We are now ready to prove Theorem~\ref{thm:t1_has_analytic_sol}.

\begin{proof}[Proof of Theorem~\ref{thm:t1_has_analytic_sol}.]
    We prove this result by first developing two different characterizations of the optimal solution set of problem \eqref{eq:expected_risk_prob_one_round} and then applying their equivalence.
    
    First note that problem \eqref{eq:expected_risk_prob_one_round} is an optimization problem with one variable $q_1$ over a constrained domain $[q_0, \infty)$. Furthermore, the objective function is continuous everywhere, meaning that there are two possible scenarios:
    \begin{itemize}
        \item $q_1^*$ satisfies $f(q_1) = c/P$. 
        
        \item $q_1^* := q_0$ and the optimal value is $P(1 - F(q_0))$.
    \end{itemize}
    The first scenario is obtained by taking the derivative of the objective function and setting it to $0$.
    The second scenario comes from the boundary condition. Moreover, note that the objective is unbounded as $q_1 \rightarrow +\infty$ meaning the above scenario is the only boundary condition that we need to consider.

    Next, note that problem \eqref{eq:expected_risk_prob_one_round} is equivalent to the following optimization problem
    \begin{align} \label{eq:expected_risk_prob_one_round_recast}
    \begin{split}
        \min_{q_1, \epsilon} \quad & c (q_1 - q_0) + P \epsilon \\
        \st \quad   & \epsilon \geq 1 - F(q_1) \\
                    & q_1 \geq q_0
    \end{split}
    \end{align}

    \begin{align*}
        \min_\epsilon c(F^{-1}(1- \epsilon) - q_0) + P\epsilon \qquad \st \epsilon \geq 1 - F(q_0)
    \end{align*}
    
    For any fixed $\epsilon$, problem \eqref{eq:expected_risk_prob_one_round_recast} is equivalent to problem \eqref{eq:constrained_expected_risk}, and therefore by Lemma \ref{lem:constrained_expected_risk_is_convex}, the problem is convex optimization problem. 
    
    We can optimize problem \eqref{eq:expected_risk_prob_one_round_recast} by breaking into two cases. 
    First, for any fixed $\epsilon \geq 1 - F(q_0)$, setting $q_1^*(\epsilon) = 0$ attains a feasible solution and mnimizes the objective to $P\epsilon$. 
    
    Second, for any fixed $\epsilon \leq 1 - F(q_0)$,  Lemma \ref{lem:constrained_expected_risk_has_opt} states that $q_1^*(\epsilon) = F^{-1}(1 - \epsilon)$ is a corresponding optimal solution and the objective function reduces to
    \begin{align*}
        c\left( F^{-1}( 1 - \epsilon ) - q_0 \right) + P \epsilon.
    \end{align*}
    Moreover, from the original formulation \eqref{eq:expected_risk_prob_one_round}, we can substitute $q_1^*(\epsilon)$ and obtain $f(F^{-1}(1 - \epsilon) = c / P$.

    Finally, problem \eqref{eq:expected_risk_prob_one_round_recast} is optimized via the second case if and only if there exists a feasible $\epsilon \leq 1 - F(q_0)$ that satisfies
    \begin{align*}
        c\left( F^{-1}( 1 - \epsilon ) - q_0 \right) + P \epsilon \leq P (1 - F(q_0)).
    \end{align*}
    We can rewrite this condition as follows. Let $q_1 > q_0$ and assume that \eqref{eq:expected_risk_prob_one_round_required_condition} holds. Then, 
    \begin{align*}
        & c (q_1 - q_0) \leq P F(q_1) - P F(q_0) \\
        \Rightarrow\quad & c (q_1 - q_0) - P F(q_1) \leq  - P F(q_0) \\
        \Rightarrow\quad & c (q_1 - q_0) + P ( 1 -  F(q_1)) \leq P (1 - F(q_0)).
    \end{align*}
    Let $\epsilon = 1 - F(q_1)$. Since $F(q_1) \geq F(q_0)$, we have $\epsilon \leq 1 - F(q_0)$, meaning that there is a feasible $q_1 > q_0$ to problem \eqref{eq:expected_risk_prob_one_round_recast} with lower objective function value than $q_0$. 
    Thus, assumption \eqref{eq:expected_risk_prob_one_round_required_condition} guarantees that problem \eqref{eq:expected_risk_prob_one_round_recast} has an optimal solution $q_1^* = F^{-1}(1 - \epsilon^*)$ where $\epsilon^*$ must satisfy $f(F^{-1}(1 - \epsilon^*) = c / P$. 
    Conversely, if \eqref{eq:expected_risk_prob_one_round_required_condition} is not satisfiable for any $q_1 > q_0$, then we can use the same steps to show that $q_1^* = q_0$ is an optimal solution to the problem.
\end{proof}

%% file: sections/app_multivariate.tex
\section{Optimal Data Collection with Multiple Sources}
\label{sec:app_multivariate}

The multi-variate data collection problem considers multiple sources delivering different types of data required to train a model. 
Consider $K$ data sets with $q^1, \dots, q^K$ points in each, respectively. 
Rather than collecting up to $q_t$ data points in each round, we optimize a vector $\bq_t \in \field{R}^K_+$ where each element $q^k_t$ refers to how much data we need from the $k$-th source.
Furthermore, the minimum data requirement is now a vector $\bD^*$.

This problem can be solved using the same general approach outlined in Algorithm~\ref{alg:data_collection}, but with two changes.
First, we require a multi-variate version of the PDF and CDF of $\bD^*$. This necessitates new neural scaling law regression models for dealing with multiple data sources.
Second, we modify the optimization problem from Appendix~\ref{sec:app_opt} to accommodate decision vectors.
We highlight the above two steps in this section.

\subsection{A General Multi-variate Neural Scaling Law}

In order to construct a PDF and CDF in the multi-variate setting, we follow the same general steps as in Algorithm~\ref{alg:data_collection}.
We first collect a data set of performance statistics $\set{R} := \{ (\bq_r, V(\dataset^1_{q^1_r}, \dataset^2_{q^2_r}, \cdots, \dataset^K_{q^K_r} ) \}_{r=1}^R$ as before.
We then use bootstrap resamples of this data set to fit parameters $\btheta^*$ to a regression model $\hat{v}(q^1, \dots, q^K; \btheta)$ and then solve for 
\begin{align*}
    \hat{\bq} := \argmin_{\bq} \{ \bc^\tpose \bq \;|\; \hat{v}(q^1, \dots, q^K; \btheta^*) \geq V^* \}.
\end{align*}
Finally, we fit a density estimation model over our data set of $\hat{\bq}$.

The key challenge to this approach however is in designing a multi-variate regression function.
To the best of our knowledge, the neural scaling law literature has not explored general power law models that can accommodate $K$ different types of data for arbitrary tasks.

We propose an easy-to-implement baseline regression model by adding the contributions of each data set being used.  Then, our additive regression model is
\begin{align*}
    \hat{v}(q^1, \dots, q^K; \btheta) := \sum_{k=1}^K \hat{v}_k(q^k; \btheta_k)
\end{align*}
where $\hat{v}_k(q_k; \btheta_k)$ can be any single-variate regression model for estimating score. For instance, consider $K=2$ data types with power law regression models for each data type. Our multi-variate regression model becomes
\begin{align*}
    \hat{v}(q^1, q^2; \btheta) = \theta_{1,0} (q^1)^{\theta_{1, 1}} + \theta_{2,0} (q^2)^{\theta_{2, 1}} + \theta_3.
\end{align*}
We may also incorporate other base models in the same way, such as the logarithmic or arctan functions introduced in~\citet{mahmood2022howmuch}.

The benefit of this additive regression model is that we can easily fit it via least squares minimization using a regression data set $\set{R} := \{ ( q^1_r, q^2_r, V(\dataset^1_{q^1_r}, \dataset^2_{q^2_r}) \}_{r=1}^R$. 
Furthermore, additive models are simple and offer interpretable explanations on the contributions of each data type to performance by assuming that each data set has an independent effect. 
Finally, additive models are common in many other tasks, such as when estimating the valuation of specific data points~\citep{ghorbani2019data}.

We remark that some recent research has explored neural scaling laws for specific tasks with multiple data types.
For instance,~
\citet{mikami2021scaling} explore a $K=2$ power law for transfer learning from synthetic to real domains, where they use a multiplicative component that captures an interaction between real and synthetic data sets. 
Because scaling laws in multi-variate settings remain an open area of study, if there exist specific structural regression functions for a given application with different types of data, then such functions should be used in place of the additive model. 
Moreover, our downstream optimization model operates independently of the regression model, as long as the regression model can be re-trained with bootstrap samples in order to facilitate density estimation.

\subsection{The Optimization Problem with Multiple Decisions}

Just as in the single-variate case, problem~\eqref{eq:expected_risk_multivariate} has a differential objective function but a series of lower bound constraints. We can use the same approach highlighted in Appendix~\ref{sec:app_opt} to reformulate this problem and remove the constraints. We summarize this reformulation below.

For each $t$, let $\bd_t = \bq_t - \bq_{t-1}$ be the additional data collected in each round. Then, we recursively re-define $\bq_t = \bq_0 + \sum_{r=1}^t \bd_r$ and re-write the problem to
\begin{align*}
    \min_{\bd_1, \cdots \bd_T} \quad & \bc^\tpose \sum_{t=1}^T \bd_t \left( 1 - F\left(\bq_0 + \sum_{r=1}^{t-1} \bd_r \right) \right) + P \left( 1 - F\left(\bq_0 + \sum_{r=1}^T \bd_r\right) \right) \\ 
    \st\; \quad & \bd_1, \dots, \bd_T \geq 0.
\end{align*}
The above problem can now be solved using off-the-shelf gradient descent.

%% file: sections/app_experiments.tex
\section{Simulation Experiment Setup}
\label{sec:app_experiments}

The most intuitive approach of validating our data collection problem is by repeatedly sampling from a data set, training a model, and solving the optimization problem. However, performing a large set of such experiments over many data sets becomes computationally intractable. 
Instead, we follow the approach introduced in~\citet{mahmood2022howmuch}, which proposes a simulation model of the data collection problem.
This section summarizes the simulation setup.

The simulation replicates the steps in Algorithm~\ref{alg:data_collection} except with one key difference. In the simulation, we replace the score function $V(\dataset)$ with a \emph{ground truth} function $v\gt(q)$ that serves as an oracle which reports the expected score of the model trained with $q$ data points. 
Thus, rather than having to collect data and train a model in each round, we evaluate $v\gt(q_t)$ and treat this as the current model score. 
The optimization and regression models do not have access to $v\gt(q)$.

\subsection{A Piecewise-Linear Ground Truth Approximation}

In order to build a ground truth function, we first use the sub-sampling procedure in Algorithm~\ref{alg:data_collection} to collect performance statistics over subsets of the entire training data set. 
Using these observed statistics, we then build a piecewise-linear model of the ground truth. 
Below, we first highlight how to construct a piecewise-linear model when given a set of data set sizes and their corresponding scores. In the next subsection, we will detail the exact data collection process.

\noindent\textbf{The Single-variate ($K=1$) Case.}
\citet{mahmood2022howmuch} develop a ground truth function as follows. Let $q_0 \leq q_1 \leq q_2 \leq \cdots$ be a series of data set sizes and let $\dataset_{q_0} \subset \dataset_{q_1} \subset \dataset_{q_2} \subset \cdots$ be their corresponding sets. Then, consider the following piecewise-linear function:
\begin{align*}
    v\gt(q) := \begin{cases} \displaystyle
        \frac{V(\dataset_{q_0})}{q_0} n , & q \leq q_0 \\  \displaystyle
        \frac{V(\dataset_{q_t}) - V(\dataset_{q_{t-1}})}{q_t - q_{t-1}} \left( q - q_t \right) + V(\dataset_{q_{t-1}}), & q_{t-1} \leq q \leq q_t
    \end{cases}
\end{align*}
This function is concave and monotonically increasing, which follows the general trend of real learning curves~\cite{hestness2017deep}. Furthermore~\citet{mahmood2022howmuch} show that given sufficient resolution, i.e., enough data subsets, this piecewise linear function is an accurate approximation of the true learning curve $V(\dataset)$.

\noindent\textbf{The Multi-variate $(K=2)$ Case.}
In the previous $K=1$ case, the ground truth was formed by taking linear approximations between different subset sizes. When $K > 1$, we have multiple subsets that are used to evaluate the score $V(\dataset^1, \dots, \dataset^K)$.

We focus specifically on $K=2$ in our numerical experiments and propose a generalization of the previous piecewise-linear function. Here, rather than building lines on the intervals between subsequent sets, we build planes on triangular intervals. Specifically, let $q^1_0 \leq q^1_1 \leq q^1_2 \leq \cdots$ and  $q^2_0 \leq q^2_1 \leq q^2_2 \leq \cdots$ be two series of data set sizes, and consider the grid 
\begin{align*}\begin{array}{cccc}
    (q^1_0, q^2_0)  & (q^1_1, q^2_0)   & (q^1_2, q^2_0)     & \cdots \\
    (q^1_0, q^2_1)  & (q^1_1, q^2_1)   & (q^1_2, q^2_1)     & \cdots \\
    (q^1_0, q^2_2)  & (q^1_1, q^2_2)   & (q^1_2, q^2_2)     & \cdots \\
    \vdots          & \vdots           & \vdots             & \ddots \\
    \end{array}
\end{align*}
For each tuple $(q^1_s, q^2_t)$ in the above grid, let $V(\dataset^1_{q^1_s}, \dataset^2_{q^2_t})$ be the score of a model trained on two data sets of the corresponding respective sizes.

For each index $(s, t)$ we fit linear models on the corresponding lower right and upper left triangles. First, let $(\underline{\alpha}(s, t), \underline{\beta}(s, t), \underline{\gamma}(s, t))$ be the parameters of the plane defined by the lower triangle $\{ (q^1_{s-1}, q^2_t), (q^1_{s}, q^2_{t-1}), (q^1_{s}, q^2_t) \}$, i.e., the unique solution to the following linear system:
\begin{align*}
    \begin{pmatrix}
        q^1_{s-1}   & q^2_t      & 1 \\
        q^1_{s}     & q^2_{t-1}  & 1 \\
        q^1_{s}     & q^2_{t}    & 1 \\
    \end{pmatrix}
    \begin{pmatrix}
        \underline{\alpha}(s, t) \\
        \underline{\beta}(s, t) \\
        \underline{\gamma}(s, t)
    \end{pmatrix}
    = 
    \begin{pmatrix}
        V(\dataset^1_{q^1_{s-1}}, \dataset^2_{q^2_t}) \\
        V(\dataset^1_{q^1_{s}}, \dataset^2_{q^2_{t-1}}) \\
        V(\dataset^1_{q^1_{s}}, \dataset^2_{q^2_t}) \\
    \end{pmatrix}
\end{align*}
Thus, for any data set sizes $(q^1, q^2)$ in this triangle, we evaluate the ground truth by the linear model $\underline{\alpha}(s, t) q^1 + \underline{\beta}(s, t) q^2 + \underline{\gamma}(s, t)$. 
Similarly, let $(\overline{\alpha}(s, t), \overline{\beta}(s, t), \overline{\gamma}(s, t))$ be the parameters of the plane defined by the upper triangle $\{ (q^1_{s}, q^2_t), (q^1_{s+1}, q^2_{t}), (q^1_{s}, q^2_{t+1}) \}$, i.e., the unique solution to the following linear system:
\begin{align*}
    \begin{pmatrix}
        q^1_{s}   & q^2_t        & 1 \\
        q^1_{s+1} & q^2_{t}      & 1 \\
        q^1_{s}   & q^2_{t+1}    & 1 \\
    \end{pmatrix}
    \begin{pmatrix}
        \overline{\alpha}(s, t) \\
        \overline{\beta}(s, t) \\
        \overline{\gamma}(s, t)
    \end{pmatrix}
    = 
    \begin{pmatrix}
        V(\dataset^1_{q^1_{s}}, \dataset^2_{q^2_t}) \\
        V(\dataset^1_{q^1_{s+1}}, \dataset^2_{q^2_{t}}) \\
        V(\dataset^1_{q^1_{s}}, \dataset^2_{q^2_{t+1}}) \\
    \end{pmatrix}
\end{align*}
Similarly for any $(q^1, q^2)$ in this triangle, the ground truth is obtained by the linear model $\overline{\alpha}(s, t) q^1 + \overline{\beta}(s, t) q^2 + \overline{\gamma}(s, t)$.

Finally, we define our ground truth function $v\gt(q^1, q^2)$. For any $q^1 \geq q^1_0, q^2 \geq q^2_0$, this function first identifies the interval $[q^1_s, q^1_{s+1}] \times [q^2_t, q^2_{t+1}]$ in which the point lies. Then, the function assigns a score based on whether the point lies in the upper left or the lower right triangle in this interval. We write this function as
\begin{align*}
    v\gt(q^1, q^2) := & \begin{cases} \displaystyle
        \overline{\alpha}(s, t) q^1 + \overline{\beta}(s, t) q^2 + \overline{\gamma}(s, t) \\ \displaystyle
            \qquad\qquad\qquad\qquad\qquad \text{ if } \norm{(q^1, q^2) - (q^1_s, q^2_t)} \leq \norm{(q^1, q^2) - (q^1_{s+1}, q^2_{t+1})}, \\  \displaystyle
        \underline{\alpha}(s+1, t+1) q^1 + \underline{\beta}(s+1, t+1) q^2 + \underline{\gamma}(s+1, t+1) \\
            \qquad\qquad\qquad\qquad\qquad \text{ otherwise},
    \end{cases}
    \\
    & \text{for } q^1_s \leq q^1 \leq q^1_{s+1} \;,\; q^2_t \leq q^2 \leq q^2_{t+1}.
\end{align*}

\begin{table}[t!]
    \centering
    \footnotesize
\resizebox{\linewidth}{!}{
\begin{tabular}{lll cc}
\toprule
             Data set & Task & Score & \multicolumn{2}{c}{Full data set size} \\
\midrule
    CIFAR-10~\cite{krizhevsky2009learning}  & Classification & Accuracy & \multicolumn{2}{c}{$50,000$}     \\
    CIFAR-100~\cite{krizhevsky2009learning} & Classification & Accuracy & \multicolumn{2}{c}{$50,000$}      \\
    ImageNet~\cite{deng2009imagenet} & Classification & Accuracy & \multicolumn{2}{c}{$1,281,167$}   \\ 
    BDD100K~\cite{yu2020bdd100k} & Semantic Segmentation & Mean IoU & \multicolumn{2}{c}{$7,000$}          \\
    nuScenes~\cite{nuscenes2019} & BEV Segmentation & Mean IoU & \multicolumn{2}{c}{$28,130$}    \\ 
    VOC~\cite{pascal-voc-2007, pascal-voc-2012} & 2-D Object Detection & Mean AP & \multicolumn{2}{c}{$16,551$}       \\ 
\midrule
    CIFAR-100~\cite{krizhevsky2009learning} & Classification & Accuracy & $25,000$ (Classes 0-49) & $25,000$ (Classes 50-99) \\
    BDD100K~\cite{yu2020bdd100k} &  Semantic Segmentation & Mean IoU & $7,000$ (Labeled) & $70,000$ (Unlabeled) \\
\bottomrule
\end{tabular}
}
    \caption{Data sets, tasks, and score functions considered.}
    \label{tab:tasks}
    \vspace{-2em}
\end{table}

\rebut{
\noindent\textbf{For $K > 2$.}
The piecewise linear approximations grow increasingly complex as the dimension $K$ increases.
Furthermore, the number of subsets of data set sizes required to create a piecewise linear approximation increases exponentially with $K$. 
Specifically for $k \in \{1,\dots, K\}$, let $M_k$ denote the number of subsets (i.e., $|\{q_0^k, q_1^k, \dots, q_{M_k}^k \}|$) of a data set that we consider when creating subsets. For each combination of $K$ subsets, we must then train a model and evaluate it's performance to record $V(\dataset^1, \dots, \dataset^K)$. Thus, we must subsample and train our model for $O(\prod_{k}M_k)$ combinations. 
This can quickly become computationally prohibitive. 
}

\subsection{Data Collection}
\label{sec:app_data_collection}

We now summarize the data collection and training process used to create the above piecewise-linear functions for each data set and task.
All models were implemented using PyTorch and trained on machines with up to eight NVIDIA V100 GPU cards. 
Table~\ref{tab:tasks} details each task and data set size.

\noindent\textbf{Image Classification Tasks.} For all experiments with CIFAR-10 and CIFAR-100, we use a ResNet18~\cite{he2016deep} following the same procedure as in~\citet{coleman2019selection}. For ImageNet, we use a ResNet34~\cite{he2016deep} using the procedure in~\citet{coleman2019selection}. 
All models are trained with cross entropy loss using SGD with momentum. We evaluate all models on Top-1 Accuracy.

For all experiments, we set the initial data set at $q_0 = 10\%$ of the data. In data collection, we create five subsets containing $2\%, 4\%, \cdots, 10\%$ of the training data, five subsets containing $12\%, 14\%, \cdots, 20\%$ of the training data, and eight subsets containing $30\%, 40\%, \cdots, 100\%$ of the data. Note that we use higher granularity in the early stage as this is where the dynamics of the learning curve vary the most. With more data, the learning curve eventually has a nearly zero slope.
For each subset, we train our respective model and evaluate performance.

\noindent\textbf{VOC.}
We use the Single-Shot Detector 300 (SSD300)~\cite{liu2016ssd} based on a VGG16 backbone~\cite{simonyan2014very}, following the same procedure as in~\citet{elezi2021towards}. All models are trained using SGD with momentum. We evaluate all models on mean AP.

For all experiments, we set the initial data set at $q_0 = 10\%$ of the data. 
In data collection, we sample twenty subsets at $5\%$ intervals, i.e., $5\%, 10\%, 15\%, \cdots, 100\%$ of the training data.

\noindent\textbf{BDD100K.}
We use Deeplabv3 \cite{chen2017rethinking} with ResNet50 backbone. We use random initialization for the backbone. We use the original data set split from \citet{yu2020bdd100k} with $7,000$ and $1,000$ data points in the train and validation sets respectively. The evaluation metrics is mean Intersection over Union (IoU).
We follow the same protocol used in the Image classification tasks to create our subsets of data.

\noindent\textbf{nuScenes.}
We use the ``Lift Splat'' architecture~\cite{liftsplat}, which is used for BEV segmentation from driving scenes, following the steps from the original paper to train this model. We evaluate on mean IoU. Our data collection procedure follows the same steps and percentages of the data set as used for BDD100K and the Image classification tasks.

\noindent\textbf{CIFAR-100 (2 Types).}
We partition this data set into two subsets $\dataset^1$ and $\dataset^2$ of $25,000$ images each containing the first 50 and last 50 classes, respectively. We then train a ResNet18~\citep{he2016deep} using different fractions of the two subsets. We follow the same training procedure as in the single-variate case except with one difference. Since some of the data sets will naturally be imbalanced (e.g., if we train with half of the first subset and all of the second subset), we employ a class-balanced cross entropy loss using the inverse frequency of samples per class.

For each $\dataset^k$ subsets, respectively, we follow the same subsampling procedure used in the single-variate case. That is, we let $q^1_0 = 10\%$ of the first data subset and $q^2_0 =10\%$ of the second data subset. For each subset, we create $10$ subsampled sets at intervals of $2\%, 4\%, 6\%, \cdots, 20\%$ of the respective data subset. We then create eight further subsampled sets at $30\%, 40\%, \cdots, 100\%$ of the respective data subset.
Finally, we train our model and evaluate the score on every combination of the subsampled subsets of $\dataset^1 \times \dataset^2$.

\noindent\textbf{BDD100K (Semi-supervised).}
For this task, we consider semi-supervised segmentation via pseudo-labeling the unlabeled data set in BDD100K. The data is partitioned into two subsets $\dataset^1$ and $\dataset^2$ containing $7,000$ labeled and $70,000$ unlabeled scenes. As before, we use the Deeplabv3~\citep{chen2017rethinking} architecture with a ResNet50 backbone. Here however, we:
\begin{enumerate}
    \item First train with a labeled subset of $\dataset^1$ via supervised learning.
    \item Pseudo-label an unlabeled subset of $\dataset^2$ using the trained model.
    \item Re-train the model with the labeled subset and the pseudo-labeled subset.
\end{enumerate}
We follow the same procedure as in the single-variate case for both training steps, except we weigh the unlabeled data by $0.2$ to reduce its contribution to the loss.

Training via semi-supervised learning on BDD100K requires long compute times, so we reduce the number of subsets used in this experiment. 
For the labeled set $\dataset^1$, we create subsets with $5\%, 10\%, 15\%, 20\%, 40\%, 60\%, 80\%, 100\%$ of the data.
For the unlabeled set $\dataset^2$, we create subsets with $0\%, 10\%, 25\%, 50\%, 100\%$ of the data.
Note that we have five settings of unlabeled data since we include the case of training with no unlabeled data as well.

\begin{table}[!t]
    \centering
    \resizebox{\linewidth}{!}{
    \begin{tabular}{p{0.35\linewidth} p{0.65\linewidth}}
        \toprule
        Parameter & Setting \\
        \midrule
        Optimizer   & GD with Momentum ($\beta = 0.9$), Adam ($\beta_0, \beta_1 = 0.9, 0.999$) \\
        Learning rate   & $0.005, \dots, 500$ \\
        Number of bootstrap samples $B$ & $500$ \\
        Number of regression subsets $R$ & See Appendix~\ref{sec:app_data_collection} \\
        Density Estimation Model    & KDE for $K=1$, GMM for $K=2$ \\
        \multirow{2}{*}{KDE Bandwidth}  & $20000, \dots, 20000000$ for ImageNet \\
                                        & $200, \dots, 4000$ for all others \\
        GMM number of clusters          & $4, \dots, 10$ \\   
        \bottomrule
    \end{tabular}
    }
    \caption{Summary of hyperparameters used in our experiments.}
    \label{tab:hyperparameters}
\end{table}

\subsection{LOC Implementation}

For all experiments, we initialize with $10$\% of the training data set. We consider $T=1,3,5$ rounds and sweep a range of $V^*$. We provide a summary of parameters in Table~\ref{tab:hyperparameters}.

For the experiments with $K=1$, we model the data requirement PDF $f(q)$ in each round of the problem as follows. We first draw $B=500$ bootstrap resamples of the current training statistics $\set{R}$, where $\set{R} = \{ (rq_0/R, V(\dataset_{rq_0/R})) \}_{r=1}^R \cup \{ (q_s, V(\dataset_{q_s})) \}_{s=1}^t$ contains all of the measured statistics up to the initial data set (e.g., for CIFAR-10, this includes performance with $2\%, 4\%, \cdots, 10\%$ of the data), and the previous collected data. The latter is obtained by calling our piecewise-linear ground truth approximation. 
For each bootstrap resample, we fit a power regression model $\hat{v}(q;\btheta) = \theta_0 q^{\theta_1} + \theta_2$ and solve for the estimated minimum data requirement. We then use our set of estimates to fit a Kernel Density Estimation (KDE) model after gridsearching for the best bandwidth parameter.

For the experiments with $K=2$, we use the same above procedure but fit Gaussian Mixture Models (GMM) due to their having an easily computable CDF via the Gaussian $\mathrm{erf}(\cdot)$, rather than numerically integrating the PDF.
We grid-search over the number of mixture components for the GMM model.

We optimize over problems~\eqref{eq:expected_risk_prob} and~\eqref{eq:expected_risk_multivariate} using gradient descent techniques. 
Depending on the current state and data set, different hyperparameters perform better.
As a result, we perform extensive hyperparameter tuning every time we need to solve the optimization problem. Here, we sweep over gradient descent with momentum and Adam with learning rates between $0.005$ to $500$.

We initialize each problem with $\bq_t$ equal to the baseline regression solution and $\bq_{t + s} = \bq_t / (s + 1)$ for all $1 \leq s \leq T - t$. That is, we set the initial value for future collection amounts to be fractions of the initial value of the immediate amount of data to collect. We identified this initialization by manually inspecting the solutions found by LOC, consequently it improves the conditioning of the loss landscape relative to other random initialization schemes.

%% file: sections/app_experiment_results.tex
\section{Additional Numerical Results}
\label{sec:app_experiment_results}

This section contains expanded results of our numerical experiments and further ablations. Our key results include:

$\bullet$ In Appendix~\ref{sec:app_estimating_data_requirement_distribution}, we evaluate the effectiveness of estimating $F(q)$ by plotting the estimated learning curves as well as the empirical histograms used to model the data requirement distribution.

$\bullet$  In Appendix~\ref{sec:app_experiment_robustness}, we explore the sensitivity of our optimization algorithm to variations in the cost and penalty parameters. In all except one instance, LOC consistently maintains a low total cost and failure rate.
    
$\bullet$  In Appendix~\ref{sec:app_experiment_k_equals_two}, we explore the multi-variate LOC (i.e., $K=2$) for problems where we have a small number of $T=1, 3$ rounds. The baseline fails for almost all instances of $T=1$, whereas LOC maintains a low failure rate.
    
$\bullet$  In Appendix~\ref{sec:app_experiment_alternate_regression_functions}, we consider variants of LOC where we use different regression functions to estimate the data requirement distribution. Our optimization framework can be deployed on top of any regression function to reduce the failure rate.

\rebut{

\subsection{Estimating the Data Requirement Distribution $F(q)$}
\label{sec:app_estimating_data_requirement_distribution}

To estimate $F(q)$, we first create an ensemble of estimated learning curves, which we then invert to obtain an empirical distribution of estimated values for $D^*$.
Figure~\ref{fig:app_learning_curves} plots our bootstrap resampled estimated learning curves versus the ground truth performance for the first round of data collection when we have access to an initial $\dataset_{q_0}$ containing $10\%$ of the full data set. 
As noted in~\citet{mahmood2022howmuch}, the mean estimated learning curve diverges from the ground truth. 
However, by bootstrap resampling an ensemble of learning curves, we can cover the ground truth with some probability.

Figure~\ref{fig:app_histograms} plots the empirical histograms of estimated $D^*$ as well as the estimated $F(q)$ obtained via KDE on CIFAR-10 with three different values for $V^*$.
Although the mode of the estimated distribution is far from the ground truth $D^*$, the estimated distribution assigns some probability to the ground truth region. 
LOC optimizes over this estimated $F(q)$, which allows us to conservatively collect data and reduce the chances of failure.
}

\begin{figure*}[!t]
\begin{center}
\begin{minipage}{0.16\linewidth}\includegraphics[width=1\textwidth]{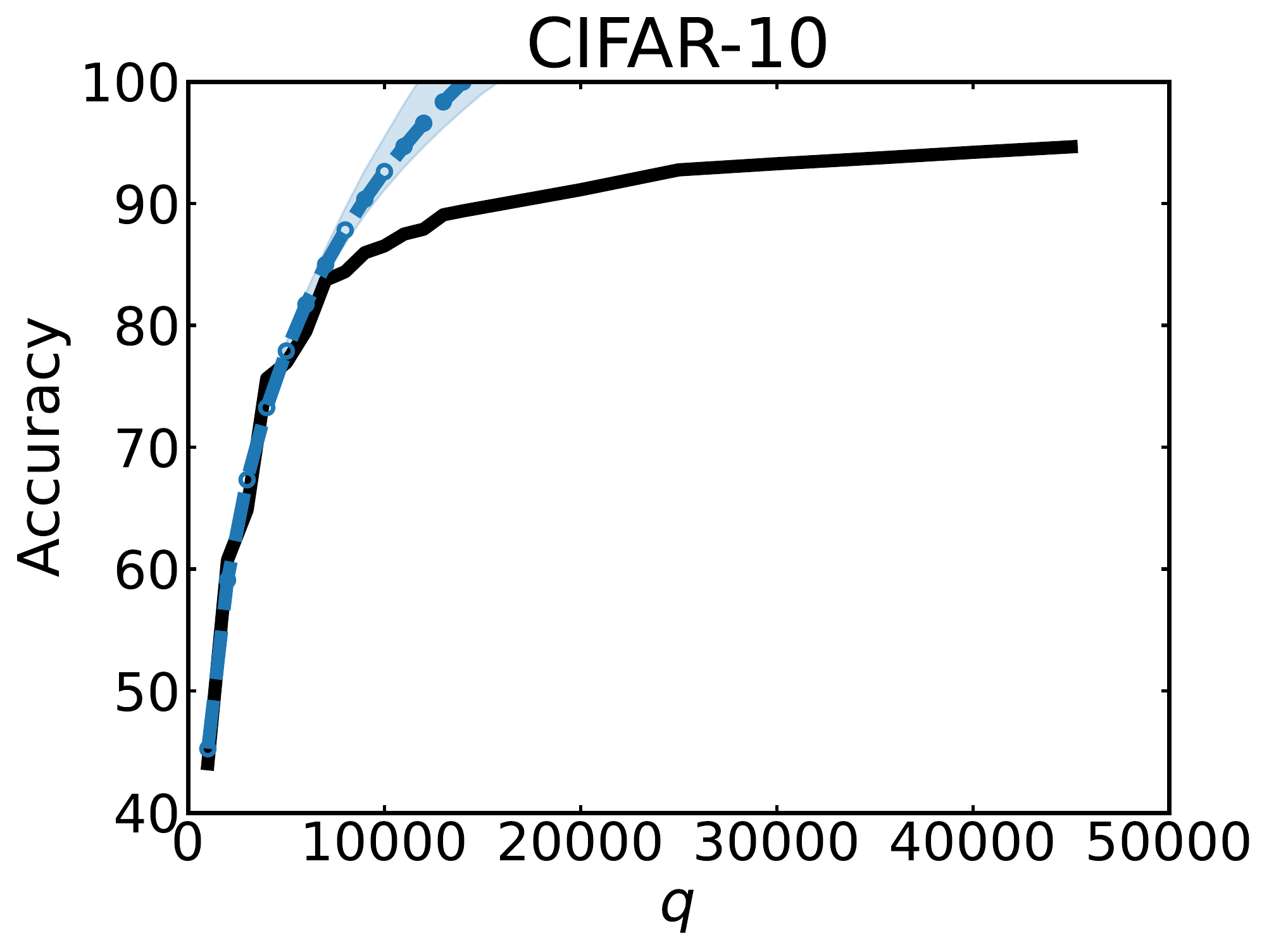}\end{minipage}
\begin{minipage}{0.16\linewidth}\includegraphics[width=1\textwidth]{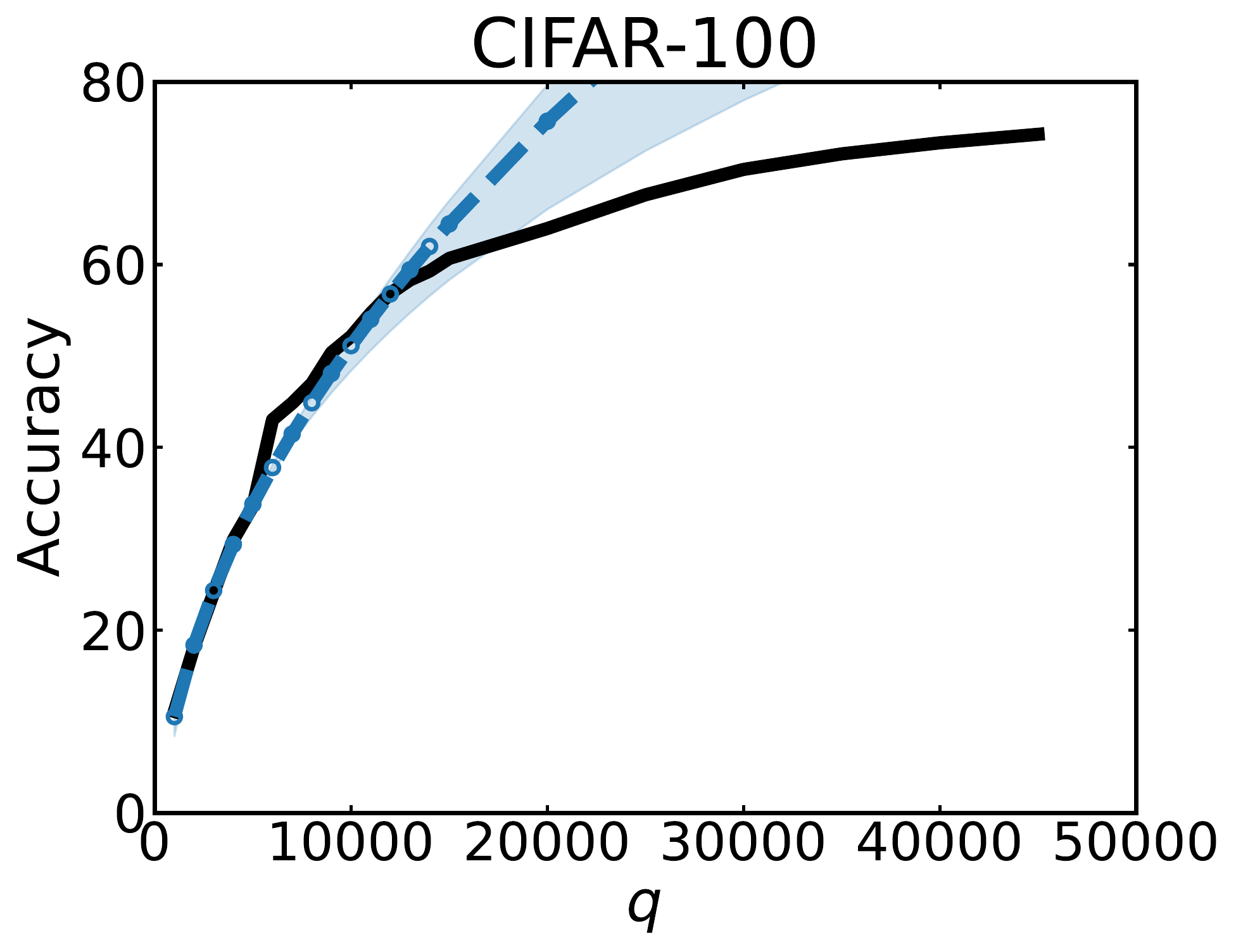}\end{minipage}
\begin{minipage}{0.16\linewidth}\includegraphics[width=1\textwidth]{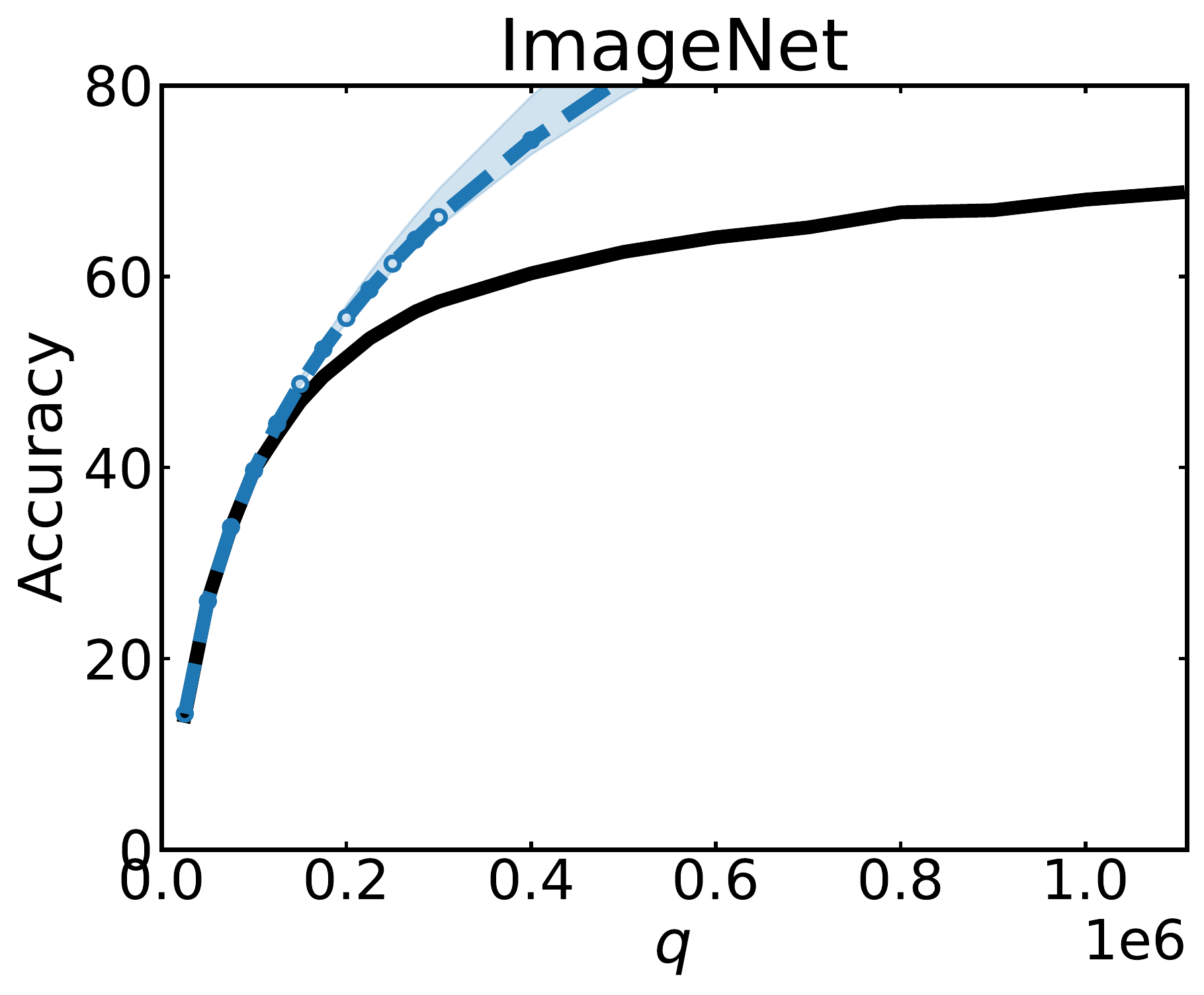}\end{minipage}
\begin{minipage}{0.16\linewidth}\includegraphics[width=1\textwidth]{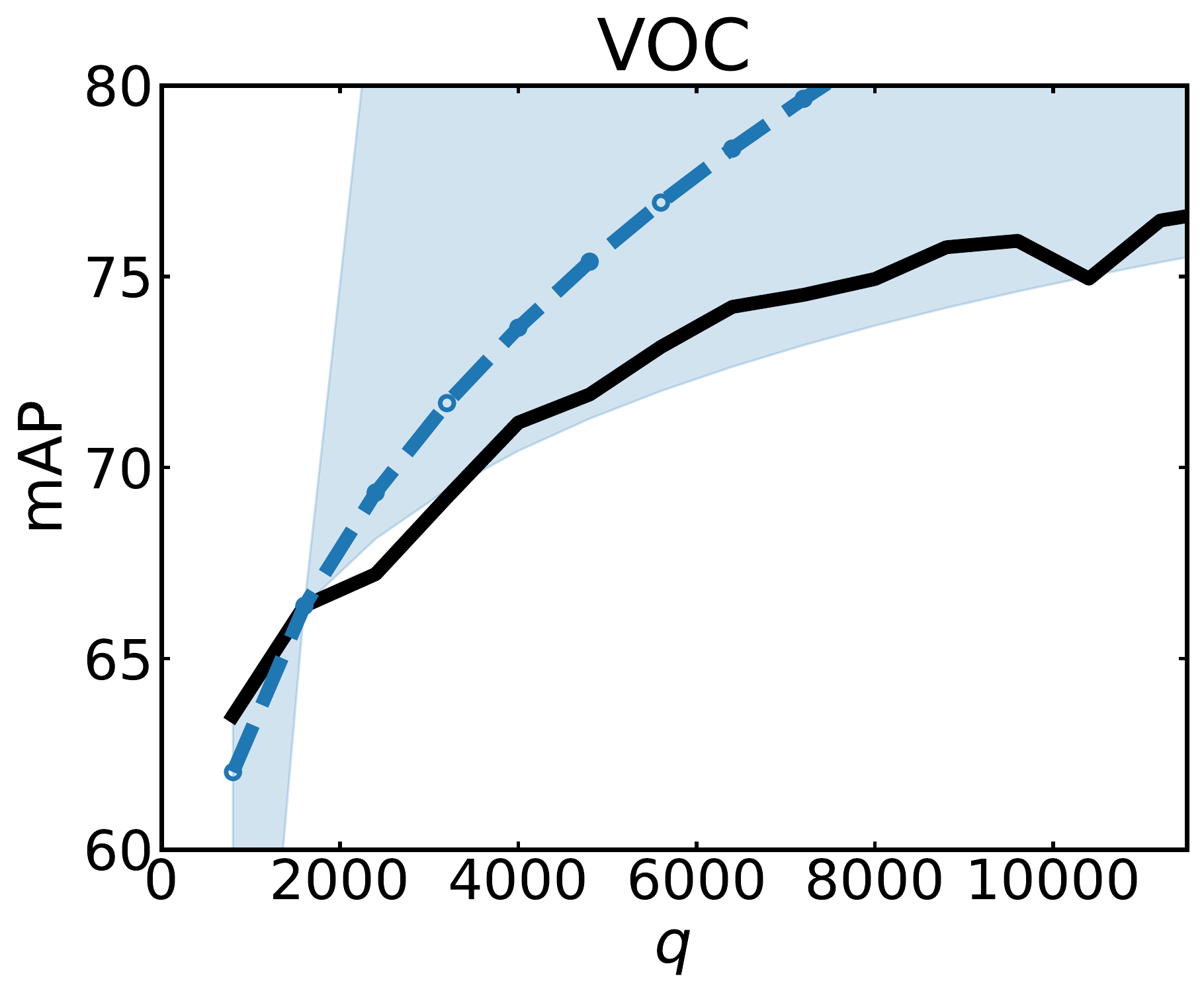}\end{minipage}
\begin{minipage}{0.16\linewidth}\includegraphics[width=1\textwidth]{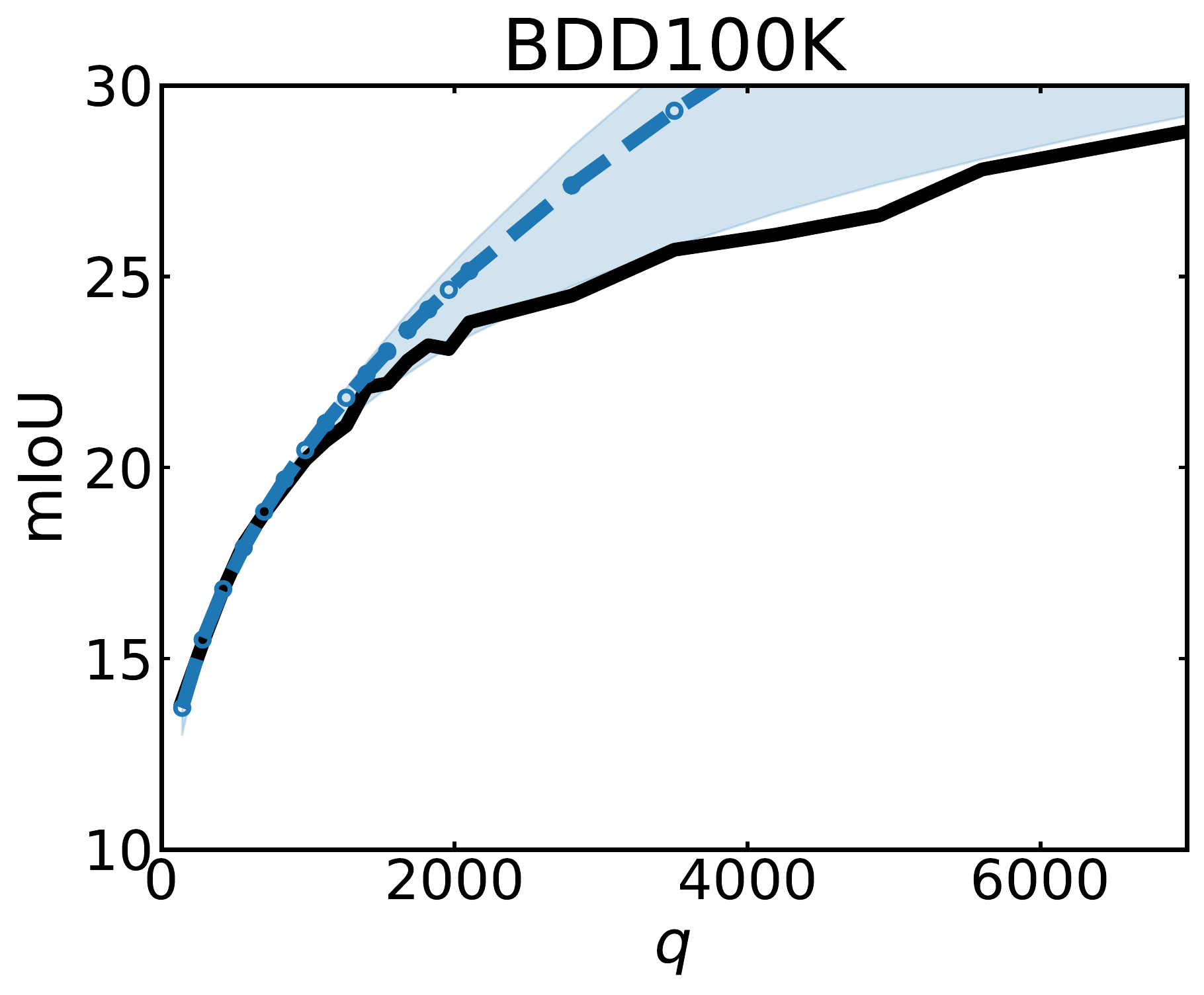}\end{minipage}
\begin{minipage}{0.16\linewidth}\includegraphics[width=1\textwidth]{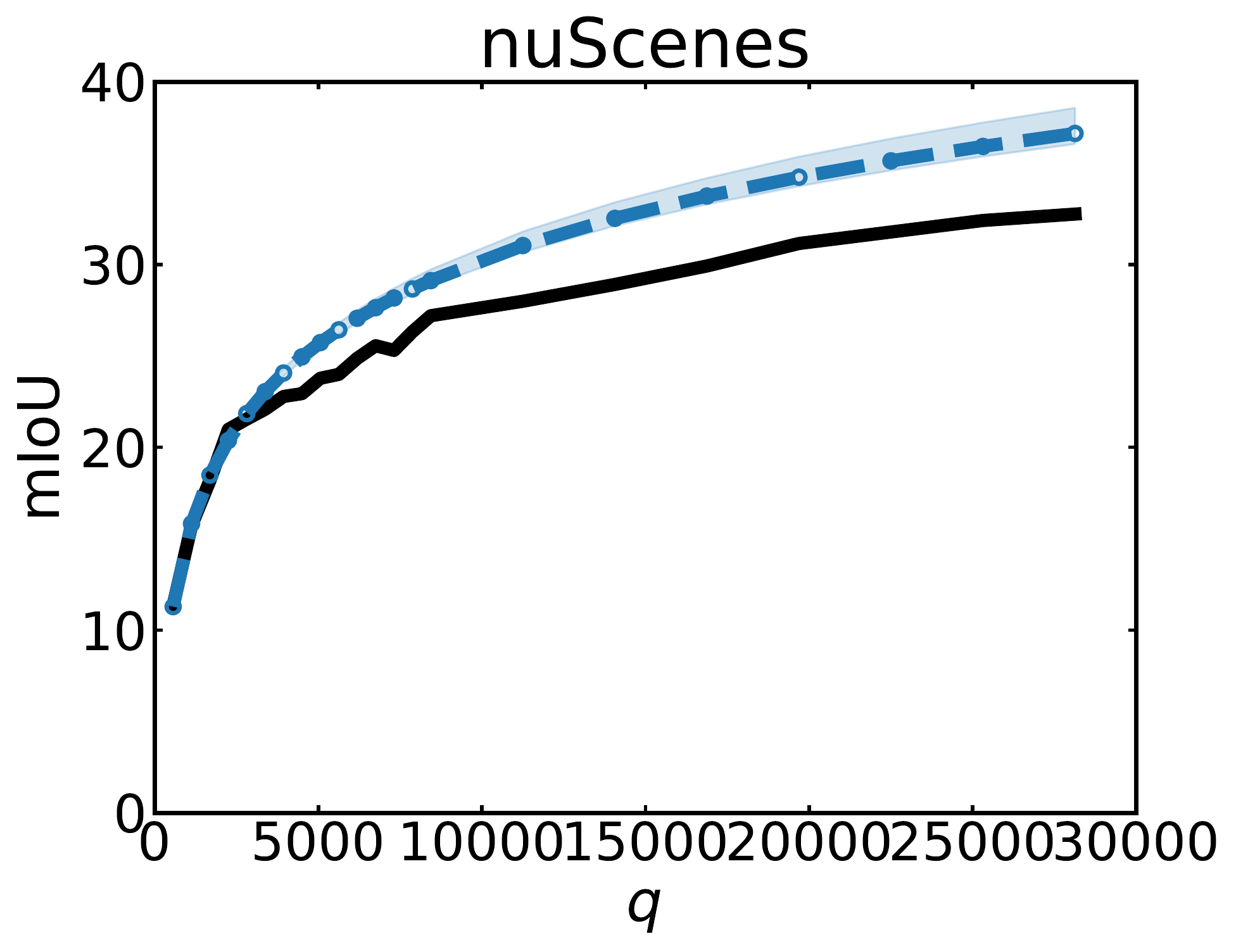}\end{minipage}
\vspace{-3mm}
\caption{\label{fig:app_learning_curves}
\rebut{
For a fixed seed, ground truth learning curves (black) and the estimated power law learning curves (blue) obtained via bootstrapping and ensembling. The shaded region represents the $95$ percentile of the ensemble and the dashed blue line represents the mean of the regression functions. The mean is consistently higher than the unknown ground truth, whereas the shaded region can at times cover it.
}
}
\end{center}
\vspace{-4mm}
\end{figure*}

\begin{figure*}[!t]
\begin{center}
\begin{minipage}{0.30\linewidth}\includegraphics[width=1\textwidth]{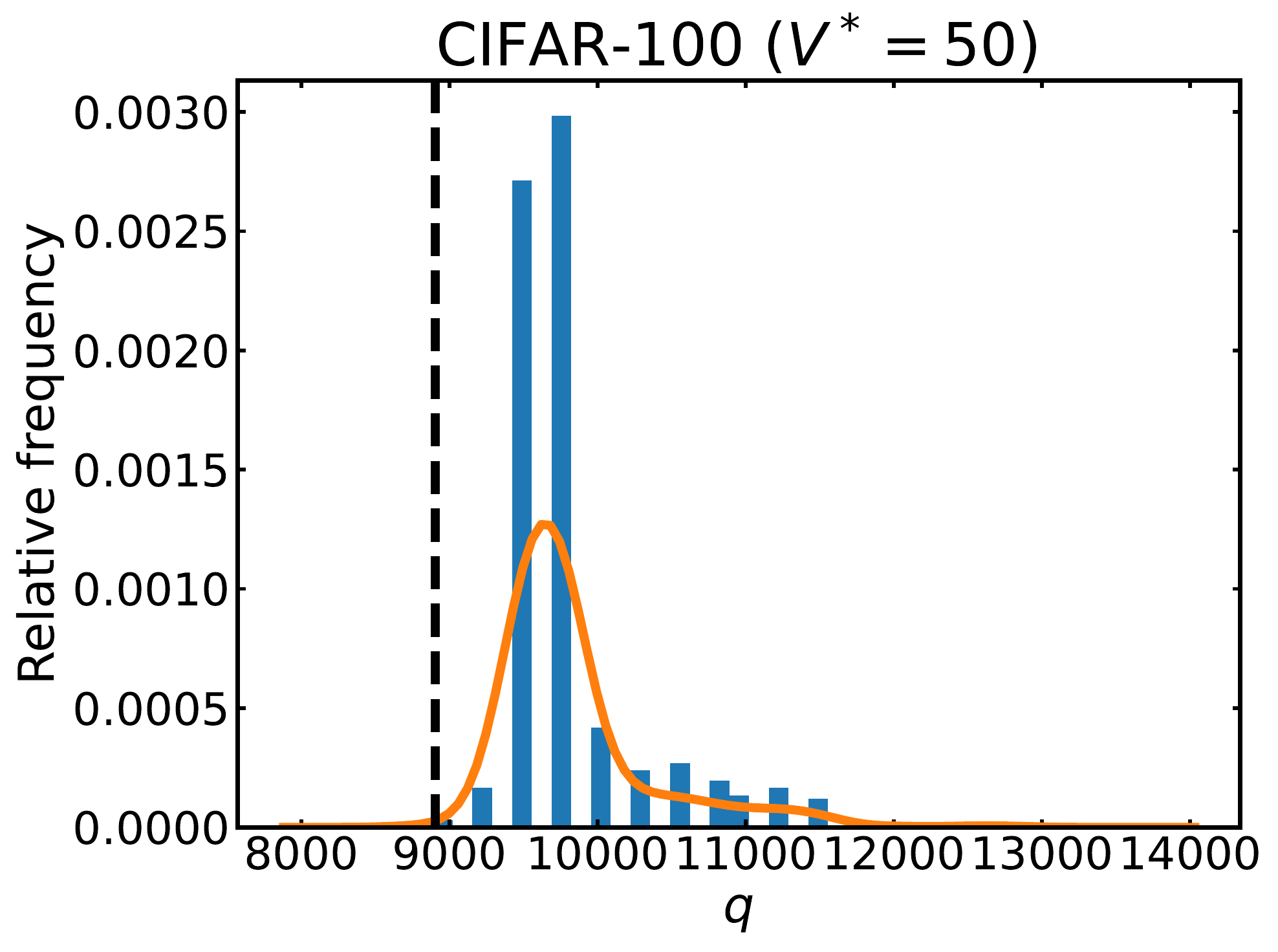}\end{minipage}
\begin{minipage}{0.30\linewidth}\includegraphics[width=1\textwidth]{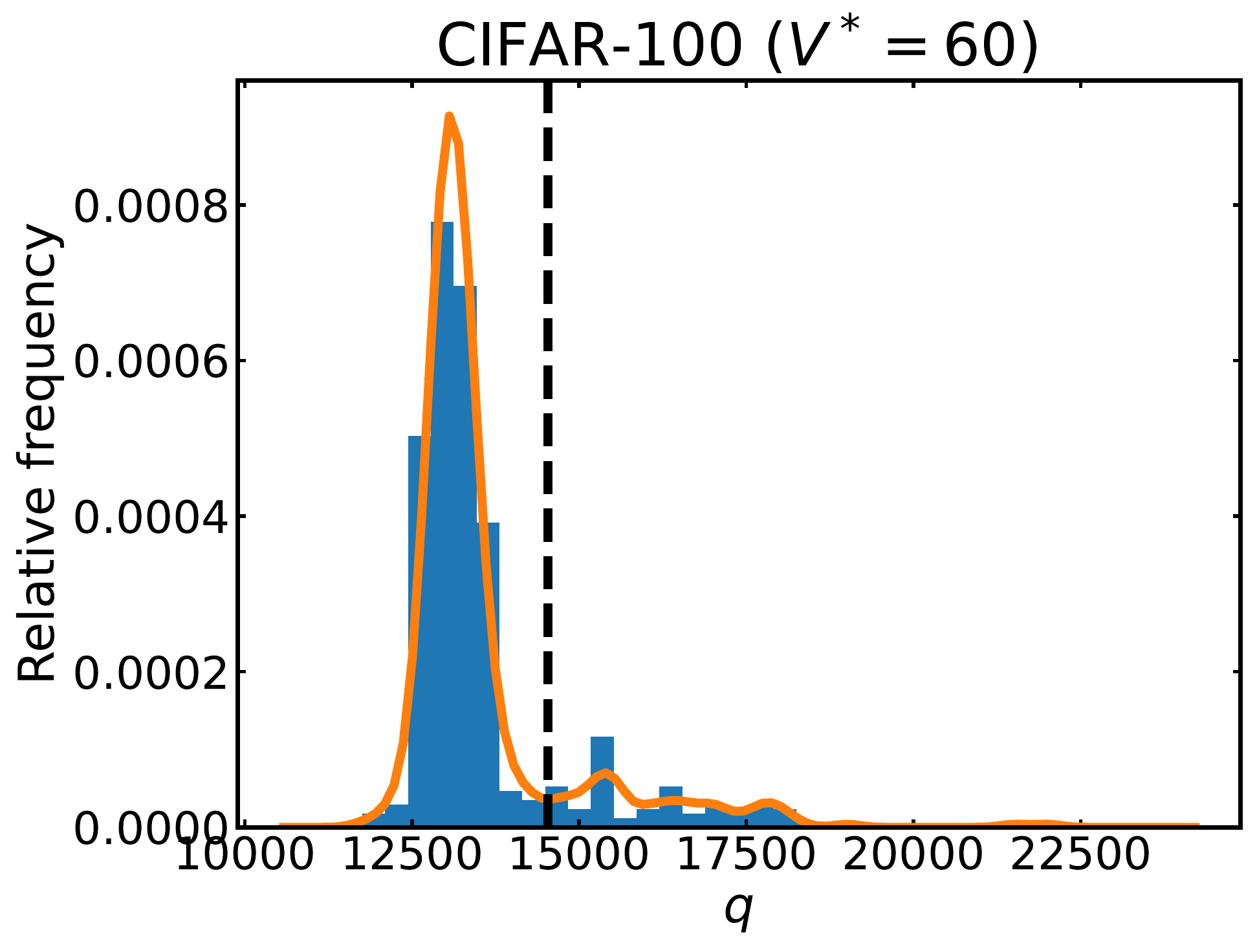}\end{minipage}
\begin{minipage}{0.30\linewidth}\includegraphics[width=1\textwidth]{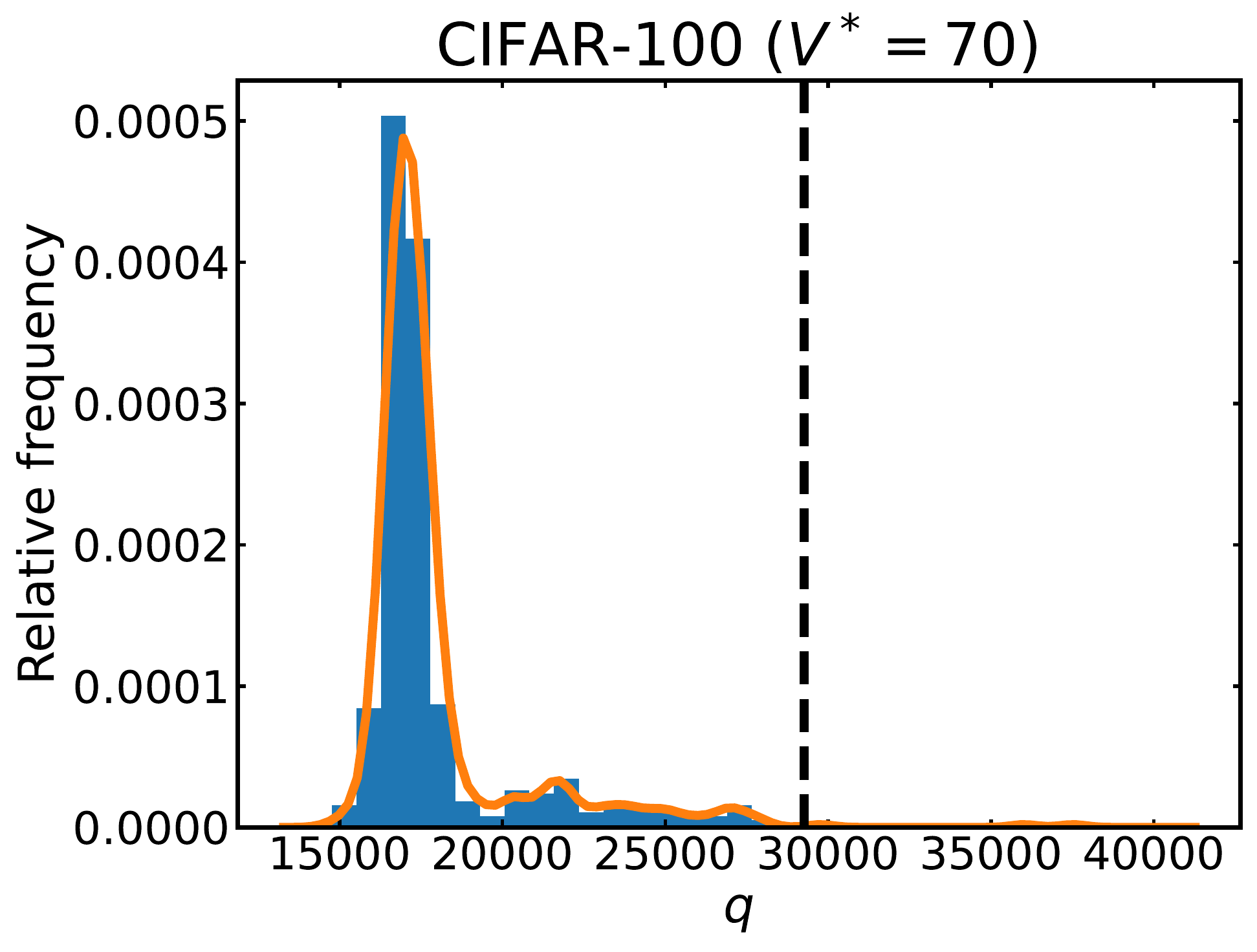}\end{minipage}
\vspace{-3mm}
\caption{\label{fig:app_histograms}
\rebut{
For a fixed seed, the histogram of estimates of $D^*$ from different bootstrapped models (blue bars), the estimated $F(q)$ (orange curve), and the ground truth $D^*$ (black dashed line). 
Each plot corresponds to a different $V^*$ for CIFAR-100 (see Figure~\ref{fig:app_learning_curves} for the learning curve). 
With higher targets, regression (i.e., collecting the mean of the distribution) will lead to larger under-estimations.
}
}
\end{center}
\vspace{-4mm}
\end{figure*}

\subsection{Robustness to the Cost and Penalty Parameters}
\label{sec:app_experiment_robustness}

\begin{figure*}[!t]
\begin{center}
\begin{minipage}{0.3\linewidth}\includegraphics[width=1\textwidth]{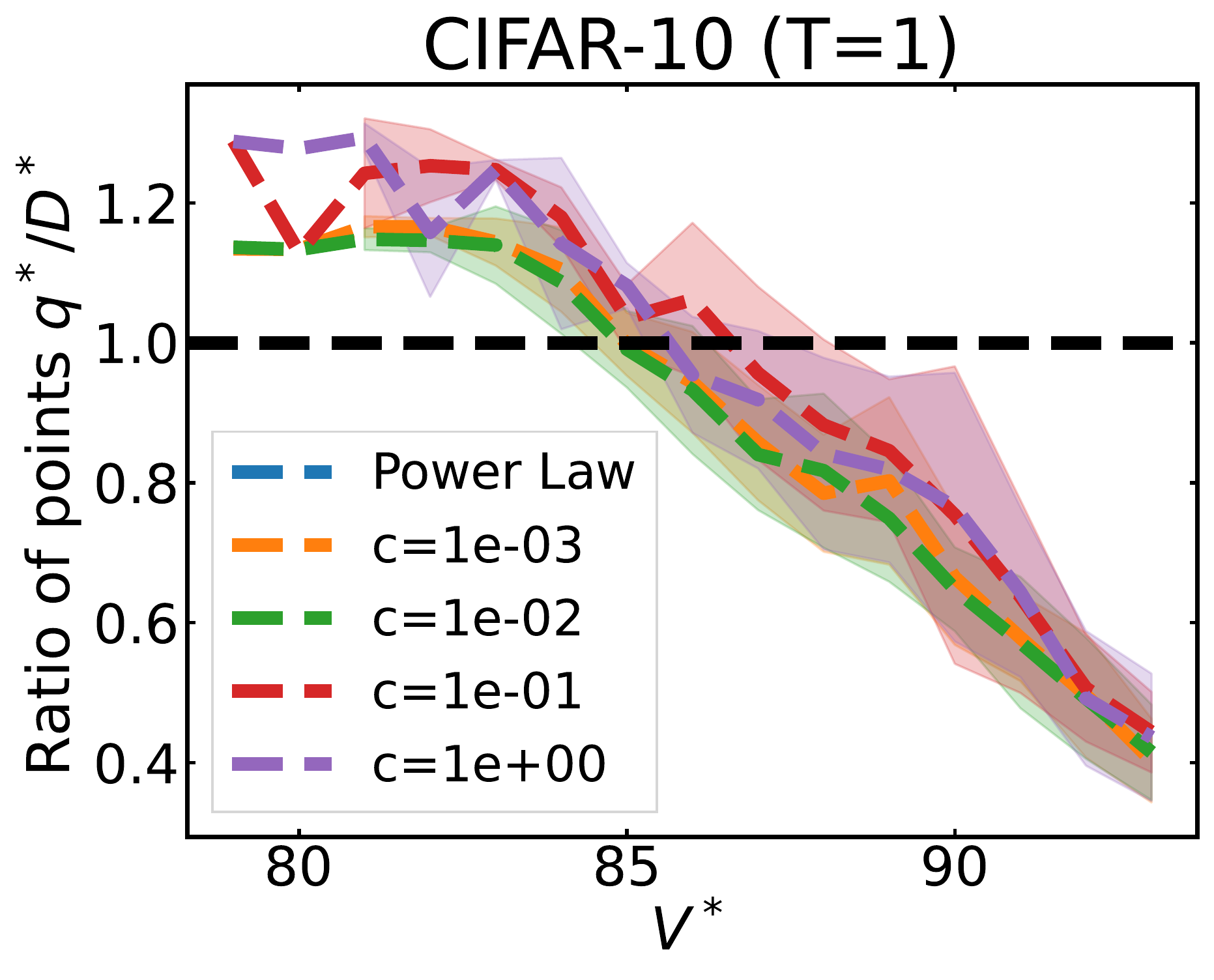}\end{minipage}
\begin{minipage}{0.3\linewidth}\includegraphics[width=1\textwidth]{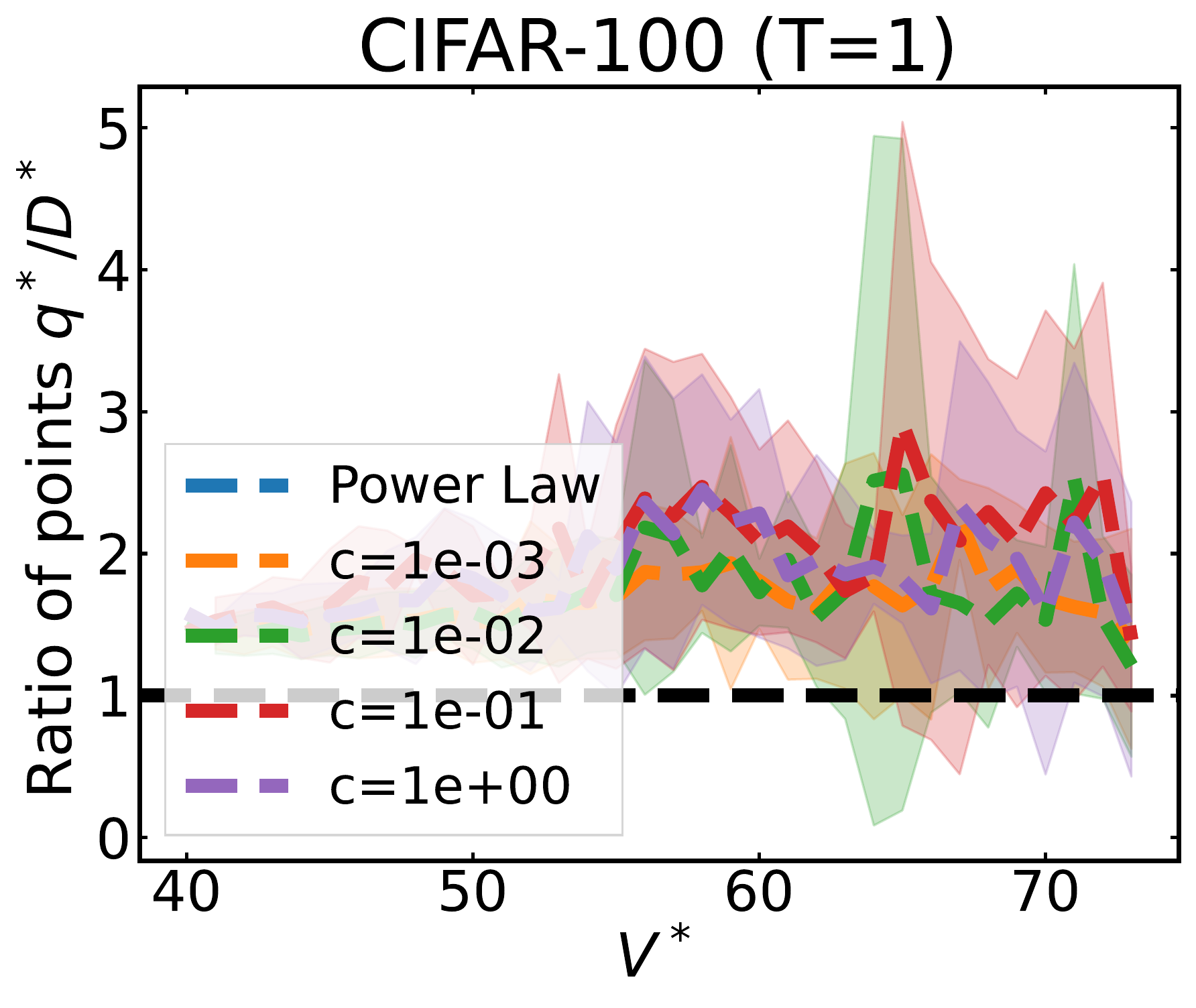}\end{minipage}
\begin{minipage}{0.3\linewidth}\includegraphics[width=1\textwidth]{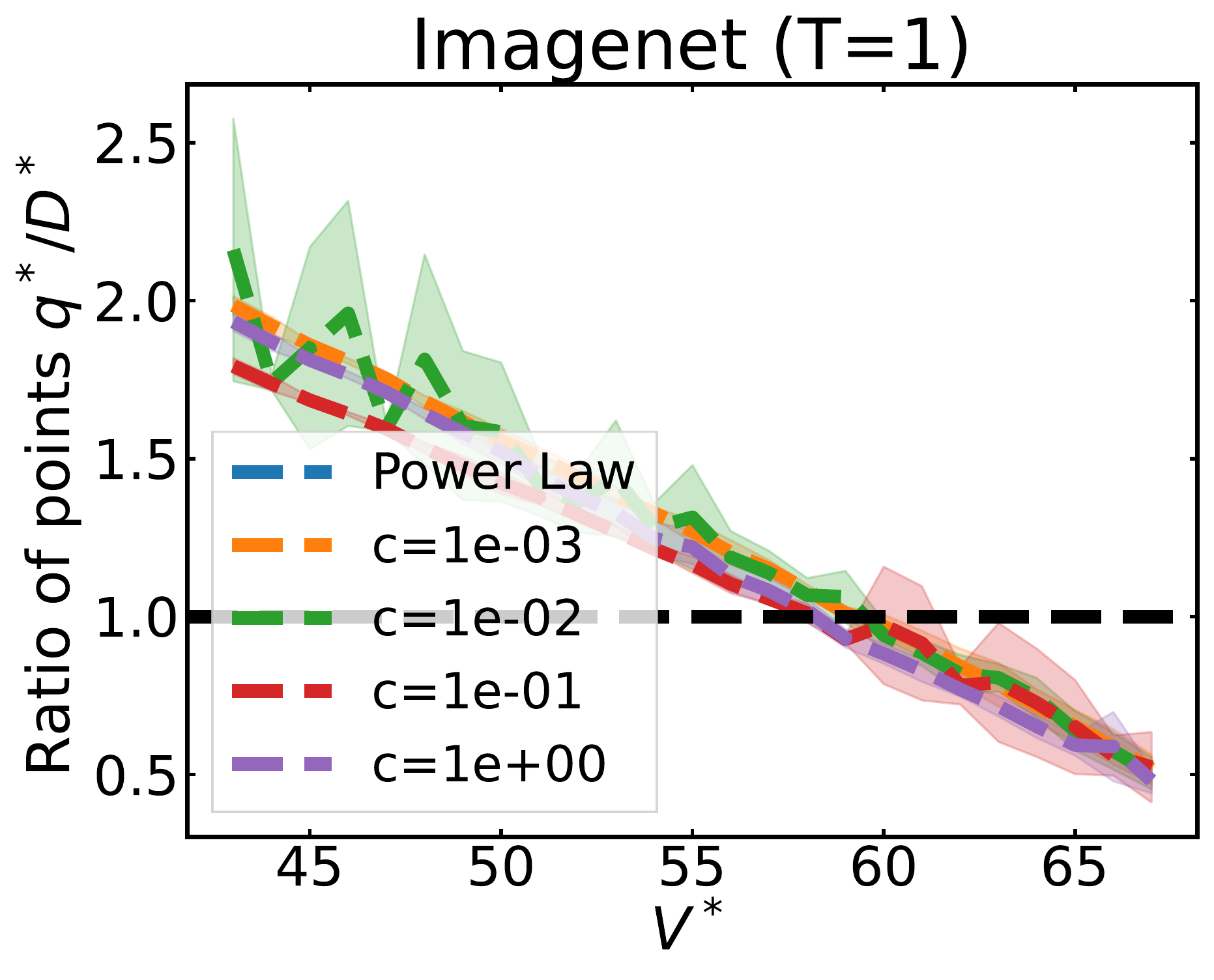}\end{minipage}
\begin{minipage}{0.3\linewidth}\includegraphics[width=1\textwidth]{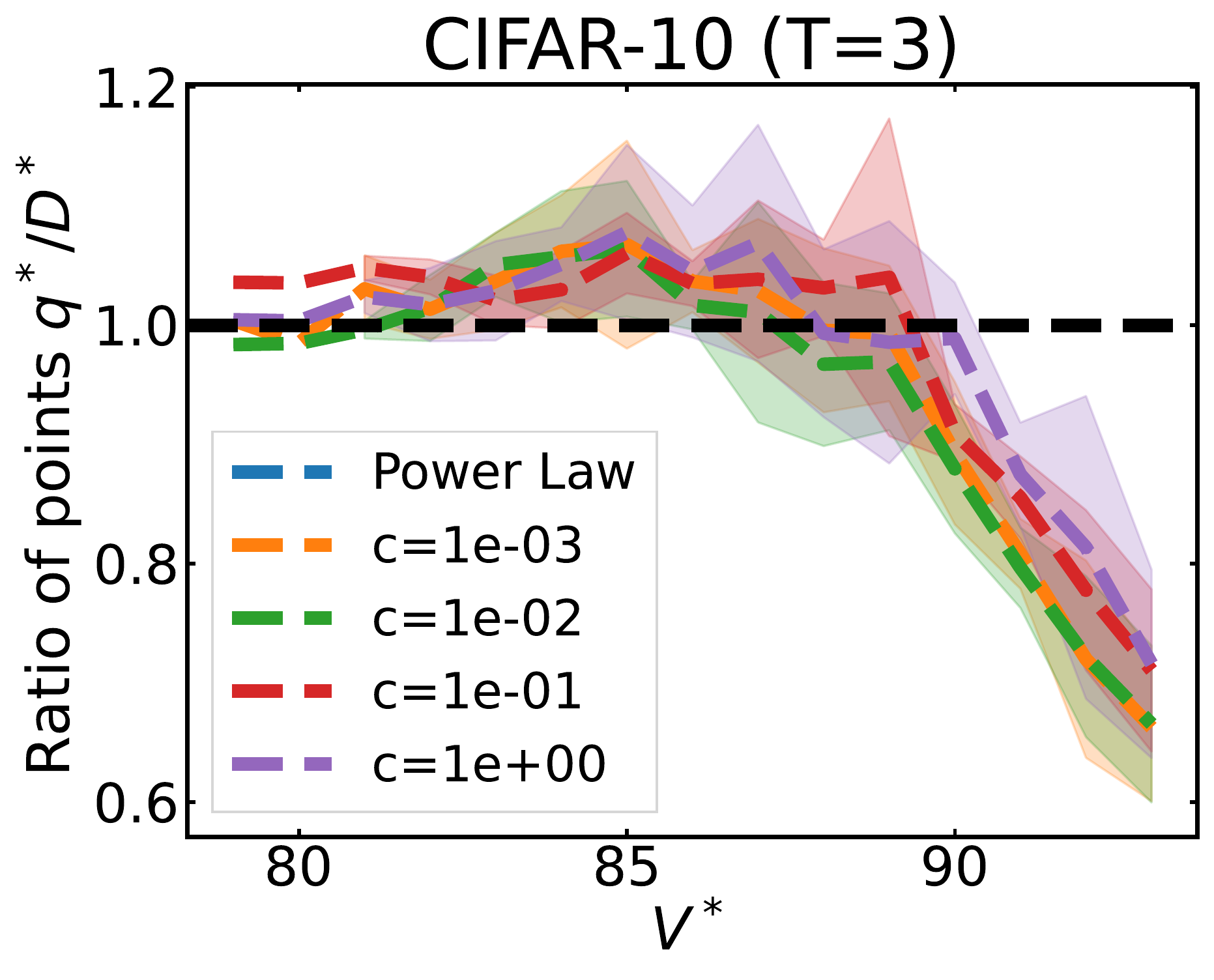}\end{minipage}
\begin{minipage}{0.3\linewidth}\includegraphics[width=1\textwidth]{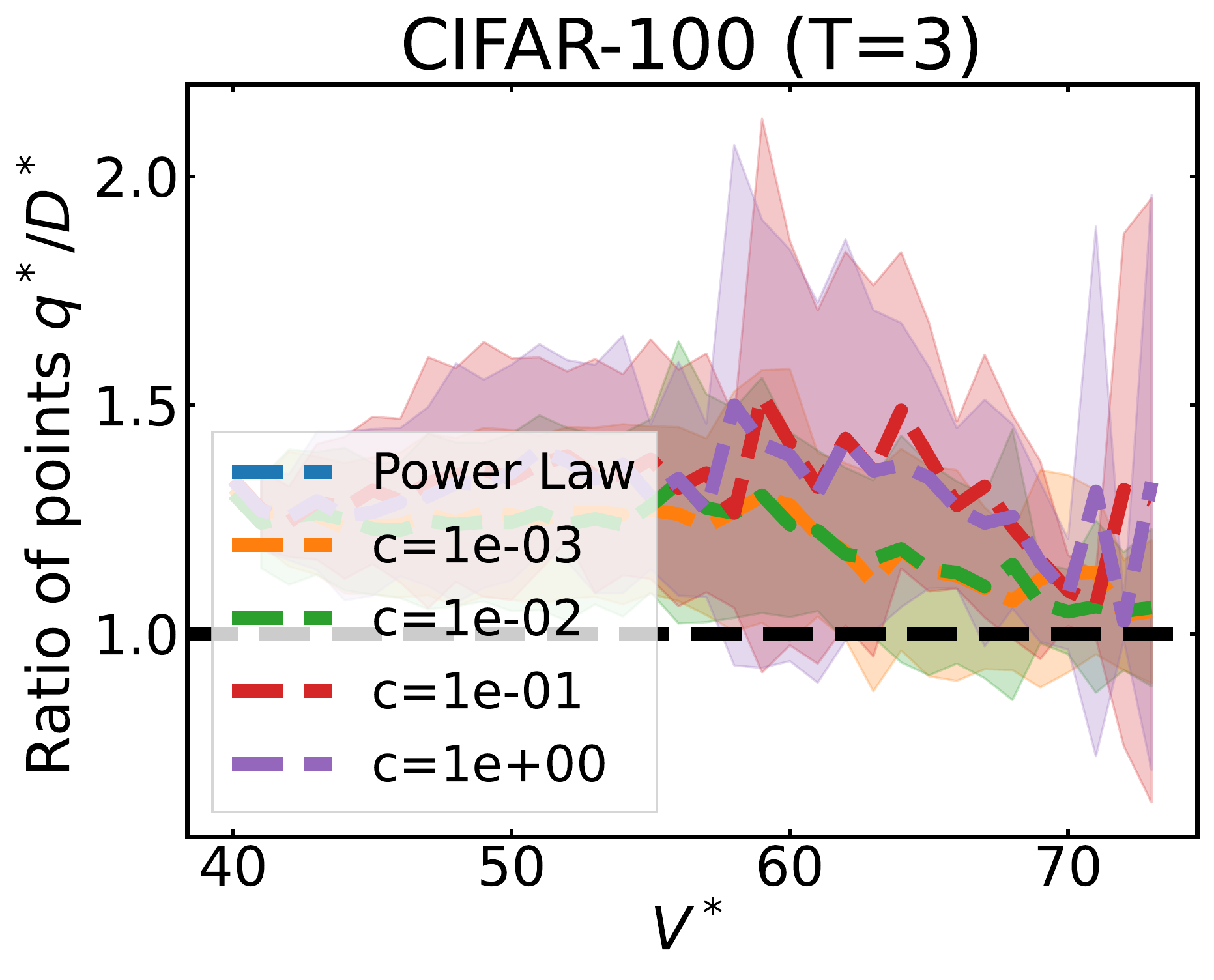}\end{minipage}
\begin{minipage}{0.3\linewidth}\includegraphics[width=1\textwidth]{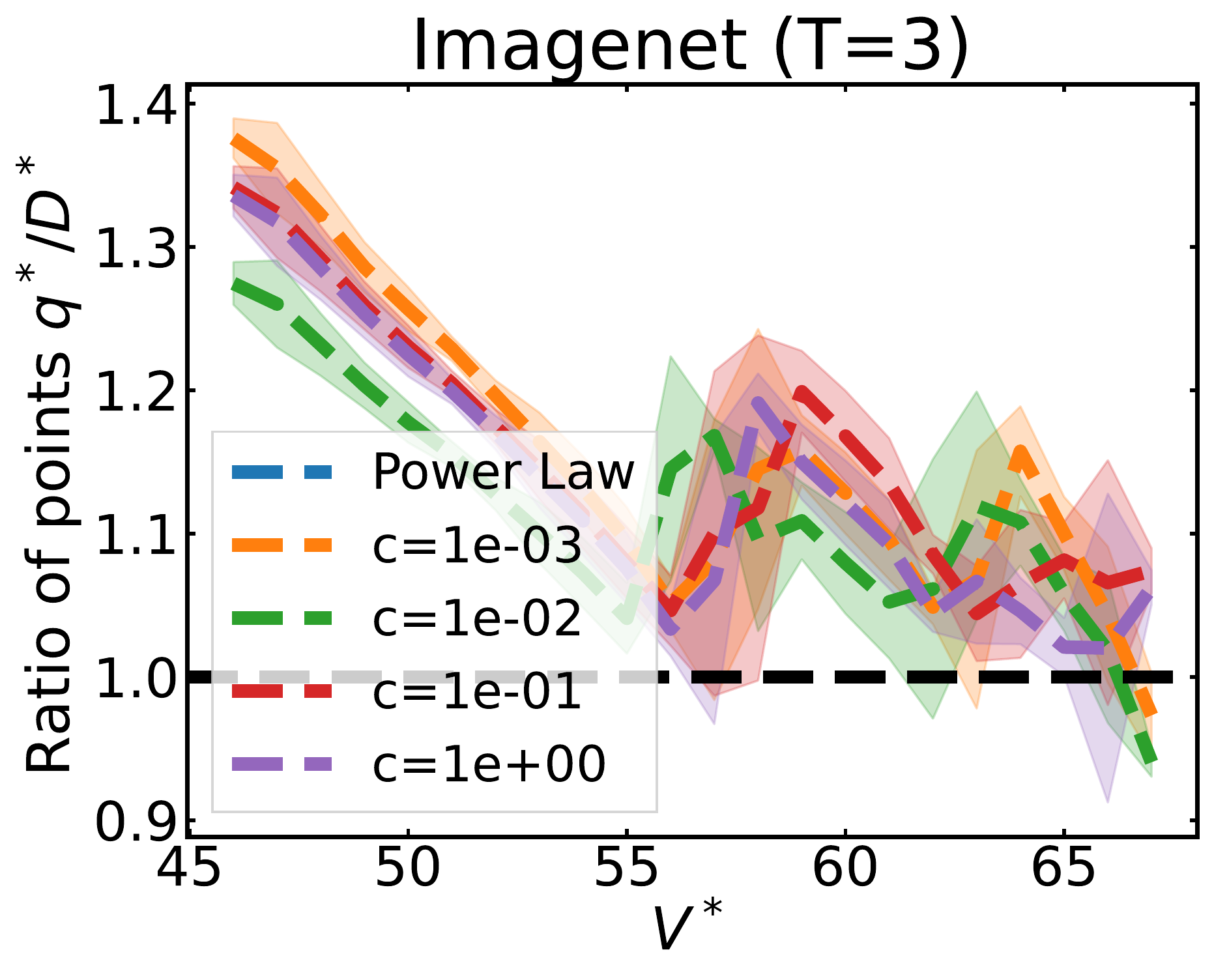}\end{minipage}
\begin{minipage}{0.3\linewidth}\includegraphics[width=1\textwidth]{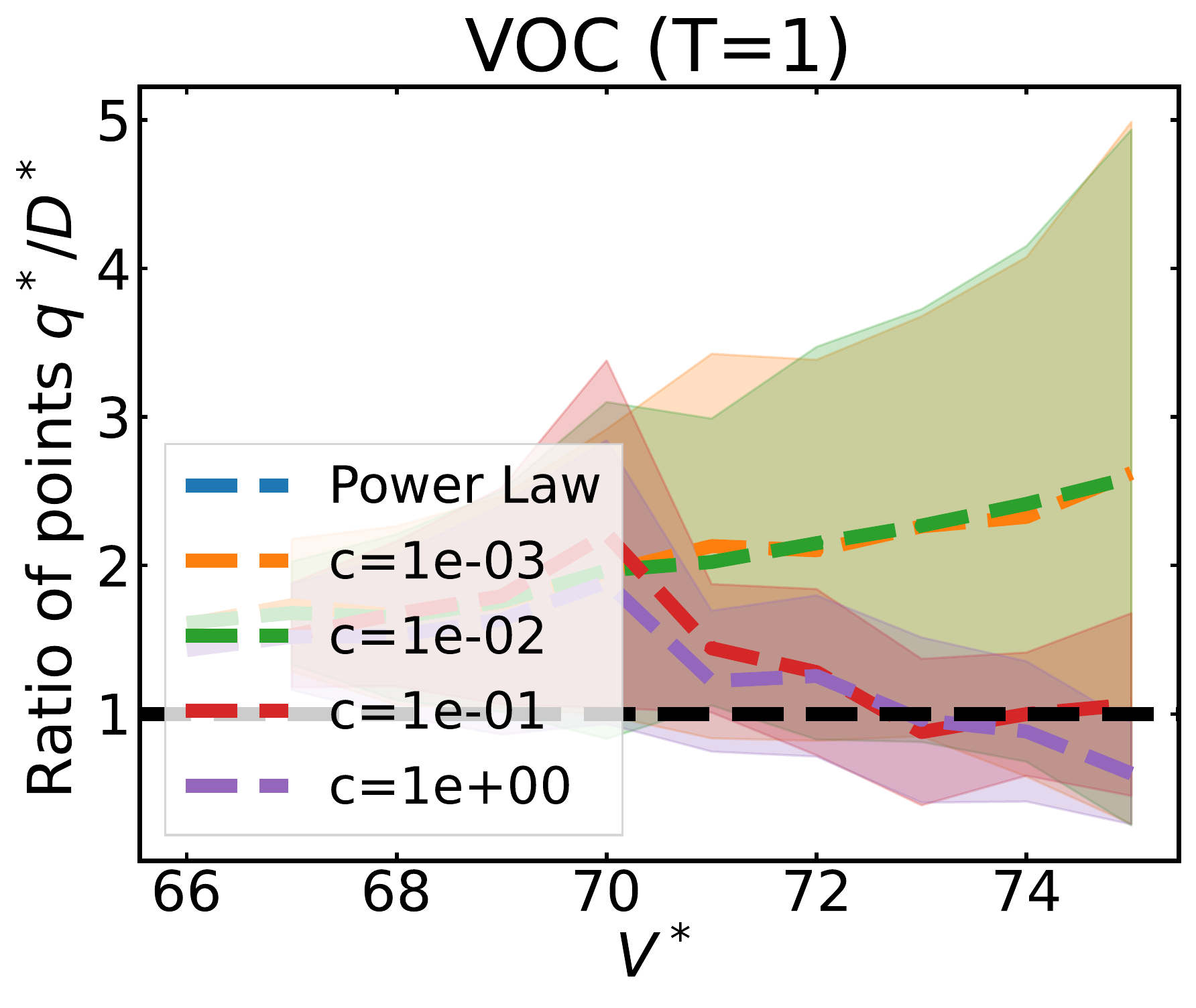}\end{minipage}
\begin{minipage}{0.3\linewidth}\includegraphics[width=1\textwidth]{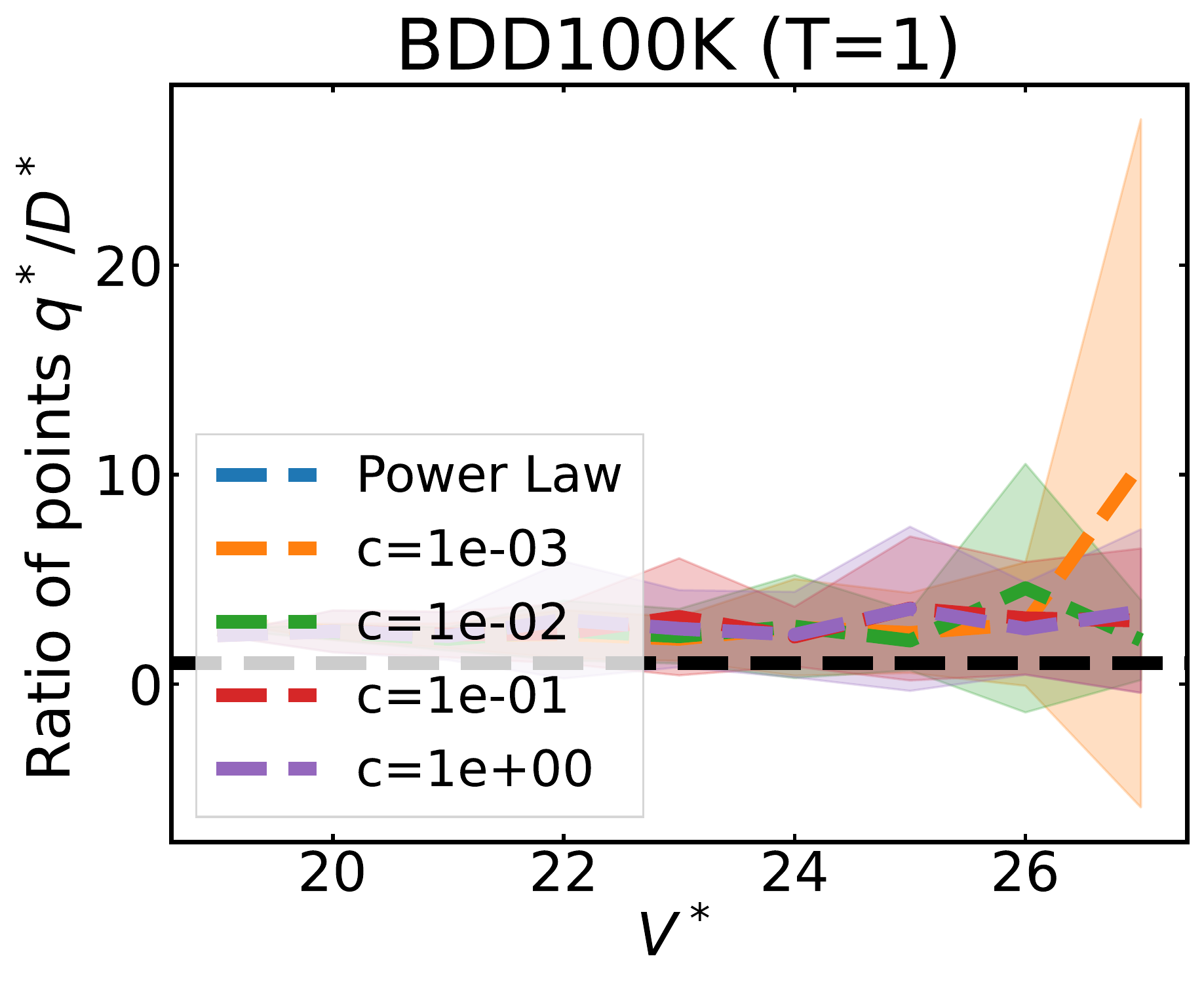}\end{minipage}
\begin{minipage}{0.3\linewidth}\includegraphics[width=1\textwidth]{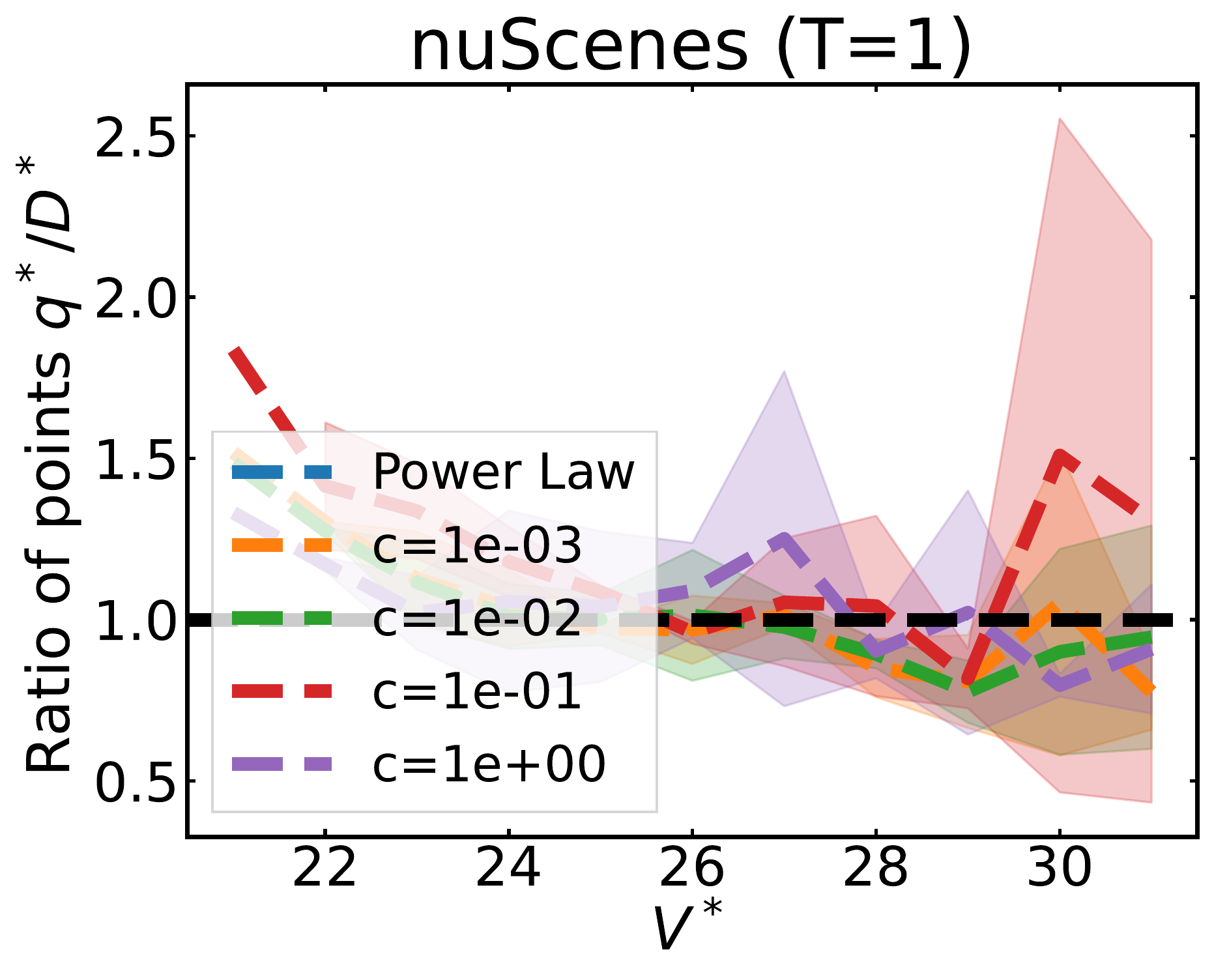}\end{minipage}
\begin{minipage}{0.3\linewidth}\includegraphics[width=1\textwidth]{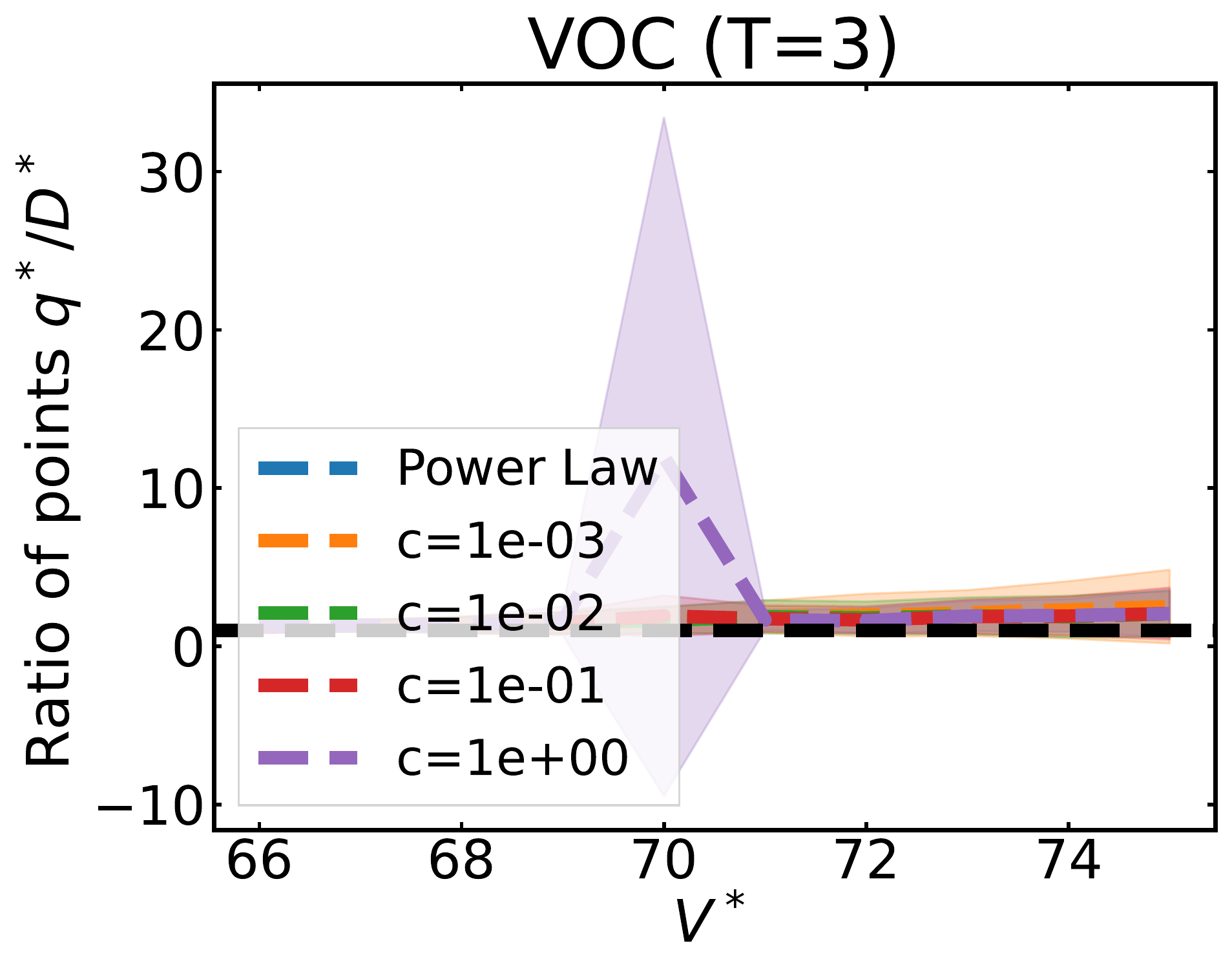}\end{minipage}
\begin{minipage}{0.3\linewidth}\includegraphics[width=1\textwidth]{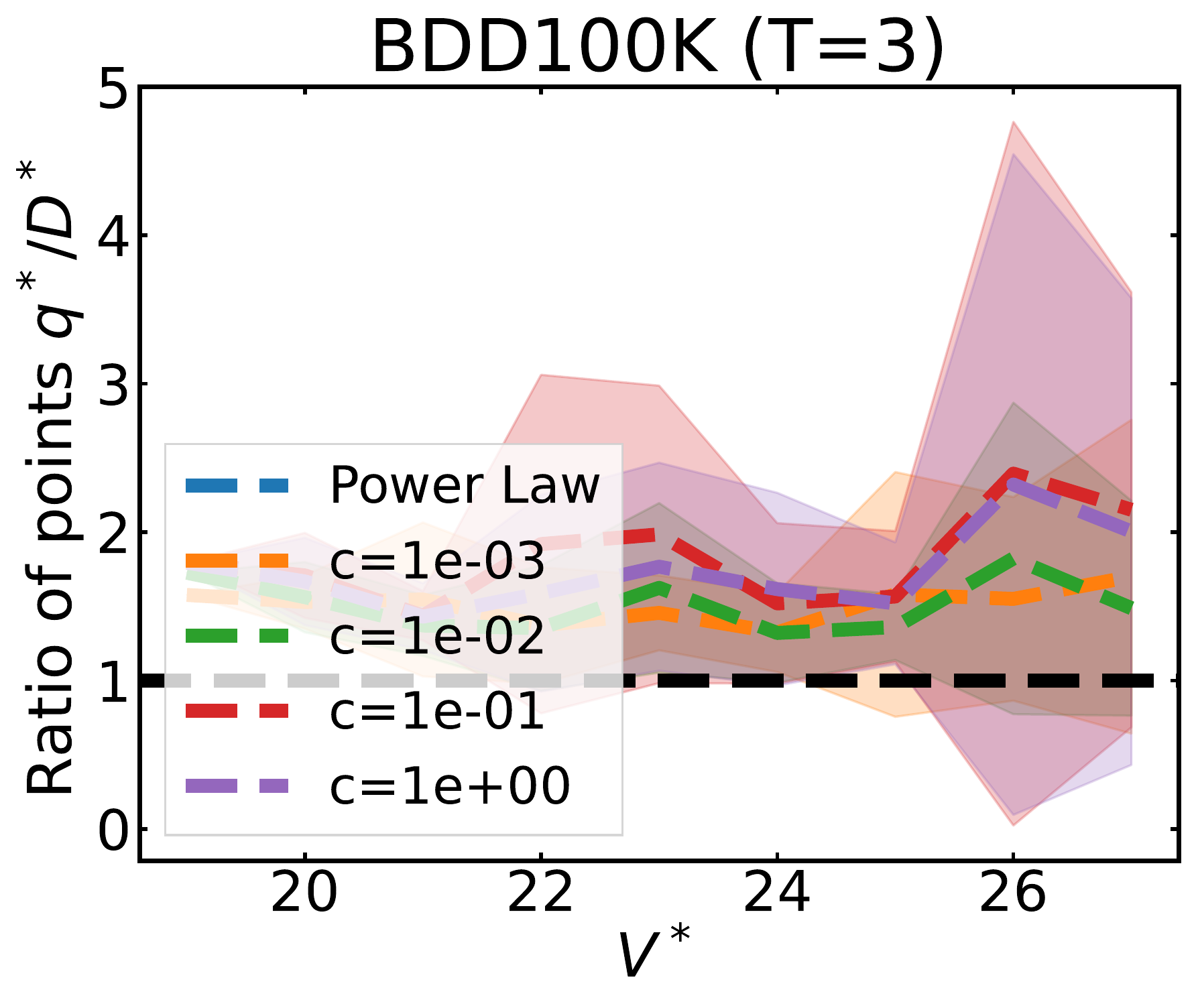}\end{minipage}
\begin{minipage}{0.3\linewidth}\includegraphics[width=1\textwidth]{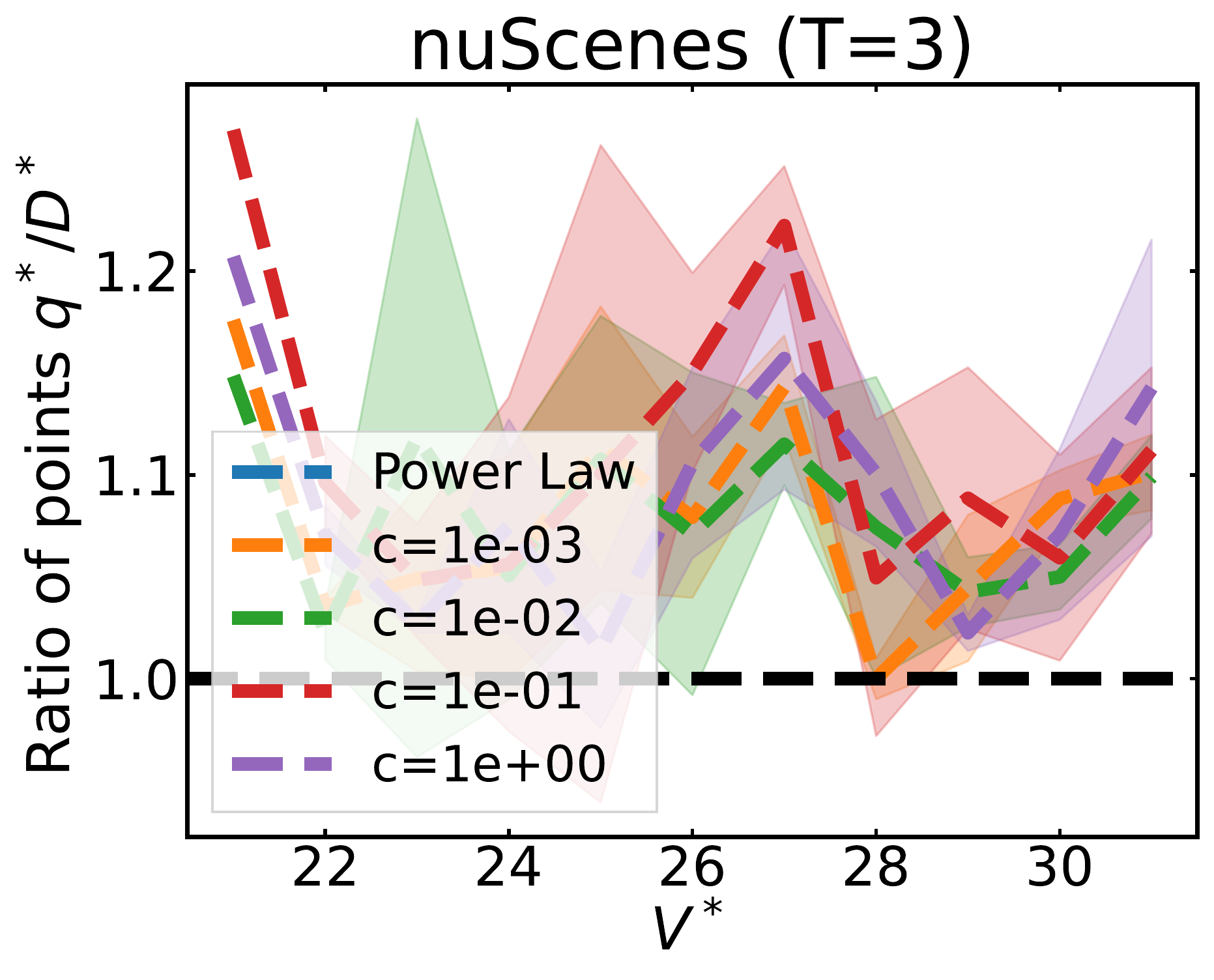}\end{minipage}
\vspace{-3mm}
\caption{\label{fig:app_sweep_cost}
Mean $\pm$ standard deviation of the ratio of data collected $q_T^* / D^*$ for different $V^*$ when we sweep the cost parameter from $0.001$ to $1$ and fix $P=10^7$. 
We show $T=1, 3$ and refer to the main paper for $T=5$.
The dashed black line corresponds to collecting exactly the minimum data requirement.
}
\end{center}
\vspace{-4mm}
\end{figure*}

\begin{figure*}[!t]
\begin{center}
\begin{minipage}{0.3\linewidth}\includegraphics[width=1\textwidth]{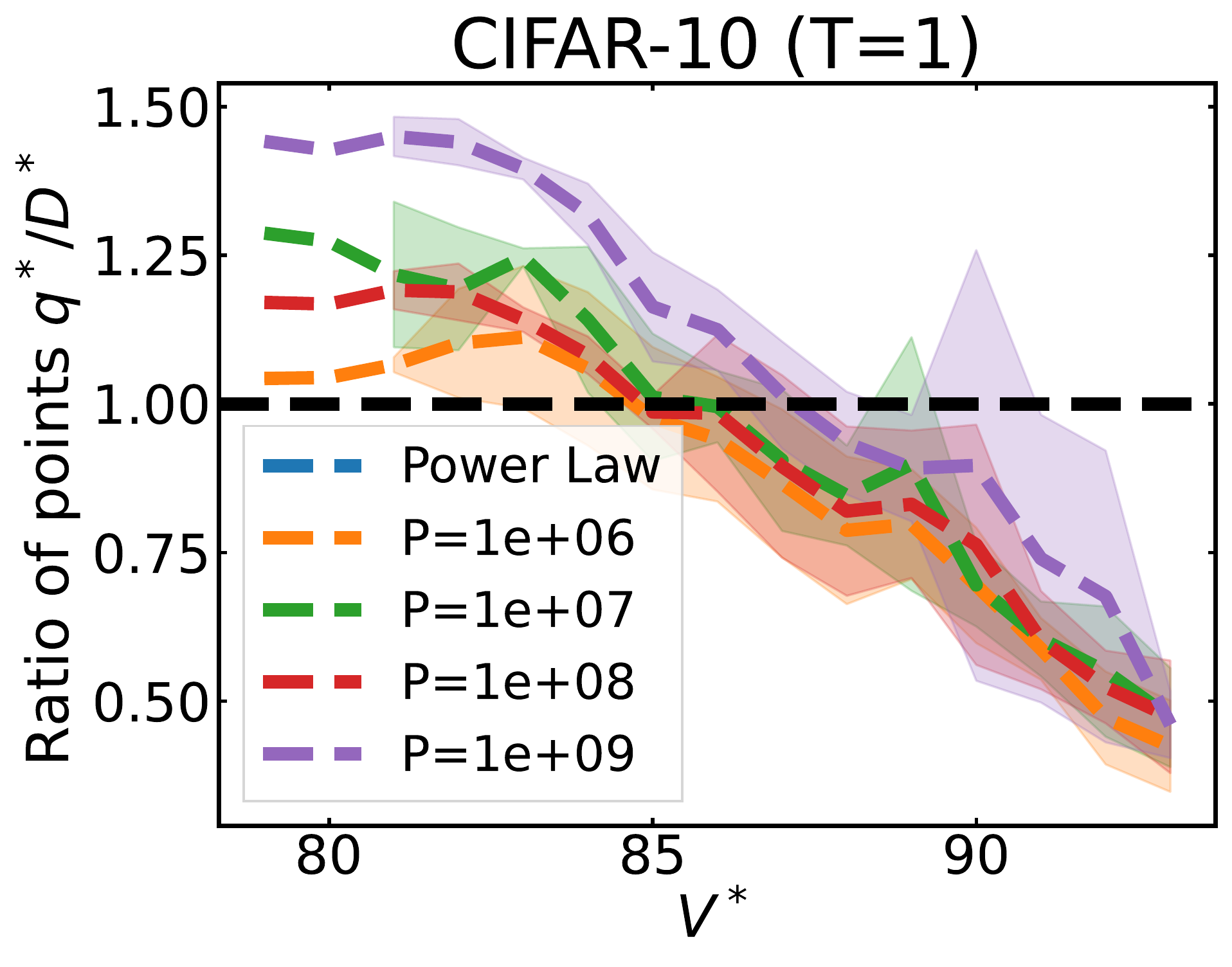}\end{minipage}
\begin{minipage}{0.3\linewidth}\includegraphics[width=1\textwidth]{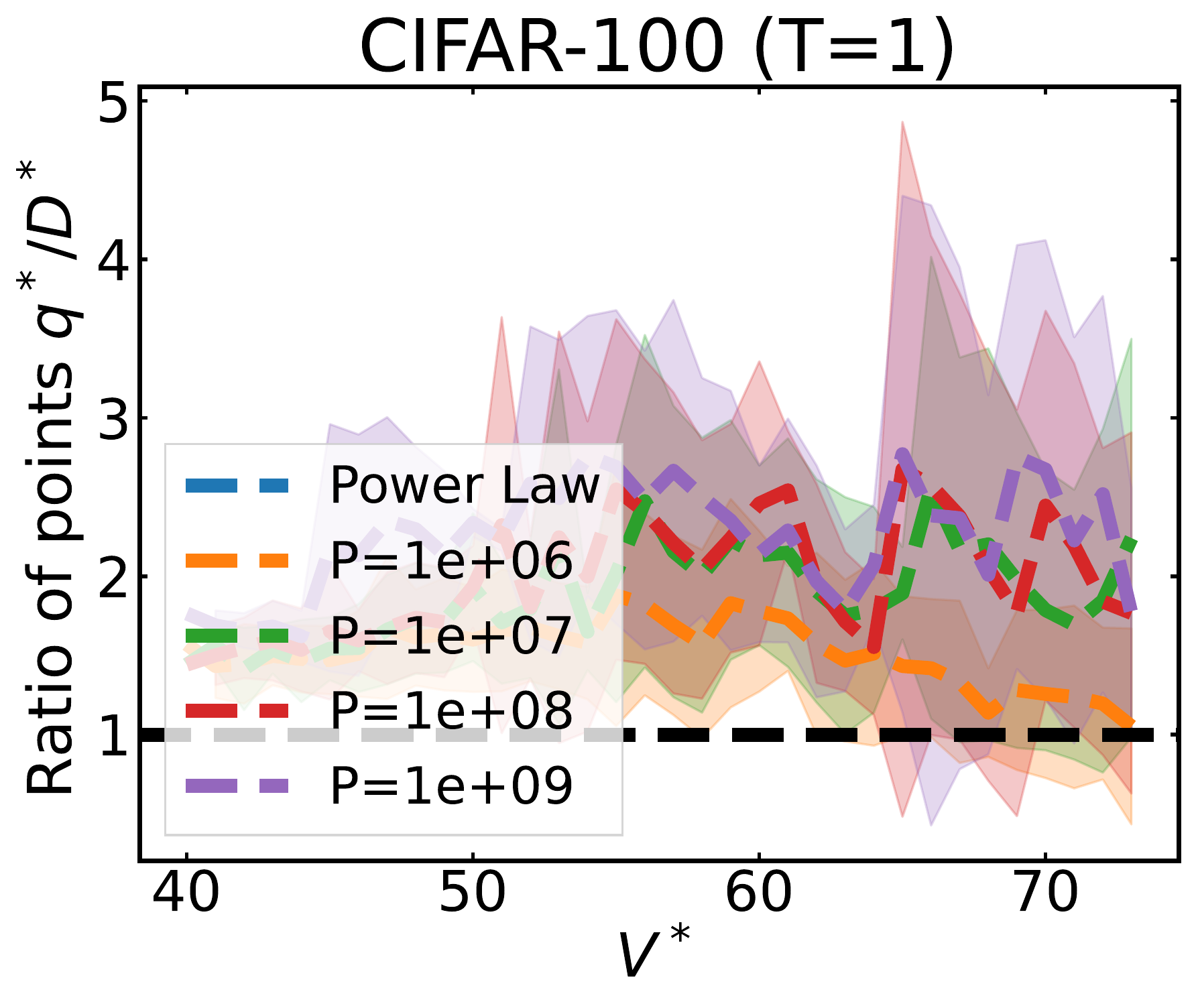}\end{minipage}
\begin{minipage}{0.3\linewidth}\includegraphics[width=1\textwidth]{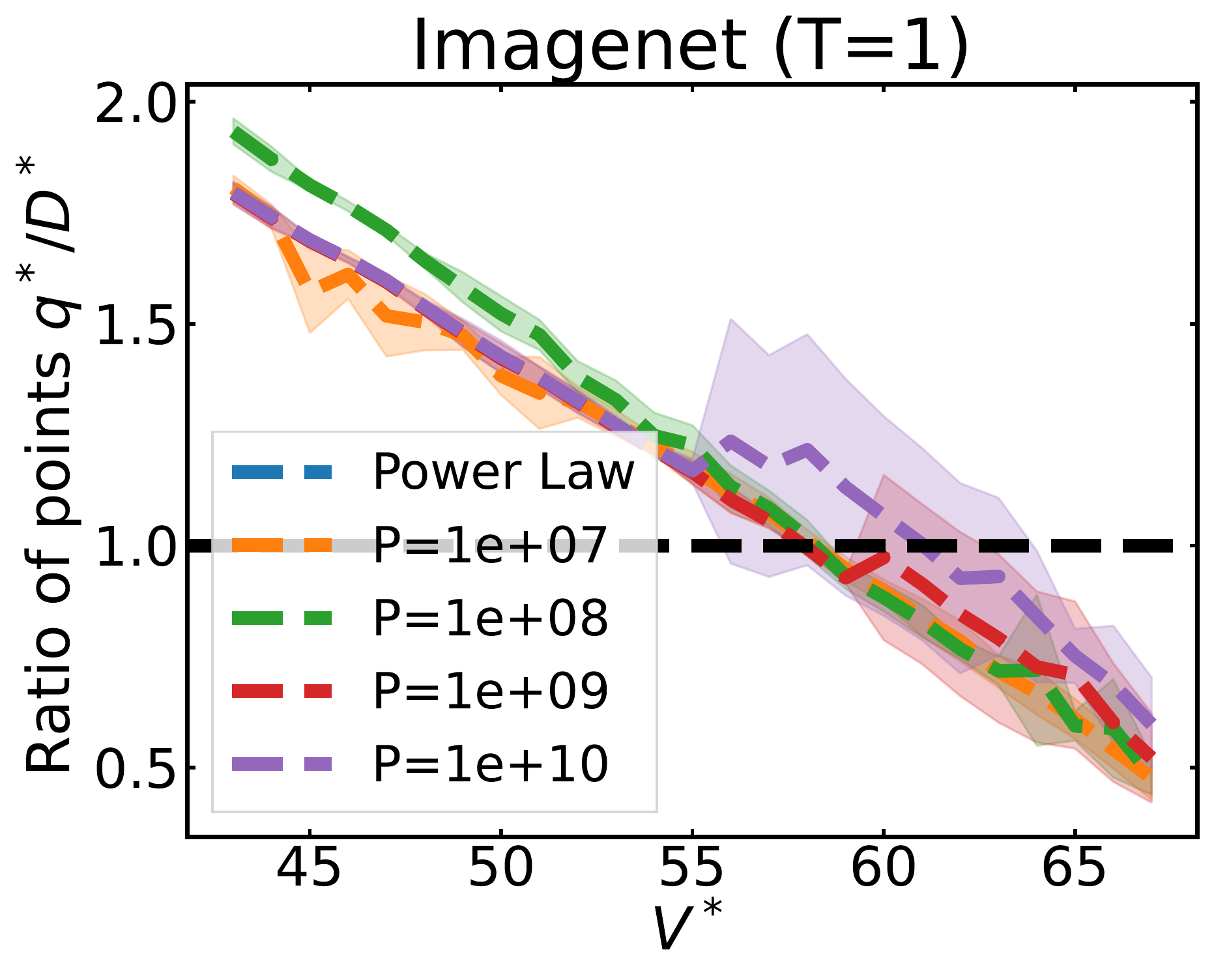}\end{minipage}
\begin{minipage}{0.3\linewidth}\includegraphics[width=1\textwidth]{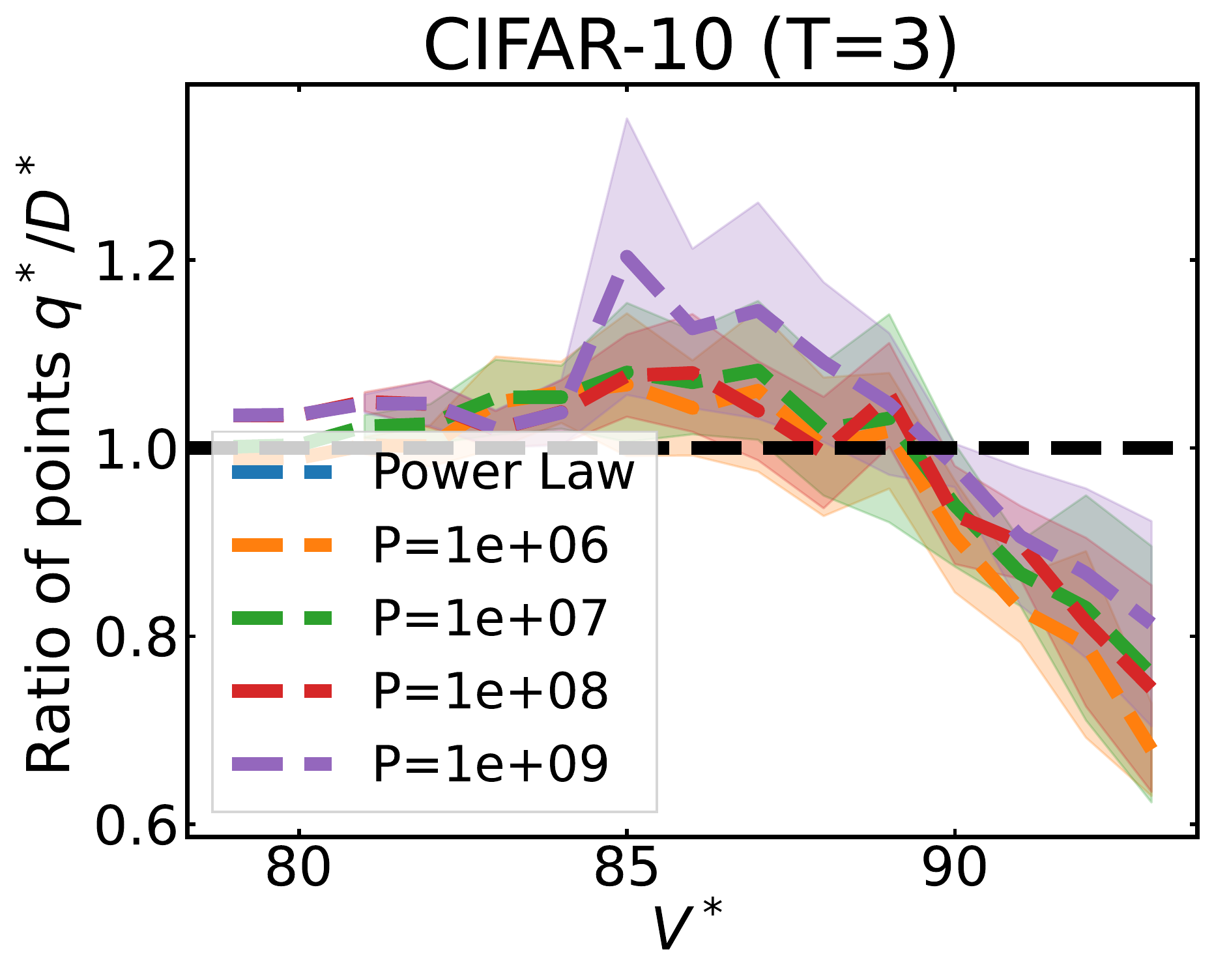}\end{minipage}
\begin{minipage}{0.3\linewidth}\includegraphics[width=1\textwidth]{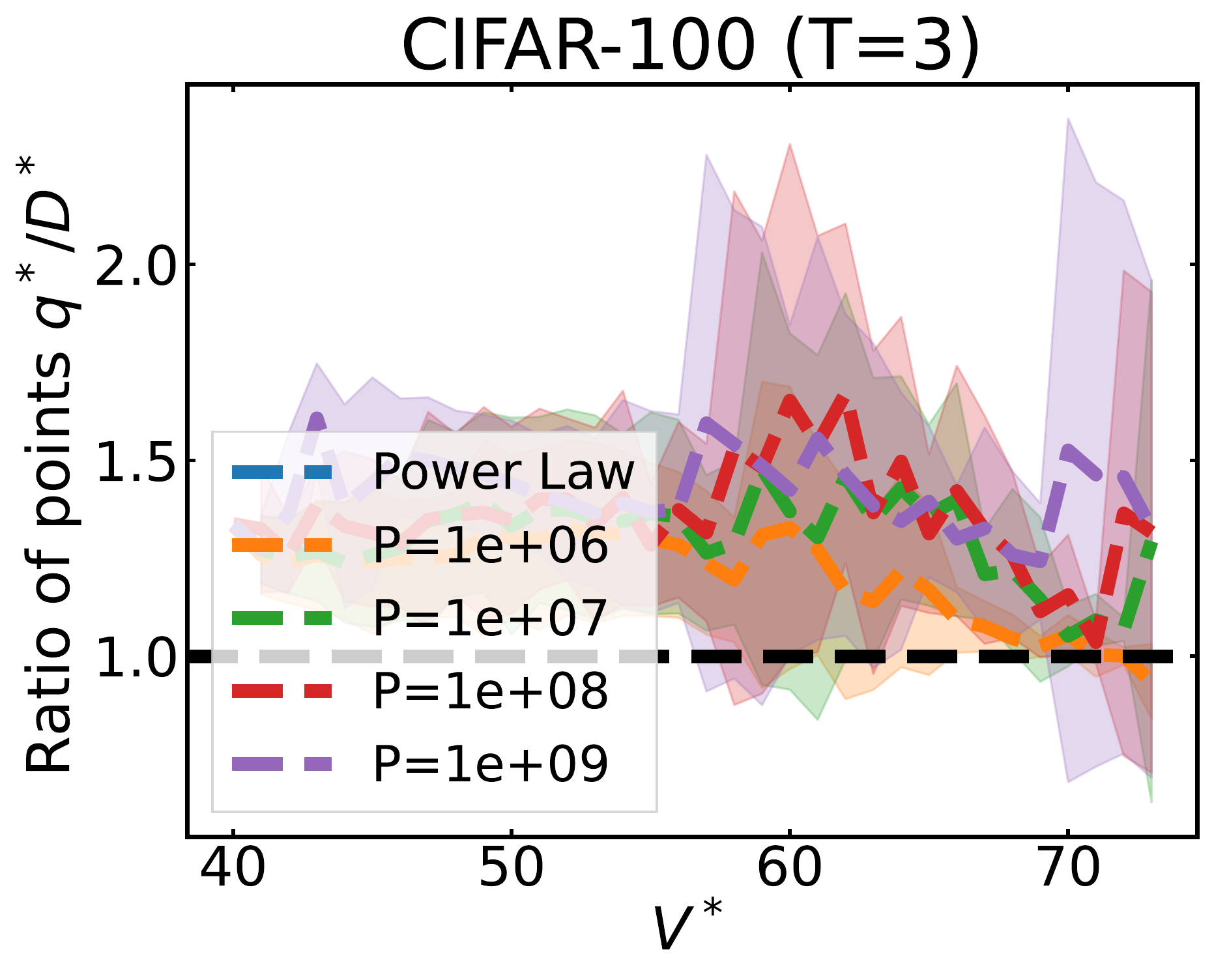}\end{minipage}
\begin{minipage}{0.3\linewidth}\includegraphics[width=1\textwidth]{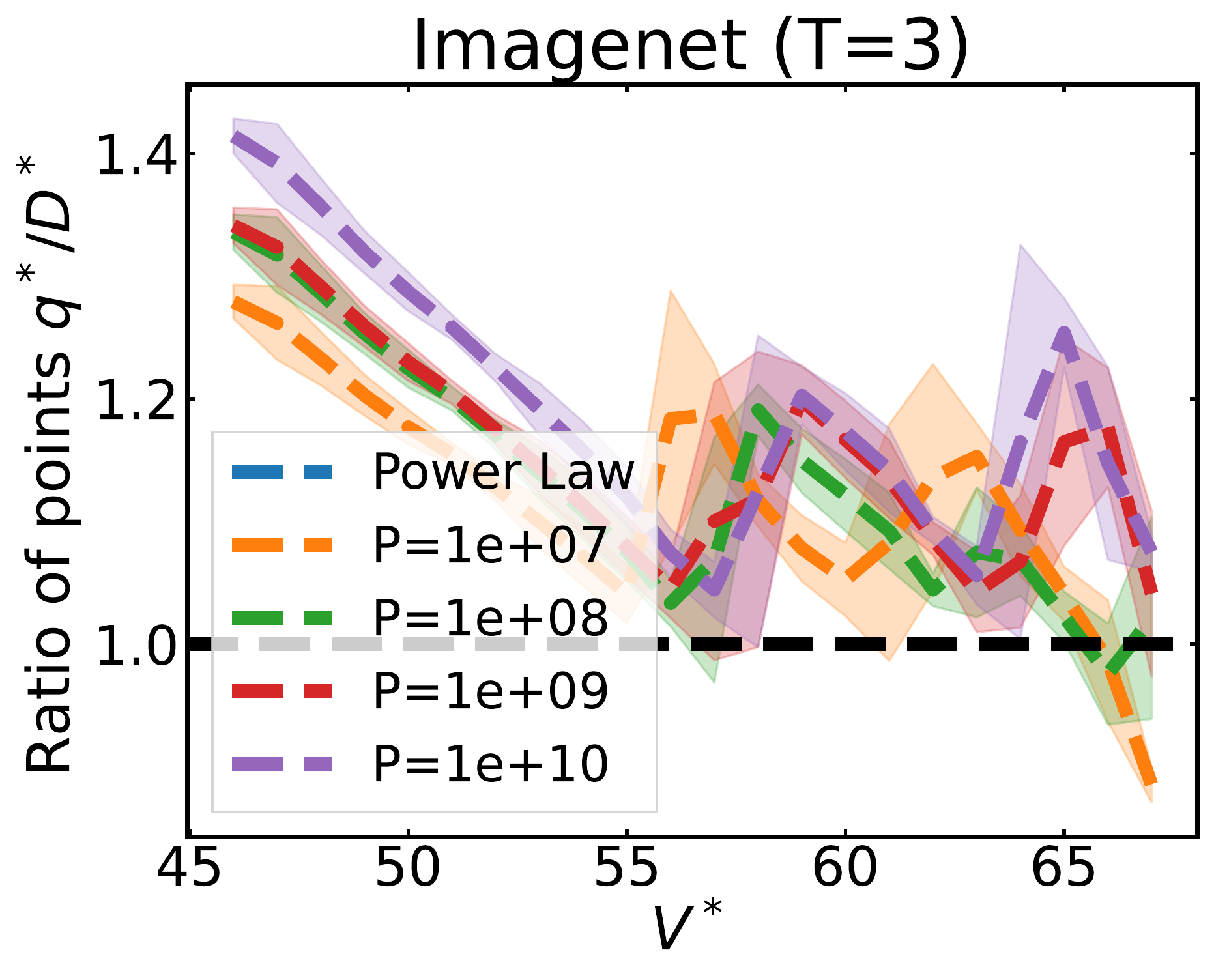}\end{minipage}
\begin{minipage}{0.3\linewidth}\includegraphics[width=1\textwidth]{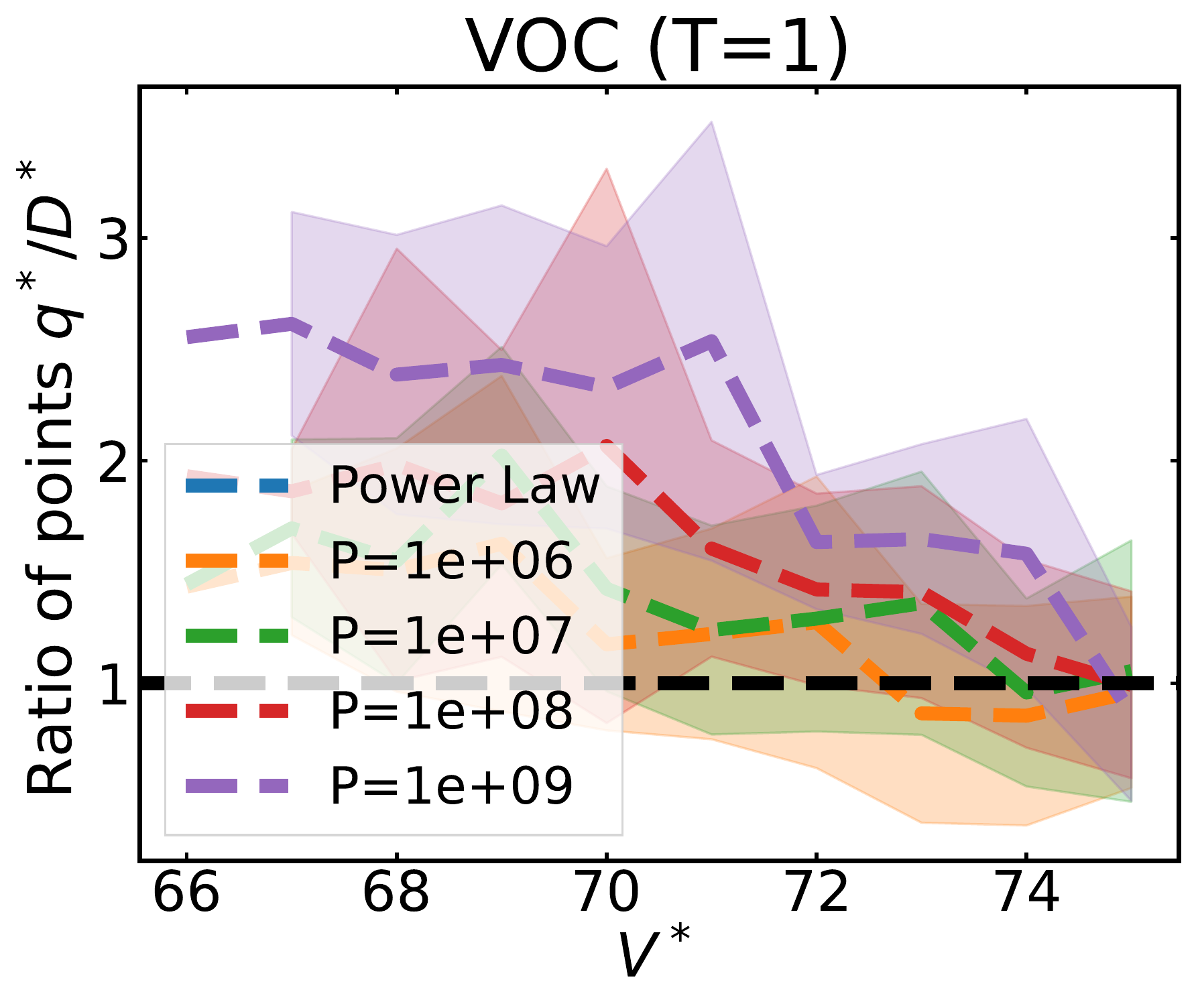}\end{minipage}
\begin{minipage}{0.3\linewidth}\includegraphics[width=1\textwidth]{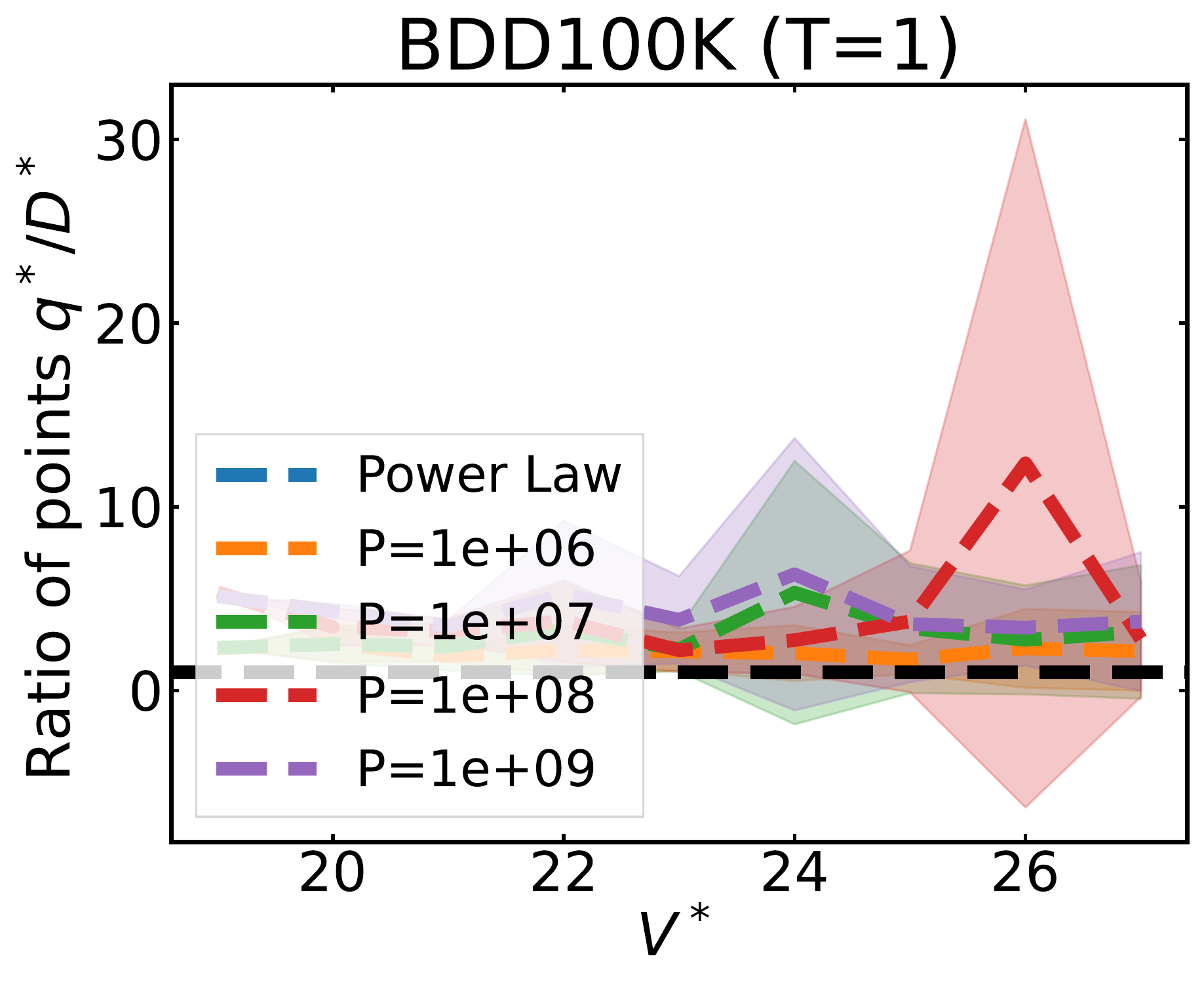}\end{minipage}
\begin{minipage}{0.3\linewidth}\includegraphics[width=1\textwidth]{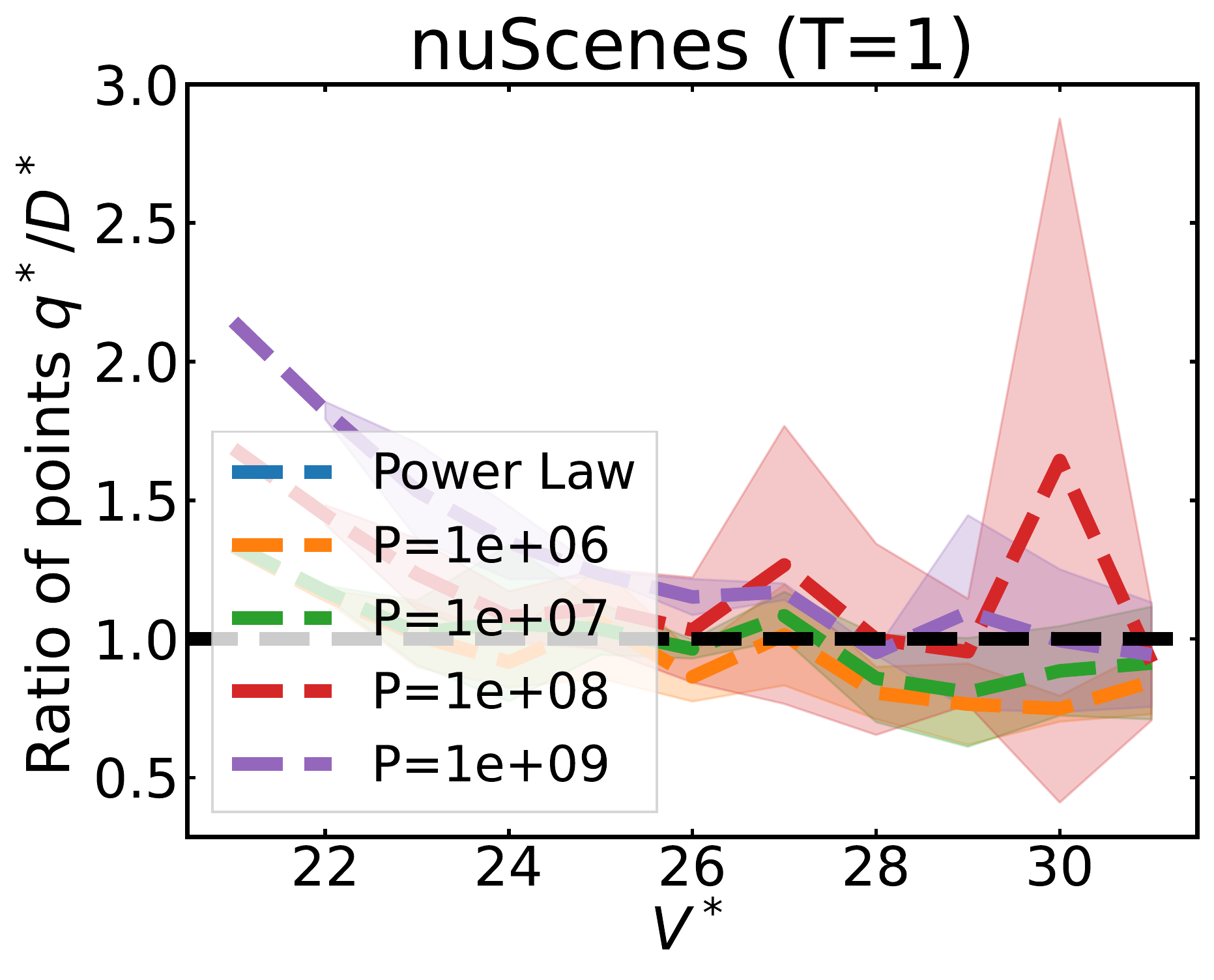}\end{minipage}
\begin{minipage}{0.3\linewidth}\includegraphics[width=1\textwidth]{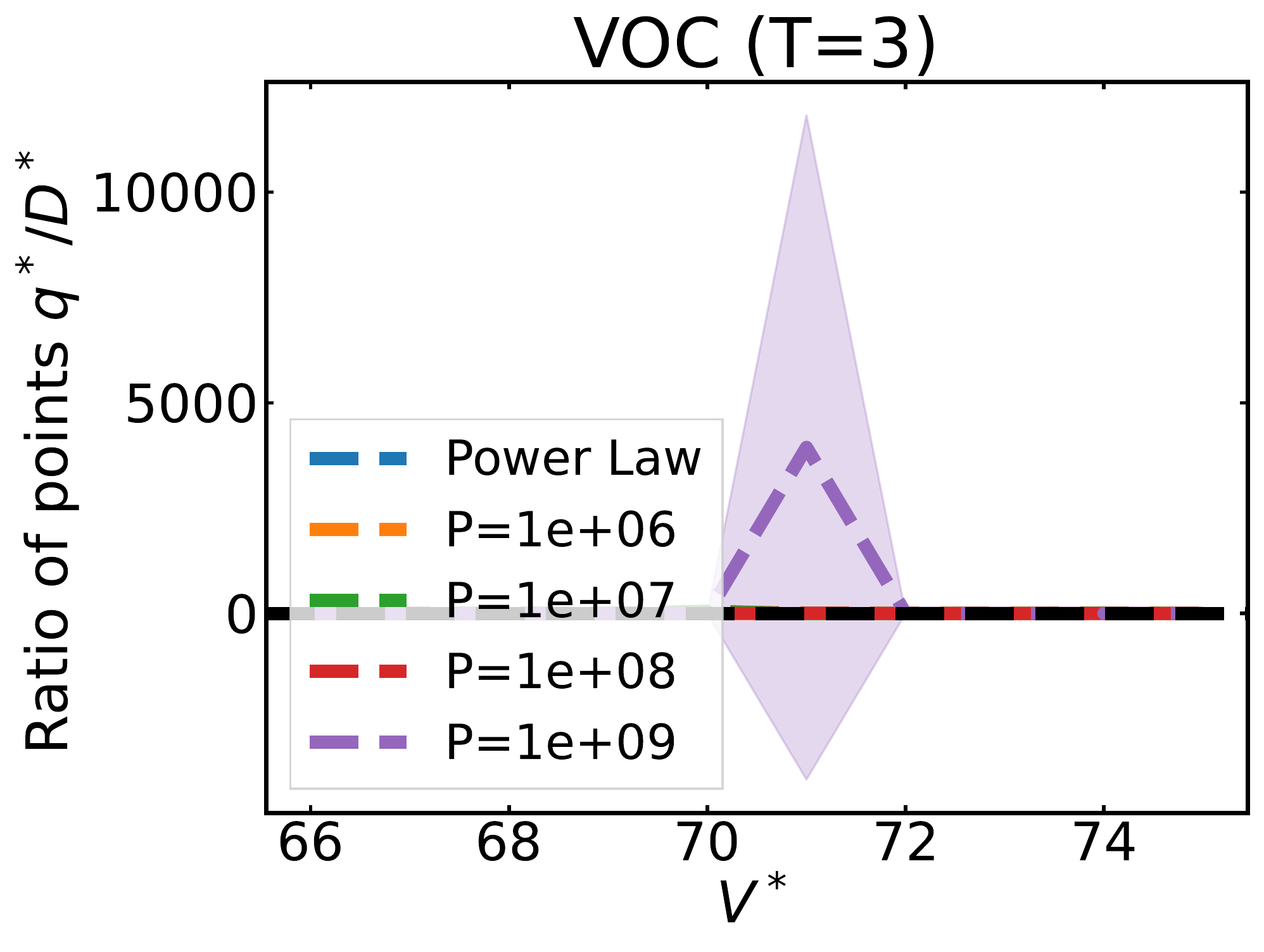}\end{minipage}
\begin{minipage}{0.3\linewidth}\includegraphics[width=1\textwidth]{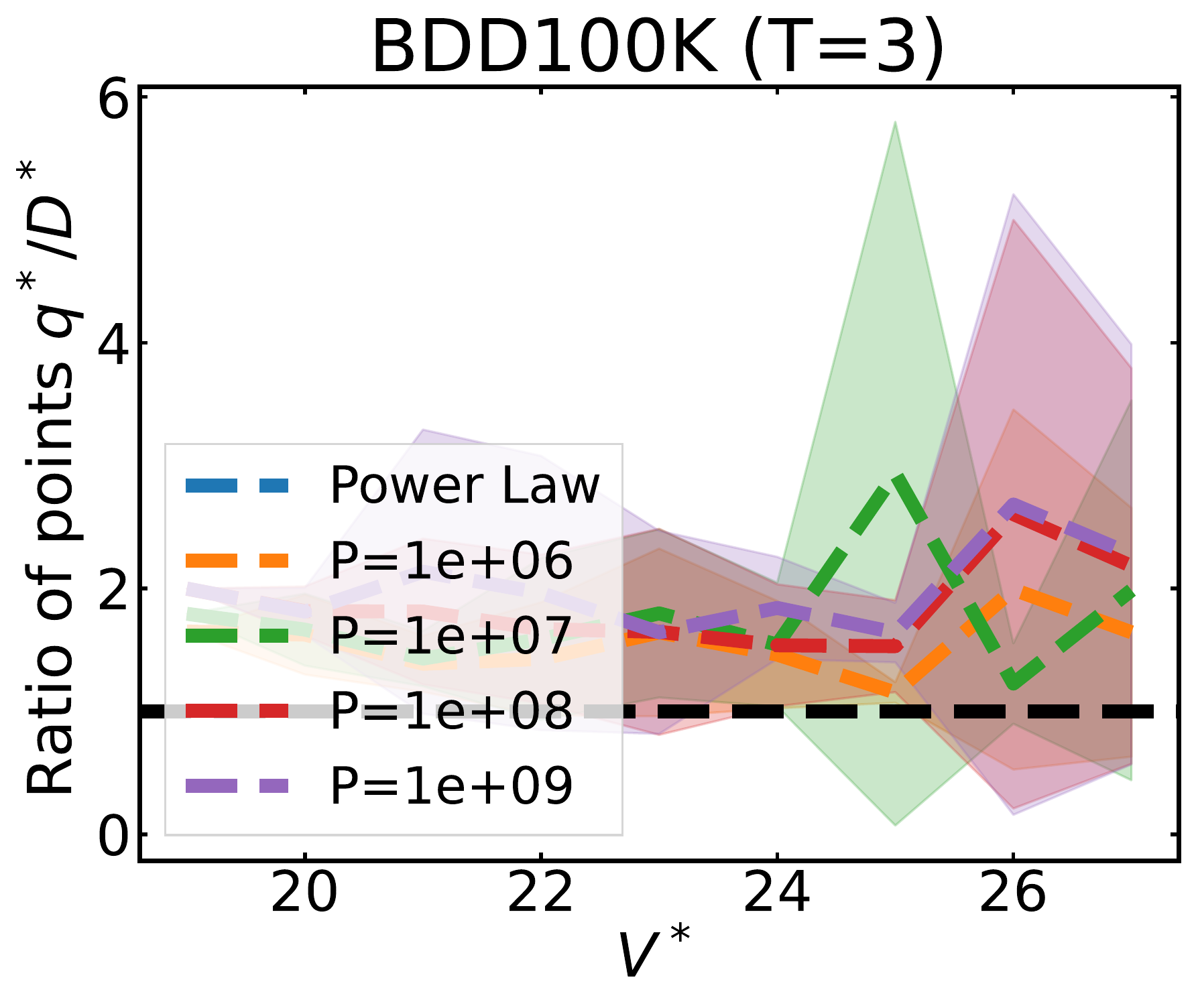}\end{minipage}
\begin{minipage}{0.3\linewidth}\includegraphics[width=1\textwidth]{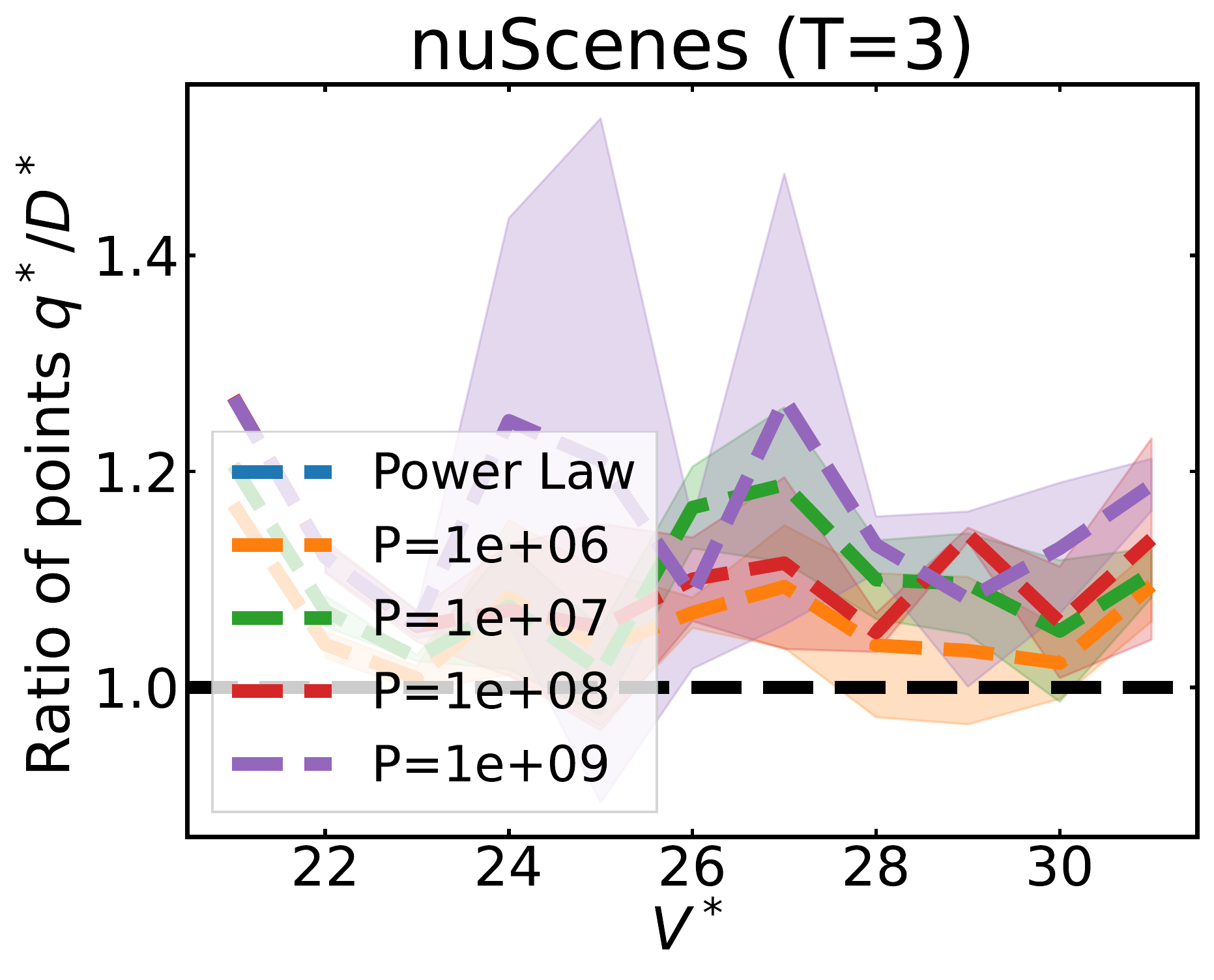}\end{minipage}
\vspace{-3mm}
\caption{\label{fig:app_sweep_penalty}
Mean $\pm$ std of the ratio of data collected $q_T^* / D^*$ for different $V^*$ when we sweep the penalty parameter from $10^6$ to $10^9$ and fix $c=1$. 
We show $T=1, 3$ and refer to the main paper for $T=5$.
The dashed black line corresponds to collecting exactly the minimum data requirement.
}
\end{center}
\vspace{-4mm}
\end{figure*}

Figure~\ref{fig:app_sweep_cost} expands the cost parameter sweep from Figure~\ref{fig:sweep_cost_and_penalty} (Top row) to the settings of $T=1, 3$. For nearly all settings, LOC remains stable to variations in the cost parameter. 
Nonetheless, careful parameter selection becomes important as $T$ decreases.
This is due to the fact that for low costs, the total amount of data collected increases as $T$ decreases (e.g., $c = 0.001$ for BDD100K).
Furthermore, Figure~\ref{fig:app_sweep_penalty} expands the penalty parameter sweep from Figure~\ref{fig:sweep_cost_and_penalty} (Bottom row). 
Here, we observe similar properties to the cost parameter sweep.

Although LOC is relatively stable on all other data sets, our results demonstrate some extreme results for VOC, potentially due to noise in the simulation. For example in Figure~\ref{fig:app_sweep_penalty}, setting $P=10^9$, $V^*=71$, and $T=3$ led to collecting $10,000$ times the minimum data requirement. 
Such a situation is unrealistic in a production-level implementation, since in a real implementation, we could impose further constraints onto problem~\eqref{eq:expected_risk_prob}, such as upper bounds on the total amount of data permissible.

\begin{figure*}[!t]
\begin{center}
\begin{minipage}{0.24\linewidth}\includegraphics[width=1\textwidth]{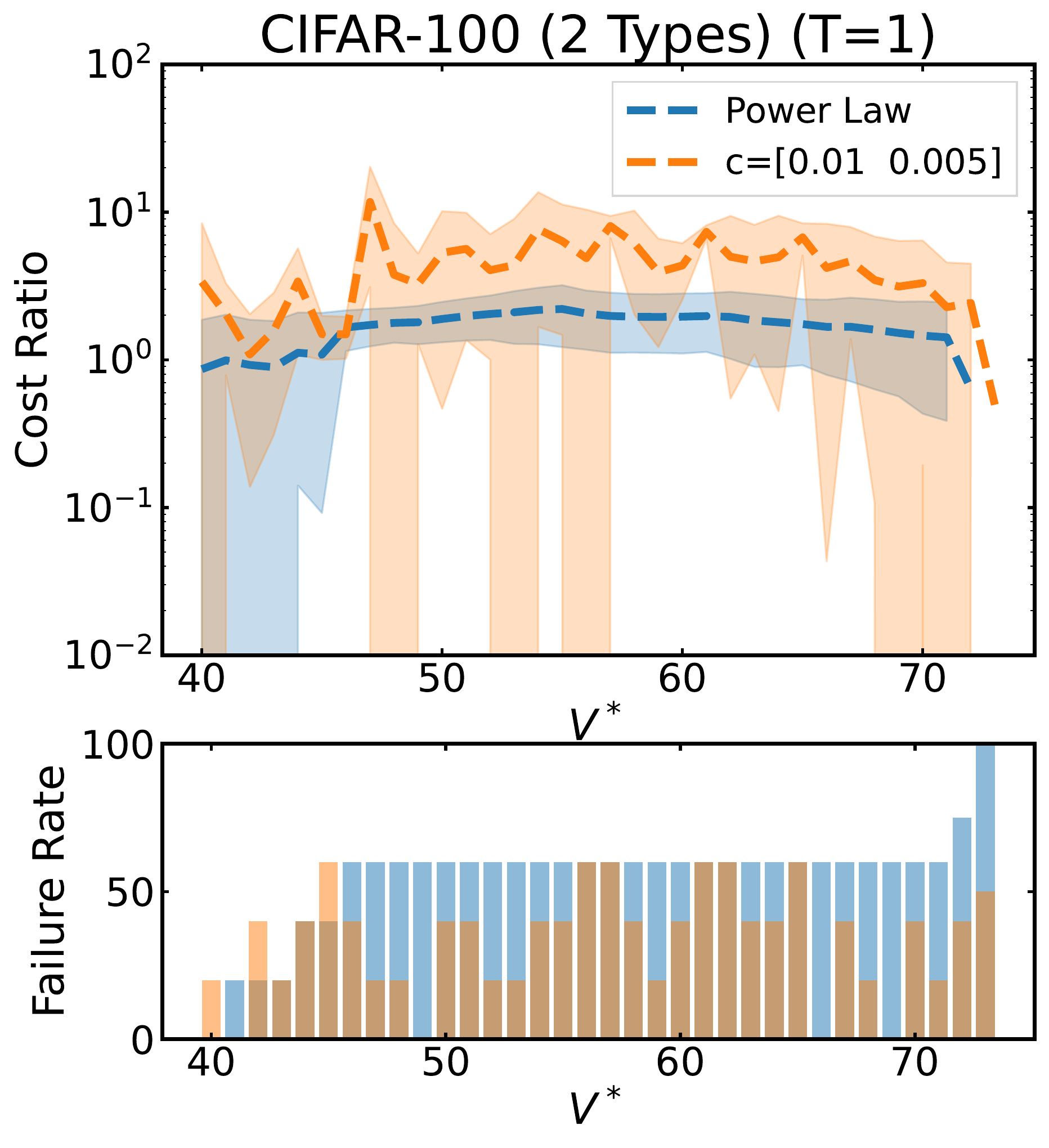} \end{minipage}
\begin{minipage}{0.24\linewidth}\includegraphics[width=1\textwidth]{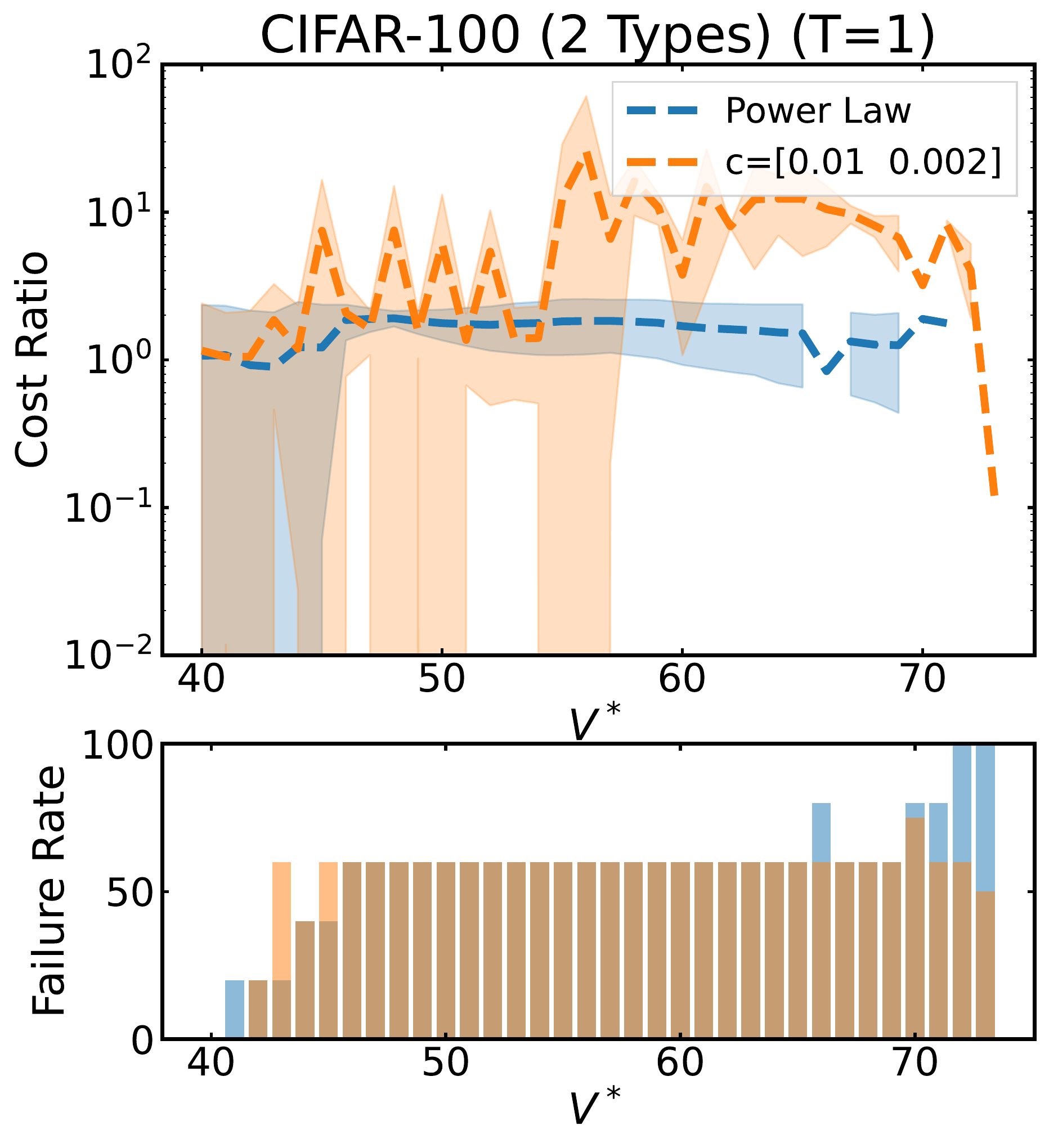} \end{minipage}
\begin{minipage}{0.24\linewidth}\includegraphics[width=1\textwidth]{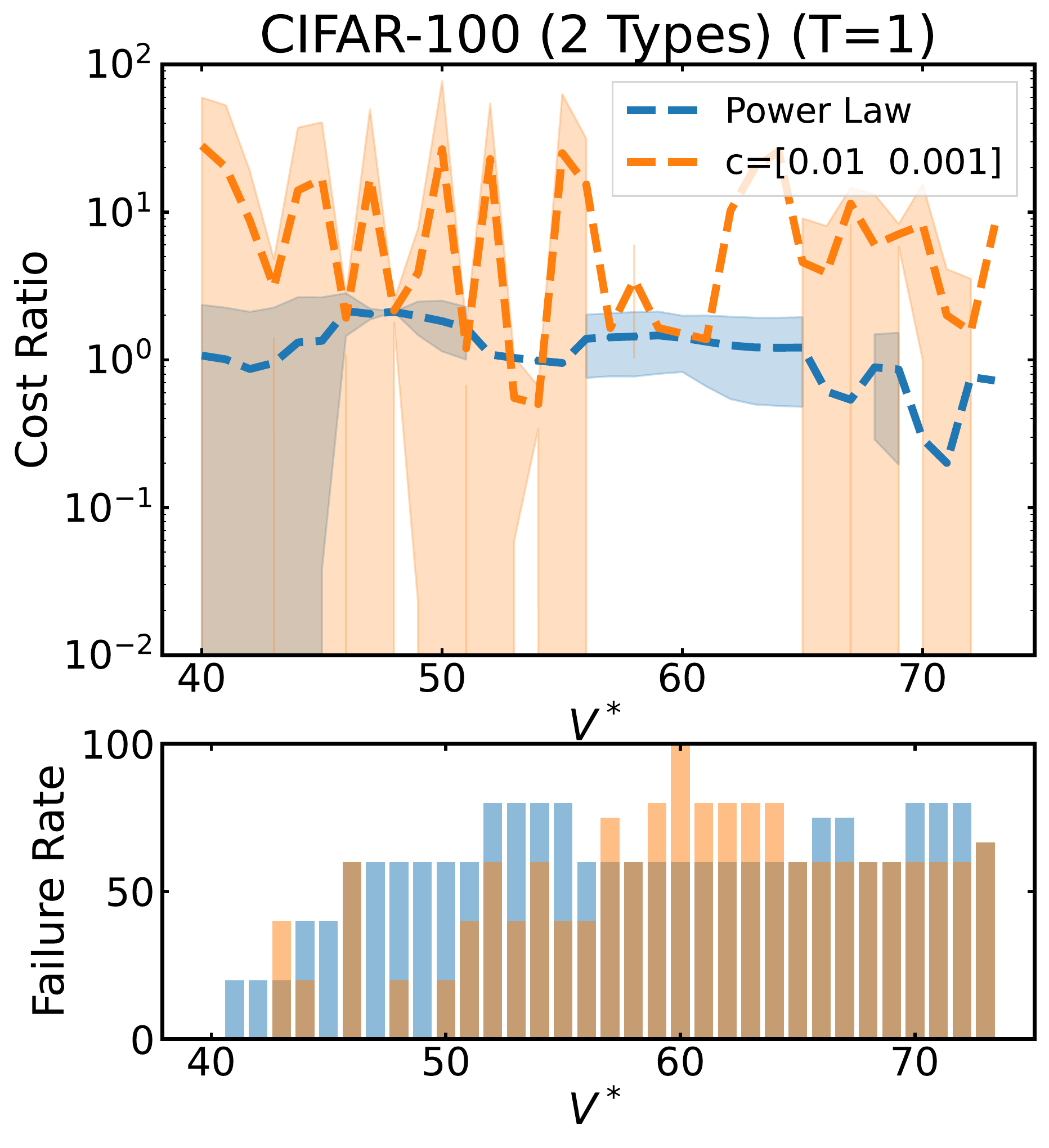} \end{minipage}
\begin{minipage}{0.24\linewidth}\includegraphics[width=1\textwidth]{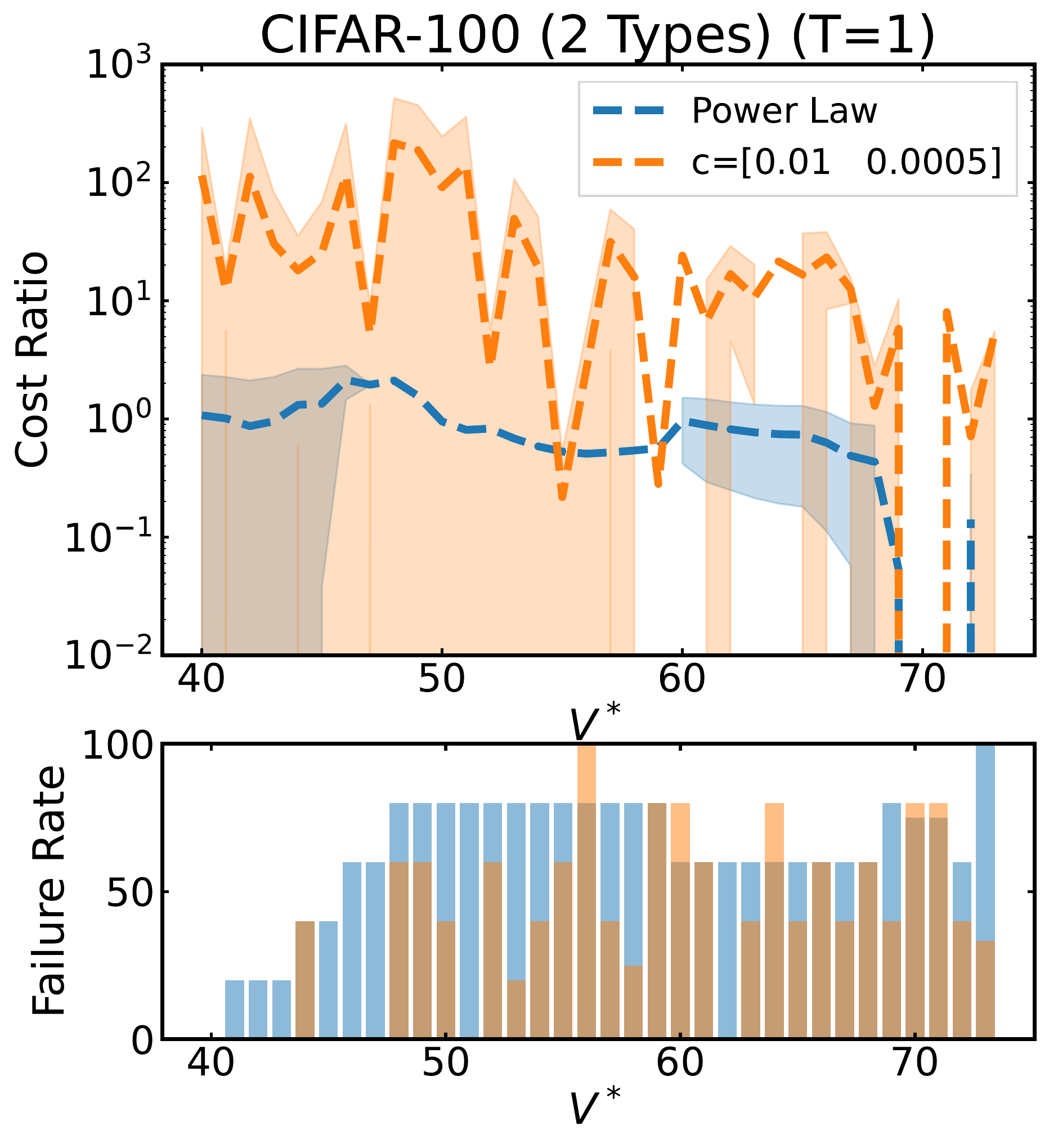} \end{minipage}
\begin{minipage}{0.24\linewidth}\includegraphics[width=1\textwidth]{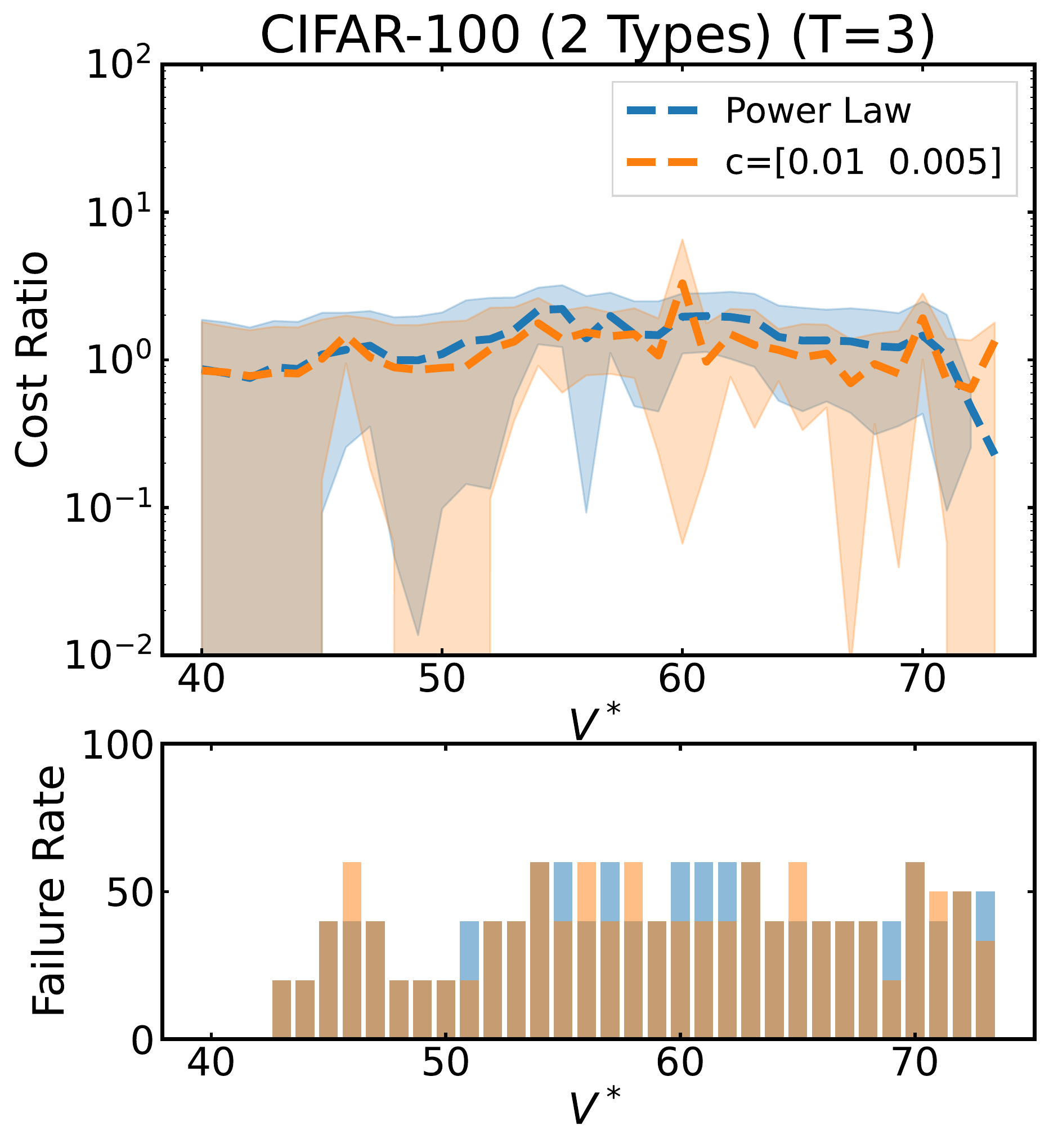} \end{minipage}
\begin{minipage}{0.24\linewidth}\includegraphics[width=1\textwidth]{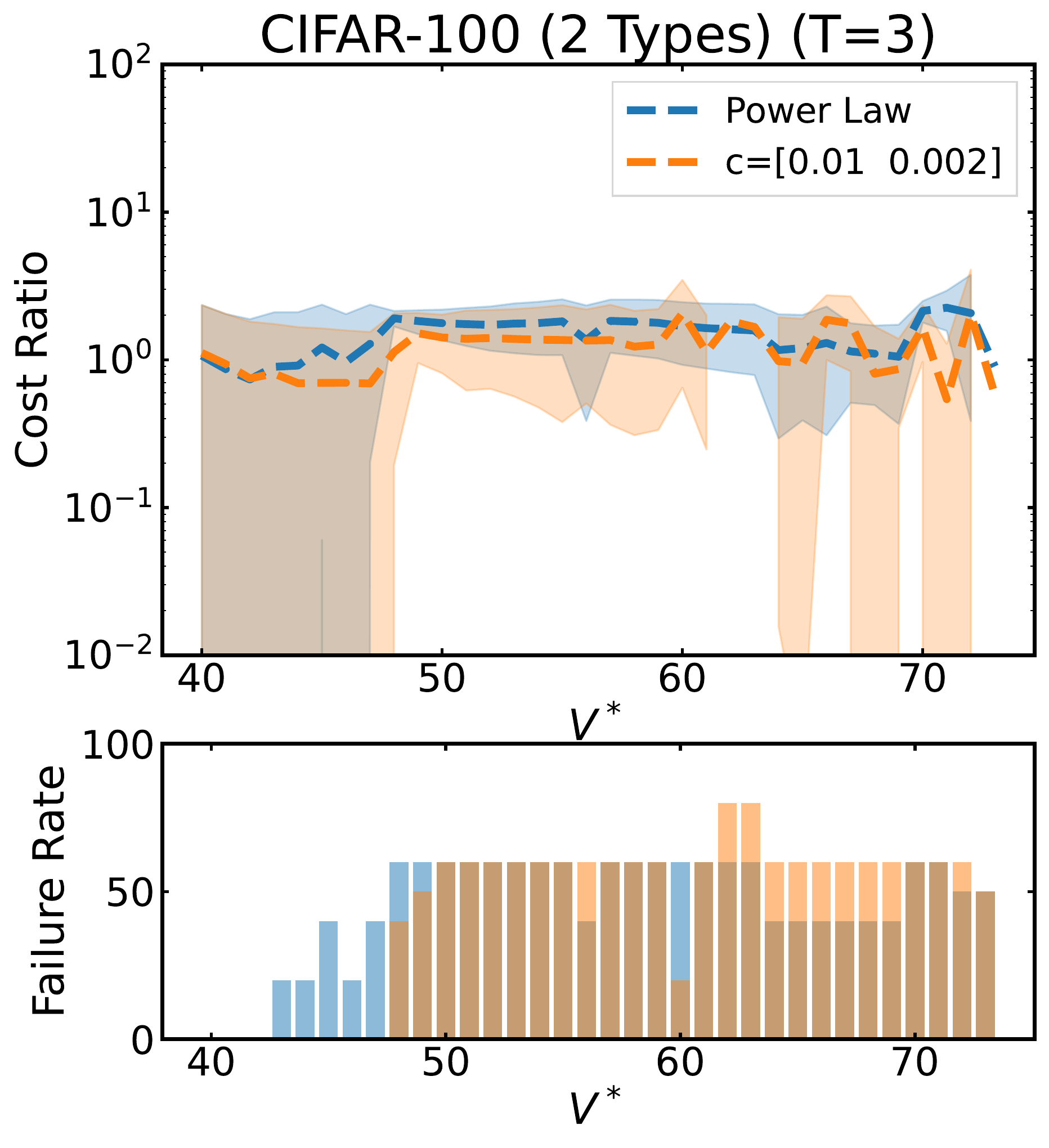} \end{minipage}
\begin{minipage}{0.24\linewidth}\includegraphics[width=1\textwidth]{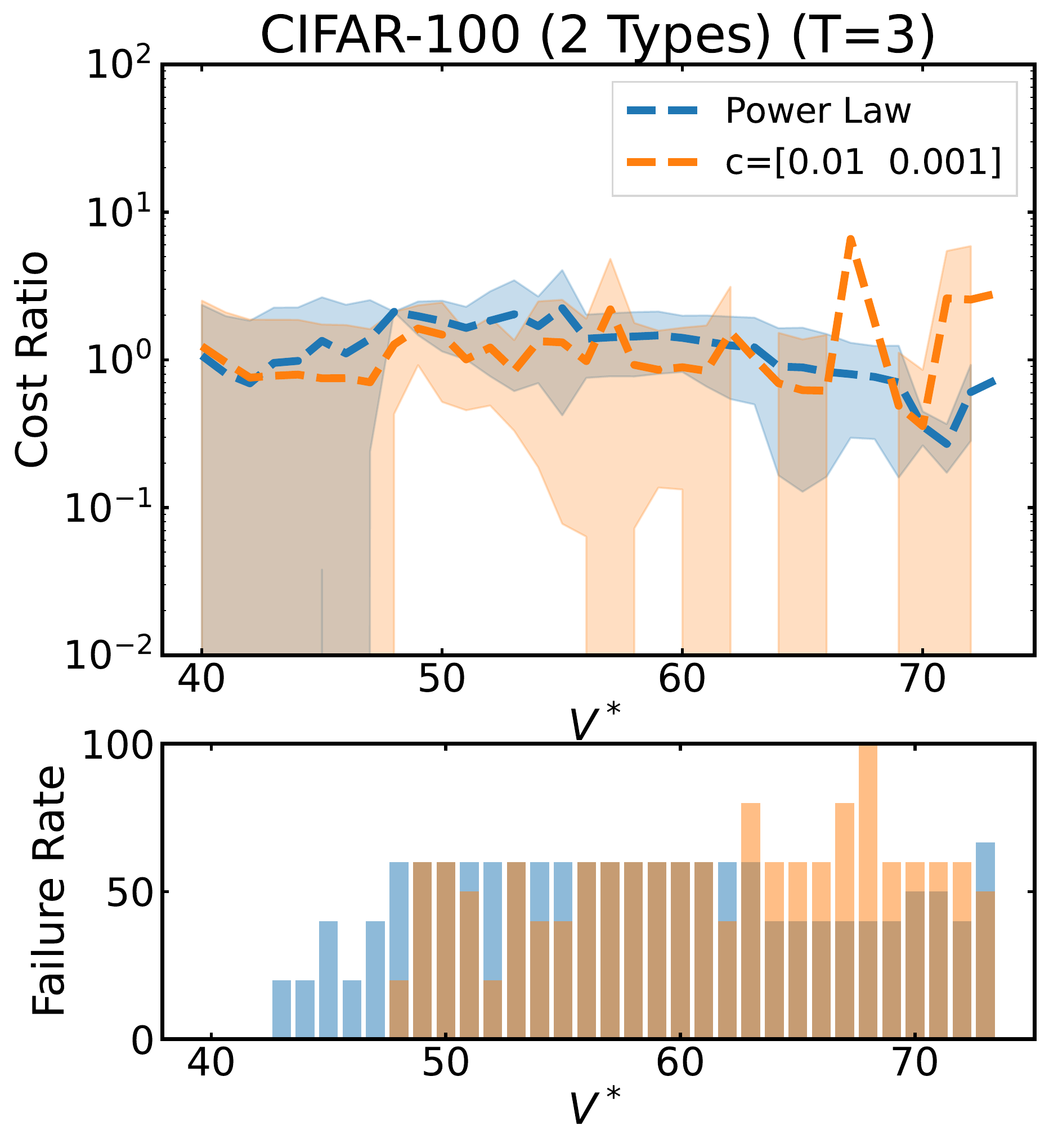} \end{minipage}
\begin{minipage}{0.24\linewidth}\includegraphics[width=1\textwidth]{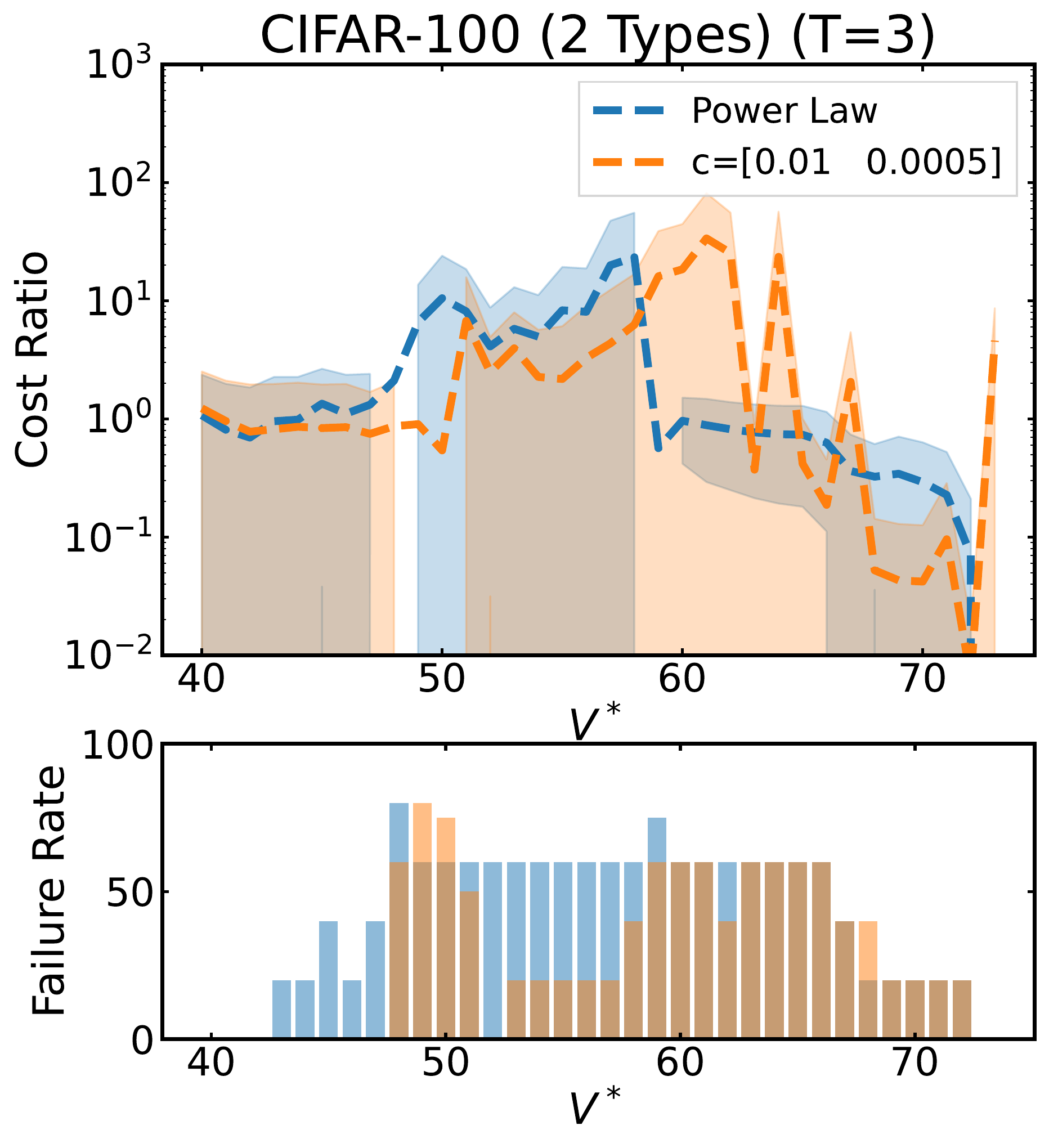} \end{minipage}
\vspace{-4mm}
\caption{\label{fig:app_head_to_head_two_d_cifar}
For experiments on CIFAR-100 with two data types, mean $\pm$ standard deviation over 5 seeds of the cost ratio $\bc^\tpose (\bq^*_T - \bq_0) / \bc^\tpose (\bD^* - \bq_0) - 1$ and failure rate for different $V$ after removing $99$-th percentile outliers. 
We fix $c_0=1$ and $P=10^{13}$.
The rows correspond to $T=1, 3$ (see the main paper for $T=5$) and the columns correspond to $c_1 = c_0/2, c_0/5, c_0/10, c_0/20$. 
}
\end{center}
\vspace{-4mm}
\end{figure*}

\begin{figure*}[!t]
\begin{center}
\begin{minipage}{0.24\linewidth}\includegraphics[width=1\textwidth]{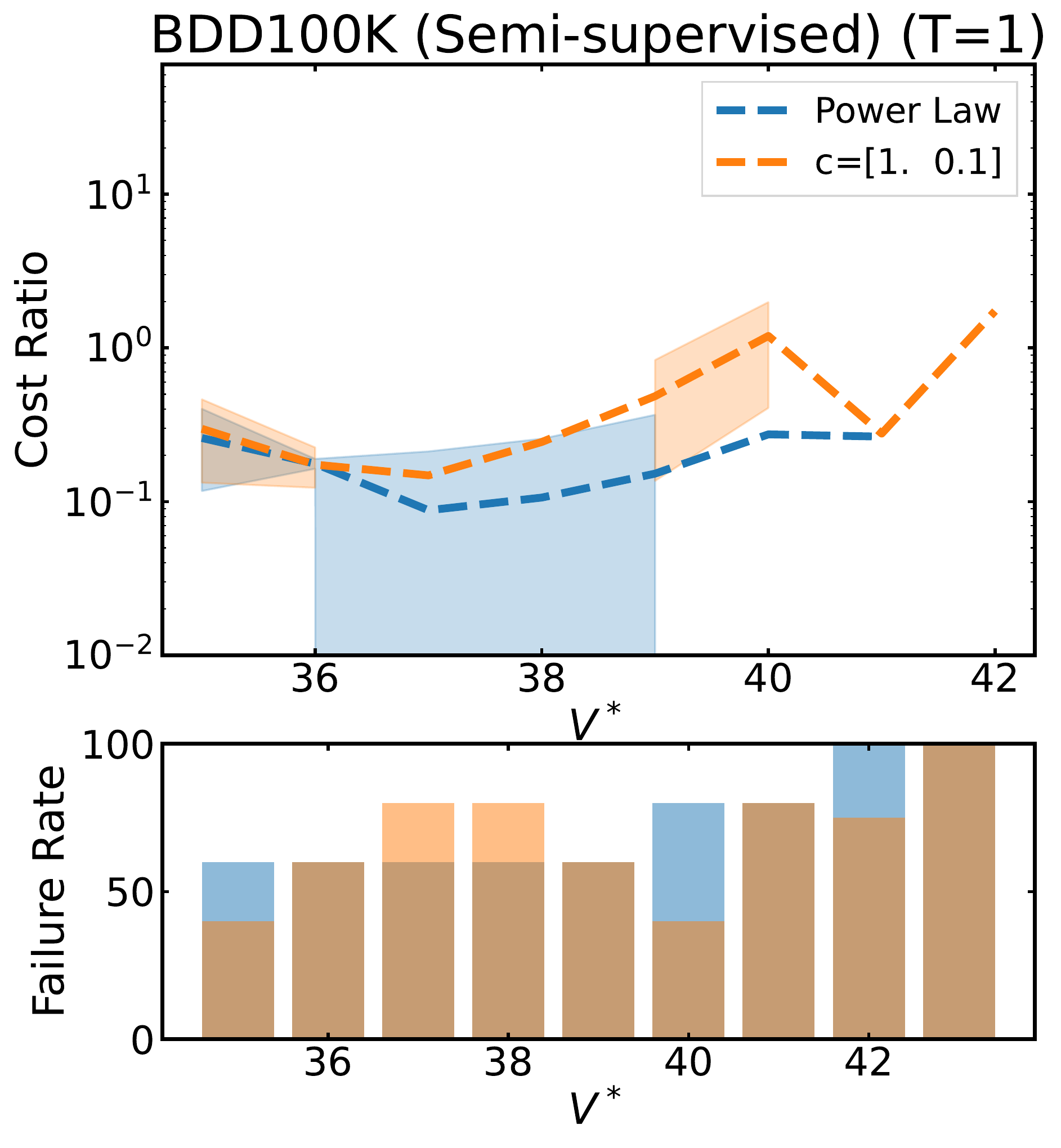} \end{minipage}
\begin{minipage}{0.24\linewidth}\includegraphics[width=1\textwidth]{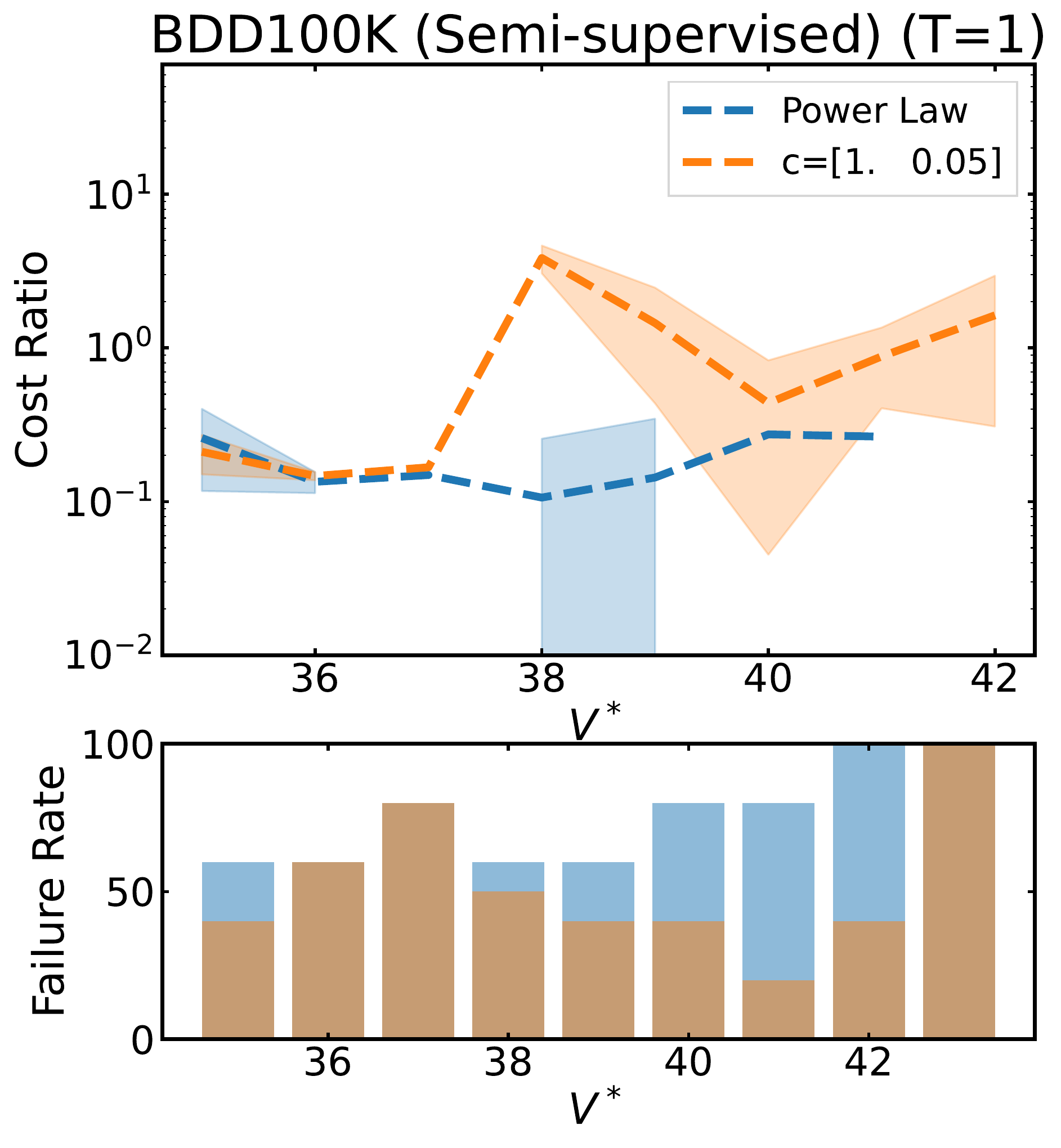} \end{minipage}
\begin{minipage}{0.24\linewidth}\includegraphics[width=1\textwidth]{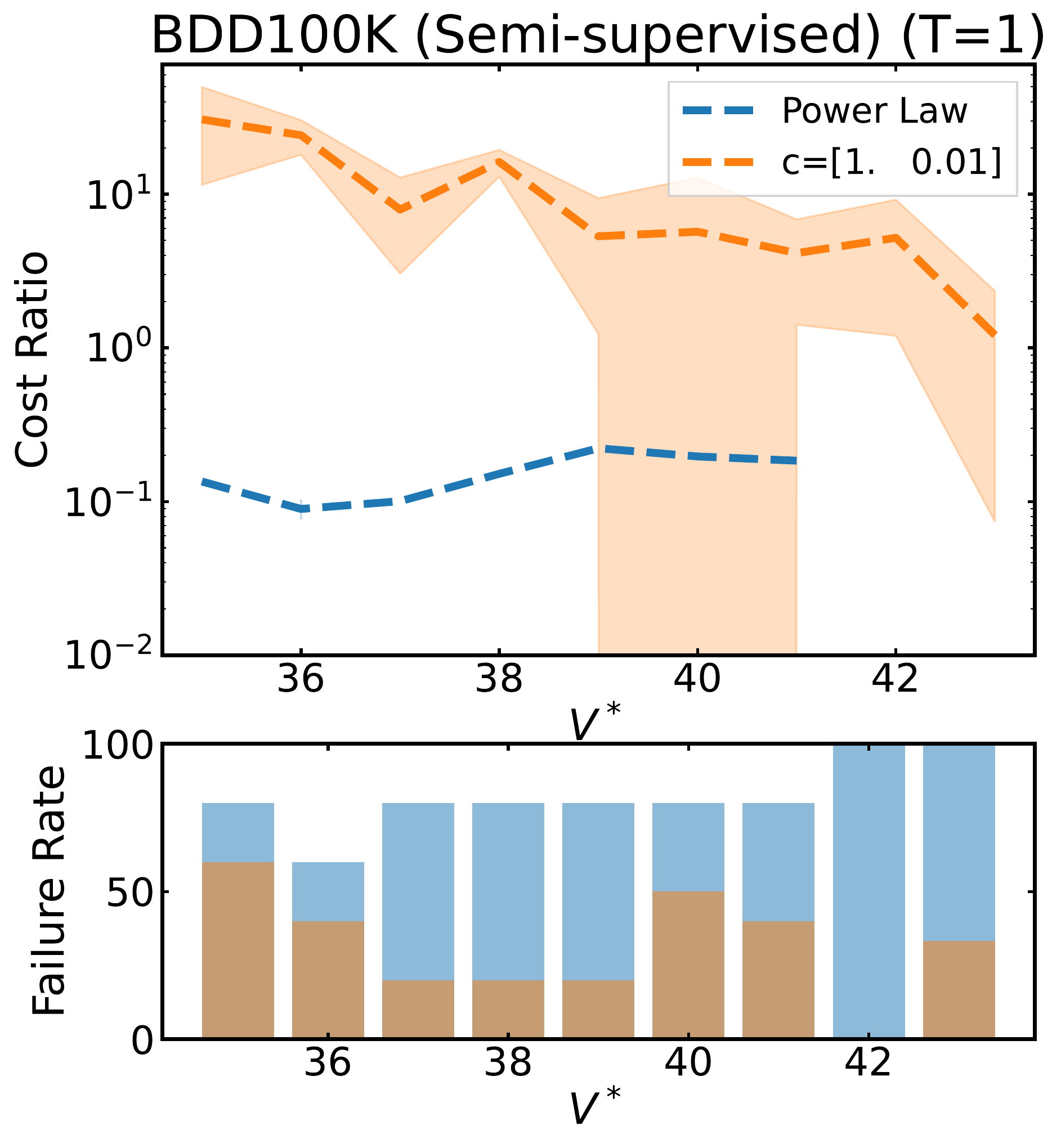} \end{minipage}
\begin{minipage}{0.24\linewidth}\includegraphics[width=1\textwidth]{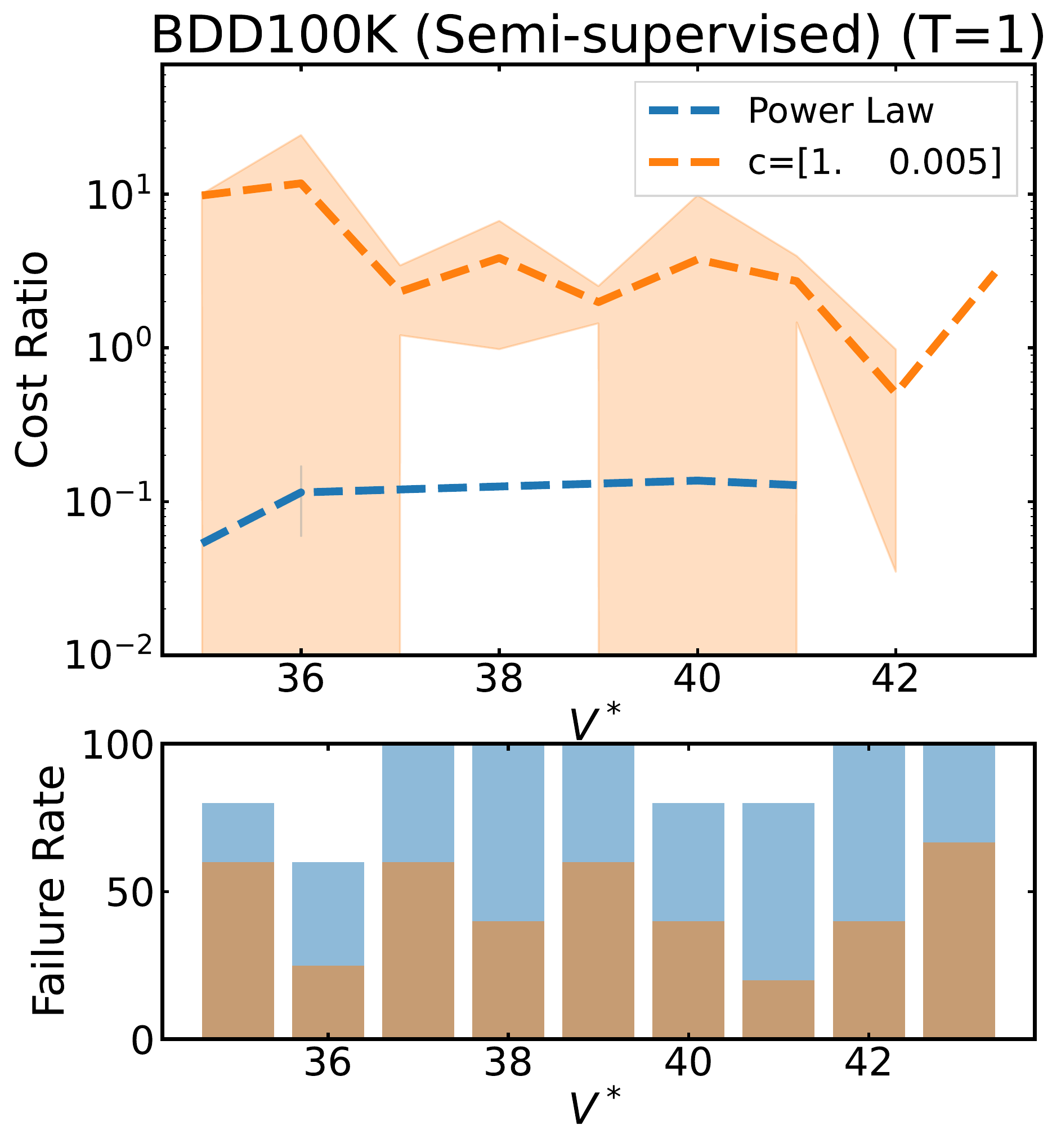} \end{minipage}
\begin{minipage}{0.24\linewidth}\includegraphics[width=1\textwidth]{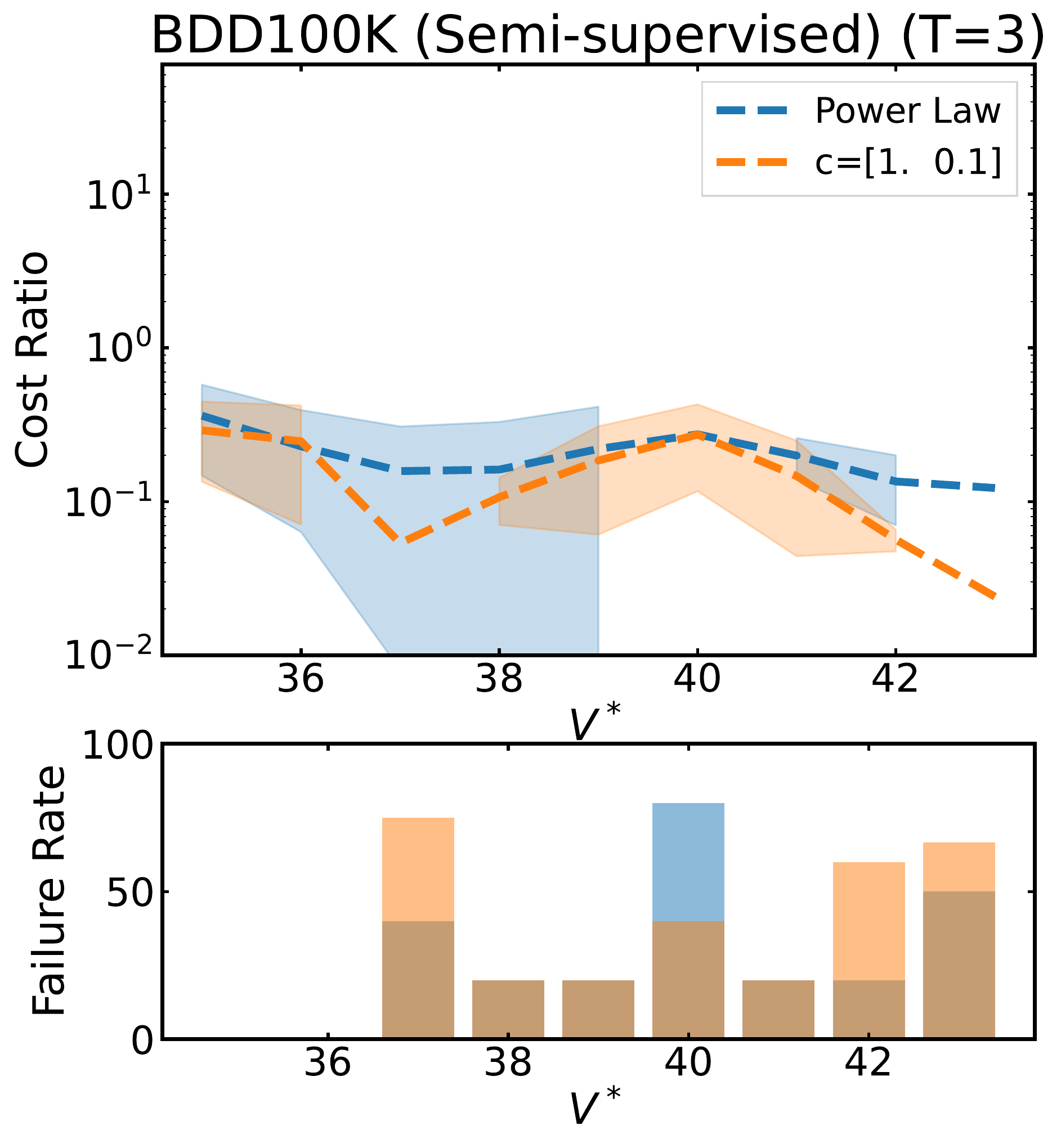} \end{minipage}
\begin{minipage}{0.24\linewidth}\includegraphics[width=1\textwidth]{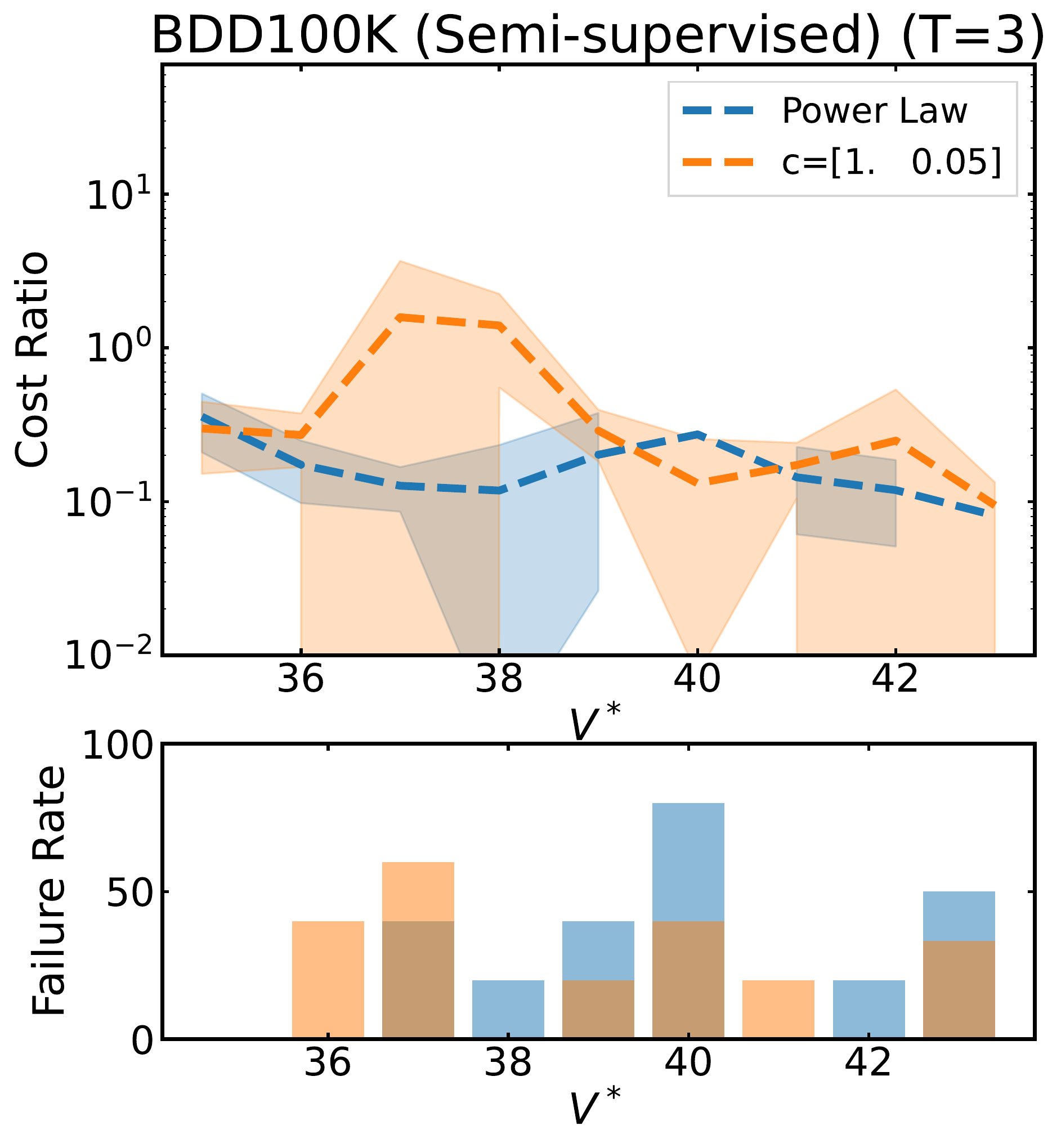} \end{minipage}
\begin{minipage}{0.24\linewidth}\includegraphics[width=1\textwidth]{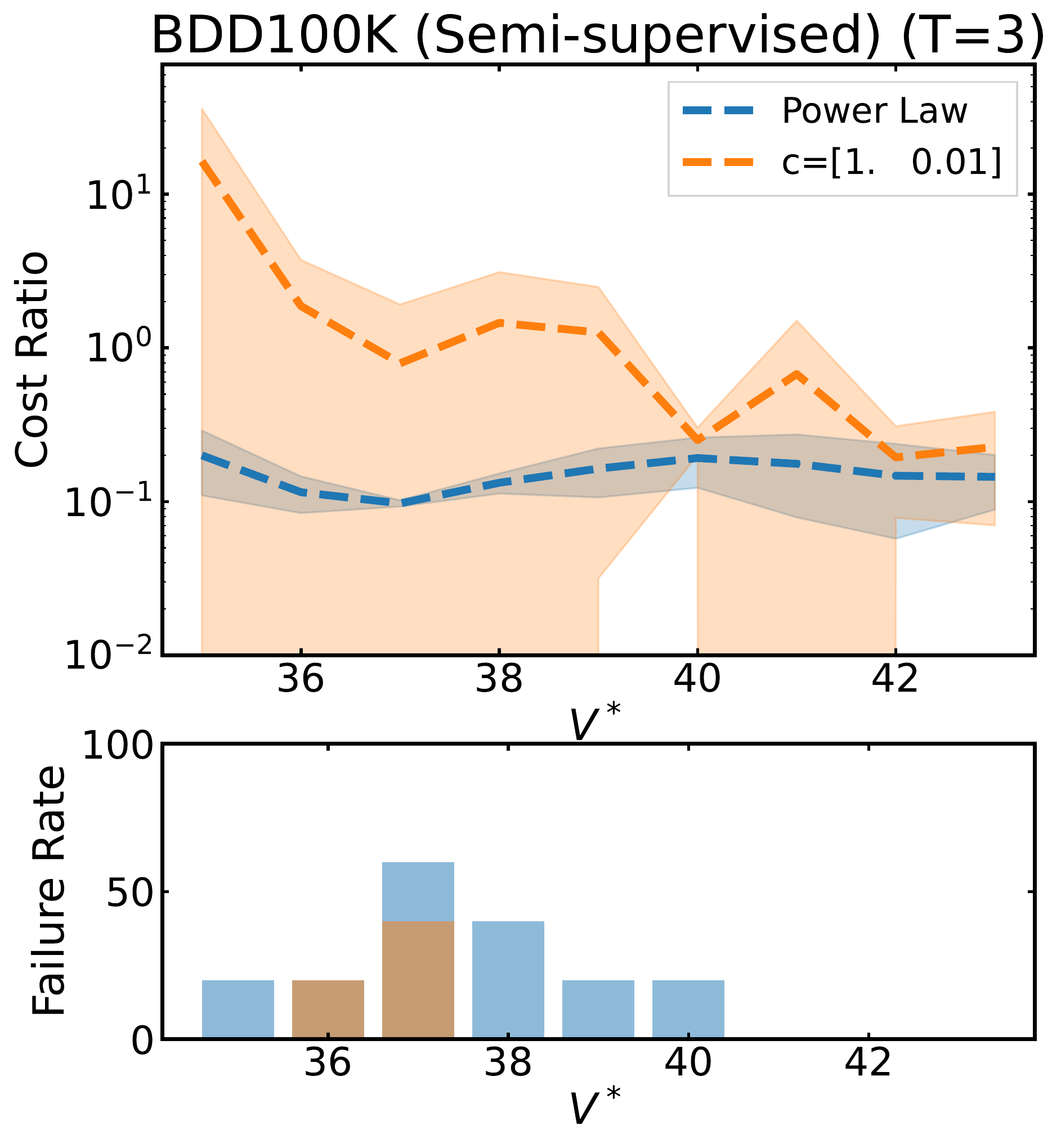} \end{minipage}
\begin{minipage}{0.24\linewidth}\includegraphics[width=1\textwidth]{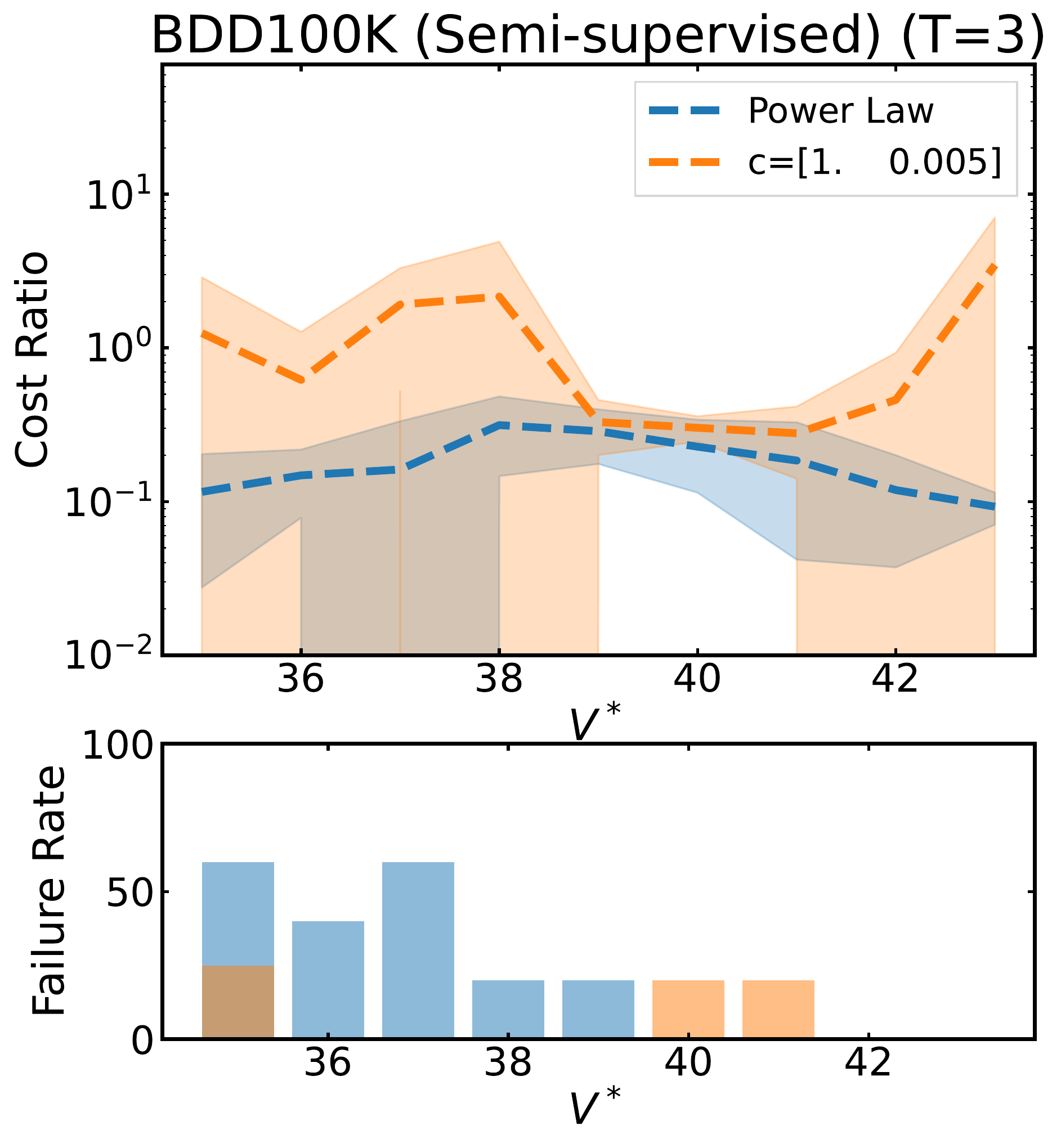} \end{minipage}
\vspace{-4mm}
\caption{\label{fig:app_head_to_head_two_d_bdd}
For experiments on BDD100K with two data types, mean $\pm$ standard deviation over 5 seeds of the cost ratio $\bc^\tpose (\bq^*_T - \bq_0) / \bc^\tpose (\bD^* - \bq_0) - 1$ and failure rate for different $V$ after removing $99$-th percentile outliers. 
We fix $c_0=1$ and $P=10^{13}$.
The rows correspond to $T=1, 3$ (see the main paper for $T=5$) and the columns correspond to $c_1 = c_0/2, c_0/5, c_0/10, c_0/20$. 
}
\end{center}
\vspace{-4mm}
\end{figure*}

\subsection{The Value of Optimization over Estimation when $K = 2$}
\label{sec:app_experiment_k_equals_two}

Figure~\ref{fig:app_head_to_head_two_d_cifar} and Figure~\ref{fig:app_head_to_head_two_d_bdd} expand Figure~\ref{fig:two_d_head_to_head} to $T=1, 3$ rounds. The results validate the summary observations from Table~\ref{tab:two_d_summary} in that the baseline has considerably higher failure rates versus LOC. In particular for BDD100K at $T=1$, the baseline fails consistently for four out of five random seeds. 
On the other hand, recall that LOC admits a higher cost ratio compared to the baseline when $T=1$. We can observe now that this high cost ratio is due to the method incurring high cost for a few target $V^*$ values. 
This behavior is similar to the observation above on VOC with high penalties at $T=3$.

\begin{table}[!t]
    \centering
\resizebox{\linewidth}{!}{
\begin{tabular}{lccccc}
\toprule
  Regression Function   &  $T$ & \multicolumn{2}{c}{Regression} & \multicolumn{2}{c}{LOC} \\ \cmidrule(lr){3-4}\cmidrule(lr){5-6}
           &      & Failure rate & Cost ratio & Failure rate & Cost ratio \\
\midrule
\multirow{3}{*}{Logarithmic $\hat{v}(q;\btheta) = \theta_0 \log (q + \theta_1) + \theta_2$}
  &  1 &                $43\%$ &                 $0.19$ &              $\fir{2\%}$ &               $1.17$ \\
  &  3 &                $37\%$ &                 $0.17$ &              $\fir{2\%}$ &               $0.54$ \\
  &  5 &                $34\%$ &                 $0.16$ &              $\fir{1\%}$ &               $0.39$ \\
\midrule
\multirow{3}{*}{Arctan $\hat{v}(q;\btheta) = \frac{200}{\pi} \arctan (\theta_0 \frac{\pi}{2} q + \theta_1) + \theta_2$}
  &  1 &                $23\%$ &                 $3.31$ &               $\fir{0\%}$ &               $5.56$ \\
  &  3 &                $15\%$ &                 $3.01$ &               $\fir{0\%}$ &               $3.92$ \\
  &  5 &                $12\%$ &                 $2.90$ &               $\fir{0\%}$ &               $3.60$ \\
 \midrule
\multirow{3}{*}{Algebraic Root $\hat{v}(q;\btheta) = \frac{100 q}{1 + | \theta_0 q |^{\theta_1})^{1/\theta_1}} +  \theta_2$}
  & \rebut{ 1 } &               \rebut{ $52\%$ } &              \rebut{   $0.11$ } &            \rebut{  $\fir{23\%}$ } &               \rebut{ $0.81$ }  \\
  & \rebut{ 3 } &               \rebut{ $44\%$ } &              \rebut{    $0.1$ } &            \rebut{   $\fir{2\%}$ } &               \rebut{ $0.87$ } \\
  & \rebut{ 5 } &               \rebut{ $44\%$ } &              \rebut{    $0.1$ } &            \rebut{   $\fir{2\%}$ } &               \rebut{ $0.54$ } \\
\bottomrule
\end{tabular}
}
%
    \caption{For experiments on CIFAR-100, average cost ratio $\bc^\tpose (\bq^*_T - \bq_0) / \bc^\tpose (\bD^* - \bq_0) - 1$ and failure rate measured over a range of $V^*$ and $T$. 
    We fix $c = 1$ and $P=10^7$.
    The best performing failure rate for each setting is bolded. The cost ratio is measured only for instances that achieve $V^*$. LOC consistently reduces the average failure rate, almost consistently down to $0\%$.}
    \label{tab:one_d_other_methods}
\vspace{-5mm}
\end{table}

\begin{table}[!t]
\begin{minipage}{0.69\linewidth}
\rebut{
\resizebox{\linewidth}{!}{
\begin{tabular}{clccccc}
\toprule
  & Data set &  $T$ & \multicolumn{2}{c}{Regression With Correction~\cite{mahmood2022howmuch}} & \multicolumn{2}{c}{LOC} \\ \cmidrule(lr){4-5}\cmidrule(lr){6-7}
  &          &      & Failure rate & Cost ratio & Failure rate & Cost ratio \\
\midrule
\multirow{6}{*}{\rotatebox[origin=c]{90}{Class.}}    
& \multirow{3}{*}{CIFAR-100}    &  1 &   $      14\%$ &   $\seco{0.94}$ &        $\fir{4\%}$ &    $0.99$ \\
&                               &  3 &   $\fir{ 1\%}$ &   $\seco{0.23}$ &        $     3\% $ &    $0.31$ \\
&                               &  5 &   $\fir{ 0\%}$ &   $\seco{0.17}$ &        $     2\% $ &    $0.19$ \\
  \cmidrule{2-7} 
 & \multirow{3}{*}{Imagenet}    &  1 &   $\fir{ 7\%}$ &          $1.03$ &        $     37\% $ &    $\seco{0.49}$ \\
 &                              &  3 &   $\fir{ 0\%}$ &          $0.21$ &        $      5\% $ &    $\seco{0.16}$ \\
 &                              &  5 &   $\fir{ 0\%}$ &          $0.14$ &        $      2\% $ &    $\seco{0.10}$ \\
 \midrule
  \multirow{6}{*}{\rotatebox[origin=c]{90}{Seg.}}    
  & \multirow{3}{*}{BDD100K}    &  1 &   $\fir{ 4\%}$ &          $4.03$ &        $     12\% $ &    $\seco{2.03}$ \\
  &                             &  3 &   $\fir{ 0\%}$ &          $1.02$ &        $\fir{ 0\%}$ &    $\seco{0.72}$ \\
  &                             &  5 &   $\fir{ 0\%}$ &          $0.62$ &        $\fir{ 0\%}$ &    $\seco{0.35}$ \\
  \cmidrule{2-7} 
 & \multirow{3}{*}{nuScenes}    &  1 &   $\fir{ 0\%}$ &          $27.2$ &        $     52\% $ &    $\seco{0.16}$ \\
 &                              &  3 &   $\fir{ 0\%}$ &          $0.75$ &        $\fir{ 0\%}$ &    $\seco{0.09}$ \\
 &                              &  5 &   $\fir{ 0\%}$ &          $0.30$ &        $\fir{ 0\%}$ &    $\seco{0.04}$ \\
 \midrule
 \multirow{3}{*}{\rotatebox[origin=c]{90}{Det.}}    
      & \multirow{3}{*}{VOC}    &  1 &   $\fir{ 0\%}$ &          $44.6$ &        $     25\% $ &    $\seco{0.56}$ \\
      &                         &  3 &   $\fir{ 0\%}$ &          $7.02$ &        $\fir{ 0\%}$ &    $\seco{1.10}$ \\
      &                         &  5 &   $\fir{ 0\%}$ &          $3.98$ &        $\fir{ 0\%}$ &    $\seco{0.84}$ \\
\bottomrule
\end{tabular}
}
}
\end{minipage}
\hfill
\begin{minipage}{0.28\linewidth}
    \caption{
    \rebut{
    Comparison against the correction factor-based Power Law Regression of~\citet{mahmood2022howmuch} 
    using the same setup as in Table~\ref{tab:one_d_summary}.
    The best performing cost ratio is underlined and the best performing failure rate for each setting is bolded. 
    Although the baseline is designed specifically to achieve low failure rates, LOC often can achieve competitive failure rates while reducing the cost ratios by an order of magnitude.
    }
    }
    \label{tab:one_d_with_tolerance}
\end{minipage}
\vspace{-2mm}
\end{table}

\subsection{LOC with Alternative Regression Functions}
\label{sec:app_experiment_alternate_regression_functions}

\citet{mahmood2022howmuch} show that we can use other regression functions instead of the power law to estimate the data requirement. Moreover, some functions tend to consistently over- or under-estimate the requirement. 
LOC can be deployed on top of any such regression function, since the regression function is only used to generate bootstrap samples.

Table~\ref{tab:one_d_other_methods} highlights experiments on CIFAR-100 with \rebut{three} alternative regression functions that were used by~\citet{mahmood2022howmuch}.
For both functions, we observe the same trends seen in Table~\ref{tab:one_d_summary}. That is, LOC reduces the failure rate down to approximately zero, at a marginal relative increase in cost. 

\rebut{

Noting that Power Law Regression often leads to failure,~\citet{mahmood2022howmuch} also propose a correction factor heuristic wherein they learn a parameter $\tau$ such that if the data collection problem requires a target performance $V^*$, we should instead aim to collect enough data to meet $V^* + \tau$. 
In order to learn this correction factor, we require a pre-existing data set upon which we can simulate a data collection policy.
\citet{mahmood2022howmuch} suggest setting $\tau$ such that we can achieve the data requirement $V^*$ for any $V^*$ on the pre-existing data set, and then fixing this parameter for new data sets.

Table~\ref{tab:one_d_with_tolerance} compares LOC (i.e., repeating Table~\ref{tab:one_d_summary}) with the Correction factor-based Power Law regression baseline of~\citet{mahmood2022howmuch}. Following the original paper, we tune $\tau$ using CIFAR-10 and apply it on all other data sets. 
The correction factor is designed to minimize the failure rate and thus, achieves nearly $0\%$ failure rate for all settings, but often at high cost ratios. 
On the other hand, LOC achieves generally low failure rates and low cost ratios. Specifically, for $T=3, 5$, we are competitive with the baseline on failure rates for most tasks while obtaining up to an order of magnitude decrease in costs. For $T=1$, we typically admit higher failure rates; however for the segmentation and detection tasks, we obtain up multiple orders of magnitude lower costs.
Finally, note that this baseline requires a similar prior task to be effective. For example, the baseline outperforms us on cost and failure rate both only on CIFAR-100, since it is tuned on CIFAR-10. On the other hand, LOC does not require this prior data set to be effective as evidence by its performance on non-classification tasks.

}

%% file: main.bbl
\begin{thebibliography}{58}
\providecommand{\natexlab}[1]{#1}
\providecommand{\url}[1]{\texttt{#1}}
\expandafter\ifx\csname urlstyle\endcsname\relax
  \providecommand{\doi}[1]{doi: #1}\else
  \providecommand{\doi}{doi: \begingroup \urlstyle{rm}\Url}\fi

\bibitem[Acuna et~al.(2018)Acuna, Ling, Kar, and Fidler]{acuna2018efficient}
David Acuna, Huan Ling, Amlan Kar, and Sanja Fidler.
\newblock Efficient interactive annotation of segmentation datasets with
  polygon-rnn++.
\newblock In \emph{Proceedings of the IEEE Conference on Computer Vision and
  Pattern Recognition (CVPR)}, June 2018.

\bibitem[Mahmood et~al.(2022{\natexlab{a}})Mahmood, Lucas, Acuna, Li, Philion,
  Alvarez, Yu, Fidler, and Law]{mahmood2022howmuch}
Rafid Mahmood, James Lucas, David Acuna, Daiqing Li, Jonah Philion, Jose~M.
  Alvarez, Zhiding Yu, Sanja Fidler, and Marc~T. Law.
\newblock How much more data do we need? estimating requirements for downstream
  tasks.
\newblock In \emph{2022 IEEE Conference on Computer Vision and Pattern
  Recognition}. Ieee, 2022{\natexlab{a}}.

\bibitem[Frey and Fisher(1999)]{frey1999modeling}
Lewis~J Frey and Douglas~H Fisher.
\newblock Modeling decision tree performance with the power law.
\newblock In \emph{Seventh International Workshop on Artificial Intelligence
  and Statistics}. PMLR, 1999.

\bibitem[Gu et~al.(2001)Gu, Hu, and Liu]{gu2001modelling}
Baohua Gu, Feifang Hu, and Huan Liu.
\newblock Modelling classification performance for large data sets.
\newblock In \emph{International Conference on Web-Age Information Management},
  pages 317--328. Springer, 2001.

\bibitem[Hestness et~al.(2017)Hestness, Narang, Ardalani, Diamos, Jun,
  Kianinejad, Patwary, Ali, Yang, and Zhou]{hestness2017deep}
Joel Hestness, Sharan Narang, Newsha Ardalani, Gregory Diamos, Heewoo Jun,
  Hassan Kianinejad, Md~Patwary, Mostofa Ali, Yang Yang, and Yanqi Zhou.
\newblock Deep learning scaling is predictable, empirically.
\newblock \emph{arXiv preprint arXiv:1712.00409}, 2017.

\bibitem[Rosenfeld et~al.(2020)Rosenfeld, Rosenfeld, Belinkov, and
  Shavit]{rosenfeld2019constructive}
Jonathan~S Rosenfeld, Amir Rosenfeld, Yonatan Belinkov, and Nir Shavit.
\newblock A constructive prediction of the generalization error across scales.
\newblock In \emph{International Conference on Learning Representations}, 2020.

\bibitem[Kaplan et~al.(2020)Kaplan, McCandlish, Henighan, Brown, Chess, Child,
  Gray, Radford, Wu, and Amodei]{kaplan2020scaling}
Jared Kaplan, Sam McCandlish, Tom Henighan, Tom~B Brown, Benjamin Chess, Rewon
  Child, Scott Gray, Alec Radford, Jeffrey Wu, and Dario Amodei.
\newblock Scaling laws for neural language models.
\newblock \emph{arXiv preprint arXiv:2001.08361}, 2020.

\bibitem[Hoiem et~al.(2021)Hoiem, Gupta, Li, and
  Shlapentokh-Rothman]{hoiem2021learning}
Derek Hoiem, Tanmay Gupta, Zhizhong Li, and Michal Shlapentokh-Rothman.
\newblock Learning curves for analysis of deep networks.
\newblock In \emph{International Conference on Machine Learning}, pages
  4287--4296. PMLR, 2021.

\bibitem[Bahri et~al.(2021)Bahri, Dyer, Kaplan, Lee, and
  Sharma]{bahri2021explaining}
Yasaman Bahri, Ethan Dyer, Jared Kaplan, Jaehoon Lee, and Utkarsh Sharma.
\newblock Explaining neural scaling laws.
\newblock \emph{arXiv preprint arXiv:2102.06701}, 2021.

\bibitem[Bisla et~al.(2021)Bisla, Saridena, and
  Choromanska]{bisla2021theoretical}
Devansh Bisla, Apoorva~Nandini Saridena, and Anna Choromanska.
\newblock A theoretical-empirical approach to estimating sample complexity of
  dnns.
\newblock In \emph{Proceedings of the IEEE/CVF Conference on Computer Vision
  and Pattern Recognition}, pages 3270--3280, 2021.

\bibitem[Mikami et~al.(2021)Mikami, Fukumizu, Murai, Suzuki, Kikuchi, Suzuki,
  Maeda, and Hayashi]{mikami2021scaling}
Hiroaki Mikami, Kenji Fukumizu, Shogo Murai, Shuji Suzuki, Yuta Kikuchi, Taiji
  Suzuki, Shin-ichi Maeda, and Kohei Hayashi.
\newblock A scaling law for synthetic-to-real transfer: How much is your
  pre-training effective?
\newblock \emph{arXiv preprint arXiv:2108.11018}, 2021.

\bibitem[Acuna et~al.(2021)Acuna, Zhang, Law, and Fidler]{acuna2021f}
David Acuna, Guojun Zhang, Marc~T Law, and Sanja Fidler.
\newblock f-domain adversarial learning: Theory and algorithms.
\newblock In \emph{International Conference on Machine Learning}, pages 66--75.
  PMLR, 2021.

\bibitem[Prakash et~al.(2021)Prakash, Debnath, Lafleche, Cameracci, State,
  Birchfield, and Law]{Prakash_2021_ICCV}
Aayush Prakash, Shoubhik Debnath, Jean-Francois Lafleche, Eric Cameracci,
  Gavriel State, Stan Birchfield, and Marc~T. Law.
\newblock Self-supervised real-to-sim scene generation.
\newblock In \emph{Proceedings of the IEEE/CVF International Conference on
  Computer Vision (ICCV)}, pages 16044--16054, October 2021.

\bibitem[Acuna et~al.(2022)Acuna, Law, Zhang, and Fidler]{acuna2022domain}
David Acuna, Marc~T Law, Guojun Zhang, and Sanja Fidler.
\newblock Domain adversarial training: A game perspective.
\newblock In \emph{International Conference on Learning Representations}, 2022.

\bibitem[Jones et~al.(2003)Jones, Carley, and Harrison]{jones2003introduction}
S~Jones, S~Carley, and M~Harrison.
\newblock An introduction to power and sample size estimation.
\newblock \emph{Emergency Medicine Journal: EMJ}, 20\penalty0 (5):\penalty0
  453, 2003.

\bibitem[Sun et~al.(2017)Sun, Shrivastava, Singh, and Gupta]{sun2017revisiting}
Chen Sun, Abhinav Shrivastava, Saurabh Singh, and Abhinav Gupta.
\newblock Revisiting unreasonable effectiveness of data in deep learning era.
\newblock In \emph{Proceedings of the IEEE International Conference on Computer
  Vision}, pages 843--852, 2017.

\bibitem[Figueroa et~al.(2012)Figueroa, Zeng-Treitler, Kandula, and
  Ngo]{figueroa2012predicting}
Rosa~L Figueroa, Qing Zeng-Treitler, Sasikiran Kandula, and Long~H Ngo.
\newblock Predicting sample size required for classification performance.
\newblock \emph{BMC Medical Informatics and Decision Making}, 12\penalty0
  (1):\penalty0 1--10, 2012.

\bibitem[Viering and Loog(2021)]{viering2021shape}
Tom Viering and Marco Loog.
\newblock The shape of learning curves: a review.
\newblock \emph{arXiv preprint arXiv:2103.10948}, 2021.

\bibitem[Zhai et~al.(2022)Zhai, Kolesnikov, Houlsby, and
  Beyer]{zhai2021scaling}
Xiaohua Zhai, Alexander Kolesnikov, Neil Houlsby, and Lucas Beyer.
\newblock Scaling vision transformers.
\newblock In \emph{Proceedings of the IEEE/CVF Conference on Computer Vision
  and Pattern Recognition}, 2022.

\bibitem[Settles(2009)]{settles2009active}
Burr Settles.
\newblock Active learning literature survey.
\newblock 2009.

\bibitem[Sener and Savarese(2018)]{sener2017active}
Ozan Sener and Silvio Savarese.
\newblock Active learning for convolutional neural networks: A core-set
  approach.
\newblock In \emph{International Conference on Learning Representations}, 2018.

\bibitem[Yoo and Kweon(2019)]{yoo2019learning}
Donggeun Yoo and In~So Kweon.
\newblock Learning loss for active learning.
\newblock In \emph{Proceedings of the IEEE/CVF Conference on Computer Vision
  and Pattern Recognition}, pages 93--102, 2019.

\bibitem[Sinha et~al.(2019)Sinha, Ebrahimi, and Darrell]{sinha2019variational}
Samarth Sinha, Sayna Ebrahimi, and Trevor Darrell.
\newblock Variational adversarial active learning.
\newblock In \emph{Proceedings of the IEEE/CVF International Conference on
  Computer Vision}, pages 5972--5981, 2019.

\bibitem[Mahmood et~al.(2022{\natexlab{b}})Mahmood, Fidler, and
  Law]{mahmood2021low}
Rafid Mahmood, Sanja Fidler, and Marc~T. Law.
\newblock Low-budget active learning via wasserstein distance: An integer
  programming approach.
\newblock In \emph{International Conference on Learning Representations},
  2022{\natexlab{b}}.

\bibitem[Jiang et~al.(2020)Jiang, Foret, Yak, Roy, Mobahi, Dziugaite, Bengio,
  Gunasekar, Guyon, and Neyshabur]{jiang2020neurips}
Yiding Jiang, Pierre Foret, Scott Yak, Daniel~M Roy, Hossein Mobahi,
  Gintare~Karolina Dziugaite, Samy Bengio, Suriya Gunasekar, Isabelle Guyon,
  and Behnam Neyshabur.
\newblock Neurips 2020 competition: Predicting generalization in deep learning.
\newblock \emph{arXiv preprint arXiv:2012.07976}, 2020.

\bibitem[Jiang et~al.(2021)Jiang, Natekar, Sharma, Aithal, Kashyap,
  Subramanyam, Lassance, Roy, Dziugaite, Gunasekar, et~al.]{jiang2021methods}
Yiding Jiang, Parth Natekar, Manik Sharma, Sumukh~K Aithal, Dhruva Kashyap,
  Natarajan Subramanyam, Carlos Lassance, Daniel~M Roy, Gintare~Karolina
  Dziugaite, Suriya Gunasekar, et~al.
\newblock Methods and analysis of the first competition in predicting
  generalization of deep learning.
\newblock In \emph{NeurIPS 2020 Competition and Demonstration Track}, pages
  170--190. PMLR, 2021.

\bibitem[Smith(1918)]{smith1918standard}
Kirstine Smith.
\newblock On the standard deviations of adjusted and interpolated values of an
  observed polynomial function and its constants and the guidance they give
  towards a proper choice of the distribution of observations.
\newblock \emph{Biometrika}, 12\penalty0 (1/2):\penalty0 1--85, 1918.

\bibitem[Cohn(1993)]{cohn1993neural}
David Cohn.
\newblock Neural network exploration using optimal experiment design.
\newblock \emph{Advances in neural information processing systems}, 6, 1993.

\bibitem[Emery and Nenarokomov(1998)]{emery1998optimal}
Ashley~F Emery and Aleksey~V Nenarokomov.
\newblock Optimal experiment design.
\newblock \emph{Measurement Science and Technology}, 9\penalty0 (6):\penalty0
  864, 1998.

\bibitem[Bertsimas et~al.(2015)Bertsimas, Johnson, and
  Kallus]{bertsimas2015power}
Dimitris Bertsimas, Mac Johnson, and Nathan Kallus.
\newblock The power of optimization over randomization in designing experiments
  involving small samples.
\newblock \emph{Operations Research}, 63\penalty0 (4):\penalty0 868--876, 2015.

\bibitem[Carneiro et~al.(2020)Carneiro, Lee, and Wilhelm]{carneiro2020optimal}
Pedro Carneiro, Sokbae Lee, and Daniel Wilhelm.
\newblock Optimal data collection for randomized control trials.
\newblock \emph{The Econometrics Journal}, 23\penalty0 (1):\penalty0 1--31,
  2020.

\bibitem[Zhang(2022)]{zhang2022dynamic}
Hao Zhang.
\newblock Dynamic learning and decision making via basis weight vectors.
\newblock \emph{Operations Research}, 2022.

\bibitem[Haixiang et~al.(2017)Haixiang, Yijing, Shang, Mingyun, Yuanyue, and
  Bing]{haixiang2017learning}
Guo Haixiang, Li~Yijing, Jennifer Shang, Gu~Mingyun, Huang Yuanyue, and Gong
  Bing.
\newblock Learning from class-imbalanced data: Review of methods and
  applications.
\newblock \emph{Expert systems with applications}, 73:\penalty0 220--239, 2017.

\bibitem[Van~Engelen and Hoos(2020)]{van2020survey}
Jesper~E Van~Engelen and Holger~H Hoos.
\newblock A survey on semi-supervised learning.
\newblock \emph{Machine Learning}, 109\penalty0 (2):\penalty0 373--440, 2020.

\bibitem[Ben-David et~al.(2010)Ben-David, Blitzer, Crammer, Kulesza, Pereira,
  and Vaughan]{ben2010theory}
Shai Ben-David, John Blitzer, Koby Crammer, Alex Kulesza, Fernando Pereira, and
  Jennifer~Wortman Vaughan.
\newblock A theory of learning from different domains.
\newblock \emph{Machine learning}, 79\penalty0 (1):\penalty0 151--175, 2010.

\bibitem[Krizhevsky(2009)]{krizhevsky2009learning}
Alex Krizhevsky.
\newblock Learning multiple layers of features from tiny images.
\newblock 2009.

\bibitem[Deng et~al.(2009)Deng, Dong, Socher, Li, Li, and
  Fei-Fei]{deng2009imagenet}
Jia Deng, Wei Dong, Richard Socher, Li-Jia Li, Kai Li, and Li~Fei-Fei.
\newblock Imagenet: A large-scale hierarchical image database.
\newblock In \emph{2009 IEEE Conference on Computer Vision and Pattern
  Recognition}, pages 248--255. Ieee, 2009.

\bibitem[He et~al.(2016)He, Zhang, Ren, and Sun]{he2016deep}
Kaiming He, Xiangyu Zhang, Shaoqing Ren, and Jian Sun.
\newblock Deep residual learning for image recognition.
\newblock In \emph{Proceedings of the IEEE Conference on Computer Vision and
  Pattern Recognition}, pages 770--778, 2016.

\bibitem[Chen et~al.(2017)Chen, Papandreou, Schroff, and
  Adam]{chen2017rethinking}
Liang-Chieh Chen, George Papandreou, Florian Schroff, and Hartwig Adam.
\newblock Rethinking atrous convolution for semantic image segmentation.
\newblock \emph{arXiv preprint arXiv:1706.05587}, 2017.

\bibitem[Yu et~al.(2020)Yu, Chen, Wang, Xian, Chen, Liu, Madhavan, and
  Darrell]{yu2020bdd100k}
Fisher Yu, Haofeng Chen, Xin Wang, Wenqi Xian, Yingying Chen, Fangchen Liu,
  Vashisht Madhavan, and Trevor Darrell.
\newblock Bdd100k: A diverse driving dataset for heterogeneous multitask
  learning.
\newblock In \emph{Proceedings of the IEEE/CVF conference on Computer Vision
  and Pattern Recognition}, pages 2636--2645, 2020.

\bibitem[Caesar et~al.(2020)Caesar, Bankiti, Lang, Vora, Liong, Xu, Krishnan,
  Pan, Baldan, and Beijbom]{nuscenes2019}
Holger Caesar, Varun Bankiti, Alex~H Lang, Sourabh Vora, Venice~Erin Liong,
  Qiang Xu, Anush Krishnan, Yu~Pan, Giancarlo Baldan, and Oscar Beijbom.
\newblock nuscenes: A multimodal dataset for autonomous driving.
\newblock In \emph{Proceedings of the IEEE/CVF Conference on Computer Vision
  and Pattern Recognition}, pages 11621--11631, 2020.

\bibitem[Philion and Fidler(2020)]{liftsplat}
Jonah Philion and Sanja Fidler.
\newblock Lift, splat, shoot: Encoding images from arbitrary camera rigs by
  implicitly unprojecting to 3d.
\newblock In \emph{Proceedings of the European Conference on Computer Vision},
  2020.

\bibitem[Everingham et~al.({\natexlab{a}})Everingham, Van~Gool, Williams, Winn,
  and Zisserman]{pascal-voc-2007}
Mark Everingham, Luc Van~Gool, Christopher~KI Williams, John Winn, and Andrew
  Zisserman.
\newblock The {PASCAL} {V}isual {O}bject {C}lasses {C}hallenge 2007 {(VOC2007)}
  {R}esults.
\newblock
  http://www.pascal-network.org/challenges/VOC/voc2007/workshop/index.html,
  {\natexlab{a}}.

\bibitem[Everingham et~al.({\natexlab{b}})Everingham, Van~Gool, Williams, Winn,
  and Zisserman]{pascal-voc-2012}
Mark Everingham, Luc Van~Gool, Christopher~KI Williams, John Winn, and Andrew
  Zisserman.
\newblock The {PASCAL} {V}isual {O}bject {C}lasses {C}hallenge 2012 {(VOC2012)}
  {R}esults.
\newblock
  http://www.pascal-network.org/challenges/VOC/voc2012/workshop/index.html,
  {\natexlab{b}}.

\bibitem[Liu et~al.(2016)Liu, Anguelov, Erhan, Szegedy, Reed, Fu, and
  Berg]{liu2016ssd}
Wei Liu, Dragomir Anguelov, Dumitru Erhan, Christian Szegedy, Scott Reed,
  Cheng-Yang Fu, and Alexander~C Berg.
\newblock Ssd: Single shot multibox detector.
\newblock In \emph{Proceedings of the European Conference on Computer Vision},
  pages 21--37. Springer, 2016.

\bibitem[van~der Meulen and McCall(2018)]{van_der_meulen_mccall_2018}
Rob van~der Meulen and Thomas McCall.
\newblock Gartner says nearly half of cios are planning to deploy artificial
  intelligence, Feb 2018.
\newblock URL
  \url{https://www.gartner.com/en/newsroom/press-releases/2018-02-13-gartner-says-nearly-half-of-cios-are-planning-to-deploy-artificial-intelligence}.

\bibitem[ven(2019)]{venturebeat_2019}
Why do 87\% of data science projects never make it into production?, Jul 2019.
\newblock URL
  \url{https://venturebeat.com/2019/07/19/why-do-87-of-data-science-projects-never-make-it-into-production/}.

\bibitem[Mor{\'e}(1978)]{more1978levenberg}
Jorge~J Mor{\'e}.
\newblock The levenberg-marquardt algorithm: implementation and theory.
\newblock In \emph{Numerical analysis}, pages 105--116. Springer, 1978.

\bibitem[Virtanen et~al.(2020)Virtanen, Gommers, Oliphant, Haberland, Reddy,
  Cournapeau, Burovski, Peterson, Weckesser, Bright, {van der Walt}, Brett,
  Wilson, Millman, Mayorov, Nelson, Jones, Kern, Larson, Carey, Polat, Feng,
  Moore, {VanderPlas}, Laxalde, Perktold, Cimrman, Henriksen, Quintero, Harris,
  Archibald, Ribeiro, Pedregosa, {van Mulbregt}, and {SciPy 1.0
  Contributors}]{2020SciPy-NMeth}
Pauli Virtanen, Ralf Gommers, Travis~E. Oliphant, Matt Haberland, Tyler Reddy,
  David Cournapeau, Evgeni Burovski, Pearu Peterson, Warren Weckesser, Jonathan
  Bright, St{\'e}fan~J. {van der Walt}, Matthew Brett, Joshua Wilson, K.~Jarrod
  Millman, Nikolay Mayorov, Andrew R.~J. Nelson, Eric Jones, Robert Kern, Eric
  Larson, C~J Carey, {\.I}lhan Polat, Yu~Feng, Eric~W. Moore, Jake
  {VanderPlas}, Denis Laxalde, Josef Perktold, Robert Cimrman, Ian Henriksen,
  E.~A. Quintero, Charles~R. Harris, Anne~M. Archibald, Ant{\^o}nio~H. Ribeiro,
  Fabian Pedregosa, Paul {van Mulbregt}, and {SciPy 1.0 Contributors}.
\newblock {{SciPy} 1.0: Fundamental Algorithms for Scientific Computing in
  Python}.
\newblock \emph{Nature Methods}, 17:\penalty0 261--272, 2020.
\newblock \doi{10.1038/s41592-019-0686-2}.

\bibitem[Puterman(2014)]{puterman2014markov}
Martin~L Puterman.
\newblock \emph{Markov decision processes: discrete stochastic dynamic
  programming}.
\newblock John Wiley \& Sons, 2014.

\bibitem[Bertsekas(2012)]{bertsekas2012dynamic}
Dimitri Bertsekas.
\newblock \emph{Dynamic programming and optimal control: Volume I}, volume~1.
\newblock Athena scientific, 2012.

\bibitem[Easley and Kiefer(1988)]{easley1988controlling}
David Easley and Nicholas~M Kiefer.
\newblock Controlling a stochastic process with unknown parameters.
\newblock \emph{Econometrica: Journal of the Econometric Society}, pages
  1045--1064, 1988.

\bibitem[Smallwood and Sondik(1973)]{smallwood1973optimal}
Richard~D Smallwood and Edward~J Sondik.
\newblock The optimal control of partially observable markov processes over a
  finite horizon.
\newblock \emph{Operations research}, 21\penalty0 (5):\penalty0 1071--1088,
  1973.

\bibitem[Zhao et~al.(2021)Zhao, Liu, Anandkumar, and Yue]{zhao2021active}
Eric Zhao, Anqi Liu, Animashree Anandkumar, and Yisong Yue.
\newblock Active learning under label shift.
\newblock In \emph{International Conference on Artificial Intelligence and
  Statistics}, pages 3412--3420. PMLR, 2021.

\bibitem[Ghorbani and Zou(2019)]{ghorbani2019data}
Amirata Ghorbani and James Zou.
\newblock Data shapley: Equitable valuation of data for machine learning.
\newblock In \emph{International Conference on Machine Learning}, pages
  2242--2251. PMLR, 2019.

\bibitem[Coleman et~al.(2020)Coleman, Yeh, Mussmann, Mirzasoleiman, Bailis,
  Liang, Leskovec, and Zaharia]{coleman2019selection}
Cody Coleman, Christopher Yeh, Stephen Mussmann, Baharan Mirzasoleiman, Peter
  Bailis, Percy Liang, Jure Leskovec, and Matei Zaharia.
\newblock Selection via proxy: Efficient data selection for deep learning.
\newblock In \emph{International Conference on Learning Representations}, 2020.

\bibitem[Simonyan and Zisserman(2015)]{simonyan2014very}
Karen Simonyan and Andrew Zisserman.
\newblock Very deep convolutional networks for large-scale image recognition.
\newblock \emph{International Conference on Learning Representations}, 2015.

\bibitem[Elezi et~al.(2022)Elezi, Yu, Anandkumar, Leal-Taixe, and
  Alvarez]{elezi2021towards}
Ismail Elezi, Zhiding Yu, Anima Anandkumar, Laura Leal-Taixe, and Jose~M
  Alvarez.
\newblock Not all labels are equal: Rationalizing the labeling costs for
  training object detection.
\newblock \emph{Proceedings of the IEEE/CVF Conference on Computer Vision and
  Pattern Recognition}, 2022.

\end{thebibliography}
